\documentclass[twoside,11pt]{article}
\usepackage{jmlr2e}
\usepackage{natbib}
\usepackage[utf8]{inputenc} 
\usepackage[T1]{fontenc}    
\usepackage{hyperref}       
\usepackage{url}            
\usepackage{booktabs}       
\usepackage{amsfonts}       
\usepackage{xcolor}
\usepackage{enumitem} 
\usepackage{microtype}      
\usepackage{graphicx}
\usepackage{nicefrac}       
\usepackage{amsmath}
\usepackage{amsfonts}
\usepackage{comment}
\usepackage{bbm}
\usepackage{amssymb}
\usepackage{wrapfig}
\usepackage{xcolor}

\usepackage{amsmath,amsfonts,bm}

\def\eqref#1{equation~\ref{#1}}

\def\1{\bm{1}}

\def\vg{{\bm{g}}}

\def\vr{{\bm{r}}}

\def\vv{{\bm{v}}}
\def\vw{{\bm{w}}}
\def\vx{{\bm{x}}}
\def\vy{{\bm{y}}}
\def\vz{{\bm{z}}}
\def\vphi{{\bm{\phi}}}

\def\mA{{\bm{A}}}
\def\mB{{\bm{B}}}

\def\mG{{\bm{G}}}
\def\mH{{\bm{H}}}
\def\mI{{\bm{I}}}
\def\mJ{{\bm{J}}}

\def\mM{{\bm{M}}}

\def\mW{{\bm{W}}}
\def\mX{{\bm{X}}}

\def\meps{{\bm{\epsilon}}}

\DeclareMathAlphabet{\mathsfit}{\encodingdefault}{\sfdefault}{m}{sl}
\SetMathAlphabet{\mathsfit}{bold}{\encodingdefault}{\sfdefault}{bx}{n}

\newcommand{\Var}{\mathrm{Var}}

\DeclareMathOperator{\Tr}{Tr}

\usepackage{float}        
\usepackage{subcaption}     
\usepackage{graphicx}      
\usepackage{dcolumn}
\usepackage{bm}
\usepackage{color}
\specialcomment{Correction}{%
\begingroup\color{black}}{%
\endgroup}
\specialcomment{Correctionunsure}{%
\begingroup\color{black}}{%
\endgroup}
\usepackage[textsize=footnotesize]{todonotes}

\usepackage{enumitem}
\usepackage{algorithm}
\usepackage{algorithmic}
\usepackage{footnote}

\DeclareMathVersion{normal2}

\jmlrheading{1}{2020}{1-48}{4/00}{10/00}{Diego Granziol, Stefan Zohren and Stephen Roberts}

\ShortHeadings{An RMT Approach to Scaling DNN Learning Rates}{Granziol, Zohren and Roberts}
\firstpageno{1}

\begin{document}
\title{Learning Rates as a Function of Batch Size: A Random Matrix Theory Approach to Neural Network Training}
\author{\name Diego Granziol \email diego@robots.ox.ac.uk \\
\addr AI Theory Lab\\
Huawei\\\
Gridiron building, 1 Pancras Square, Kings Cross, London, N1C 4AG
\AND
\name  Stefan Zohren \email zohren@robots.ox.ac.uk \\
\name  Stephen Roberts \email sjrob@robots.ox.ac.uk \\
\addr Machine Learning Research Group and Oxford-Man Institute for Quantitative Finance\\
University of Oxford\\\
25 Walton Well Rd, Oxford OX2 6ED, UK
}

\editor{Simon Lacoste-Julien}

\maketitle
\begin{abstract}
We study the effect of mini-batching on the loss landscape of deep neural networks using spiked, field-dependent random matrix theory.  We demonstrate that the magnitude of the extremal values of the batch Hessian are larger than those of the empirical Hessian. We also derive similar results for the Generalised Gauss-Newton matrix approximation of the Hessian. As a consequence of our theorems we derive an analytical expressions for the maximal learning rates as a function of batch size, informing \textcolor{black}{practical training regimens} for both stochastic gradient descent (linear scaling) and adaptive algorithms, such as Adam (square root scaling), 
\textcolor{black}{for smooth, non-convex deep neural networks}.
Whilst the linear scaling for stochastic gradient descent has been derived under more restrictive conditions, which we generalise, the square root scaling rule for adaptive optimisers is, to our knowledge, completely novel. 
We validate our claims on the VGG/WideResNet architectures on the CIFAR-$100$ and ImageNet datasets.
\begin{Correction}
Based on our investigations of the sub-sampled Hessian we develop a stochastic Lanczos quadrature based on the fly learning rate and momentum learner, which avoids the need for expensive multiple evaluations for these key hyper-parameters and shows good preliminary results on the Pre-Residual Architecure for CIFAR-$100$.
\end{Correction}
\end{abstract}
\begin{keywords}
Deep Learning Theory, Random Matrix Theory, Loss Surfaces, Neural Network Training, Learning Rate Scaling, Adam, Adaptive Optimization, Square root rule
\end{keywords}

\section{Introduction}

Deep Learning has taken computer vision and natural language processing tasks by storm. The observation that different critical points on the loss surface post similar test set performance has spawned an explosion of theoretical \citep{choromanska2015loss,choromanska2015open,pennington2017geometry} and empirical interest \citep{papyan2018full,ghorbani2019investigation,li2017visualizing,sagun2016eigenvalues,sagun2017empirical,wu2017towards}, in deep learning loss surfaces, typically through study of the eigenspectrum of the Hessian. 
Scalar metrics of the Hessian, such as the trace/spectral norm, have been related to generalisation \citep{keskar2016large,li2017visualizing}. Under a Bayesian \citep{mackay2003information} and minimum description length framework \citep{hochreiter1997flat}, flatter minima generalise better than sharp minima.  
\textcolor{black}{
This has, however, been disputed recently \citep{dinh2017sharp} due to a perceived lack of parameterisation invariance, with further work considering a parameterisation invariant flatness metric \citep{tsuzuku2020normalized}.}
Theoretical work on the Hessian of neural networks has shown that all local minima are close to the global minimum \citep{choromanska2015loss} and that critical points of high index (i.e those with many negative eigenvalues) have high loss values \citep{pennington2017geometry}. second-order optimisation methods \citep{bottou2018optimization}, use the Hessian (or positive semi definite approximations thereof, such as the Fisher information matrix). They more efficiently navigate along narrow and sharp valleys, making significantly more progress per iteration  \citep{martens2010deep,martens2012training,martens2015optimizing,dauphin2014identifying} than first-order methods.

A crucial part of practical deep learning is the concept of sub-sampling or mini-batching. Instead of using the entire dataset of size $N$ to evaluate the loss, gradient or Hessian at each training iteration, only a small randomly chosen subset of size $B \ll N$ is used. This allows faster progress and lessens the computational burden tremendously. However, despite its widespread use in optimisation, the precise characterisation of the effects of mini-batching on the loss landscape and implications thereof, has not been thoroughly investigated. 
In this paper we show that: 
\begin{itemize}
\item Under assumptions consistent with the optimisation paradigm, the fluctuations in the Hessian due to mini-batching can be modelled as a random matrix;
\item For the feed forward, fully connected network with cross-entropy loss we expect the full Hessian to be low-rank and we provide extensive experiments
along with a theoretical derivation to back up this assertion.
\item When the eigenvalues of the full dataset Hessian are well separated from the fluctuations matrix (which we define in Section \ref{subsec:noise}) due to mini-batching, the extremal eigenvalues of the batch Hessian are given by the extremal eigenvalues of the full Hessian plus a term proportional to the ratio of the \emph{Hessian variance} to the batch size. We verify this empirically for the VGG-$16$ network \citep{simonyan2014very} on the CIFAR-$100$ dataset;
\begin{Correction}
\item By a natural extension of our framework we can (and experimentally do) investigate the nature of the Hessian under the data generating distribution, which is a natural object when considering the true risk surface and generalisation;
\end{Correction}
\item Our rigorous theoretical results predicts initial perfect scaling, diminishing returns and stagnation when increasing the batch size of stochastic gradient descent training \citep{golmant2018computational,shallue2018measuring}. This result is crucial for understanding how to alter learning rate schedules when exploiting large batch training and data-parallelism, or when using limited GPU capacity for small or mobile devices. Whilst this result has been experimentally verified and derived previously \citep{goyal2017accurate,smith2017don}, the setting here is much more general and less restrictive than in previous work;
\begin{Correction}
\item As a consequence of our analysis of the batch Hessian, we provide a Lanczos algorithm based learning rate and momentum learner, which we show works effectively in training neural networks out of the box on a preliminary example.
\end{Correction}
\item For adaptive-gradient methods where the damping parameter is fixed to a small value (such as the Adam default settings) we derive and verify the efficacy of a square root learning rate scaling with batch size. Specifically we mean that we expect a similar performance and training stability as we increase/decrease the learning rate with the square root of the batch size increase/decrease. 
\begin{Correction}
\item We explicitly experimentally validate our proposed scaling rules, by scaling the largest learning rate which trains without divergence on the VGG-$16$ \citep{simonyan2014very} architecture, for a batch size of $128$, with no weight decay and batch normalisation. We show that alternative scaling rules break down and fail to train in the regime where they predict more aggressive scalings (larger learning rates) than our rules. 
\item We show that alternate scaling rules when they are more conservative, give sub-optimal validation errors and hence can be considered sub-optimal from a practical perspective. We relate this to the similarity of paths taken throughout the loss landscape. Where similar paths result in similar validation/test performance.
\end{Correction}
\end{itemize}
The paper is structured as follows. The relevance of our work, key contributions and relationships to prior literature is detailed in Section \ref{sec:motivation}. \begin{Correction}
Section \ref{sec:illustration} Illustrates the main result for practitioners.
Section \ref{sec:rmttheory} details the random matrix theory framework modelling the noise due to mini-batching -- it states the assumptions, lemmas and proofs. Section \ref{sec:mainresult} gives the theoretical main result. Section \ref{sec:GGNtheory} extends the framework from Section \ref{sec:rmttheory} to strictly positive-definite matrices such as the Generalised Gauss-Newton matrix, along with a theoretical and empirical investigation on the low rank approximation of the full Dataset Hessian in Section \ref{sec:lowrankapprox}. 
Section \ref{sec:experiments} provides experimental validation for the theoretical claims. We discuss why we expect similar trajectories in weight space to give similar validation curves in Section \ref{sec:testaccrmt}.
We then derive and verify as consequence of our framework a linear scaling rule for SGD in Section \ref{sec:scaling} along with a square root scaling rule for Adam in Section \ref{sec:adamlrscale} as a function of batch size. We discuss the Hesssian under the data generating distribution in Section \ref{sec:truelosssurface} and why for classification we always expect outliers in the spectra in Section \ref{sec:whywehaveoutliers}.
\end{Correction}
Finally, we conclude in Section \ref{sec:conclusion}. Several appendices provide further details as referred to in the main text.

\section{Motivation}
\label{sec:motivation}
For samples drawn independently from the training set, the stochastic gradient  $\vg_{i}(\vw) \in \mathbb{R}^{P \times 1}$ in expectation is equal to the empirical gradient $\mathbb{E}(\vg_{i}(\vw)) = \vg(\vw)$ \citep{boyd_vandenberghe_2009,nesterov2013introductory}. However, for the sample inverse Hessian $\mH^{-1}_{i}(\vw) \in \mathbb{R}^{P \times P}$, we note that $\mathbb{E}(\mH^{-1}_{i}(\vw)) \neq \mH^{-1}(\vw)$, as inversion is not a linear operation.
By the spectral theorem, every Hermitian matrix, can be represented by its spectrum $\mH(\vw) = \sum_{i}^{P}\lambda_{i}\vphi_{i}\vphi_{i}^{T}$ and hence the spectrum of $\mH_{i}(\vw)$ differs from that of $(1/N)\sum_{i=1}^{N}\mH(\vw)$ or that of $\mathbb{E}(H(\vw))$. Whilst this problem may at first seem intractable, under specific assumptions about the \emph{matrix of fluctuations}, which characterises how the Hessian of a single sample varies from that of the full dataset, we can evaluate this difference in spectrum analytically. In this paper we develop this idea with two different assumptions. We show that our theory well describes the perturbations between the batch and full data Hessians for large neural networks (VGG) with millions of parameters on regularly used datasets (CIFAR-100). We show that as consequences of our theorems, scaling rules as a function of batch size for both stochastic gradient descent and adaptive optimisers (which are different) follow naturally. \begin{Correction}
We analyse the scaling rules, which are derived from our work, on other common networks and datasets, such as Residual networks \citep{he2016deep} and ImageNet. We note that other concurrent analytical works on the Hessian have also used the VGG net as a reference network \citep{papyan2020traces}.
\end{Correction}

\subsection{Practical Applicability}
How the loss surface changes as a function of mini-batch size, is of general interest to the greater problem of understanding deep learning. In particular, in the following we detail three practical applications which we identify.
\paragraph{Second-order optimisation:}
Mini-batching is prevalent in all \citep{martens2015optimizing,dauphin2014identifying} deep learning second-order optimisation methods. 
\textcolor{black}{
Certain proofs of convergence for this class of methods explicitly require similarity between the spectra of the sub-sampled and full dataset Hessians \citep{roosta2016sub}. Hence, understanding the spectral perturbations due to mini-batching is important for some theoretical results regarding second-order methods. We note, however, that alternative proof methods \citep{bollapragada2019exact,moritz2016linearly} don't require such assumptions.
}
\paragraph{Gradient-based optimisation: }
For gradient methods on convex functions, the convergence rate, optimal and maximal learning rates are functions of the Lipschitz constant \citep{nesterov2013introductory}, which is the infimum of the eigenvalues of the Hessian in the weight manifold. Hence understanding the largest eigenvalue perturbation due to mini-batching also has direct implications for their stability and convergence. Our framework prescribes a linear scaling rule up to a threshold for stochastic gradient descent. The works in \citet{krizhevsky2014one,goyal2017accurate} also prescribe a linear scaling of the learning rate with batch size, however it is justified under the unrealistic assumption that the gradient is the same at all points in weight space.
\citet{jain2017parallelizing} show linear parallelisation and then thresholding for least squares linear regression, assuming strong convexity. Our result holds for more general losses and does not assume strong convexity. Other work which considers the effect of batch sizes on learning rate choices and various optimisation algorithms, considers a constant as opposed to evolving Hessian and relies on assumptions of co-diagonalisability of the Hessian and covariance of the gradients \citep{zhang2019algorithmic}, which is not necessary in our framework. 
\paragraph{Adaptive gradient optimisation:}For adaptive or stochastic second-order methods using small damping and small learning rates, our theory prescribes a square root scaling procedure. \citet{hoffer2017train} also prescribe a square root scaling based on the co-variance of the gradients, for stochastic gradient descent (SGD) but not for adaptive methods. Our analysis expressly shows that the ways in which SGD and adaptive-gradient methods traverse the loss surface differ and this alters the optimal learning rate scaling as we increase the batch size.
\begin{Correction}
To the best of our knowledge no work has considered the difference in learning rate scalings between adaptive and non adaptive methods. In this work we expressly show (and empirically validated) that whilst for SGD we expect a linear learning rate scaling to hold as we increase/decrease the batch size up to a threshold, for Adam with small numerical stability constant (as is typical in practice) we expect a square root scaling rule.
\end{Correction}
\subsection{Related Work}
To the best of our knowledge no prior work has theoretically or empirically compared the Hessian of the full dataset and that of a mini-batch and the consequences thereof. 
\paragraph{Previous Loss Landscape Work:} Previous works focusing on the loss landscape structure as a function of loss value \citep{choromanska2015loss,pennington2017geometry} assume normality and independence of the inputs and weights and often even more assumptions, such as i.i.d. Hessian elements and free addition \citep{pennington2017geometry} which means that we can simply add the spectra of two matrices. Removing these assumptions is considered a major open problem \citep{choromanska2015open}, addressed in the deep linear case with squared loss \citep{kawaguchi2016deep}. Furthermore, the Hessian spectra are not compatible with outliers, extensively observed in practice \citep{sagun2016eigenvalues,sagun2017empirical,ghorbani2019investigation,papyan2018full}. We address both concerns, by considering a field dependence structure \citep{gotze2012semicircle}, non-identical element variances and modelling the outliers explicitly as low-rank perturbations \citep{benaych2011eigenvalues}. This may be of more general use to the community outside of our applications. 
\begin{Correction}
\paragraph{Similar Scaling Rules:}\citet{smith2017bayesian} derived optimal learning rate scalings, which were found to be linear by considering the scale of gradient noise and (assuming independent draws) the central limit theorem. This work was further extended (and experimentally verified) in \citet{smith2017don}. This raises the question \emph{Why should we consider the impact of curvature as opposed to gradient variance?} One simple pedagogical reason includes a holistic understanding in the limit of full-dataset training. In \citet{smith2017bayesian} the noise scale is given by a factor $\frac{N-B}{NB}$, where $N,B$ denote the dataset size and batch size respectively. In the case where $N=B$, even when there is no noise, learning rate choices are dictated by the local curvature. This is already well known in the stochastic (convex and otherwise) optimisation literature \citep{rakhlin2011making,shamir2013stochastic,lacoste2012simpler,harvey2019tight}, where proofs typically set a learning rate of $\frac{1}{\lambda t}$, where $\lambda,t$ denotes the Lipshitz constant (which is an upper bound on the local Hessian maximum eigenvalue) and the iteration number respectively, showing the importance of considering curvature. As a consequence of this, we implement and present an online learning rate and momentum learner which uses the local sub-sampled curvature estimate to estimate appropriate values for these two coefficients.
Another practical consideration, to the best of our knowledge novel in this paper, is the difference in learning rate scaling for adaptive methods compared to that of gradient descent. This forms a key contribution and motivation for our framework. Because our fine-grained analysis allows for an understanding of what happens to different regions of the spectrum when sub-sampling, we predict a new phenomenon unexplored in previous literature. Whilst \citet{smith2017don} argue that a linear scaling rate can also be used \footnote{Figure 4b page 5}, we note from the corresponding figure in their text that, before the final sharp learning rate drop, the test accuracy for Adam diverges significantly as the learning rate drops and batch size increases. This implies that the linear scaling rate does not hold and hence warrants further investigation and in the authors opinion a novel framework. In this paper we show how a curvature based approach identifies that, for adaptive methods, a different regime holds compared to that of SGD. We experimentally validate this observation. 
As a further potential practical use case, which could form the basis of future work, our framework naturally extends to stochastic second-order optimisation methods \citep{nocedal2006numerical} such as KFAC \citep{martens2015optimizing}. These approximate the eigenvalue/eigenvectors pairs of the batch Hessian. Hence, an understanding of how the eigenvalue/eigenvector estimations vary as a function of batch size becomes useful. 
\paragraph{Hessian Analysis of DNNs:}\citet{papyan2020traces} provides an extensive analysis of Deep Neural Network Hessians, developing an attribution strategy to various elements of the observed spectra which they empirically verify.  Specifically this work builds upon that of \citet{papyan2018full}, which shows that the spectral outliers are attributable to the covariance of gradient class means and demonstrates that a mini-bulk, separated from the main bulk and outliers, is attributable to the cross-class gradient covariance and, further, that the main bulk is attributable to the within-class covariance. The paper demonstrates this experimentally by leveraging linear algebraic tools to plot the spectrum of $\log \mH$ and by removing the components due to the within and cross class covariance from the spectrum. The paper also shows that increasing separation of the spectral outliers from the bulk distribution occurs with network depth and that. Furthermore they show for softmax regression on a Gaussian mixture dataset, that separation of the spectral outliers from the mini-bulk and separation of the mini-bulk from the bulk can be analytically related to generalistaion. 
The work also provides an alternative matrix to KFAC \cite{martens2015optimizing} for second order optimisation called CFAC, which is shown to be a better approximation to the Generalised Gauss Newton matrix. \citet{ghorbani2019investigation} re-introduce the Lanczos \citep{meurant2006lanczos} algorithm to the machine learning community and validate its accuracy to double precision using only a limited number ($m=90$) of Hessian vector products. They use this tool to investigate the Hessian spectral density on Imagenet and conclude that there remains significant negative spectral mass at the end of training and that the optimisation landscape seems to be smoother without residual connections. They also discuss spectral outlier suppression due to batch normalisation and argue that increasing the gradient contribution to flatter directions is inherently beneficial to the optimisation process. 
Whilst both of these works similarly focus on the Hessian and use similar tools to evaluate the spectrum, our principal focus is on how the sub-sampled batch Hessian deviates from the empirical (and or population) Hessian and the impacts this has on network training and hence the focus of the work, theoretical basis and approach are very different.
\end{Correction}
Some of the ideas in this work are inspired by earlier unfinished work on the true loss surface \citep{granzioldeep2018}. 

\begin{Correction}

\section{Illustration of the Key Result}
\label{sec:illustration}
We illustrate our key result (formalised in Theorems \ref{theorem:mainresult} and \ref{theorem:batchtheoremgnn} in Sections \ref{sec:rmttheory} \& \ref{sec:GGNtheory}) in Figure \ref{fig:stylisedfact2}.
\begin{figure}[h!]
	\centering
	\begin{subfigure}{0.23\linewidth}
		\includegraphics[trim={0cm 0cm 0cm 0cm},clip, width=1\textwidth]{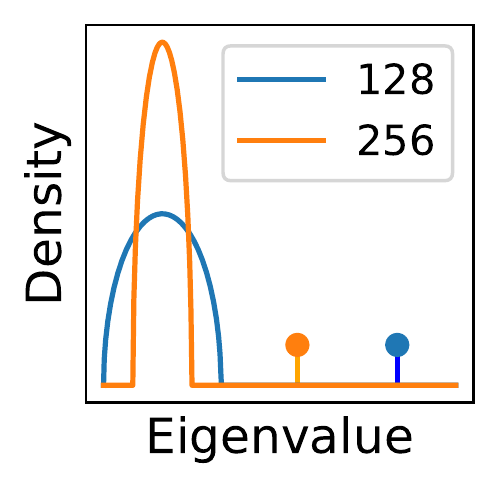}		
		\caption{Wigner Linear}
		\label{subfig:sepoutliersmallbatch2}
	\end{subfigure}
	\begin{subfigure}{0.23\linewidth}
		\includegraphics[trim={0cm 0cm 0cm 0cm},clip, width=1\textwidth]{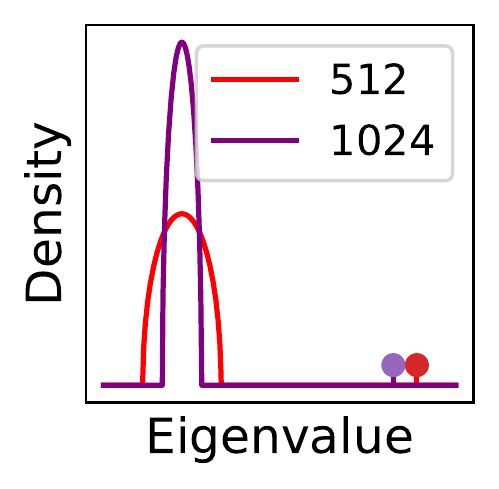}	
		\caption{Wigner Threshold}
		\label{subfig:sepoutlierbigbatch2}
	\end{subfigure}
	\begin{subfigure}{0.23\linewidth}
		\includegraphics[trim={0cm 0cm 0cm 0cm},clip, width=1\textwidth]{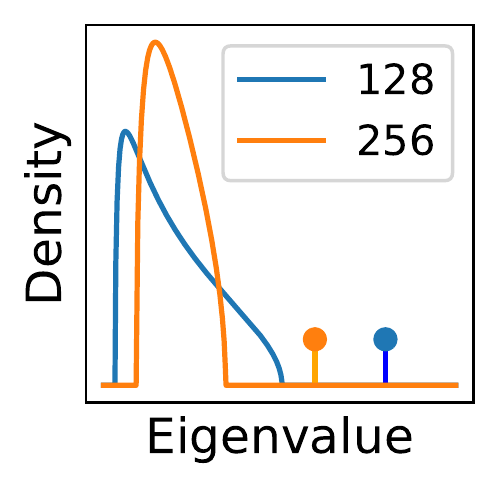}	
		\caption{MP Linear}
		\label{subfig:sepoutliersmallbatch}
	\end{subfigure}
	\begin{subfigure}{0.23\linewidth}
		\includegraphics[trim={0cm 0cm 0cm 0cm},clip, width=1\textwidth]{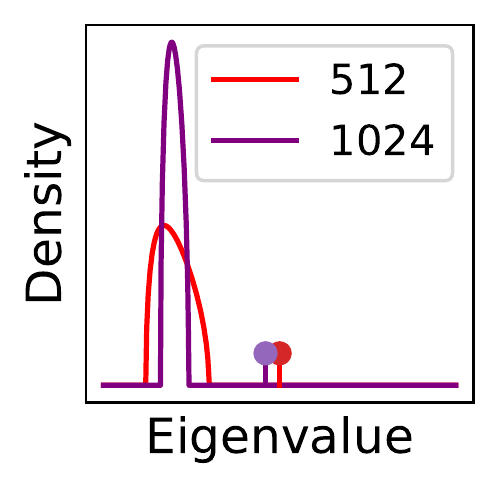}		
		\caption{MP Threshold}
		\label{subfig:sepoutlierbigbatch}
	\end{subfigure}
	\caption{\textbf{Variation of the spectral norm with batch size. Spectral norm decreases linearly until a threshold with batch size increase for both the Wigner and Machenko-Pastur noise models.} The continuous region (bulk) corresponds to the fluctuation matrix induced by mini-batching, shown as a Wigner semicircle (a \& b) or Marchenko-Pastur (MP - c \& d), whose width depends on the square root of the batch size). The largest eigenvalue of the batch Hessian is shown as a single peak, which decreases in magnitude as the batch size increases.}
	\label{fig:stylisedfact2}
\end{figure}
If the largest Hessian eigenvalue is well separated from the fluctuation matrix (continuous spectral density), as shown in Figures \ref{subfig:sepoutliersmallbatch2} \& \ref{subfig:sepoutliersmallbatch}, then increasing the batch size, which reduces the spectral width of the fluctuation matrix (which in turn reduces with the square root of the batch size), will have an approximately linear effect in reducing the spectral norm. This will hold up to a threshold, shown in Figures \ref{subfig:sepoutlierbigbatch2} \& \ref{subfig:sepoutlierbigbatch}, after which the spectral norm no longer appreciably changes in size. This is because the perturbation due to minibatch sampling no longer dominates the magnitude of the eigenvalue from the full dataset Hessian. We discuss the prevalence and origin of spectral outliers in Deep Neural Network spectra in Section \ref{sec:whywehaveoutliers}.
\end{Correction}

\section{Random matrix theoretic approach to the Batch Hessian}
\label{sec:rmttheory}
For an input, output pair $[\vx,\vy] \in [\mathbb{R}^{d_{x}},\mathbb{R}^{d_{y}}]$ and a given prediction function $h(\cdot;\cdot):\mathbb{R}^{d_{x}}\times \mathbb{R}^{P} \rightarrow \mathbb{R}^{d_{y}}$, we consider the family of prediction functions parameterised by a weight vector $\vw$, i.e., $\mathcal{H}:= \{h(\cdot;\vw):\vw \in \mathbb{R}^{P} \}$ with a given loss function $\ell(h(\vx;\vw), \vy): \mathbb{R}^{d_{y}} \times \mathbb{R}^{d_{y}} \rightarrow \mathbb{R}$. 
In conjunction with statistical learning theory terminology, we denote the loss over our data generating distribution $\psi(\vx,\vy)$, as the \emph{true risk}.
\begin{equation}
R_{true}(\vw) = \int\ell(h(\vx;\vw),\vy)d\psi(\vx,\vy),   
\end{equation}
with corresponding gradient $\vg_{true}(\vw)= \nabla R_{true}(\vw)$  and Hessian $\mH_{true}(\vw) = \nabla^{2} R_{true}(\vw) \in \mathbb{R}^{P\times P}$.
Given a dataset of size $N$, we only have access to the \emph{empirical risk} 
\begin{equation}
R_{emp}(\vw) = \sum_{i=1}^{N}\frac{1}{N}\ell(h(\vx_{i};\vw),\vy_{i}),    
\end{equation}
empirical gradient $\vg_{emp}(\vw) = \nabla R_{emp}(\vw)$ and empirical Hessian $\mH_{emp}(\vw) = \nabla^{2} R_{emp}(\vw).$
To further reduce computation cost, often only the batch risk
\begin{equation}
R_{batch}(\vw) = \frac{1}{B}\sum_{i=1}^{B}\ell(h(\vx_{i};\vw),\vy_{i}),
\end{equation}
(where $B \ll N$ \textcolor{black}{belongs to the batch})
and the gradients $\vg_{batch}(\vw)$, Hessians $\mH_{batch}(\vw)$ thereof are accessed. The Hessian describes the curvature at that point in weight space $\vw$ and hence the risk surface can be studied through the Hessian. 
\subsection{Properties of the fluctuation matrix}
\label{subsec:noise}
We write the stochastic batch Hessian as the deterministic empirical Hessian plus a perturbation due to the sampling noise. 
\begin{equation}
\mH_{batch}(\vw) = \mH_{emp}(\vw) + \meps(\vw)\footnote{Note that although we could write $\mH_{emp}(\vw) = \mH_{batch}(\vw) - \meps(\vw)$, this treatment is not symmetric as $\mH_{batch}(\vw)$ is dependent on $\meps(\vw)$, whereas $\mH_{emp}(\vw)$ is not.}
\end{equation}
Rewriting the fluctuation matrix as
$\meps(\vw) \equiv \mH_{batch}(\vw) - \mH_{emp}(\vw)$, we can infer 

\begin{equation}
\label{eq:noisematrix}
\begin{aligned}
	& \meps(\vw) = \bigg(\frac{1}{B}-\frac{1}{N}\bigg)\sum_{j=1}^{B}\nabla^{2} \ell(\vx_{j},\vw;\vy_{j}) - \frac{1}{N}\sum_{i=B+1}^{N}\nabla^{2} \ell(\vx_{i},\vw;\vy_{i}) \\
	& \text{thus   } \mathbb{E}(\meps(\vw)_{j,k}) = 0 \text{   and   } \thinspace \mathbb{E}(\meps(\vw)_{j,k})^{2} = \bigg(\frac{1}{B}-\frac{1}{N}\bigg)\mathrm{Var}[\nabla^{2} \ell(\vx,\vw;\vy)_{j,k}]. \\ 
\end{aligned}
\end{equation}
Where $B$ is the batch size and $N$ the total dataset size 
\begin{Correctionunsure}
and we use the fact that for each sample a given Hessian element has a common mean and variance. Note that we sample without replacement. This avoids the pathological case where we sample the same element $B$ times and hence have no variance reduction. Sampling without replacement is typical in deep learning and hence the relevant case for our investigations. We implicitly assume that each sample can be considered an independent draw from the data generating distribution $\psi(\vx,\vy)$. Without this assumption the variance could scale differently. The expectation is taken with respect to $\psi(\vx,\vy)$.
\end{Correctionunsure}
In order for the variance in Equation \ref{eq:noisematrix} to exist, the elements of $\nabla^{2} \ell(\vw,\vw;\vy)$ must obey sufficient moment conditions. This can either be assumed as a technical condition, or alternatively derived under the more familiar condition of $L$-Lipschitz continuity, as shown with the following Lemma
\begin{lemma}
\label{lemma:boundedhessianelements}
For a Lipschitz-continuous empirical risk gradient and almost everywhere twice differentiable loss function $\ell(h(\vx;\vw),\vy)$, the elements of the fluctuation matrix $\meps(\vw)_{j,k}$ are strictly bounded in the range $- \sqrt{P}L \leq \meps(\vw)_{j,k}\leq \sqrt{P}L$. Where $P$ is the number of model parameters and $L$ is \textcolor{black}{the smoothness constant}.
\end{lemma}
\begin{proof}
As the gradient of the empirical risk is $L$ Lipschitz  \textcolor{black}{continuous and} the empirical risk 
\textcolor{black}{is the sum over the samples}, the gradient of the batch risk is also Lipschitz continuous. As the difference of two Lipschitz functions is also Lipschitz, by the fundamental theorem of calculus and the definition of Lipschitz continuity the largest eigenvalue $\lambda_{max}$ of the fluctuation matrix $\meps(\vw)$ must be smaller than $L$. Hence using the Frobenius norm we can upper bound the matrix elements of $\meps(\vw)$
\begin{equation}
	\begin{aligned}
		& \text{Tr}(\meps(\vw)^{2})  = \sum_{j,k=1}^{P}\meps(\vw)_{j,k}^{2}
		= \meps(\vw)_{j=j',k=k'}^{2} + \sum_{j\neq j',k\neq k'}^{P}\meps(\vw)_{j,k}^{2} = \sum_{i=1}^{P}\lambda_{i}^{2} \\
		& \text{thus   } \meps(\vw)_{j=j',k=k'}^{2} \leq \sum_{i=1}^{P}\lambda_{i}^{2} \leq PL^{2} \thinspace \thinspace 
		\text{   and   }  -\sqrt{P}L \leq \meps(\vw)_{j=j',k=k'} \leq \sqrt{P}L .\\
	\end{aligned}
\end{equation}
\end{proof}
As the domain of the Hessian elements under the data generating distribution is bounded, the moments of Equation \ref{eq:noisematrix} are bounded and hence the variance exists. We can even go a step further with the following extra lemma.
\begin{lemma}
\label{lemma:normalelements}
For independent samples drawn from the data generating distribution and an $L$-Lipschitz loss $\ell$ the difference between the empirical Hessian and Batch Hessian converges element-wise to a zero mean, normal random variable with variance $\propto \frac{1}{B}-\frac{1}{N}$ for large $B,N$.
\end{lemma}
\begin{proof}
By Lemma \ref{lemma:boundedhessianelements}, the Hessian elements are bounded, hence the moments are bounded and using independence of samples and the central limit theorem \citep{stein1972}
\begin{Correction}
	\begin{equation}
		(\frac{1}{B}-\frac{1}{N})^{-1/2}[\nabla^{2} R_{emp}(\vw)-\nabla^{2} R_{batch}(\vw)]_{jk} \xrightarrow[a.s]{} \mathcal{N}(0,\sigma_{jk}^{2})    
	\end{equation}
\end{Correction}
\end{proof}
\subsection{The fluctuation matrix spectrum converges to the semi-circle law}
To derive analytic results, we employ the Kolmogorov limit \citep{bun2017cleaning}, where $P,B,N \rightarrow \infty$ but $P(\frac{1}{B}-\frac{1}{N}) = q > 0$. 

\textcolor{black}{
We preserve the shape factor $q$ to keep our results consistent with the theoretical and applied random matrix theory literature \citep{baik2006eigenvalues,bun2016rotational,bun2017cleaning}. Mathematically, the limit to infinity allows for the convergence of stochastic quantities into deterministic ones, for which we can derive exact expressions. We discuss finite size corrections, both experimentally and state the corresponding theoretical corrections, in Section \ref{subsec:stochasticrmt}. Note that for typical deep learning the number of parameters is in the millions or billions, and the dataset size is also often in the tens or thousands or millions of examples. State of the art training also utilises batch sizes in the thousands \citep{goyal2017accurate}. We experimentally demonstrate in our experiments that whilst for smaller batch sizes e.g $B=128$, stochasticity is important, we find that the mean predictions given by our framework are still accurate and useful.
}

By Lemma \ref{lemma:boundedhessianelements}, we have $\mathbb{E}(\meps(\vw)_{j,k}) = 0$ and $\mathbb{E}(\meps(\vw)^{2}_{j,k}) = \sigma^{2}_{j,k}$. To further account for dependence beyond the symmetry of the fluctuation matrix elements, we introduce the $\sigma$-algebras
\begin{equation}
\begin{aligned}
	& 	    \mathfrak{F}^{(i,j)} :=  \sigma \{ \meps(\vw)_{kl}:1\leq k \leq l \leq P, (k,l) \neq (i,j) \}, ~~~ q \leq i \leq j \leq P \\
\end{aligned}
\end{equation}
We can now state the following Theorem which is based on a general result from \citet{gotze2012semicircle}:
\begin{theorem}
\label{theorem:beastingassumptions}
Under the conditions of Lemmas \ref{lemma:boundedhessianelements} and \ref{lemma:normalelements}, where $\meps(\vw) \equiv \mH_{batch}(\vw)-\mH_{emp}(\vw) $ along with the following technical conditions:
\begin{enumerate}[label=(\roman*)]
	\item $\frac{1}{P^{2}}\sum_{i,j=1}^{P}\mathbb{E}|\mathbb{E}(\meps(\vw)_{i,j}^{2}|\mathfrak{F}^{i,j})-\sigma^{2}_{i,j}| \rightarrow 0$,
	\item $\frac{1}{P}\sum_{i=1}^{P}|\frac{1}{P}\sum_{j=1}^{P}\sigma_{i,j}^{2}-\sigma_{\epsilon}^{2}| \rightarrow 0$ 
	\item $\max_{1\leq i \leq P} \frac{1}{P}\sum_{j=1}^{P}\sigma_{i,j}^{2} \leq C$ 
\end{enumerate}
when $P \rightarrow \infty$, the limiting spectral density $p(\lambda)$ of $\meps(\vw) \in \mathbb{R}^{P\times P}$ satisfies the semicircle law $p(\lambda) = \frac{\sqrt{4\sigma_{\epsilon}^{2}-\lambda^{2}}}{2\pi\sigma_{\epsilon}^{2}}$. Where $\mathbb{E}(\meps(\vw)_{i,j}^{2}|\mathfrak{F}^{i,j})$ denotes the expectation conditioned on the sigma algebra, which is different to the unconditional expectation $\mathbb{E}(\meps(\vw)_{i,j}^{2}|\mathfrak{F}^{i,j}) \neq \mathbb{E}(\meps(\vw)_{i,j}^{2}) = \sigma^{2}_{i,j}$. 

\end{theorem}
We note that under the assumption of independence between all the elements of $\meps(\vw)$ we would have obtained the same result, as long as conditions $ii)$ and $iii)$ were obeyed. So in simple words, condition $9i)$ merely states that the dependence between the elements cannot be too large. For example completely dependent elements have a second moment expectation that scales as $P^{2}$ and hence condition $(i)$ cannot be satisfied. Condition $(ii)$ merely states that there cannot be too much variation in the variances per element and condition $(iii)$ that the variances are bounded. \textcolor{black}{Note that $\meps(\vw)$ is a function of the current iterate $\vw$ and hence its spectrum depends on the Hessian at that point.}

\begin{proof}
Lindenberg's ratio is defined as $L_{P}(\tau) := \frac{1}{P^{2}} \sum_{i,j=1}^{P}\mathbb{E}|\meps(\vw)_{i,j}|^{2}\mathbbm{1}(|\meps(\vw)_{i,j}|\geq \tau \sqrt{P})$.
By Lemma \ref{lemma:normalelements}, the tails of the normal distribution decay sufficiently rapidly such that $L_{P}(\tau) \rightarrow 0$ for any $\tau>0$ in the $P \rightarrow \infty$ limit. Alternatively, using the Frobenius identity and Lipschitz continuity $\sum_{i,j=1}^{P}\mathbb{E}|\meps(\vw)_{i,j}|^{2}\mathbbm{1}(|\meps(\vw)_{i,j}|\geq \tau \sqrt{P}) \leq \sum_{i,j}^{P}\meps(\vw)_{i,j}^{2} = \sum_{i}^{P}\lambda_{i}^{2} \leq PL^{2}$, $L_{P}(\tau) \rightarrow 0$ for any $\tau>0$. 
By Lemma \ref{lemma:normalelements} we also have $\mathbb{E}(\meps(\vw)_{i,j}|\mathfrak{F}^{i,j})=0$. Hence along with conditions $(i), (ii), (iii)$ the matrix $\meps(\vw)$ satisfies the conditions in \citet{gotze2012semicircle} and the and the limiting spectral density $p(\lambda)$ of $\meps(\vw) \in \mathbb{R}^{P\times P}$ converges to the semicircle law $p(\lambda) = \frac{\sqrt{4\sigma_{\epsilon}^{2}-\lambda^{2}}}{2\pi\sigma_{\epsilon}^{2}}$ 
\citep{gotze2012semicircle}. 
\citet{gotze2012semicircle} use the condition $\frac{1}{P}\sum_{i=1}^{P}|\frac{1}{P}\sum_{j=1}^{P}\sigma_{i,j}^{2}-1| \rightarrow 0$, however this simply introduces a simple scaling factor, which is accounted for in condition $ii)$ and the corresponding variance per element of the limiting semi-circle.
\end{proof}
\section{Main Result}
\label{sec:mainresult}
Having shown that the limiting spectral density of the fluctuations matrix converges to the semi-circle, we are now in a position to present the main result of this paper.
\begin{theorem}
\label{theorem:mainresult}
Under the assumption that $\mH_{emp}$ is of low-rank $r \ll P$, 
the extremal eigenvalues $[\lambda'_{1},\lambda'_{P}]$ of the matrix sum $\mH_{batch}(\vw) = \mH_{emp}(\vw) + \meps(\vw)$, where $\lambda'_{1}\geq\lambda'_{2}...\geq \lambda'_{P}$ and $\meps(\vw)$ is defined in Section~\ref{subsec:noise} and obeys the conditions set out in Theorem~\ref{theorem:beastingassumptions}, are given by
\begin{equation}
	\label{eq:spectralbroadeningwignernoisebatch}
	\lambda'_{1} = \left\{\begin{array}{lr}
		\lambda_{1} + \frac{P}{\mathfrak{b}}\frac{\sigma_{\epsilon}^{2}}{\lambda_{1}}, & \text{if } \lambda_{1} > \sqrt{\frac{P}{\mathfrak{b}}}\sigma_{\epsilon}\\
		2 \sqrt{\frac{P}{\mathfrak{b}}}\sigma_{\epsilon}, & \text{otherwise } \\
	\end{array}\right\} \thinspace 
	, \thinspace
	\lambda'_{P} = \left\{\begin{array}{lr}
		\lambda_{P} + \frac{P}{\mathfrak{b}}\frac{\sigma_{\epsilon}^{2}}{\lambda_{P}}, & \text{if } \lambda_{P} < -\sqrt{\frac{P}{\mathfrak{b}}}\sigma_{\epsilon}\\
		-2 \sqrt{\frac{P}{\mathfrak{b}}}\sigma_{\epsilon}, & \text{otherwise } \\
	\end{array}\right\}.
\end{equation}
where $[\lambda_{1},\lambda_{P}]$ are the extremal eigenvalues  of $\mH_{emp}(\vw)$, $\mathfrak{b} = B/(1-B/N)$ \textcolor{black}{occurs due to the random sub-sampling} and $B$ is the batch-size.\footnote{Note that the factor $\mathfrak{b}= B/(1-B/N)$ has appeared before in \citep{jastrzkebski2018relation,jain2017parallelizing}.} Recall that $\sigma_{\epsilon}$ is defined in Theorem~\ref{theorem:beastingassumptions}, through the limiting spectral density $p(\lambda)$ of $\meps(\vw)$. \textcolor{black}{This result holds in the $P,B,N \rightarrow \infty$ limit, where $P/\mathfrak{b}$ remains finite.}
\hspace{-10pt}
\end{theorem}
In order to prove Theorem \ref{theorem:mainresult} we utilise the following Lemma, which is taken from \cite{benaych2011eigenvalues} and for which we outline the proof in Appendix ~\ref{sec:notationandproof} for completeness.
\begin{lemma}
\label{theorem:rmtbookwork}
Denote by $[\lambda_{1}',\lambda_{P}']$ the extremal eigenvalues of the matrix sum $\mM = \mA+\meps(\vw)/\sqrt{P}$, where $\mA \in \mathbb{R}^{P\times P}$ is a matrix of finite rank $r$ with extremal eigenvalues $[\lambda_{1},\lambda_{P}]$ and $\meps(\vw) \in \mathbb{R}^{P\times P}$ with limiting spectral density $p(\lambda)$ satisfying the semicircle law $p(\lambda) = \frac{\sqrt{4\sigma_{\epsilon}^{2}-\lambda^{2}}}{2\pi\sigma_{\epsilon}^{2}}$. Then we have
\begin{equation}
	\lambda'_{1} = \left\{\begin{array}{lr}
		\lambda_{1} + \frac{\sigma_{\epsilon}^{2}}{\lambda_{1}}, & \text{if } \lambda_{1} > \sigma_{\epsilon}\\
		2 \sigma_{\epsilon}, & \text{otherwise } \\
	\end{array}\right\} 
	, \thinspace
	\lambda'_{P} = \left\{\begin{array}{lr}
		\lambda_{P} + \frac{\sigma_{\epsilon}^{2}}{\lambda_{P}}, & \text{if } \lambda_{P} < -\sigma_{\epsilon}\\
		-2 \sigma_{\epsilon}, & \text{otherwise } \\
	\end{array}\right\}.
\end{equation}
\end{lemma}
We now proceed with the proof of Theorem \ref{theorem:mainresult}:
\begin{proof}
The variance per element is a function of the batch size $B$ and the size of the empirical dataset $N$, as given by Lemma \ref{lemma:normalelements}. Furthermore, unravelling the dependence in $P$ (which is simply the matrix dimension) due to the definition of the Wigner matrix (shown in Appendix \ref{sec:backgroundtheory}) leads to Theorem \ref{theorem:mainresult}.
\end{proof}
\paragraph{Comments on the Proof:}Although for clarity we only focus on the extremal eigenvalues, the proof as shown in Appendix \ref{sec:backgroundtheory} holds for all outlier eigenvalues which are outside the spectrum of the fluctuation matrix.  The assumption that either $\mH_{emp}$ or $\meps(\vw)$ are low-rank is necessary to use perturbation theory in the proof. This condition could be relaxed if a substantial part of the eigenspectrum of $\mH_{emp}$ were considered to be mutually free with that of $\meps(\vw)$ \citep{bun2017cleaning}. In Section \ref{sec:lowrankapprox} we derive a bound on the rank of a feed-forward network, which we show to be small for large networks and provide extensive experimental evidence that the full Hessian is in fact low-rank.
In the special case that $\meps(\vw)_{i,j}$ are i.i.d. Gaussian, the fluctuation matrix is the Gaussian Orthogonal Ensemble, proposed as the spectral density of the Hessian by \citet{choromanska2015loss}. In this case, Theorem \ref{theorem:mainresult} can be proved more succinctly, which we detail in full in the Appendix \ref{sec:backgroundtheory}.
\begin{Correction}
\begin{remark}
	Note that whilst we have considered the framework in which the batch Hessian is considered a perturbation of the full dataset (empirical) Hessian (via an additive perturbation), we could have alternatively considered the batch Hessian to be the true Hessian (i.e the dataset under the data generating distribution) plus an additive perturbation, i.e.
	\begin{equation}
		\mH_{batch}(\vw) = \mH_{true}(\vw) + \meps(\vw).
	\end{equation}
	This might be considered appropriate if each sample is only seen once, or as is typical in deep learning, the extent of the augmentation, e.g. random flips, crops with zero padding, rotations, colour variations, additions of Gaussian noise, are so extensive that no identical (or sufficiently similar) samples are ever seen by the optimiser twice. Note that under this framework, we simply need to replace $\mathfrak{b} \rightarrow B$ in our framework and we simply replace the maximal eigenvalue $\lambda_{1}$ of the full dataset Hessian with that of the Hessian of the data generating distribution. Since such an extension is natural under the typical neural network training framework utilising many augmentations (such as random flipping, cropping with zero padding, rotations, translations, or the addition of Gaussian noise) and can be readily derived from our framework. We investigate the nature of the true Hessian in Section \ref{sec:truelosssurface}. Here we show that the empirical Hessian does indeed closely resemble that of the true Hessian. 
\end{remark}
\end{Correction}

\section{Extension to Fisher information and other positive-definite matrices}
\label{sec:GGNtheory}
In the case of Logistic regression, which is simply a $0$ hidden layer neural network with cross-entropy loss, by the diagonal dominance theorem \citep{cover2012elements}, the Hessian is semi-positive-definite and positive-definite with the use of $L2$ regularisation. Hence an underlying fluctuation matrix which contains negative eigenvalues is unsatisfactory and we extend our noise model to cover the positive semi definite case. 
Commonly used positive semi-definite approximations to the Hessian in deep learning \citep{martens2014new} include the Generalised Gauss-Newton matrix (GGN) matrix \citep{martens2010deep,martens2012training} and the Fisher information matrix \citep{martens2015optimizing,pennington2018spectrum}, both used extensively for optimisation and theoretical analysis. Hence to understand the effect of mini-batching on these practically relevant optimisers, we must also extend our framework. Below we introduce the Generalised Gauss-Newton matrix.
\paragraph{The Generalised Gauss-Newton matrix:}
For some common activations and loss functions typical in deep learning, such as the cross-entropy loss and sigmoid activation the Generalised Gauss-Newton matrix is equivalent to the Fisher information matrix \citep{pascanu2013revisiting}. The Hessian may be expressed in terms of the activation $\sigma$ at the output of the final layer $f(\vw)$ using the chain rule as $\mH = \nabla^{2} \sigma(f(\vw))$ with corresponding $(i,j)$'th component:
\begin{equation}
\label{eq:batchHessian}
\begin{aligned}
	& \mH(\vw)_{ij}  = \sum_{k=0}^{d_{y}} \sum_{l=0}^{d_{y}} \frac{\partial^{2} \sigma(f(\vw))}{\partial f_{l}(\vw)\partial f_{k}(\vw)}\frac{\partial f_{l}(\vw)}{\partial w_{j}}\frac{\partial f_{k}(\vw)}{\partial w_{i}} 
	+ \sum_{k=0}^{d_{y}}\frac{\partial \sigma(f(\vw))}{\partial w_{k}} \frac{\partial^{2} f_{k}(\vw)}{\partial w_{j}\partial w_{i}}.\\
\end{aligned}
\end{equation}
The first term on the RHS of Equation \ref{eq:batchHessian} is known as the Generalised Gauss-Newton (GGN) matrix. The rank of a product is the minimum rank of its products so the raank of the GGN matrix is upper bounded by $B\times d_{y}$.  Following \citet{sagun2017empirical} due to the convexity of the loss $\ell$ with respect to the output $f(\vw)$ we rewrite the GGN matrix per sample as
\begin{equation}
\label{eq:batchggncrossent}
\begin{aligned}
	& \sum_{k,l=0}^{d_{y}} \sqrt{\frac{\partial^{2} \sigma(f(\vw))}{\partial f_{l}(\vw)\partial f_{k}(\vw)}}\frac{\partial f_{l}(\vw)}{\partial w_{j}} \times~~\sqrt{\frac{\partial^{2} \sigma(f(\vw))}{\partial f_{l}(\vw)\partial f_{k}(\vw)}}\frac{\partial f_{k}(\vw)}{\partial w_{i}} = \mJ_{*}\mJ_{*}^{T}, \\
\end{aligned}
\end{equation}
where we define $\mJ_{*}$ in order to retain a similarity for the GGN matrix in the case of the squared loss function \citep{pennington2017geometry}, which has the form $\mG(\vw) = \mJ\mJ^{T}$. There are many potential candidate noise models due to the effect of mini-batching. Examples include the free multiplicative and information plus noise model \citep{bun2016rotational,hachem2013bilinear}. 
\begin{Correctionunsure}
Let us simply consider a mini-batching model where the transformed Jacobian, $J^{*}$, is perturbed by additive noise. Specifically,
\begin{equation}
	\mJ^{*}_{batch}(\vw) = \mJ^{*}_{emp}(\vw)+\meps(\vw).
\end{equation}
Under this framework, as $\mathbb{E}[\meps(\vw)] =0$,
\begin{equation}
	\mathbb{E}(\mJ_{*}+\meps)(\mJ_{*}+\meps)^{T} = \mJ_{*}\mJ_{*}^{T}+\mathbb{E}\meps\meps^{T}.
\end{equation}
Note that in this case  $\meps(\vw) \in \mathbb{R}^{P\times  (B\times d_{y})}$ and hence $\meps(\vw)\meps(\vw)^{T} \in \mathbb{R}^{P\times P}$. Whilst it is known that, for i.i.d. Normal entries for $\meps(\vw)$, the spectrum of $\mathbb{E}\meps\meps^{T}$ converges to the Marchenko-Pastur distribution \citep{marvcenko1967distribution}, the conditions can similarly be relaxed to those stated in Theorem \ref{theorem:mainresult} \citep{adamczak2011marchenko,gotze2015limit,o2012note}. Hence, with this assumption and in line with the previous derivation, we consider the finite rank perturbation of the Marchenko-Pastur density and arrive at the following result. For completeness, we derive the non-unit-variance Stieltjes transform of the Marchenko-Pastur distribution
in Appendix~\ref{sec:nonunitmp}. 
\begin{theorem}
\label{theorem:batchtheoremgnn}
The extremal eigenvalue $\lambda_{1}'$ of the matrix $\mG_{batch}$, where $\mG_{emp}$ has extremal eigenvalue $\lambda_{1}$, is given by
\begin{equation}
	\label{eq:spectralbroadeningwignernoise}
	\lambda'_{1} = \left\{\begin{array}{lr}
	\frac{\lambda_{1}+\sigma^{2}(1-\frac{P}{\mathfrak{b}})}{1-\frac{P\sigma^{2}}{\lambda_{1}\mathfrak{b}}}, & \text{if } \lambda_{1} > \sigma^{2}(1+\frac{P}{\mathfrak{b}})\\
		2\sigma^{2}(1+\frac{P}{\mathfrak{b}}), & \text{otherwise } \\
	\end{array}\right\}.
\end{equation}
\end{theorem}
The key conclusion is that, \emph{independent of the exact limiting spectral density of the fluctuation matrix, we can consider the extremal eigenvalues of the True Hessian, or Generalised Gauss-Newton matrix (GGN), to be a low-rank perturbation of the fluctuation matrix. This can be considered a form of universality for the proved result in Theorem \ref{theorem:mainresult}. Where the assumptions on the noise matrix may differ, but the key phenomena, that of spectra broadening, persists.}  
\end{Correctionunsure}

\paragraph{How realistic is the low-rank approximation?} Since this is a major assumption in our analysis, we investigate the experimental evidence for the low-rank nature of the empirical Hessian and empirical GGN in Section~\ref{sec:lowrankapprox} and provide a theoretical argument for feed forward neural network Hessians. 
\begin{Correction}

\section{Evaluating the Low Rank Approximation}
\label{sec:lowrankapprox}
One of the key ingredients to proving Theorem \ref{theorem:mainresult}, as shown in Section \ref{sec:rmttheory}, is the use of perturbation theory. This requires either the fluctuation matrix, or the full empirical Hessian, to be low-rank. In our work, we consider the empirical Hessian to be low-rank. The rank degeneracy of small neural networks has already been discovered and discussed in \citet{sagun2017empirical} and reported for larger networks using spectral approximations in \citet{ghorbani2019investigation,papyan2018full}. We further provide extensive experimental validation for both the VGG-$16$ and PreResNet-$110$ on the CIFAR-$100$ datasets.
However theoretical arguments rely on the Generalised Gauss-Newton matrix (GGN) decomposition. From Equation \ref{eq:batchHessian} it can be surmised that the rank of the GGN  is  bounded above by $N\times d_{y}$ (the dataset size times the number of classes). However the Hessian is the sum of the GGN and another matrix, which has not been theoretically argued to be low-rank. The rank of a sum of two matrices is upper bounded by their rank sum. Furthermore, if the dataset size becomes large (e.g. such as ImageNet with $10^{7}$ entries) and the class number also large, even the GGN bound is ineffective. We hence provide in Section \ref{subsec:rankbound} a novel theoretical argument for a Hessian rank bound for feed-forward neural networks with a cross-entropy loss. The key intuition behind our proof is that each product of weights is a rank one object. Hence, if the sum of these products can be bounded we can also bound the rank. Since the sum depends on the number of neurons, the rank bound can end up becoming very small.

\subsection{Theoretical argument for Feed Forward Networks}
\label{subsec:rankbound}
We consider a neural network with a $d_{x}$ dimensional input $\vx$. Our network has $H-1$ hidden layers and we refer to the output as the $H$'th layer and the input as the $0$'th layer. We denote the ReLU activation function as $f(x)$ where $f(x) = \max(0,x)$. Let $\mW_{i}$ be the matrix of weights between the $(i-1)$'th and $i$'th layer. For a $d_{y}$ dimensional output our $q$'th component of the output can be written as
\begin{equation}
	\vz(\vx_{i};\vw)_{q} = f(\mW_{H}^{T}f(\mW_{H-1}^{T}....f(\mW_{1}\vx))) = \prod_{l=0}^{H}\sum_{n_{i,l}=1}^{N_{l}}\sum_{i}^{d_{x}}\vx_{i}\vw_{n_{i,l},n_{i,l+1}}
\end{equation}
where $\vw_{n_{i,l},n_{i,l+1}}$ denotes the weight of the path segment connecting node $i$ in layer $l$ with node $i$ in layer $l+1$. layer $l$ has $N_{l}$ nodes. Where $n_{i,l_{0}} = x_{i}$. The Hessian, in the small loss limit tends to
\begin{equation}
	\frac{\partial^{2} \ell(h(\vx_{i};\vw),\vy_{i})}{\partial w_{\phi,\kappa}\partial w_{\theta,\nu}} \rightarrow -\sum_{m\neq c}\exp(h_{m})\bigg[\frac{\partial^{2}h_{m}}{\partial w_{\phi,\kappa}\partial w_{\theta,\nu}}+\frac{\partial h_{m}}{\partial w_{\phi,\kappa}}\frac{\partial h_{m}}{\partial w_{\theta,\nu}}\bigg].
\end{equation}
\begin{equation}
	\label{eq:expandingloss}
	\begin{aligned}
		& \bigg[\frac{\partial^{2}h_{m}}{\partial w_{\phi,\kappa}\partial w_{\theta,\nu}}+\frac{\partial h_{m}}{\partial w_{\phi,\kappa}}\frac{\partial h_{m}}{\partial w_{\theta,\nu}}\bigg] =  \prod_{l=1}^{d-1}\sum_{n_{i,l}\neq [(\phi, \kappa),(\theta, \nu)]}^{N_{i,l}}\sum_{i}^{d_{x}}\vx_{i}\vw_{n_{i,l},n_{i,l+1}} \\
		& +\bigg(\prod_{l=1}^{d-1}\sum_{n_{i,l}\neq (\theta, \nu)}^{N_{i,l}}\sum_{i}^{d_{x}}\vx_{i}\vw_{n_{i,l},n_{i,l+1}}\bigg)
		\bigg(\prod_{l=1}^{d-1}\sum_{n_{j,l}\neq (\phi, \kappa)}^{N_{j,l}}\sum_{i}^{d_{x}}\vx_{i}\vw_{n_{j,l},n_{j,l+1}}\bigg)
		\\
	\end{aligned}
\end{equation}
Each product of weights contributes an object of rank-$1$ (as shown in Section \ref{sec:motivation}). Furthermore, the rank of a product is the minimum of the constituent ranks, i.e. $\text{rank}(AB) = \min \text{rank}(A,B)$. Hence Equation \ref{eq:expandingloss} is rank bounded by $2(\sum_{l}N_{l} + d_{x})$, where  $N_{l}$ is the total number of neurons in the network. By rewriting the loss per-sample, repeating the same arguments and including the class factor, we obtain
\begin{equation}
	\frac{\partial^{2}\ell}{\partial w_{k}\partial w_{l}} = - \frac{\partial^{2}h_{q(i)}}{\partial w_{k}\partial w_{l}} + \frac{\sum_{j}\exp(h_{j})\sum_{i}\exp(h_{i})(\frac{\partial^{2}h_{i}}{\partial w_{k} \partial w_{l}}+\frac{\partial h_{i}}{\partial w_{k}}\frac{\partial h_{i}}{\partial w_{l}})-\sum_{i}\exp(h_{i})\frac{\partial h_{i}}{\partial w_{k}}\sum_{j}\frac{\partial h_{j}}{\partial w_{l}}\exp(h_{j})}{[\sum_{j}\exp(h_{j})]^{2}},
\end{equation}
and thence a rank bound of $4d_{y}(\sum_{l}N_{l} + d_{x})$. To give some context, along with a practical application of a real network and dataset, for the CIFAR-$10$ dataset, the VGG-$16$ \citep{simonyan2014very} contains $1.6 \times 10^{7}$ parameters, the number of classes is $10$ and the total number of neurons is $13,416$ and hence the bound gives us a spectral peak at the origin of at least $1-\frac{577,600}{1.6\times 10^{7}} = 0.9639$.

\subsection{Experimental Validation of Low Rank Approximation}
A full Hessian inversion with computational cost $\mathcal{O}(P^{3})$ is infeasible for large neural networks. Hence, counting the number of zero eigenvalues (which sets the degeneracy) is not feasible in this manner. Furthermore, there would still be issues with numerical precision, so a threshold would be needed for accurate counting. Hence, based on our understanding of the Lanczos algorithm, discussed in Appendix \ref{sec:lanczos}, we propose an alternative method. 
\paragraph{Lanczos:}
We know that $m$ steps of the Lanczos method, gives us an $m$-moment matched spectral approximation of the moments of $\vv^{T}\mH\vv$, where in expectation over the set of zero mean, unit variance, random vectors this is equal to the spectral density of $\mH$. \cite{meurant2006lanczos,ete} Each eigenvalue/eigenvector pair estimated by the Lanczos algorithm is called a Ritz-value/Ritz-vector. We hence take $m\gg1$, where for consistency with \citet{ghorbani2019investigation}
we take $m=100$ in our experiments\footnote{They show that $m=90$ is sufficient for double precision accuracy on an MLP MNIST example}. We then take the Ritz value closest to the origin and take that as a proxy for the zero eigenvalue and report its weight.
\paragraph{Spectral Splitting:}
One weakness of this method is that for a large value of $m$, since the Lanczos algorithm finds a discrete moment matched spectral algorithm, the spectral mass near the origin may split into multiple components. Counting the largest thereof, or closest to the origin, may not be sufficient.  We note this problem both for the PreResNet-$110$ and VGG-$16$ on the CIFAR-$100$ dataset shown in Figure \ref{fig:degenhessprob}. Significant drops in degeneracy occur at various points in training and occur in tandem with significant changes in the absolute value of the Ritz value of minimal magnitude. This suggests the aforementioned splitting phenomenon is occurring. This issue is not present in the calculation of the Generalised Gauss-Newton matrix, as the spectrum is constrained to be positive-definite, so there is a limit to the extent of splitting that may occur. In order to remedy this problem for the Hessian, we calculate the combination of the two closest Ritz values around the centre and combine their mass. We consider this mass, and the weighted average of their values, as the degenerate mass. An alternative approach could be to kernel smooth the Ritz weights at their values, but this would involve another arbitrary hyper-parameter $\sigma$ and hence we do not adopt this strategy.

\begin{figure}[h!]
	\centering
	\begin{subfigure}{0.23\linewidth}
		\includegraphics[width=1\linewidth,trim={0 0 0 0},clip]{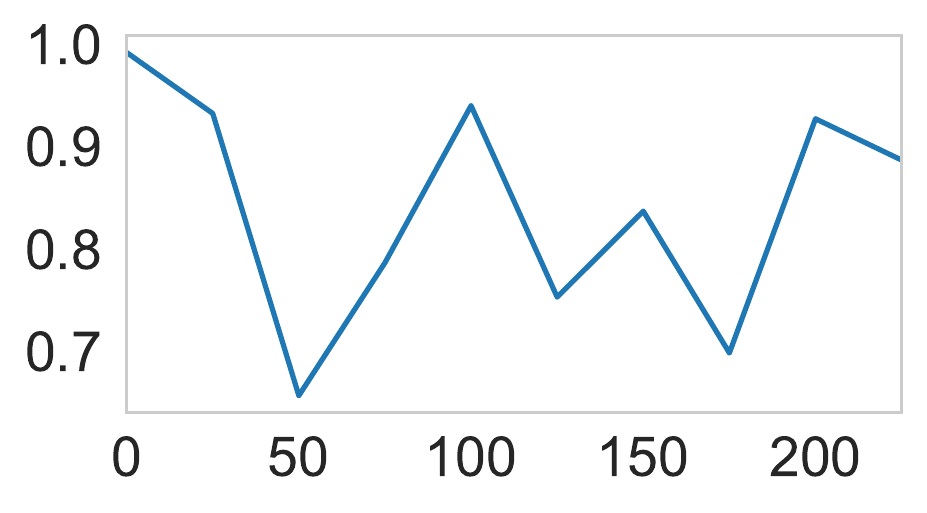}
		\caption{$\mathcal{D}$ weight P110}
		\label{subfig:hessp110prob}
	\end{subfigure}
	\begin{subfigure}{0.23\linewidth}
		\includegraphics[width=1\linewidth,trim={0 0 0 0},clip]{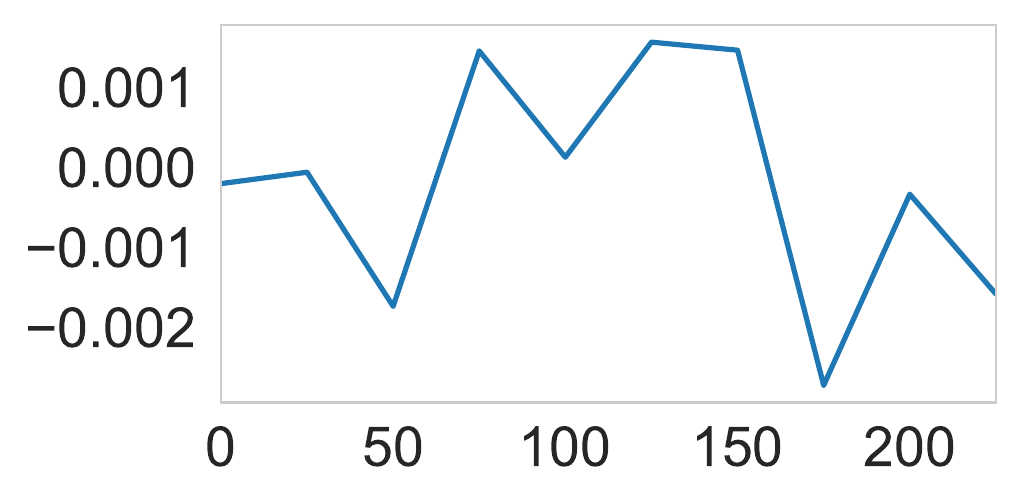}
		\caption{$\mathcal{D}$ value P110}
		\label{subfig:hessp110probv}
	\end{subfigure}
	\begin{subfigure}{0.23\linewidth}
		\includegraphics[width=1\linewidth,trim={0 0 0 0},clip]{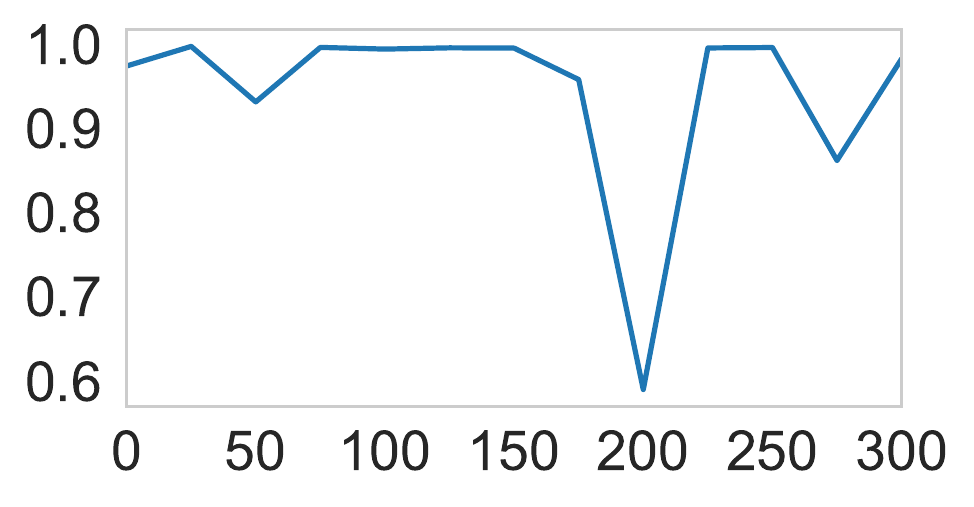}
		\caption{$\mathcal{D}$ weight VGG-$16$}
		\label{subfig:hessvgg16prob}
	\end{subfigure}
	\begin{subfigure}{0.23\linewidth}
		\includegraphics[width=1\linewidth,trim={0 0 0 0},clip]{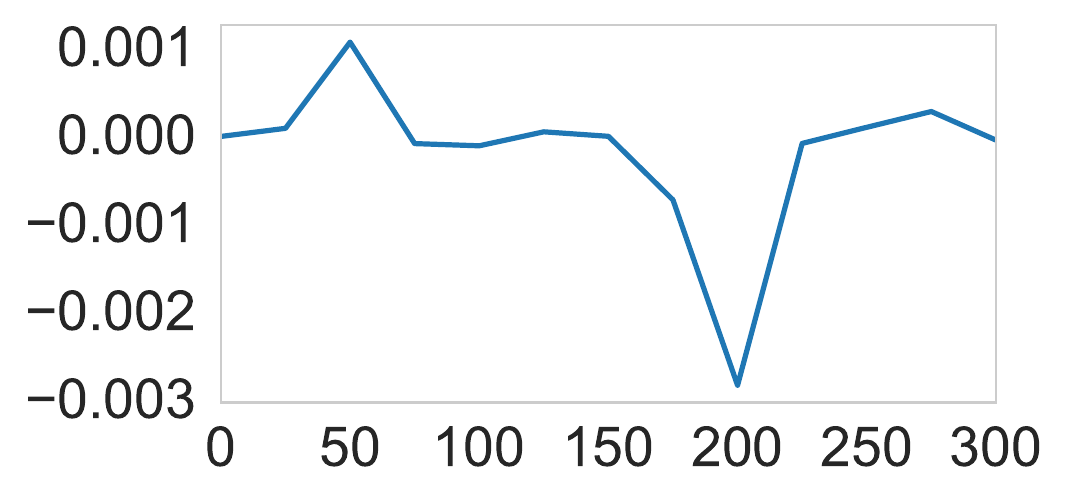}
		\caption{$\mathcal{D}$ value VGG-$16$}
		\label{subfig:hessvgg16probv}
	\end{subfigure}
	\caption{Rank degeneracy $\mathcal{D}$ (proportion of zero eigenvalues) evolution throughout training using the VGG-$16$ and PreResNet-$110$ on the CIFAR-$100$ dataset, the weight corresponds to the spectral mass of the Ritz value(s) considered to correspond to $\mathcal{D}$}
	\label{fig:degenhessprob}
	
\end{figure}

\subsection{VGG16}
\label{subsec:rankapproxvgg16}
For the VGG-$16$ model, which forms the reference model for this paper, we see that for both the Generalised Gauss-Newton matrix (GGN, shown in Figure \ref{subfig:rggndegenvgg16}) and the Hessian (shown in Figure \ref{subfig:rhessdegenvgg16}) the rank degeneracy is extremely high. For the GGN, the magnitude of the Ritz value, which we take to be the origin, is extremely close to the threshold of GPU precision, as shown in Figure \ref{subfig:0ggndegenvgg16}. For the Hessian, for which we combine the two smallest absolute value Ritz values, we find an even larger spectral degeneracy. The weighted average also gives a value very close to $0$, as shown in Figure \ref{subfig:0hessdegenvgg16}. The combined weighted average, however, is much closer to the origin than that of the lone spectral peak, shown in Figure \ref{fig:degenhessprob}, which indicates splitting, we do not get as close to the GPU precision threshold of $10^{-7}$, which we consider as a reasonable level to assume domination by numerical imprecision.

\begin{figure}[h!]
	\centering
	\begin{subfigure}{0.23\linewidth}
		\includegraphics[width=1\linewidth,trim={0 0 0 0},clip]{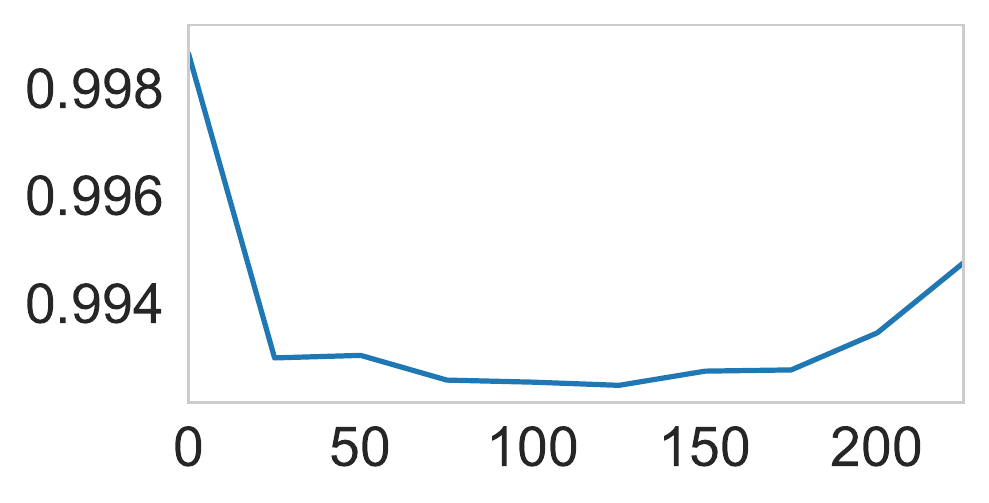}
		\caption{GGN Degeneracy}
		\label{subfig:rggndegenvgg16}
	\end{subfigure}
	\begin{subfigure}{0.23\linewidth}
		\includegraphics[width=1\linewidth,trim={0 0 0 0},clip]{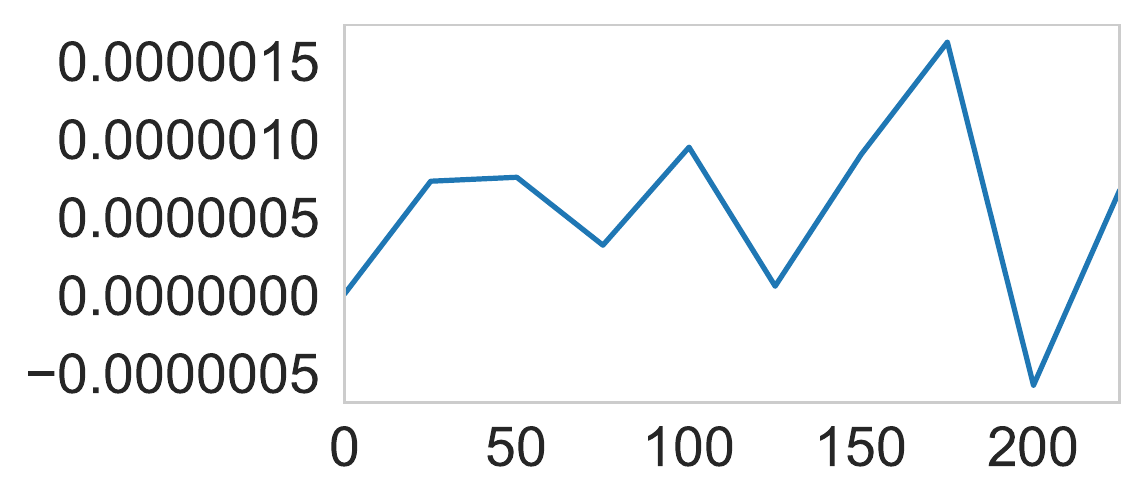}
		\caption{GGN Ritz Value}
		\label{subfig:0ggndegenvgg16}
	\end{subfigure}
	\begin{subfigure}{0.23\linewidth}
		\includegraphics[width=1\linewidth,trim={0 0 0 0},clip]{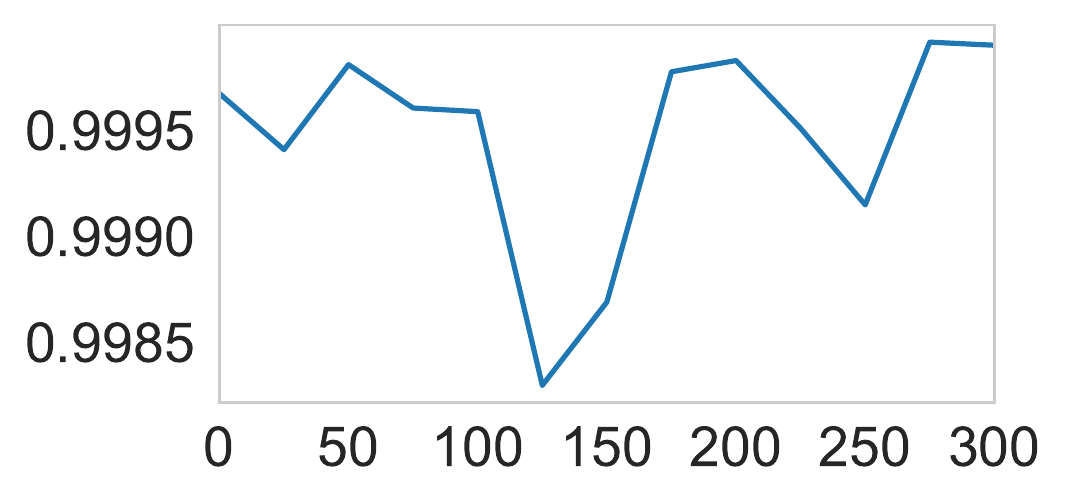}
		\caption{Hessian Degeneracy}
		\label{subfig:rhessdegenvgg16}
	\end{subfigure}
	\begin{subfigure}{0.23\linewidth}
		\includegraphics[width=1\linewidth,trim={0 0 0 0},clip]{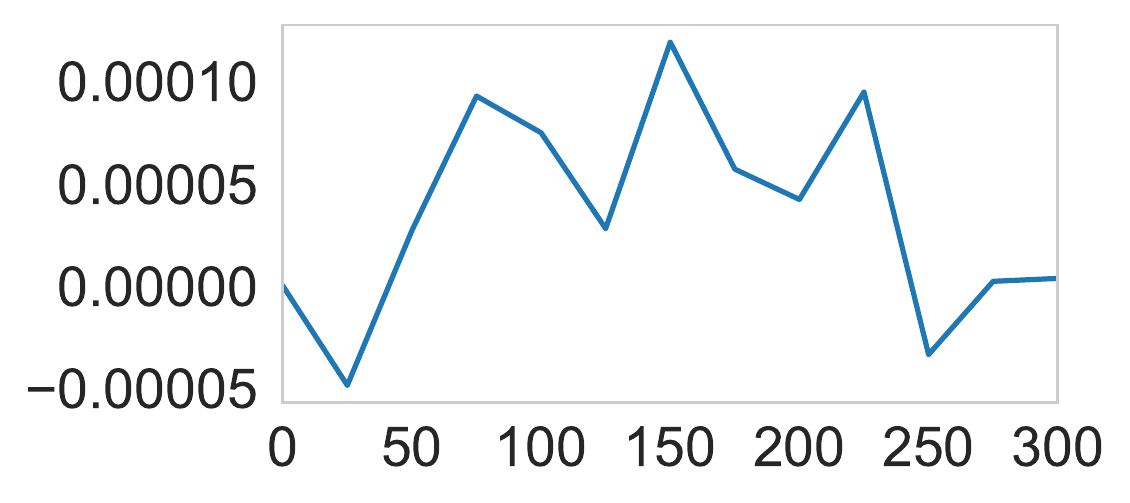}
		\caption{Hessian Ritz Value}
		\label{subfig:0hessdegenvgg16}
	\end{subfigure}
	\caption{Rank degeneracy (proportion of zero eigenvalues) evolution throughout training using the VGG-$16$ on the CIFAR-$100$ dataset, total training $225$ epochs, the Ritz value corresponds to the value of the node which we assign to $0$.}
	\label{fig:degenvgg16}
	
\end{figure}
\subsection{PreResNet110}
\label{subsec:rankapproxp110}
We repeat the same experiments in Section \ref{subsec:rankapproxvgg16} for the preactivated residual network with $110$ layers, on the same dataset. Note that, as explained in Section \ref{sec:batchnormresults}, we can calculate the spectra in both batch normalisation and evaluation mode. Hence we report results for both, with the main finding that the empirical Hessian spectra are consistent with large rank degeneracy.
\begin{figure}[h!]
	\centering
	\begin{subfigure}{0.23\linewidth}
		\includegraphics[width=1\linewidth,trim={0 0 0 0},clip]{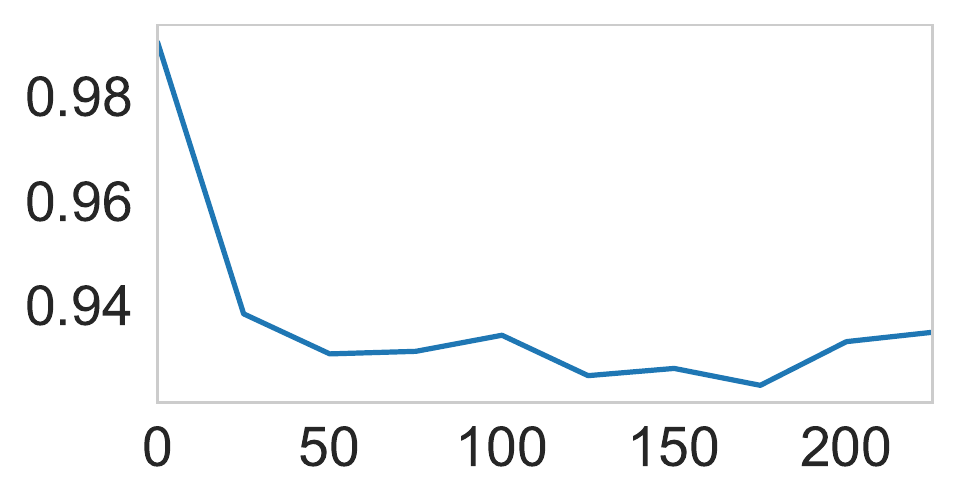}
		\caption{BN-train Degen}
		\label{subfig:ggndegenbnon}
	\end{subfigure}
	\begin{subfigure}{0.23\linewidth}
		\includegraphics[width=1\linewidth,trim={0 0 0 0},clip]{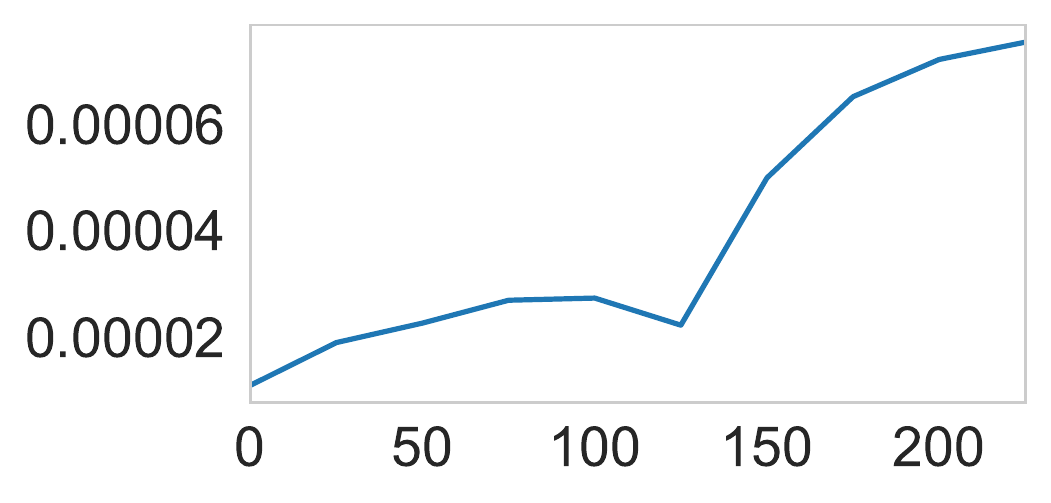}
		\caption{BN-train Ritz Val}
		\label{subfig:ggndegenbnonval}
	\end{subfigure}
	\begin{subfigure}{0.23\linewidth}
		\includegraphics[width=1\linewidth,trim={0 0 0 0},clip]{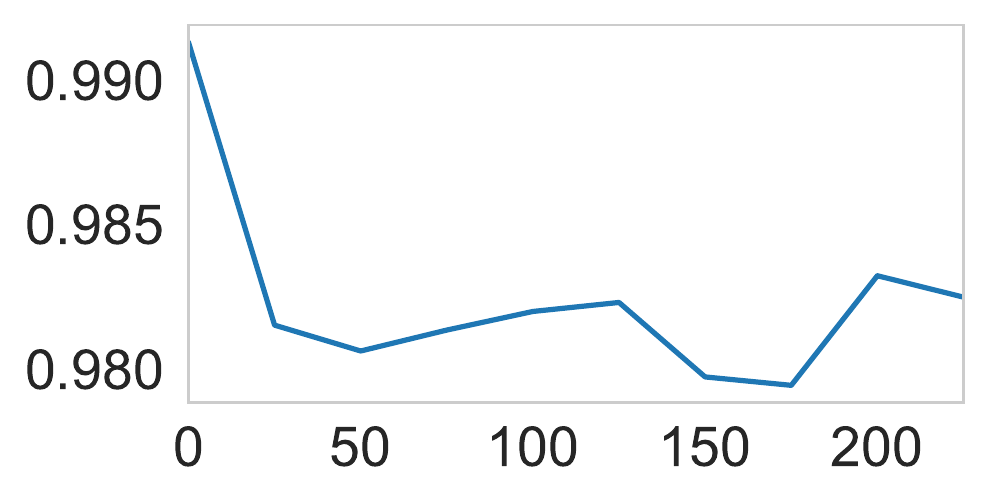}
		\caption{BN-eval Degen}
		\label{subfig:ggndegenbnoff}
	\end{subfigure}
	\begin{subfigure}{0.23\linewidth}
		\includegraphics[width=1\linewidth,trim={0 0 0 0},clip]{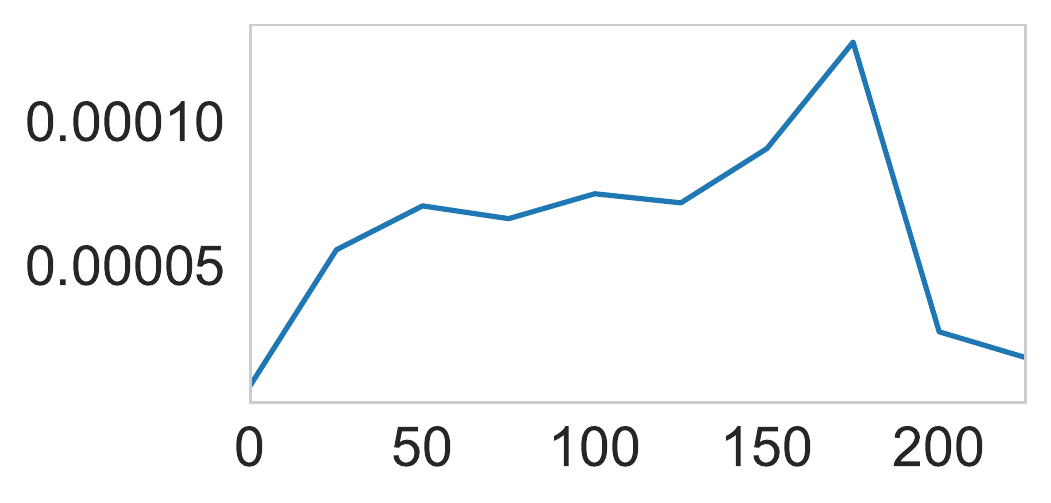}
		\caption{BN-eval Ritz Val}
		\label{subfig:ggndegenbnoffval}
	\end{subfigure}
	\caption{Generalised Gauss-Newton matrix rank degeneracy (proportion of zero eigenvalues) evolution throughout training using the PreResNet-$110$ on the CIFAR-$100$ dataset, total training $225$ epochs, the Ritz value corresponds to the value of the node which we assign to $0$.}
	\label{fig:ggndegen}
	
\end{figure}
\begin{figure}[h!]
	\centering
	\begin{subfigure}{0.23\linewidth}
		\includegraphics[width=1\linewidth,trim={0 0 0 0},clip]{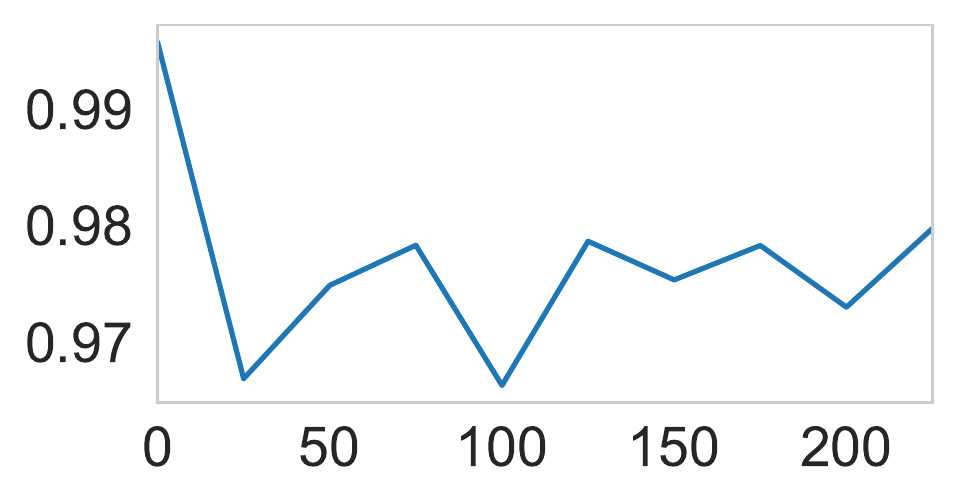}
		\caption{BN-train Degen}
		\label{subfig:hessdegenbnon}
	\end{subfigure}
	\begin{subfigure}{0.23\linewidth}
		\includegraphics[width=1\linewidth,trim={0 0 0 0},clip]{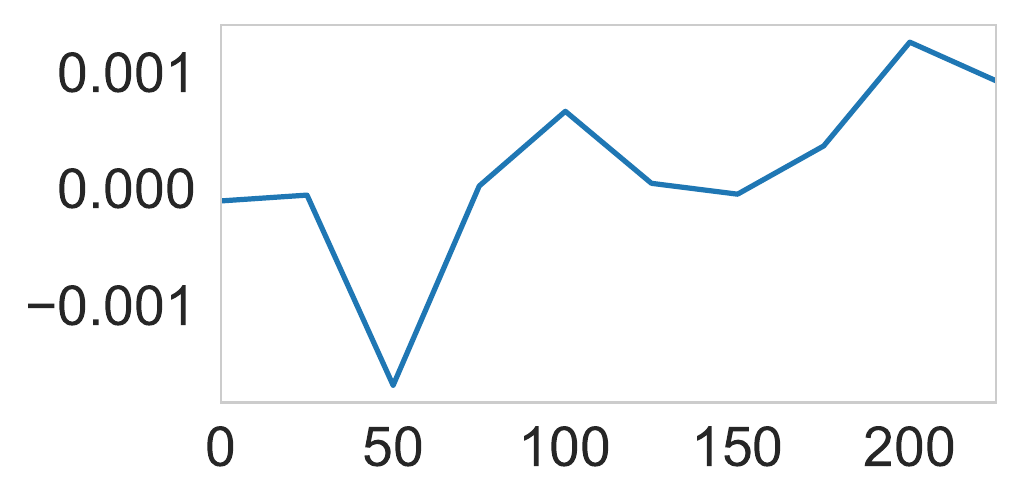}
		\caption{BN-train Ritz Val}
		\label{subfig:hessdegenbnonval}
	\end{subfigure}
	\begin{subfigure}{0.23\linewidth}
		\includegraphics[width=1\linewidth,trim={0 0 0 0},clip]{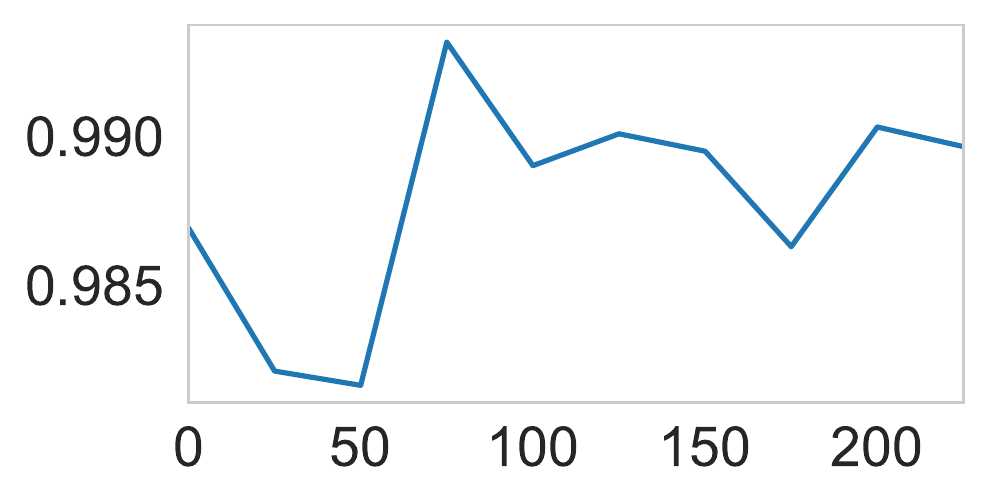}
		\caption{BN-eval Degen}
		\label{subfig:hessdegenbnoff}
	\end{subfigure}
	\begin{subfigure}{0.23\linewidth}
		\includegraphics[width=1\linewidth,trim={0 0 0 0},clip]{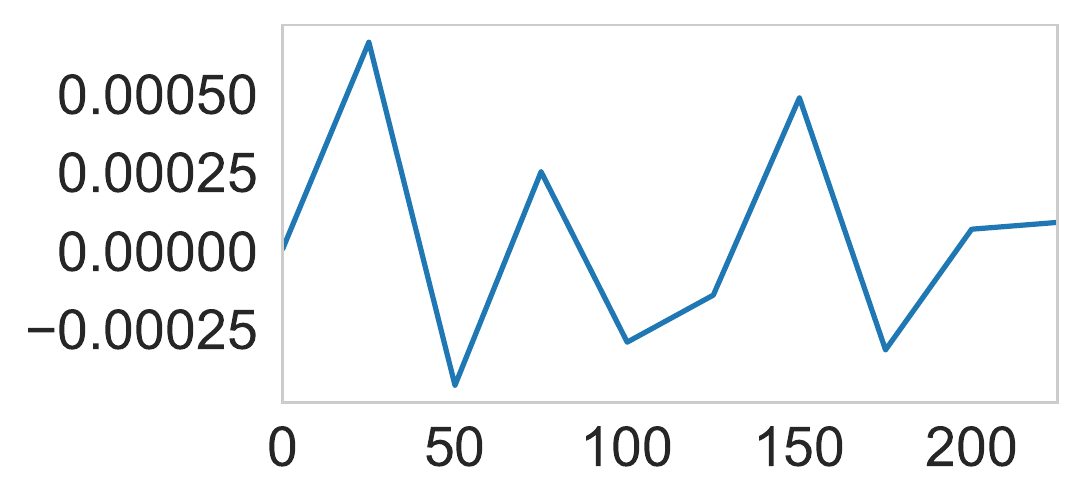}
		\caption{BN-eval Ritz Val}
		\label{subfig:hessdegenbnoffval}
	\end{subfigure}
	\caption{Hessian rank degeneracy (proportion of zero eigenvalues) evolution throughout training using the PreResNet-$110$ on the CIFAR-$100$ dataset, total training $225$ epochs, the Ritz value corresponds to the value of the node which we assign to $0$.}
	\label{fig:hessdegen}
\end{figure}
\end{Correction}

\section{Experimental Validation of the Theoretical Results}
\label{sec:experiments}
\begin{Correction}
In this section we run experiments to explictly test the validity of our derived theorems, for which we then develop practical algorithms and scaling rules in the coming sections.
\paragraph{Experimental Setup:}We use the GPU powered Lanczos quadrature algorithm \citep{gardner2018gpytorch, meurant2006lanczos}, with the Pearlmutter trick \citep{pearlmutter1994fast} for Hessian and GGN vector products, using the PyTorch \citep{paszke2017automatic} implementation of both Stochastic Lanczos Quadrature and the Pearlmutter. We then train a 16 Layer VGG CNN \citep{simonyan2014very} with $P=15291300$ parameters 
on the CIFAR-$100$ dataset (45,000 training samples and 5,000 validation samples) using SGD and K-FAC optimisers. For both SGD and K-FAC, we use the following  learning rate schedule:
\begin{equation}
	\label{eq:schedule}
	\alpha_t = 
	\begin{cases}
		\alpha_0, & \text{if}\ \frac{t}{T} \leq 0.5 \\
		\alpha_0[1 - \frac{(1 - r)(\frac{t}{T} - 0.5)}{0.4}] & \text{if } 0.5 < \frac{t}{T} \leq 0.9 \\
		\alpha_0r, & \text{otherwise.}
	\end{cases}
\end{equation}
We use a learning rate ratio $r=0.01$ and a total number of epochs budgeted $T=300$. We further use momentum set to $\rho=0.9$, a weight decay coefficient of $0.0005$ and data-augmentation on PyTorch \citep{paszke2017automatic}. We set the inversion frequency to be once per 100 iterations for K-FAC.
\end{Correction}

\paragraph{Advantages of the VGG architecture:}
For simplicity, we do not analyse the added dependence between curvature and the samples due to batch normalisation \citep{ioffe2015batch} and hence adopt as our reference model the VGG-$16$ \citep{simonyan2014very} on the CIFAR-$100$ dataset which does not utilise batch normalisation. We show in Appendix \ref{sec:batchnormresults} that many of our results also hold with batch-normalisation for ResNet architectures. We also include further results for the WideResNet architecture and the ImageNet-$32$ dataset.

\begin{Correction}
\paragraph{Estimating the Spectrum and Extremal Eigenvalues using the Lanczos Algorithm:} To plot the spectrum of the neural network we use the approach of \citet{granziol2019mlrg}, which gives a discrete, moment-matched approximation to the underlying spectrum. We use $m=100$ as the number of moments.
As discussed in \citet{granziol2019mlrg} the spectrum can be estimated consistently even using a single random vector, due to the high dimensionality of the neural network (large number of parameters). Whilst accurate bounds on the moments of the spectrum can be derived using stochastic Lanczos quadrature \citep{ubaru2017fast}, we note that these bounds are considered very loose and pessimistic \citep{ete,diego}. Whether a spectrum, or more generally a density, can be accurately estimated using its moments is known as the Hausdorff moment problem \citep{hausdorff1921summationsmethoden}. It can be shown \citep{diego} that finite matrices (such as the Hessians of Neural Networks) satisfy these conditions. Hence there can be no surprises from "bad pathological spectra" in this case. Note that in the case of infinite matrices, we would need to have bounded moment conditions and hence finite eigenvalues, but this is not relevant for our measurements here.
\end{Correction}

\subsection{Effect of spectral broadening for a typical batch size}

We plot an example effect of the spectral broadening of the Hessian due to mini-batching, for a typical batch size of $B=128$ in Figure \ref{fig:hesspert}. \emph{The magnitude of the extremal eigenvalues are significantly increased as are other outlier eigenvalues, such as the second largest}. We estimate the mean of the continuous region (bulk) of the spectrum as the position where the Ritz\footnote{This is the term used by approximate eigenvalue/eigenvector pairs by the Lanczos algorithm, as detailed in Appendix \ref{sec:lanczos}.} weight drops below $1/P$. We see that the spectral width of this continuous region also increases.
\begin{figure}[h]
\centering
\begin{subfigure}{0.49\linewidth}
	\includegraphics[trim={0cm 0cm 0cm 0cm},clip, width=1\textwidth]{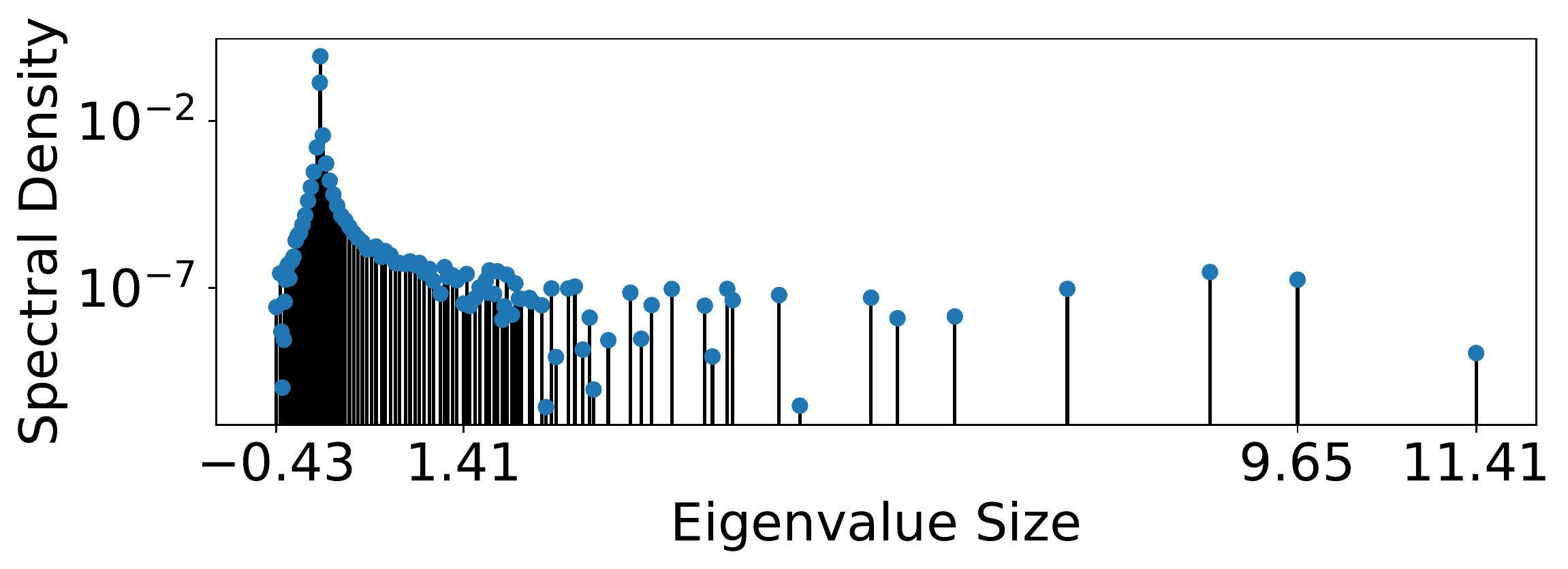}
	\caption{Empirical Hessian spectrum $N=50,000$}
	\label{subfig:vgg16batchhess}
\end{subfigure}
\begin{subfigure}{0.49\linewidth}
	\includegraphics[trim={0cm 0cm 0cm 0cm},clip, width=1\textwidth]{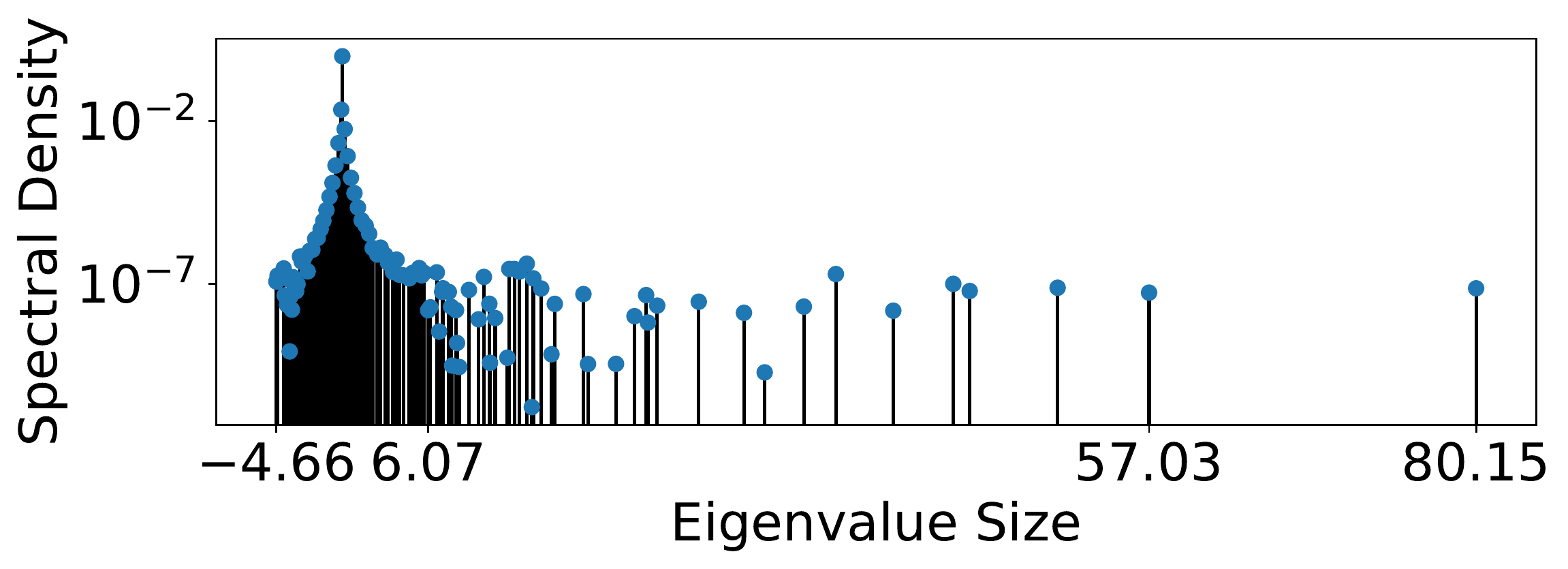}
	
	\caption{Batch Hessian spectrum $B = 128$}
	\label{subfig:vgg16emphess}
\end{subfigure}
\caption{Spectral Density of the Hessian at epoch $200$, for different sample sizes $B,N$ on a VGG-$16$ on the CIFAR-$100$ dataset. The Y-axis corresponds to $p(\lambda)$ and the X-axis to $\lambda$. The initial learning rate used is $\alpha = 0.05$, with momentum $\rho=0.9$ and weight decay $0.0005$, using the learning rate schedule in Section \ref{sec:applictions}.}
\label{fig:hesspert}

\end{figure}
We plot an example of the Generalised Gauss-Newton matrix in Figure \ref{fig:ggnpert}, which for cross-entropy loss and softmax activation is equal to the Fisher information matrix \citep{pascanu2013revisiting}. We observe identical behaviour of bulk and outlier broadening.
\begin{figure}[htbp]
\centering
\begin{subfigure}{0.48\linewidth}
	\includegraphics[trim={0cm 0cm 0cm 0cm},clip, width=1\textwidth]{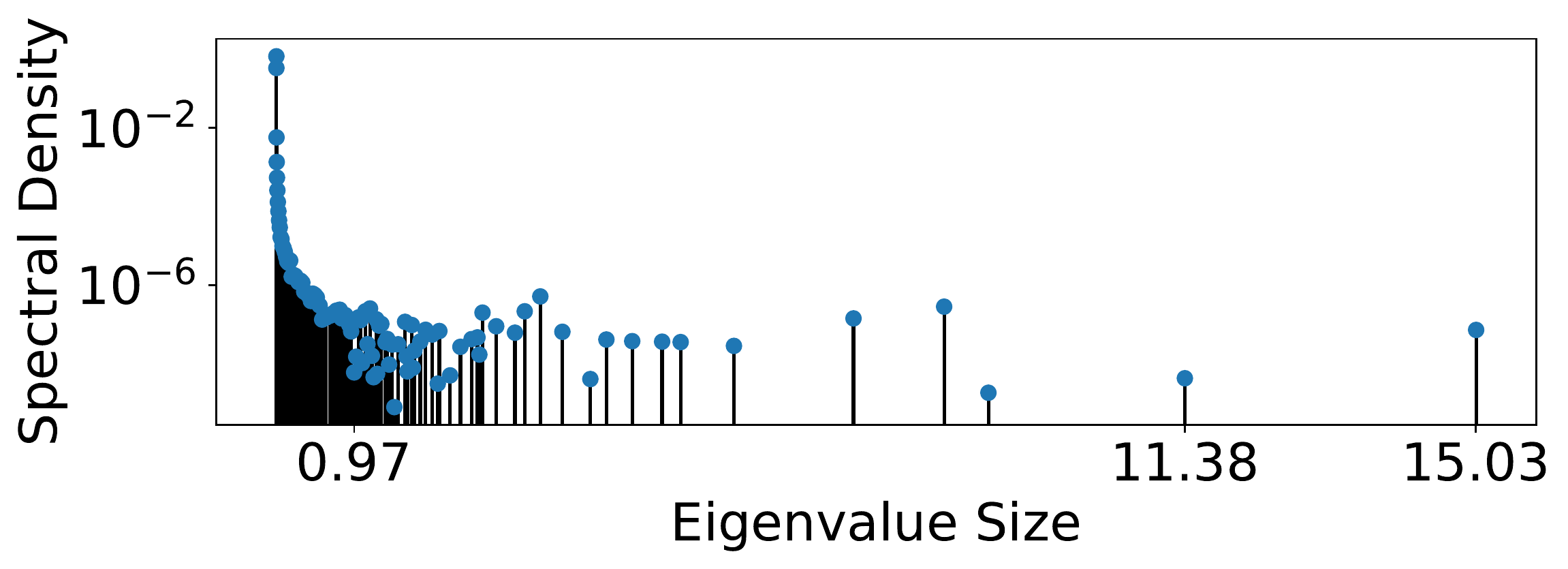}
	
	\caption{Empirical GGN spectrum $N=50,000$}
	\label{subfig:vgg16batchggn}
\end{subfigure}
\begin{subfigure}{0.48\linewidth}
	\includegraphics[trim={0cm 0cm 0cm 0cm},clip, width=1\textwidth]{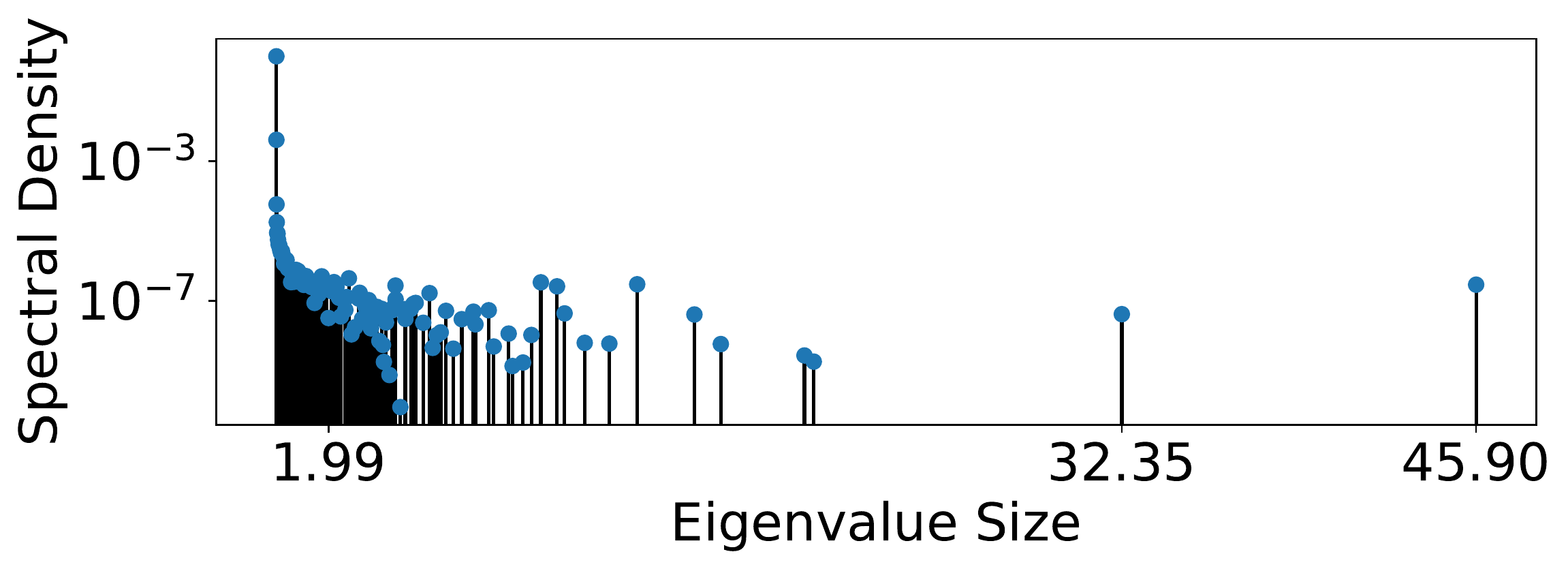}
	
	\caption{Batch GGN spectrum $B = 128$}
	\label{subfig:vgg16empggn}
\end{subfigure}
\caption{Spectral Density of the Generalised Gauss-Newton matrix (GGN) at epoch $25$ ,for different sample sizes $B,N$,  on a VGG-$16$ on the CIFAR-$100$ dataset. The Y-axis corresponds to $p(\lambda)$ and the X-axis to $\lambda$. The initial learning rate used is $\alpha = 0.05$, with momentum $\rho=0.9$ and weight decay $0.0005$, using the learning rate schedule in Section \ref{sec:applictions}.}
\label{fig:ggnpert}
\end{figure}

\subsection{Measuring the Hessian Variance}
We estimate the variance of the Hessian/GGN using stochastic trace estimation \citep{hutchinson1990stochastic,diego} in Algorithm \ref{alg:hessvar}, from which the variance per element can be inferred. 
\textcolor{black}{
Note that under the assumptions of our model, which assumes that the batch Hessian is either a deterministic full Hessian plus a stochastic fluctuations matrix (Theorem \ref{theorem:mainresult}), or alternatively the product of a stochastic fluctuations matrix and a deterministic modified Jacobian (Theorem \ref{theorem:batchtheoremgnn}), the variance of the elements of the Hessian directly leads us to the variance of the elements of the fluctuations matrix.
}
We plot the evolution of the Hessian/GGN variance throughout an SGD training cycle in Figure \ref{fig:hessvarcurve},
\textcolor{black}{
where we observe a slow initial growth, followed by explosive growth during learning rate reduction (from epoch $161$ onwards) and then reduction when the learning rate is held fixed at a low value (from epoch $270$ onwards). Because the variance of the Hessian massively increases in the later part of training (from epoch $161$ onwards) and the variance of the Hessian determines the variance of the elements of the fluctuations matrix (because the full Hessian is deterministic).
}
This Figure implies that we expect the batch Hessian extremal eigenvalues to diverge from those of the empirical Hessian during training. \textcolor{black}{
By `diverge' we specifically mean substantially larger in magnitude. This is exactly what we see in practice in Figures \ref{subfig:sgdvgg16batchhess} and \ref{subfig:sgdvgg16batchggn} for both the Hessian and the Generalised Gauss Newton. Here we plot the batch Hessian maximum eigenvalues (Batch Maxval) using a batch size of $B=128$ against the full Hessian maximum eigenvalues (Full MaxVal) over the course of training a VGG-$16$ on CIFAR-$100$. We track the Hessian variance over the trajectory to make our predictions (shown as Pert Maxval). We calculate the perturbation prediction using Theorems \ref{theorem:mainresult} and \ref{theorem:batchtheoremgnn}, where $\sigma_{\epsilon}$ is calculated using Algorithm \ref{alg:hessvar}. The full Hessian maximum eigenvalue used for the theorems and plotted is derived from using the Lanczos algorithm on the full dataset $N$. We take the average of $10$, $B=128$ batch Hessian extremal eigenvalues. We shade in the $\pm$ standard deviation of our stochastic Batch Hessian and Batch Generalised Gauss Newton eigenvalues. We also repeat the same experiment for the KFAC optimiser and show similar results, pertaining to the difference between the full and batch eigenvalues along with the ability to predict them in Figures \ref{subfig:kfacvgg16batch}, \ref{subfig:kfacvgg16batchggn}. Whilst the goal of this section is to show that Theorems \ref{theorem:mainresult},\ref{theorem:batchtheoremgnn} are accurate and representative, we note that potentially accurate and cheap estimates of the full Hessian spectral norm could be calculated using an inverse procedure, whereby we calculate the spectral norm on a data subset and then, considering the Hessian variance within a subset, estimate the full Hessian spectral norm. 
}
\begin{table}[h]
\begin{tabular}{cc}
	\begin{minipage}{.45\textwidth}
		\begin{algorithm}[H]
			\caption{Calculate Hessian Variance}
			\label{alg:hessvar}
			\begin{algorithmic}[1]
				\STATE {\bfseries Input:} Sample Hessian $\mH_{i}\in \mathbb{R}^{P\times P}$
				\STATE {\bfseries Output:} Hessian Variance $ \sigma^{2}$
				\STATE $\vv \in \mathbb{R}^{1\times P} \sim \mathcal{N}(\boldsymbol{0}, \mI)$
				\STATE Initialise $\sigma^{2}=0, i = 0$, $\vv \leftarrow \vv/||\vv||$
				\FOR{$i < N$}
				\STATE $\sigma^{2} \leftarrow \sigma^{2} + \vv^{T}\mH_{i}^{2}\vv$
				\STATE $i \leftarrow i + 1$
				\ENDFOR
				\STATE $\sigma^{2} \leftarrow \sigma^{2} - [\vv^{T}(1/N\sum_{j=1}^{N}\mH_{j})\vv]^{2}$
			\end{algorithmic}
		\end{algorithm}
		\vspace{0pt}
		\caption{Algorithm which estimates the central quantity $\sigma_{\meps}^{2}$ in Theorems~\ref{theorem:mainresult} \& ~\ref{theorem:batchtheoremgnn}.}
	\end{minipage}&
	\begin{minipage}{.52\textwidth}
		\begin{figure}[H]
			\centering
			\includegraphics[trim={0cm 0cm 0cm 0cm},clip, width=0.97\textwidth]{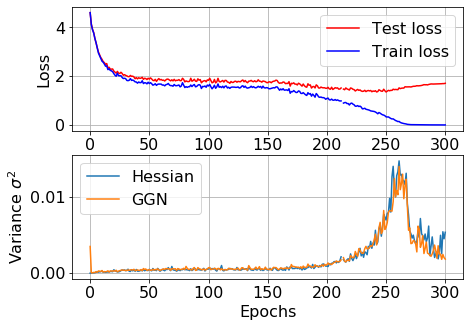}
			\vspace{0pt}
			\caption{Loss/variance evolution during SGD training for VGG-$16$ CIFAR-$100$. Learning Rate schedule specified in Sec~\ref{sec:applictions}.}
			\label{fig:hessvarcurve}
		\end{figure}
	\end{minipage}
\end{tabular}
\end{table}
\begin{figure}[h!]
\centering
\begin{subfigure}{0.24\linewidth}
	\includegraphics[trim={0cm 0cm 0cm 0cm},clip, width=1\textwidth]{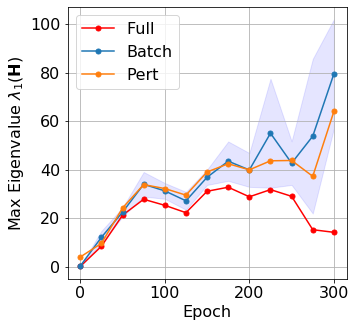}
	
	\caption{KFAC $\lambda_{1}(\mH)$}
	\label{subfig:kfacvgg16batch}
\end{subfigure}
\begin{subfigure}{0.24\linewidth}
	\includegraphics[trim={0cm 0cm 0cm 0cm},clip, width=1\textwidth]{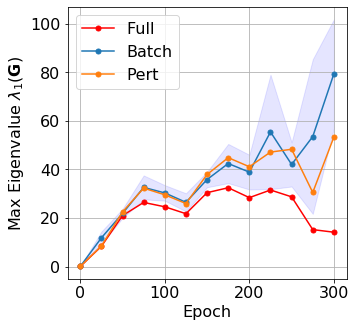}
	
	\caption{KFAC $\lambda_{1}(\mG)$}
	\label{subfig:kfacvgg16batchggn}
\end{subfigure}
\begin{subfigure}{0.24\linewidth}
	\includegraphics[trim={0cm 0cm 0cm 0cm},clip, width=1\textwidth]{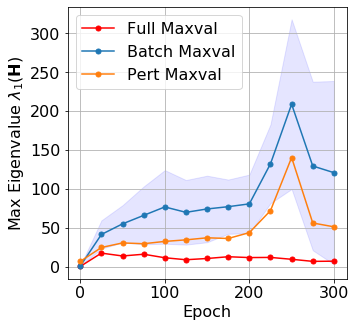}
	
	\caption{SGD $\lambda_{1}(\mH)$}
	\label{subfig:sgdvgg16batchhess}
\end{subfigure}
\begin{subfigure}{0.24\linewidth}
	\includegraphics[trim={0cm 0cm 0cm 0cm},clip, width=1\textwidth]{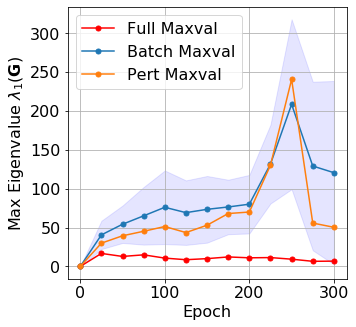}
	
	\caption{SGD $\lambda_{1}(\mG)$}
	\label{subfig:sgdvgg16batchggn}
\end{subfigure}
\caption{Evolution of the maximal eigenvalue $\lambda_{1}$ for both the Hessian $\mathbf{H}$ and the GGN matrix $\mathbf{G}$, during SGD and KFAC training on VGG-$16$ using the CIFAR-$100$ dataset. Full, Batch and Pert refer to the full, batch and the theoretically predicted Hessian eigenvalues respectively. The initial learning rate used is $\alpha = 0.05$, with momentum $\rho=0.9$ and weight decay $0.0005$ for SGD and using $\alpha=0.003$ and decoupled weight decay $0.01$ for KFAC, both using the learning rate schedule in Section \label{sec:applictions}.}
\label{fig:perttest}
\end{figure}

\begin{Correction}
\subsubsection{How important is stochasticity?}
\label{subsec:stochasticrmt}
The batch Hessian extremal eigenvalues have a large variance. This is to be expected, as our results are in the limit of $P,B \rightarrow \infty$ and corrections for finite $B$ scale as $B^{-1/4}$ for matrices with finite $4$th moments \citep{bai2008convergence}, which is $\approx 30 \%$ for $B=128$. Both the theoretical results from the additive noise process (Theorem \ref{theorem:mainresult}) and multiplicative noise process (Theorem \ref{theorem:batchtheoremgnn}) are within 1 standard deviation from the true result. They both follow the increase in variance of the Hessian in Figure \ref{fig:hessvarcurve}. 
We note that the multiplicative noise process provides a better fit. Recent work shows the Hessian outliers to be attributable to the GGN matrix component of the spectrum \citep{papyan2018full}. Hence a positive semi-definite noise process, tailored to the GGN matrix, would be expected to better estimate the outlier perturbations due to mini-batching - which we observe.

\section{Test Accuracy and Movement in the Loss Surface}
In this section we bring together intuitions on generalisation, flat minima and distance from the initialisation point. We argue that in the case that there are exponentially many local minima very close in error to that of the global minimum on the training set, similar curves on the validation not training set may give greater ability to discern whether we are appropriately scaling our learning rates with batch size. We argue that if we want to escape similarly sharp minima into similarly flat minima, we need to scale our learning rates by the decrease/increase in sharpness resulting from our increase/decrease in sub-sampling respectively. How to scale the learning rate with batch size forms the study of our next sections. 
\paragraph{Large Learning Rates and their Uses}
\label{sec:testaccrmt}
Large learning rates have been shown to induce implicit regularisation, observed in \cite{li2019towards}. In contrast, too small learning rates have been shown to lead to poor generalisation \citep{jastrzkebski2017three,berrada2018deep}. We show an example in Figure \ref{fig:smalllrtraintest} where a smaller learning rate for Adam quickly trains worse but generalises better. This is despite the fact that we train with no weight decay, hence there is no confounding $(1-\alpha\gamma)$ decay factor which depends on the learning rate. Therefore, \emph{learning the largest stable learning rate is an important practical question for neural network training}.
\begin{figure}[h!]
    \centering
	\includegraphics[width=0.8\textwidth]{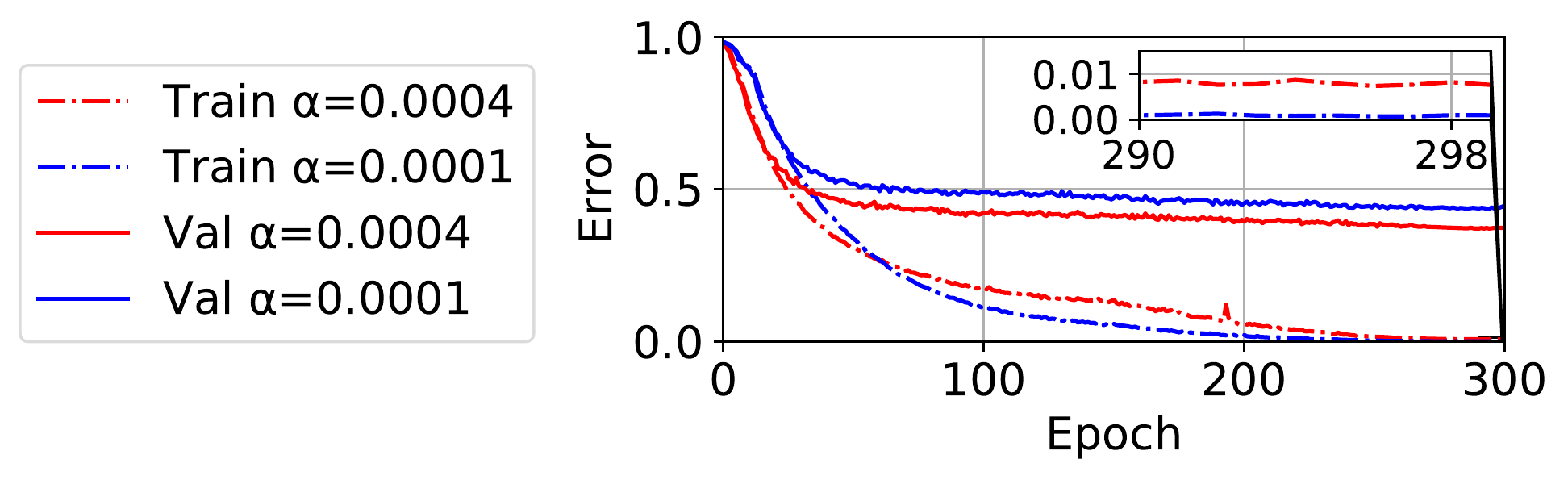}
	\caption{\textbf{Smaller learning rates train faster but generalise worse} Training and Validation Error on the VGG-$16$ on CIFAR-$100$ with different initial learning rates following the same learning rate schedule with no weight decay.}
	\label{fig:smalllrtraintest}
\end{figure}
There is no definitive answer on why large learning rates seem to correlate with better generalisation and this topic remains an active area of research. However, one potential intuitive explanation for this phenomenon is that many minima, equivalently or similarly deep in the training loss surface, may have different characteristics on the true loss surface. This is indicated by performance on the validation or test set, which can be considered unbiased estimates of the True Risk/Error. Such minima might be "flatter", generalising better under both Bayesian and minimum description length arguments \citep{hochreiter1997flat,jastrzkebski2017three,dinh2017sharp}. These arguments can be extended to include parameterisation invariance \citep{tsuzuku2020normalized}, which makes the correlation between sharpness and test accuracy more robust. 
\begin{figure}[t!]
	\centering
	\begin{subfigure}{0.12\textheight}
		\includegraphics[trim={2cm 0cm 1.2cm 0cm},clip, width=1\textwidth]{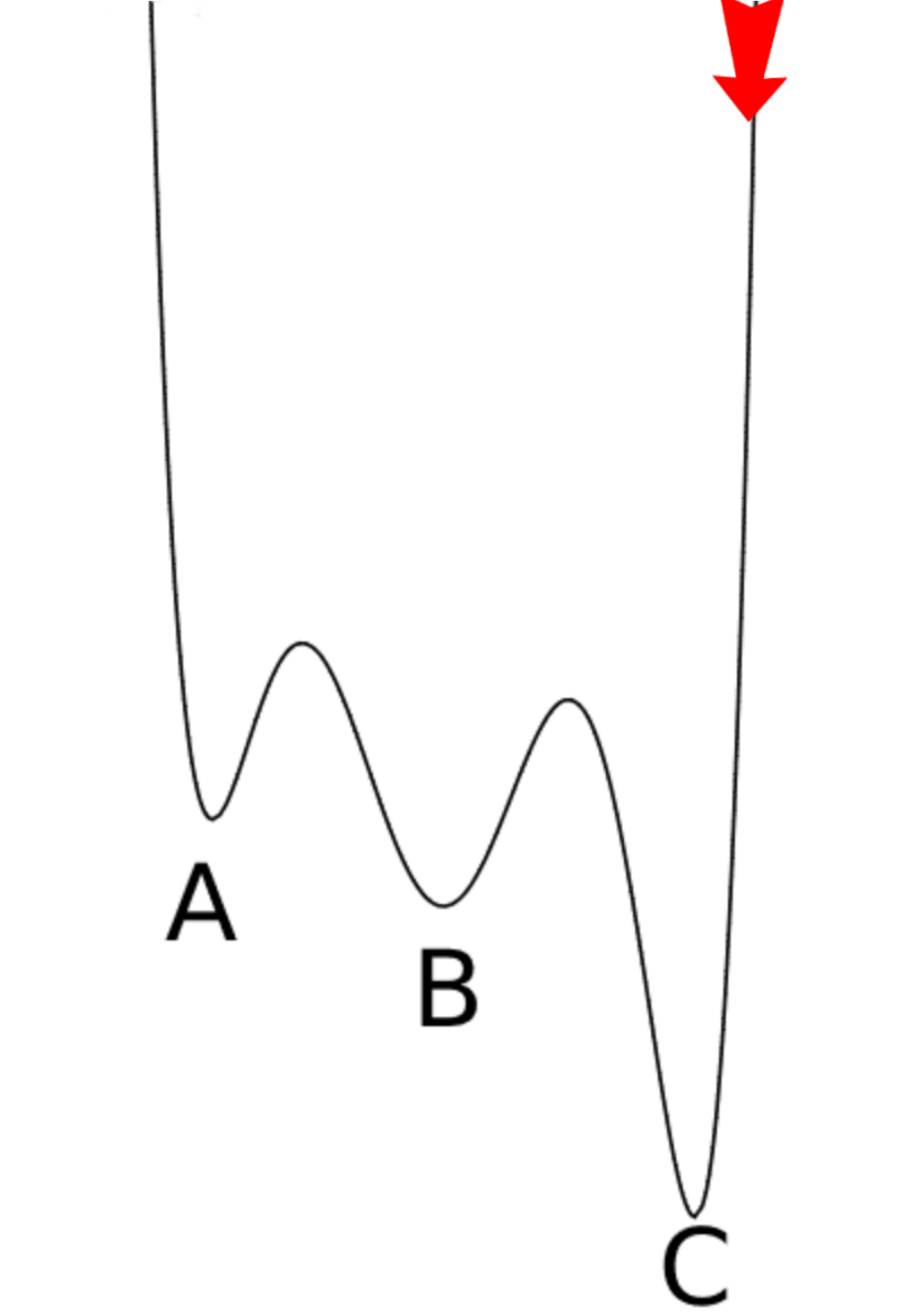}
		\caption{Small $B$}
		\label{subfig:tightcurve}
	\end{subfigure}
	\begin{subfigure}{0.24\textheight}
		\includegraphics[trim={0.6cm 0cm 0cm 0cm},clip, width=1\textwidth]{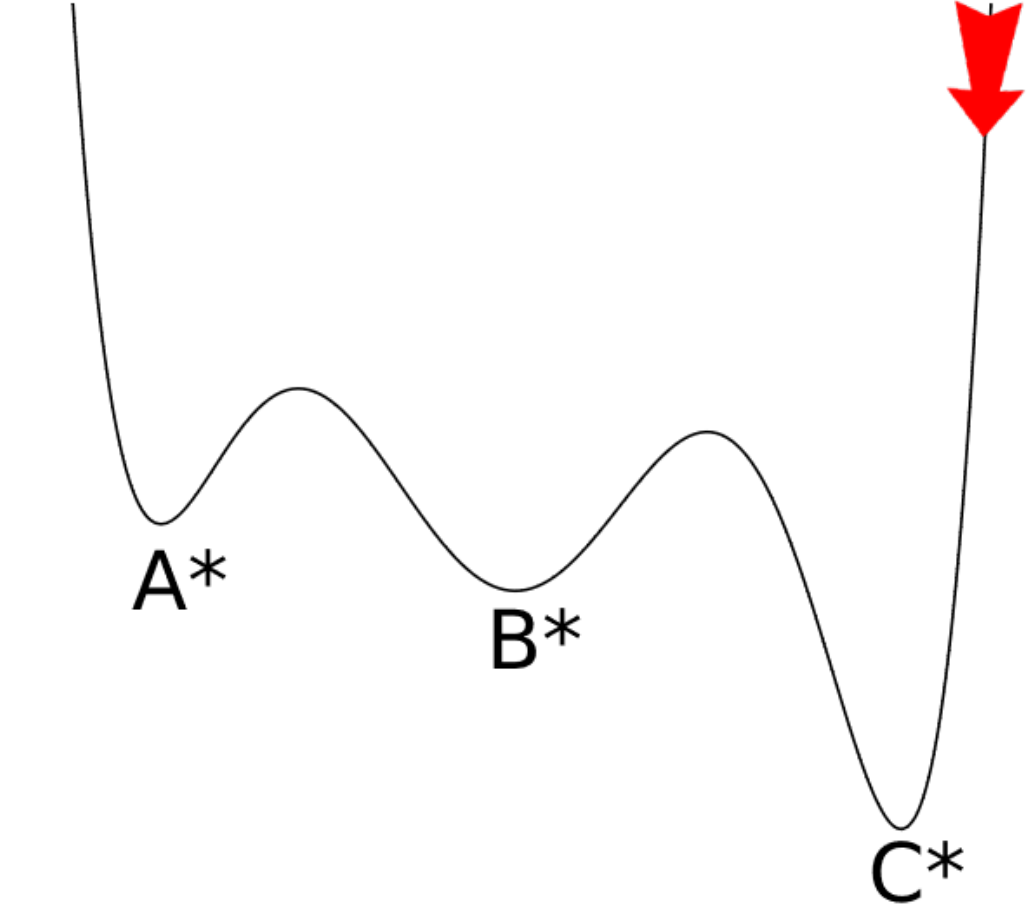}
		\caption{Large $B$}
		\label{subfig:widecurve}
	\end{subfigure}
	\begin{subfigure}{0.25\textheight}
		\tiny
		\begin{tabular}{@{}lllll@{}}
			\toprule
			$\alpha$ &Train Acc & Val Acc & Test Acc & $||\Delta_{init}||$ \\ \midrule
			0.08                 & $99.7 \%$ & $64.35\%$ & $65.36\%$  & 145.76                              \\
			0.028                  & $99.89 \%$& $61.08\%$ & $62.45\%$  & 67.44                               \\ \bottomrule
		\end{tabular}
		\label{table:distvgg}
		\caption{\textbf{Larger distances from initalisation in weight space give solutions of greater generalisation.} Various Learning Rates $\alpha$ for $B=1024$, with corresponding val/test accuracies and $L_{2}$ distance from initialisation for CIFAR-$100$ on the VGG-$16$.}
	\end{subfigure}
	\caption{Transformed Test Set Surface, going from sharper to flatter with an increase in Batch Size. Larger learning rates are more able to escape local poor quality minima which are close to the initialisation.}
\end{figure}
In this case, the practice of large learning rate SGD (or any other optimiser) can be viewed as simulated annealing, where we move around the loss surface, limiting our ability to be trapped early on in a sharp local minimum, as the maximal sharpness of the minimum in which we can be trapped is inversely proportional to our learning rate \citep{wu2018sgd}. Another conjecture, from \citet{hoffer2017train}, considers minima which have a greater distance from the initialisation surface, to generalise better. We visualise both of these concepts in Figure \ref{subfig:tightcurve}, which can be considered a one dimensional slice in the high dimensional surface. If we start with a small learning rate, we settle in minimum $C$, which is deep (i.e low training loss) but sharp and close to the initialisation point (shown as a red arrow), indicating potentially poor generalisation. If instead we start with a sufficiently large learning rate we can potentially escape such a minimum and with sufficient decay later in training end in minima $B$ or $A$, which are both flatter and further from the origin. We show evidence that these phenomena are relevant to deep learning in Figure $11c$, where we show a larger learning rate variant, trains worse, but tests and validates better and has a greater distance from the initialisation point.

\subsection{Validation Error Curves as a Proxy for Trajectories in the Loss Surface}
When increasing the batch size, our Theorems \ref{theorem:mainresult} and  \ref{theorem:batchtheoremgnn} indicate that the sharpness of the loss landscape decreases at all points in weightspace. To have a similar trajectory in weight space, where we avoid the 
"transformed" sharp minimum $C^{*}$, shown in Figure \ref{subfig:widecurve}, we must move with a larger learning rate since the extremal and bulk edge eigenvalues have decreased in size. Note that depending on whether we mainly move in outlier or bulk directions, the transformations will vary either linearly or with the square root of the batch size respectively.

Since DNNs have been shown to easily fit completely random data \citep{zhang2021understanding} and have exponentially local minima in the training loss \citep{choromanska2015loss}, which are close to the global minimum of training loss, there are likely to be many regions of low training error/loss. Hence, when scaling the learning rate in-equivalently with batch size, we could be taking very different paths in weight space, moving through very different minima and still end up with very similar training loss/error curves. Here by in-equivalently we mean as per our example not escaping from minima $C/C^{*}$ into minimum $A/A^{*}$ but instead to $B/B^{*}$.
Hence, we can consider that validation/testing error curves and NOT training error curves should serve as good proxies to identify whether similar trajectories (moving into and through minima of similar sharpness) are being taken along the multi-modal surface. 

\paragraph{Note on Trajectory Stochasticity:} We note that, since we consider trajectories in expectation, we now discuss whether trajectory stochasticity affects the core arguments presented in this paper. Deep learning initialisation with different random seeds (and hence different starting points and different gradient updates) leads to very similar validation performance. As shown in Appendix \ref{sec:diffinit}.
This is not the case, for example, with different learning rates. We thus consider the trajectory in expectation to be the critical factor and not the stochasticity.
\paragraph{Experimental Validation:}
We show in Figure $11c$ an example of a VGG-$16$ network on CIFAR-$100$, trained with different learning rates and a common batch size of $B=1024$. Despite near identical training performance across the ensemble (noting though that the lower learning rate variant appears to train better), the validation and test accuracies differ significantly with initialisation distance (again noting that these are higher for the higher learning rate variant). We further show in Appendix \ref{sec:initdist} that, for this network and a linear scaling relation (which we derive for SGD in the subsequent section) that the distance in L2 norm between initialisation and final solution remains stable across scaling, as does the final testing error and its profile.

\paragraph{Practical Consequences:}
Whilst we present both training and validation metrics in our experiments, the setup of the experiments, which uses data augmentation, an initially large learning rate (followed by a drop) and often non-zero weight decay, is specifically chosen to provide a low
test error. Predicting the learning rates required to achieve similar validation error trajectories, for a given batch size, is key as opposed to finding trajectories leading to similar training error. We show that different learning rates can give rise to (largely) indistinguishable training characteristics unless divergence occurs. 

\subsection{Experimental Design} 
Given that neural networks can be trained with a wide variety of schedules, which traverse the loss landscape in very different ways, we need an experimental design which allows us to discern, whether two trajectories in weight space are "similar" and hence whether a proposed scaling rule "works". Having already argued that learning the largest possible initial learning rate is a practical problem for deep learning as such schedules aid deep learning generalisation and that trajectories in the validation error are more meaningful to measure loss trajectory movements, we need to be clear with what we mean by "largest". We find that for the VGG \citep{simonyan2014very} networks, without batch normalisation and weight decay, there exists a learning rate value for a given schedule (we use flat and then linear decay) above which (to a certain precision in the grid search) the loss value returns NaN and training breaks. This serves as our definition of maximum and hence for this reason in our experiments we consider the VGG-$16$ as our reference network. As discussed previously, due to the interaction of batch normalisation with curvature and the lack of explicit treatment in our work on batch normalisation, this network serves as an ideal testing ground, but we conjecture our scaling rules to hold more generally and give preliminary evidence for this. For other networks, including batch normalisation \citep{ioffe2015batch} and residual layers \citep{he2016deep}, we find that there exist learning rates above which training and testing both suffer and so we use a working "maximum" which is the largest learning rate which gives a good validation error, close to what is used in practice. 
\end{Correction}


\section{SGD learning rates as a function of batch size}
\label{sec:scaling}
One key practical application of Section \ref{sec:rmttheory} for neural network training is its implications for learning rates as we alter the batch size. \begin{Correction}
Where weight decay is used, the value of $\gamma=0.0005$ is employed, giving the best validation performance on the grid of $[0,0.0001,0.0005,0.001]$ and representing a common practical starting choice.
\end{Correction}
\begin{Correction}

\subsection{Finding the Maximal Allowable Learning Rate}

The change in batch loss, to second-order approximation, for a small step in the direction of the gradient is given by 
\begin{equation}
	\label{eq:losschange}
	\begin{aligned}
		\delta L_{batch}(\vw-\alpha\nabla L) & = - \alpha ||\vg (\vw)||^{2}\bigg(1-\frac{\alpha\sum_{i}^{P}\lambda_{i}||\vphi_{i}\vg(\vw)||^{2}}{2}\bigg) \leq - \alpha ||\vg (\vw)||^{2} \bigg(1-\frac{\alpha \lambda_{1}}{2}\bigg) \\
	\end{aligned}.
\end{equation}
Typically this bound is derived for the deterministic full gradient case \citep{nesterov2013introductory} and hence $\alpha < 2/\lambda_{1}$. In our case, since the batch is stochastic and hence the gradient and Hessian (and therefore its extremal eigenvalues) are also stochastic, this bound holds in expectation i.e. $\mathbb{E}(\delta L_{batch}(\vw-\alpha\nabla L))$.
Here, $\lambda_{1}(\mH_{batch})$ is the largest eigenvalue of the batch Hessian (in expectation) which, along with all outlier eigenvalues of the batch Hessian, are given by Theorem \ref{theorem:mainresult}.
\end{Correction}
A key term in Equation \ref{eq:losschange} is the overlap between the eigenvectors and the stochastic gradient, shown to be large in practice \citep{ghorbani2019investigation,gur2018gradient}.  This indicates that the outlier broadening effect predicted by our framework (when there are well separated outliers\footnote{If there are no outliers, we expect the largest eigenvalue to decrease as the square root of the batch size.}), i.e  $\lambda_{i}* \approx \lambda_{i} + P\sigma^{2}/\mathfrak{b}\lambda_{i}$, is relevant to determining the maximal allowed learning rate. 
\textcolor{black}{
This follows as the bracketed term in Equation \ref{eq:losschange} can be written as $(1-\frac{\alpha |\vg(\vw)|^{2}}{2}\sum_{i}^{P}\lambda_{i}\beta_{i}^{2})$. Hence, if $\sum_{i}^{k}\beta_{i}^{2} \approx 1$ (where $k$ is the number of outliers) and noting that all outliers scale in a similar way, the result in expectation is similar to that of the expectation of the upper bound. This can be seen in the case where the broadening term dominates the value of the outliers from the empirical Hessian, i.e. 
\begin{equation}
	1 - \frac{\alpha |\vg(\vw)|^{2}}{2}\sum_{i}\beta_{i}^{2}\bigg(\lambda_{i}+\frac{P\sigma^{2}}{\mathfrak{b}\lambda_{i}}\bigg) \approx 1 - \frac{P\sigma^{2}\alpha |\vg(\vw)|^{2}}{2\mathfrak{b}}\sum_{i}\frac{\beta_{i}^{2}}{\lambda_{i}}.
\end{equation}
Note that, if we want to consider the difference between the batch and true Hessian, we would have $B$ instead of $\mathfrak{b}$, where $\mathfrak{b} > B$ and for $B\ll N$ $\mathfrak{b} \approx B$.}
We observe outliers in all our experiments, as shown in Figures \ref{fig:hesspert} and  \ref{fig:ggnpert}, which is consistent with previous literature \citep{ghorbani2019investigation,papyan2018full} and motivated in Section \ref{sec:whywehaveoutliers}.

\paragraph{For small batch sizes the maximal learning rate is proportional to the batch size.} As the largest allowable learning rate as shown by Equation \ref{eq:losschange} is $\propto 1/\lambda_{1}*$ and $\lambda_{1}* = \lambda_{1} + P\sigma^{2}/B\lambda_{1}$ for very small batch sizes this $ \approx P\sigma^{2}/B\lambda_{1} $, hence increasing the batch size allows a proportional increase in the maximal learning rate. This holds until the first term $\lambda_{1}$ in Theorem \ref{theorem:mainresult} is no longer negligible in comparison to the latter, $P\sigma^{2}/B\lambda_{1}$. Thereafter we cross over into the regime where, although the spectral norm still decreases with batch size, it asymptotically reaches its minimal value $\lambda_{1}$. Hence the learning rate cannot be appreciably increased despite using larger batch sizes.
To validate this empirically, we train the VGG-$16$ on CIFAR-$100$, finding the maximal learning rate at which the network trains for $B=128$. We then increase/decrease the batch size by factors of $2$, proportionally scaling the learning rate. We plot the results in Figure \ref{fig:vggexp}. 







\begin{figure}[h!]
\centering
\begin{subfigure}{0.53\linewidth}
	\includegraphics[trim={0cm 0cm 0cm 0cm},clip, width=1\textwidth]{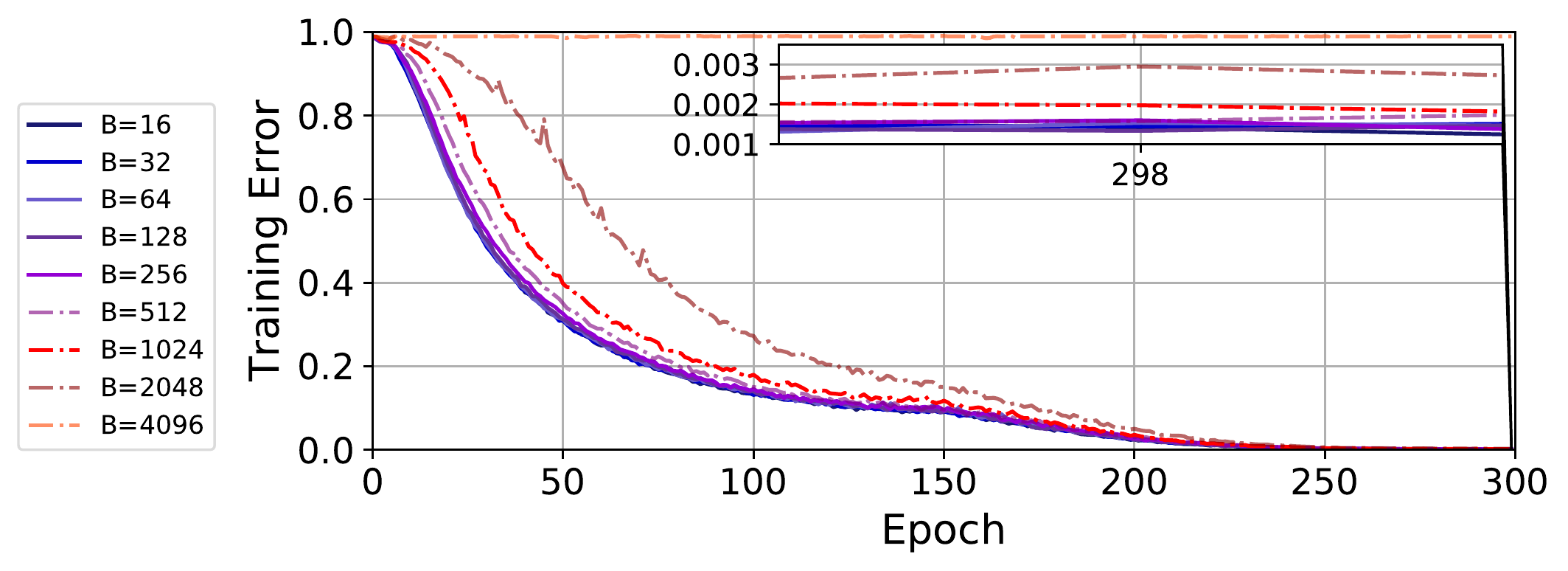}
	\caption{Training Error against Epoch}
	\label{subfig:vgg16train}
\end{subfigure}
\begin{subfigure}{0.45\linewidth}
	\includegraphics[trim={0.0cm 0cm 0cm 0cm},clip, width=1\textwidth]{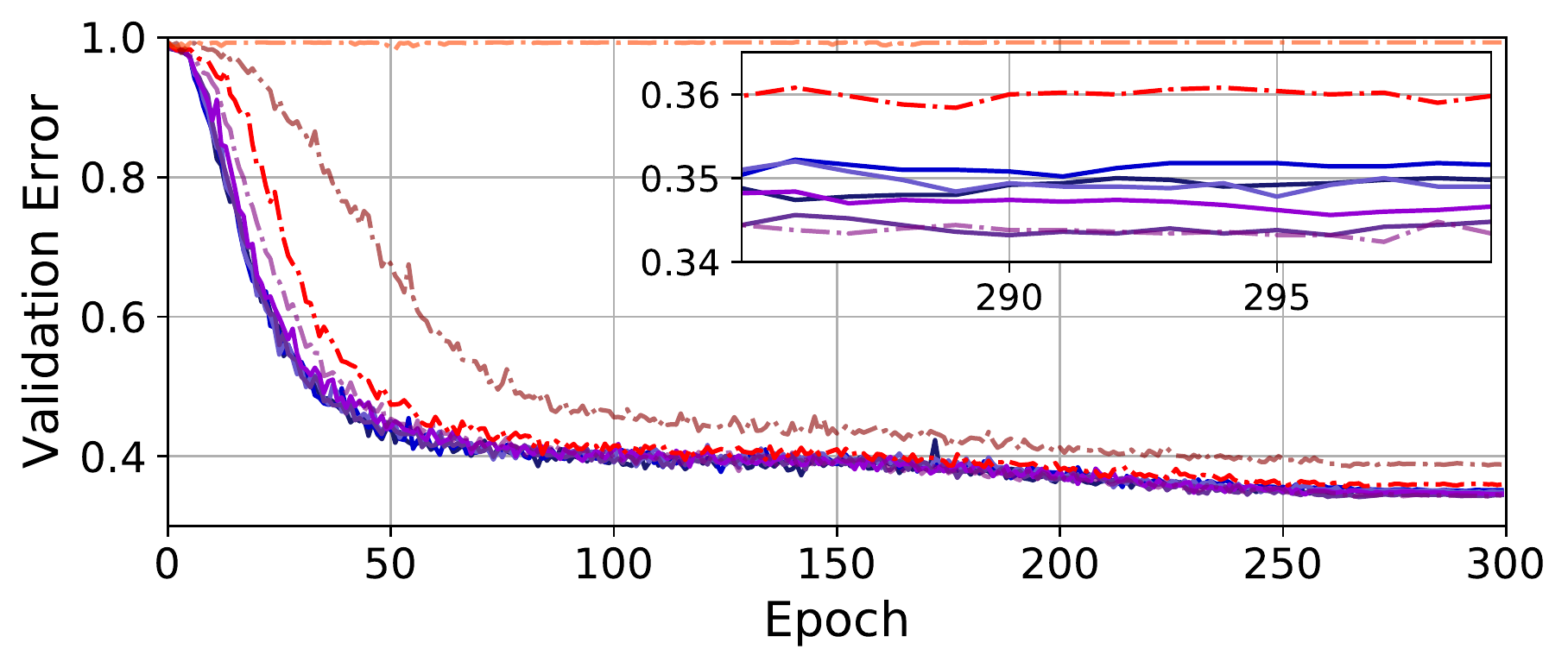}
	\caption{Validation Error against Epoch}
	\label{subfig:vgg16test}
\end{subfigure}
\caption{\textbf{Linear scaling is consistent up to a threshold.} Training and Validation error of the VGG-$16$ architecture, without batch normalisation (BN) on CIFAR-$100$, with no weight decay $\gamma=0$ and initial learning rate $\alpha_{0}=\frac{0.01B}{128}$}
\label{fig:vggexp}
\end{figure}

The 
\begin{Correction}
training and
\end{Correction}
validation accuracy remains stable for all batch size values, until a small drop for $B=1024$, a larger drop still for $B=2048$ and for $B=4096$ we see no training. 

\paragraph{One can get away with large initial learning rates.}
Another theoretical prediction is that, if the Hessian variance increases during training (as observed in Figure \ref{fig:hessvarcurve}), large learning rates which initially rapidly decrease the loss could become unstable later in training. 
\textcolor{black}{
This follows because we want to ensure, at every point in training, that the batch loss does not increase. The largest allowable learning rate which, in expectation, does not increase the batch loss is inversely proportional to the sum of the full Hessian (which is deterministic for a point in weight space) and a term proportional to the variance of the elements of the fluctuations matrix. If the variance of the Hessian (which uniquely determines under the conditions of our model the variance of the fluctuations matrix elements) increases, then we expect the largest allowable learning rate to decrease. Since, in practice, we note that the variance of the Hessian increases during training, we expect for these experiments to be able to start with a larger learning rates. This has also been noted in \citet{lewkowycz2020large}.
}

To see this, we run the same experiment but this time use twice the maximal allowed learning rate. We observe in Figure \ref{subfig:vgg16failtrain} that initially the loss decreases rapidly (far faster than in the smaller learning rate alternative in Figure \ref{subfig:vgg16train}), but that soon the training becomes unstable and diverges. This indicates that the practice of starting with an initially large learning rate and decaying it is well justified in terms of stability implied by the batch Hessian
at least for our experiments.

\begin{figure}[h!]
\centering
\begin{subfigure}{0.53\linewidth}
	\includegraphics[trim={0cm 0cm 0cm 0cm},clip, width=1\textwidth]{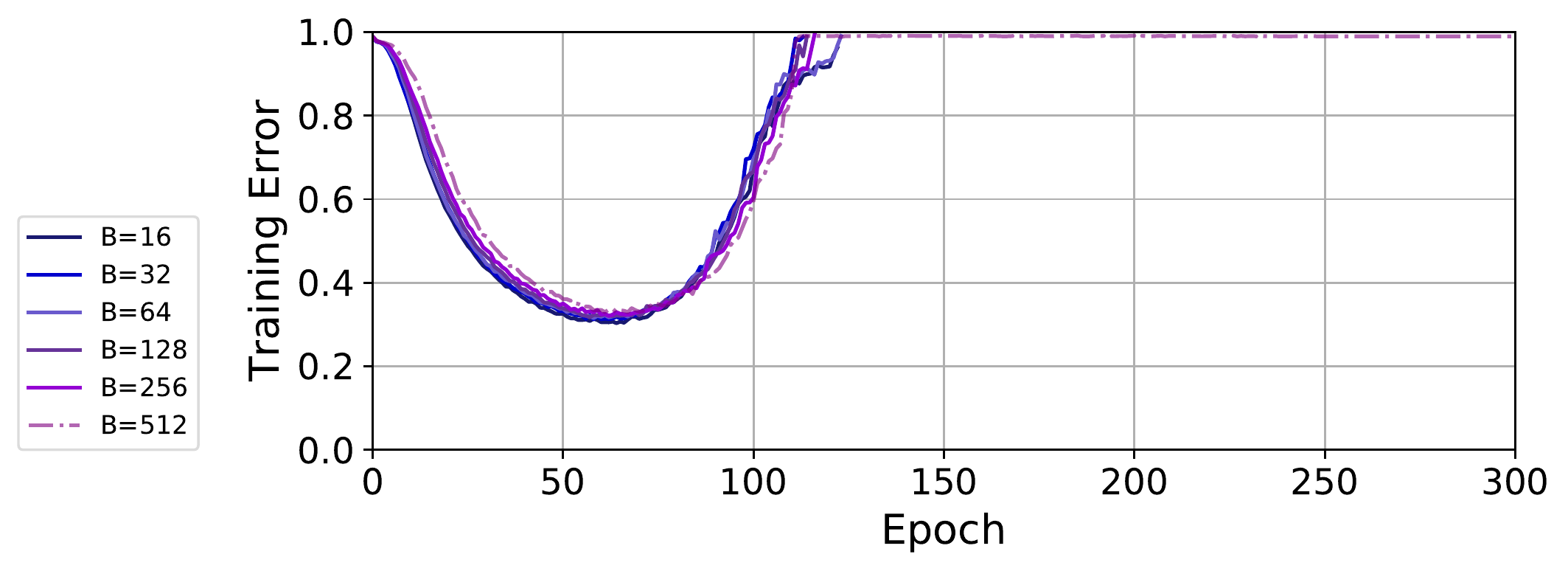}
	\caption{Training Error against Epoch}
	\label{subfig:vgg16failtrain}
\end{subfigure}
\begin{subfigure}{0.45\linewidth}
	\includegraphics[trim={0.0cm 0cm 0cm 0cm},clip, width=1\textwidth]{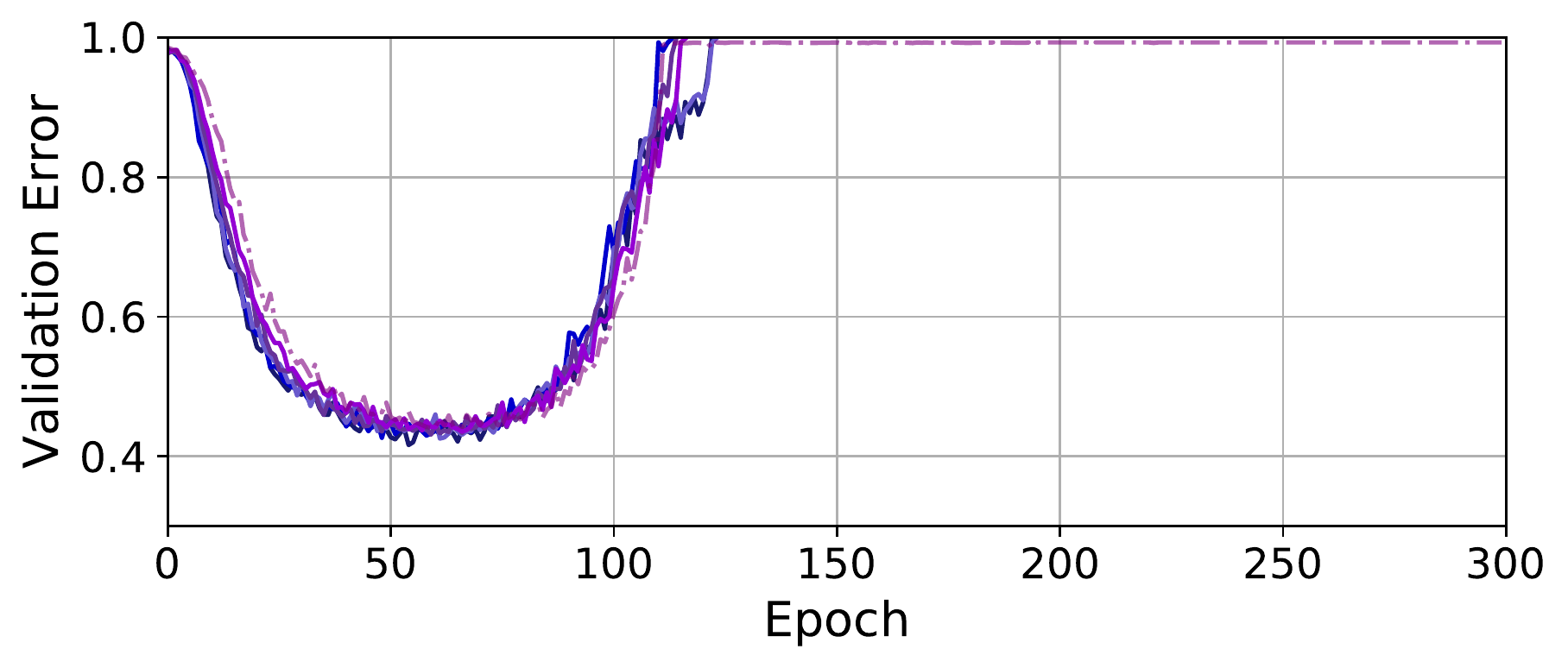}
	\caption{Validation Error against Epoch}
	\label{subfig:vgg16failtest}
\end{subfigure}
\caption{\textbf{Consistency holds for a variety of learning rates.} Training and Validation error of the VGG-$16$ architecture, without batch normalisation (BN) on CIFAR-$100$, with no weight decay $\gamma=0$ and initial learning rate $\alpha_{0}=\frac{0.02B}{128}$.}
\label{fig:vggfailed}
\end{figure}

\paragraph{Our linear scaling rule seems to hold generally for SGD.} To highlight the generality of our linear scaling rule, we include batch normalisation \citep{ioffe2015batch} and weight decay $\gamma = 0.0005$. In this case there is a greater range of permissible learning rates, so we grid search the best learning rate as defined by the validation error for $B=128$ and use our derived linear scaling rule, as shown in Figure \ref{fig:vggbnexp}, where we observe a similar pattern. We repeat the experiment on the WideResNet-$28\times10$ \citep{zagoruyko2016wide} on both the CIFAR-$100$ and ImageNet $32\times 32$ \citep{chrabaszcz2017downsampled} datasets shown in Figures \ref{fig:wrnc100}, \ref{fig:wrnimg32} respectively. Unlike the VGG model without batch normalisation, where unstable trajectories diverge or with batch normalisation do not train. Highly unstable oscillatory WideResNet trajectories converge with learning rate reduction, however they never reach peak performance. The 
\begin{Correction}
training and
\end{Correction}
test performance is stable for a variety of learning rates with fixed learning rate to batch size ratio, again 
\begin{Correction}
strongly supporting
\end{Correction}
the validity of the linear scaling rate rule until a threshold. 
\begin{figure}[h!]
\centering
\begin{subfigure}{0.53\linewidth}
	\includegraphics[trim={0cm 0cm 0cm 0cm},clip, width=1\textwidth]{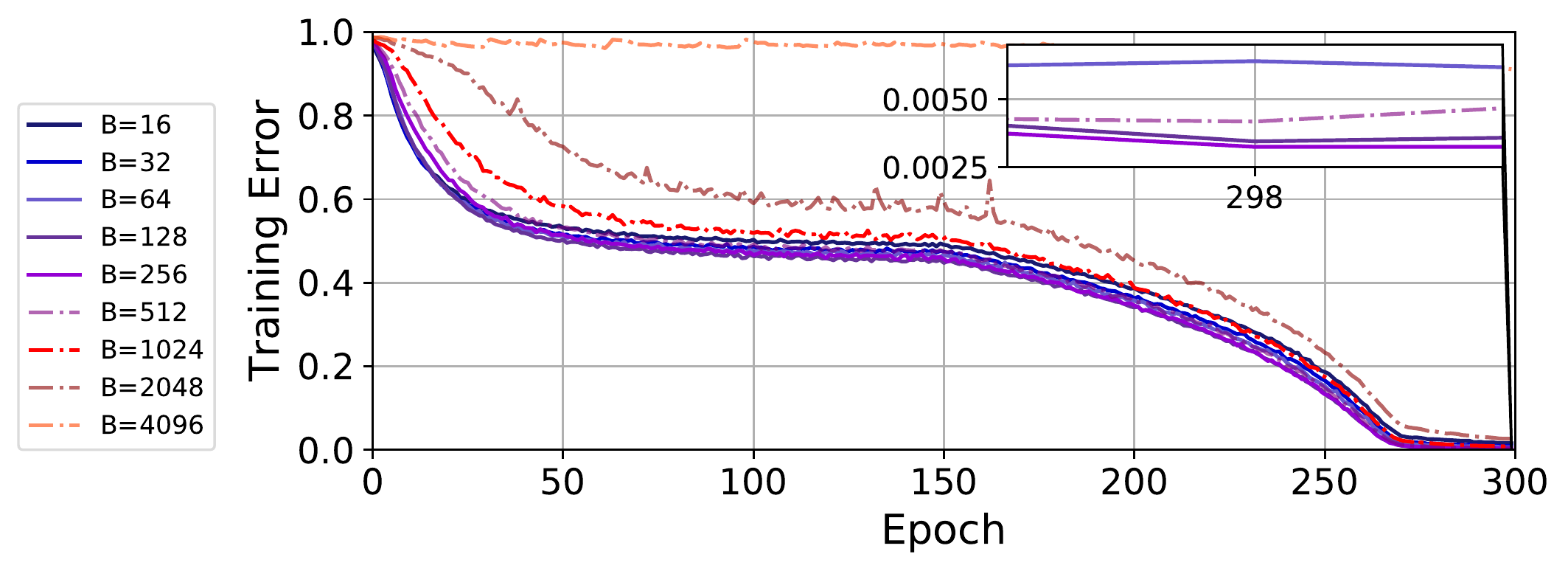}
	\caption{Training Error against Epoch}
	\label{subfig:vgg16bntrain}
\end{subfigure}
\begin{subfigure}{0.45\linewidth}
	\includegraphics[trim={0.0cm 0cm 0cm 0cm},clip, width=1\textwidth]{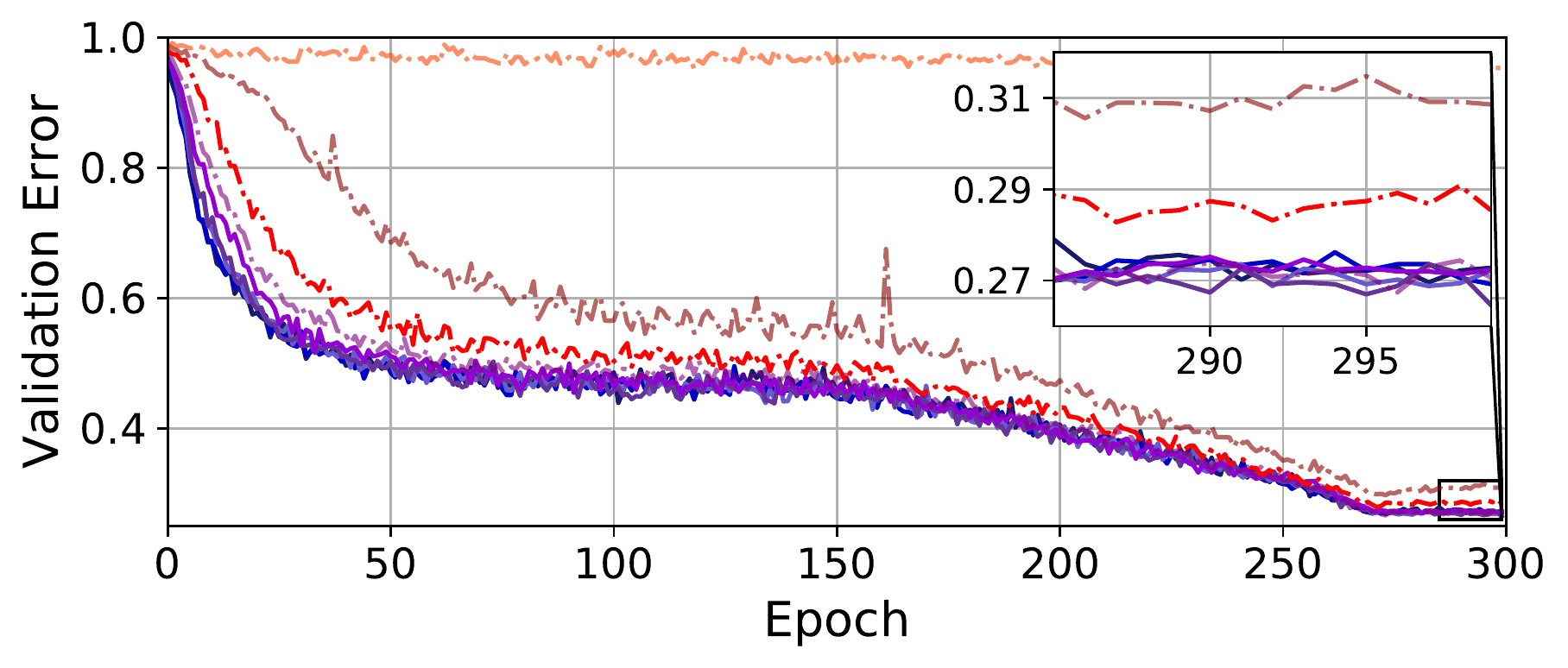}
	\caption{Validation Error against Epoch}
	\label{subfig:vgg16bntest}
\end{subfigure}
\caption{\textbf{Linear scaling is consistent up to a threshold.} Training and Validation error of the VGG-$16$ architecture, with batch normalisation (BN) on CIFAR-$100$, with no weight decay $\gamma=5e^{-4}$ and initial learning rate $\alpha_{0}=\frac{0.1B}{128}$.}
\label{fig:vggbnexp}
\end{figure}
\begin{figure}[h!]
\centering
\begin{subfigure}{0.53\linewidth}
	\includegraphics[trim={0cm 0cm 0cm 0cm},clip, width=1\textwidth]{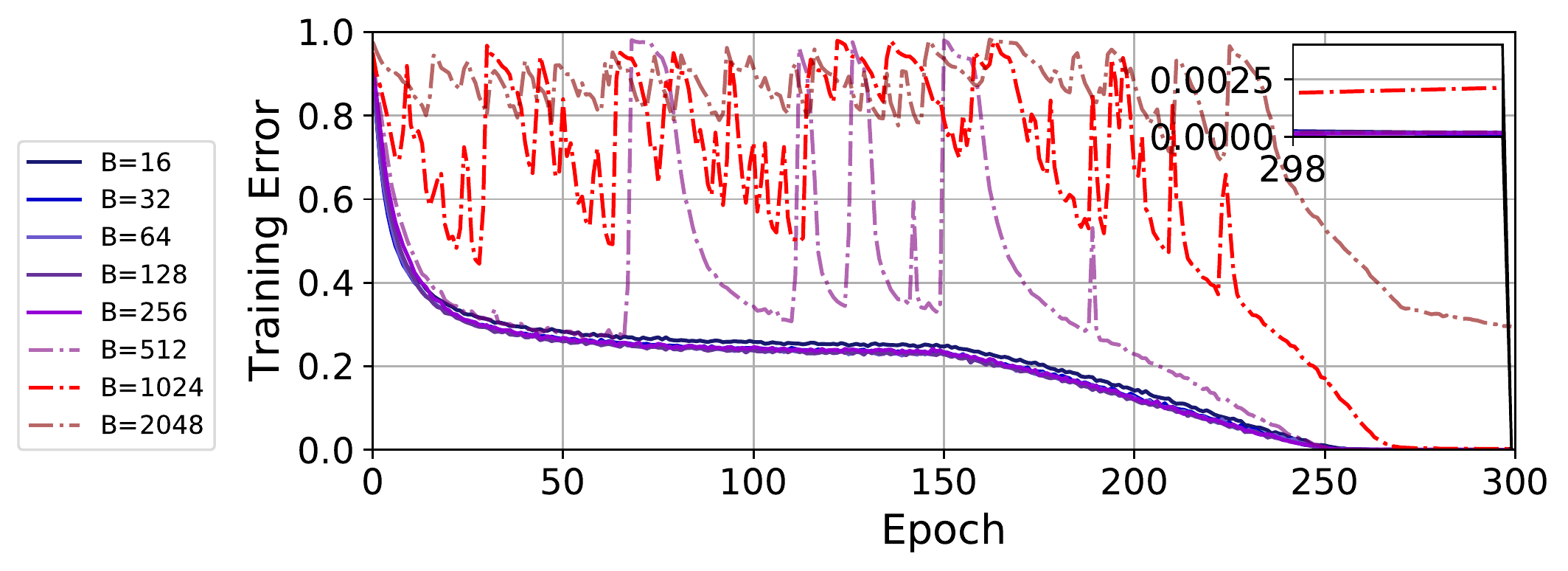}
	\caption{Training Error against Epoch}
	\label{subfig:wrnc100train}
\end{subfigure}
\begin{subfigure}{0.45\linewidth}
	\includegraphics[trim={0.0cm 0cm 0cm 0cm},clip, width=1\textwidth]{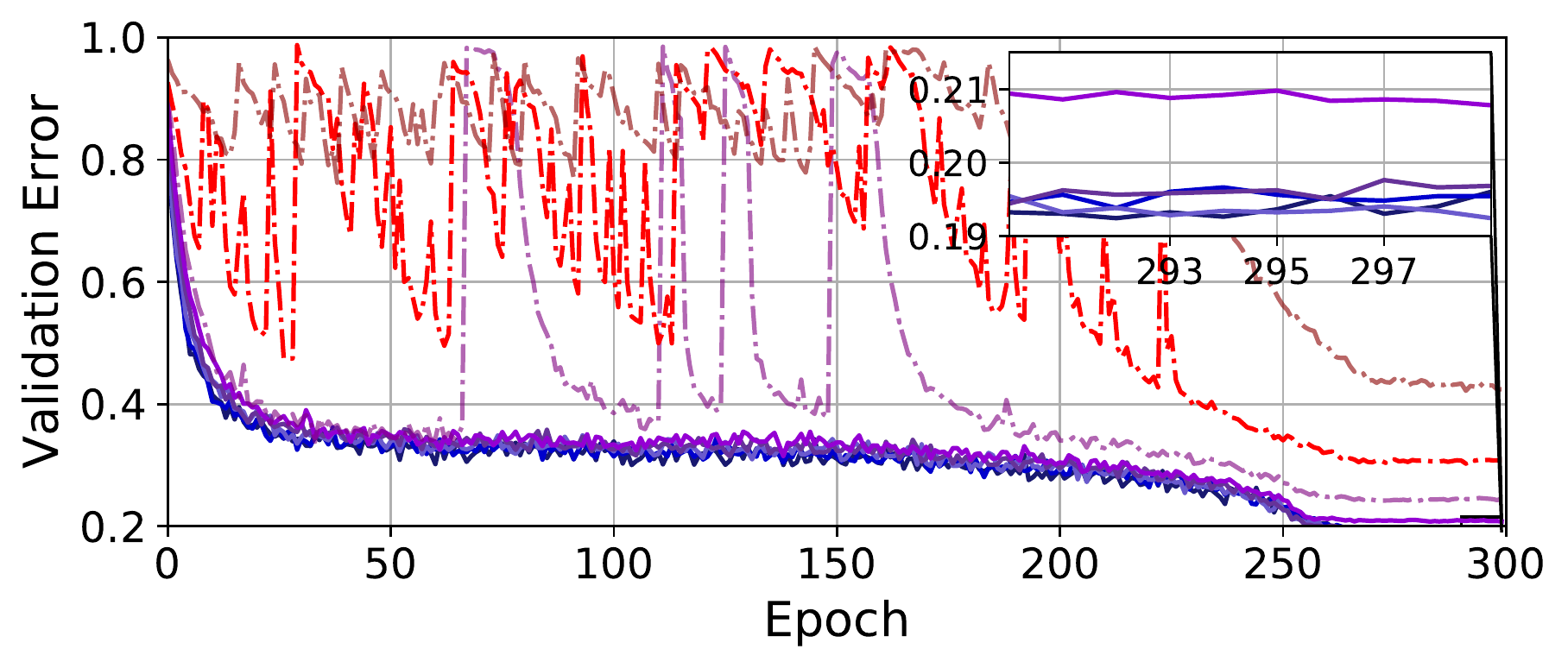}
	\caption{Validation Error against Epoch}
	\label{subfig:wrnc100test}
\end{subfigure}
\caption{\textbf{Consistency holds for a variety of learning rates.} Training and Validation error of the WideResNet-$28\times10$ architecture, with batch normalisation (BN) on CIFAR-$100$, with weight decay $\gamma=5e^{-4}$ and initial learning rate $\alpha_{0}=\frac{0.1B}{128}$.}
\label{fig:wrnc100}
\end{figure}
\begin{figure}[h!]
\centering
\begin{subfigure}{0.53\linewidth}
	\includegraphics[trim={0cm 0cm 0cm 0cm},clip, width=1\textwidth]{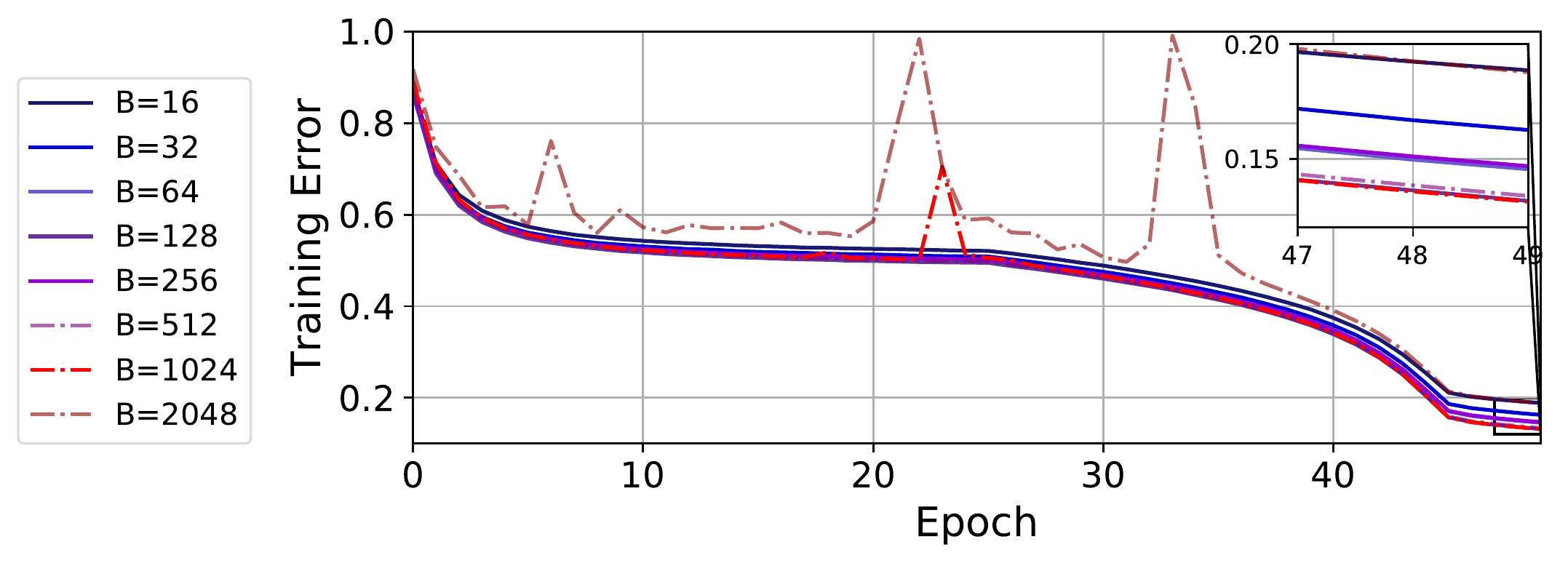}
	\caption{Training Error against Epoch}
	\label{subfig:wrnimg32train}
\end{subfigure}
\begin{subfigure}{0.45\linewidth}
	\includegraphics[trim={0.0cm 0cm 0cm 0cm},clip, width=1\textwidth]{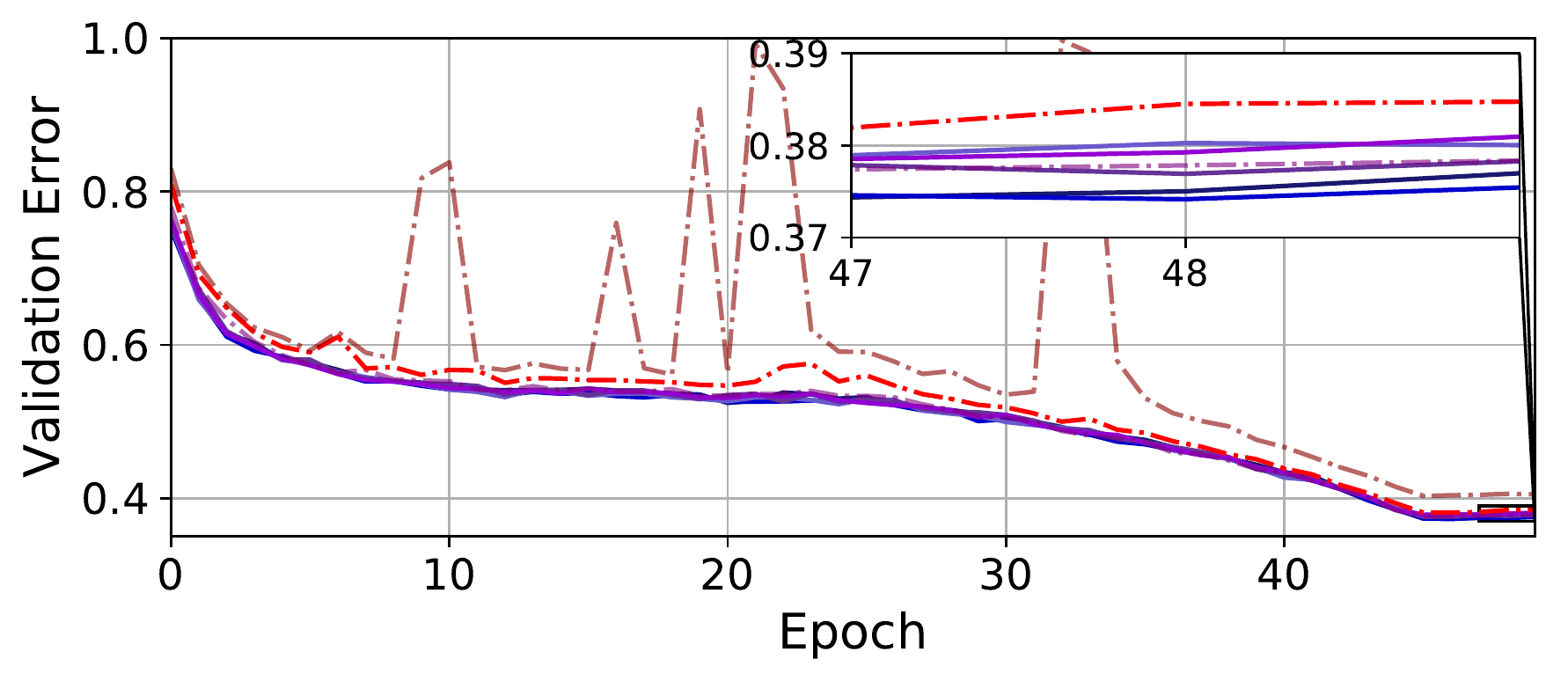}
	\caption{Validation Error against Epoch}
	\label{subfig:wrnimg32test}
\end{subfigure}
\caption{\textbf{Consistency holds for a variety of learning rates.} Training and Validation error of the WideResNet-$28\times10$ architecture, with batch normalisation (BN) on ImageNet-$32$, with weight decay $\gamma=5e^{-4}$ and initial learning rate $\alpha_{0}=\frac{0.1B}{128}$.}
\label{fig:wrnimg32}
\end{figure}

\begin{Correction}
\subsection{Estimating the "Optimal" Learning Rate and Momentum from the Spectrum}
Our theoretical anaylsis in Section \ref{sec:rmttheory} and experiments in Section \ref{sec:experiments} show that the relevant curvature estimates, when mini-batch training, are not those of the full (or true) Hessian, but rather those of the batch Hessian. This leads to the  supposition that we can estimate relevant aspects of the curvature using the Lanczos algorithm in $\mathcal{O}(mPB)$ time during training and, in effect, estimate the optimal learning and momentum rates during training. Note that, since one iteration of SGD is only of cost $\mathcal{O}(PB)$, this procedure needs only to be run irregularly (or alternatively $m$ needs to be kept very small, resulting in poor curvature estimates) for it to be competitive with multiple runs using differing learning rates and/or schedules. As a proof of concept, we run two variants of our approach using the optimality relations for both Polyak and Nesterov learning rates and momenta.
\begin{equation}
	\alpha_{Polyak} =  \frac{2}{\sqrt{\lambda_{1}}+\sqrt{\lambda_{P}}}, \thinspace \alpha_{Nesterov} = \sqrt{\frac{\lambda_{P}}{\lambda_{1}}}
\end{equation}
\begin{equation}
	\rho_{Polyak} = \bigg(\frac{\sqrt{\lambda_{1}}-\sqrt{\lambda_{P}}}{\sqrt{\lambda_{1}}+\sqrt{\lambda_{P}}}\bigg)^{2}, \rho_{Nesterov} = \bigg(\frac{\sqrt{\lambda_{1}}-\sqrt{\lambda_{P}}}{\sqrt{\lambda_{1}}+\sqrt{\lambda_{P}}}\bigg).
\end{equation}
Here, the Lipshitz and strong convexity constants are estimated locally using the Lanczos algorithm on the \emph{batch Hessian}. Note that, whilst the Hessian in our experiments has negative spectral mass at all points in weight space (and is hence not strongly convex), we conveniently can use a positive-definite approximation, as is frequently done in the second-order learning literature \citep{martens2015optimizing,dauphin2014identifying}. We run a curvature estimate using the Lanczos algorithm, seeded with a random vector every $20$ epochs, with iteration number $m=20$. Neither of these parameters was optimised. Our primary objective is to show that a batch Hessian curvature based approach to learning the learning rate and momentum can be useful out of the box and experimentally reduce and not increase the hyper-parameter burden.  
Given that, in the stochastic case, all methods must decay the learning rate and/or employ weight averaging we employ the latter \citep{izmailov2018averaging} for all methods near the end of training. It is known \citep{kushner2003stochastic} that iterate averaging gives greater robustness to the learning rate schedule and choice, whilst still leading to convergence. 
Since the smallest Ritz values are very very close to zero, which would result in a momentum $\rho = 1$, we use a heuristic to remove the smallest Ritz values, whereby if the Ritz value of largest spectral mass has more than $50\%$ of the spectral mass, it is removed and the resulting density renormalised, forming the new spectral density of interest. 
We present our results in both training and testing for the PreResNet-$110$, with weight decay of $0.0005$, in Figure \ref{fig:p110learned}. Here we compare with the tuned learning-rate schedule used in \citet{izmailov2018averaging} and described in Section \ref{sec:experiments}. The latter has an initial learning rate set to $0.1$. We show the learned learning rates and momenta for both methods in Figure \ref{fig:p110learnedhyp}. We note that, whilst the Polyak method strongly decays the learning rate, converging fast on the training set, the Nesterov variant, coupled with Nesterov Momentum, keeps the learning rate high, converging only slightly faster than the SGDSWA variant but outperforming in test error at the end. Whilst we don't expect for general non-convex problems, such as deep learning, a method such as this to out-perform all combinations of learning rates and momentum schedules, it is encouraging that such a cheap estimation approach has significant potential.
\begin{figure}[h!]
\centering
\begin{subfigure}{0.48\linewidth}
	\includegraphics[trim={0cm 0cm 0cm 0cm},clip, width=1\textwidth]{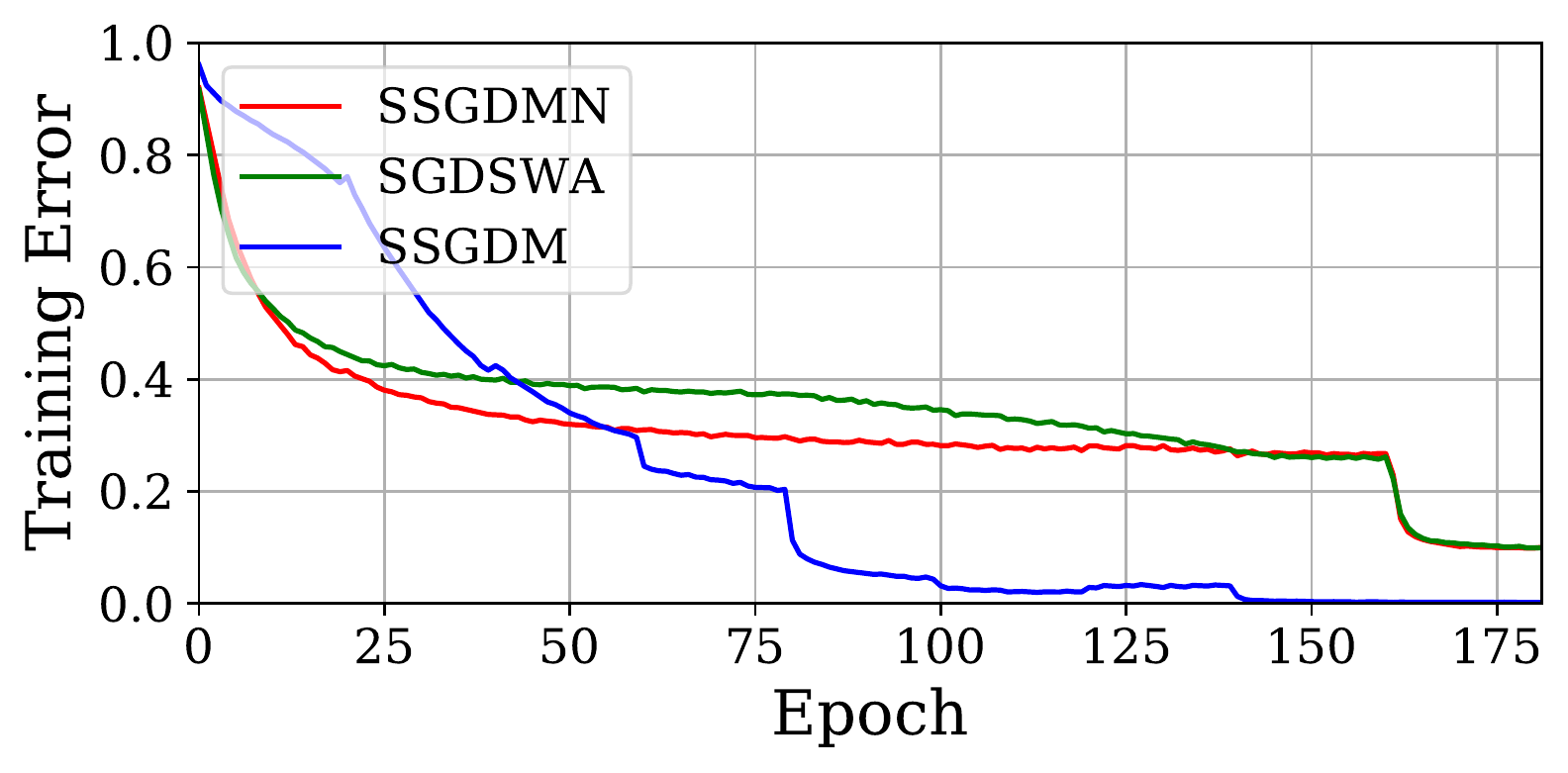}
	\caption{Training Error}
	\label{subfig:p110learnedtrain}
\end{subfigure}
\begin{subfigure}{0.48\linewidth}
	\includegraphics[trim={0cm 0cm 0cm 0cm},clip, width=1\textwidth]{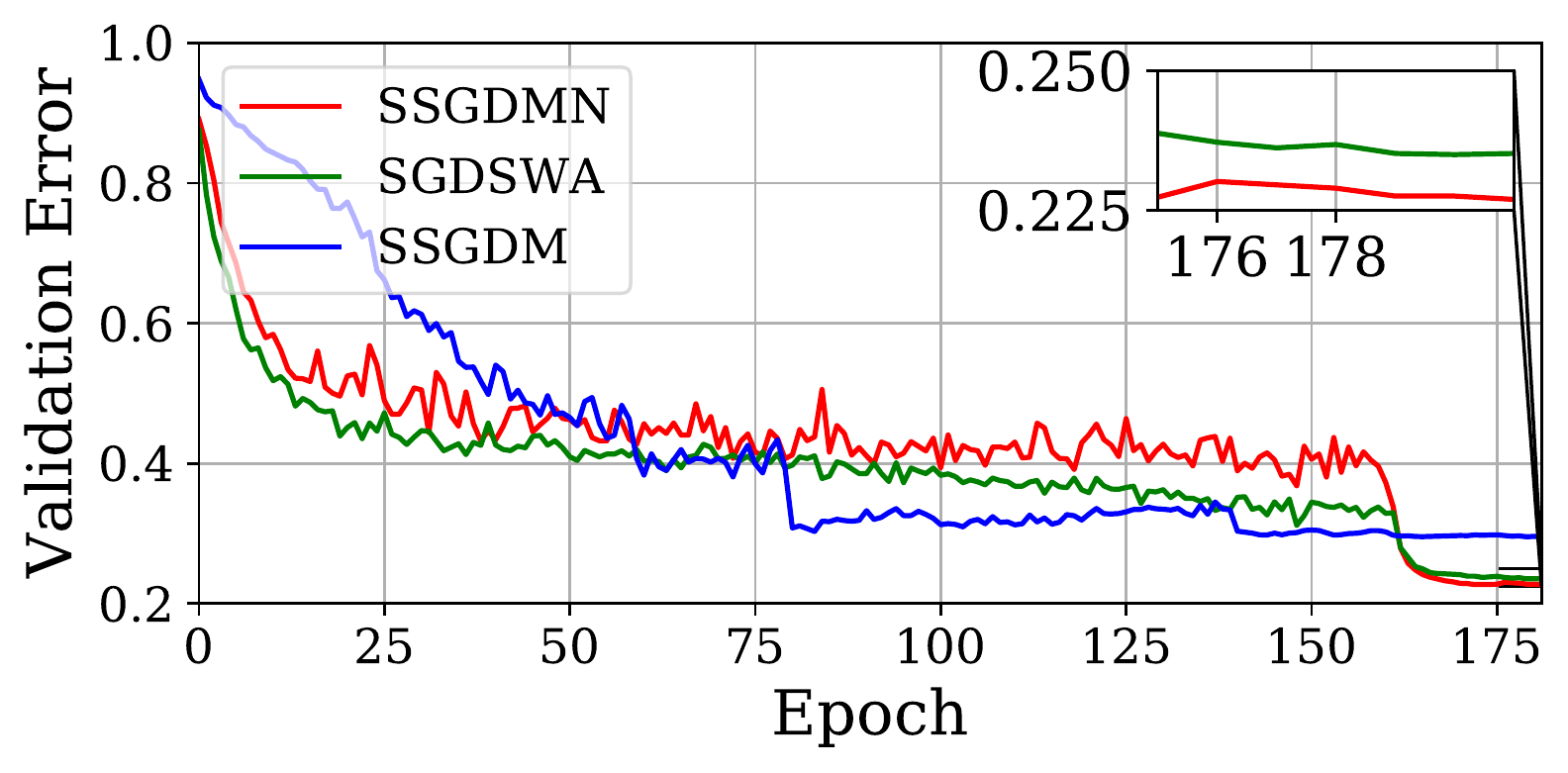}
	\caption{Validation Error}
	\label{subfig:p110learnedtest}
\end{subfigure}
\caption{\textbf{Learned Learning Rates seem Competitive with Fine Tuned} PreResNet-$110$ on the CIFAR-$100$ dataset, with weight decay $\gamma=5e^{-4}$.}
\label{fig:p110learned}
\vspace{-10pt}
\end{figure}
\begin{figure}[h!]
\centering
\begin{subfigure}{0.48\linewidth}
	\includegraphics[trim={0cm 0cm 0cm 0cm},clip, width=1\textwidth]{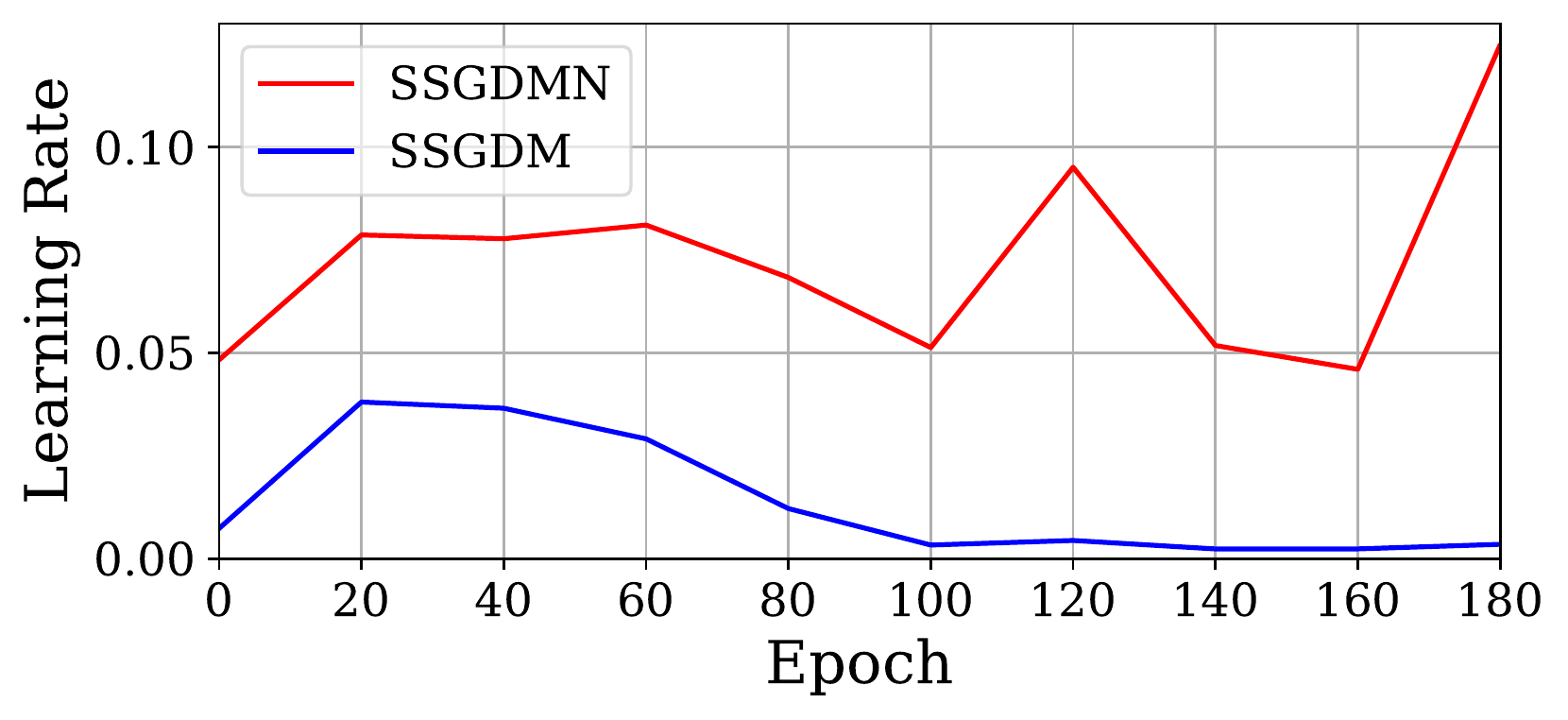}
	\caption{Learning Rate vs Epoch}
	\label{subfig:p110learnedlr}
\end{subfigure}
\begin{subfigure}{0.48\linewidth}
	\includegraphics[trim={0cm 0cm 0cm 0cm},clip, width=1\textwidth]{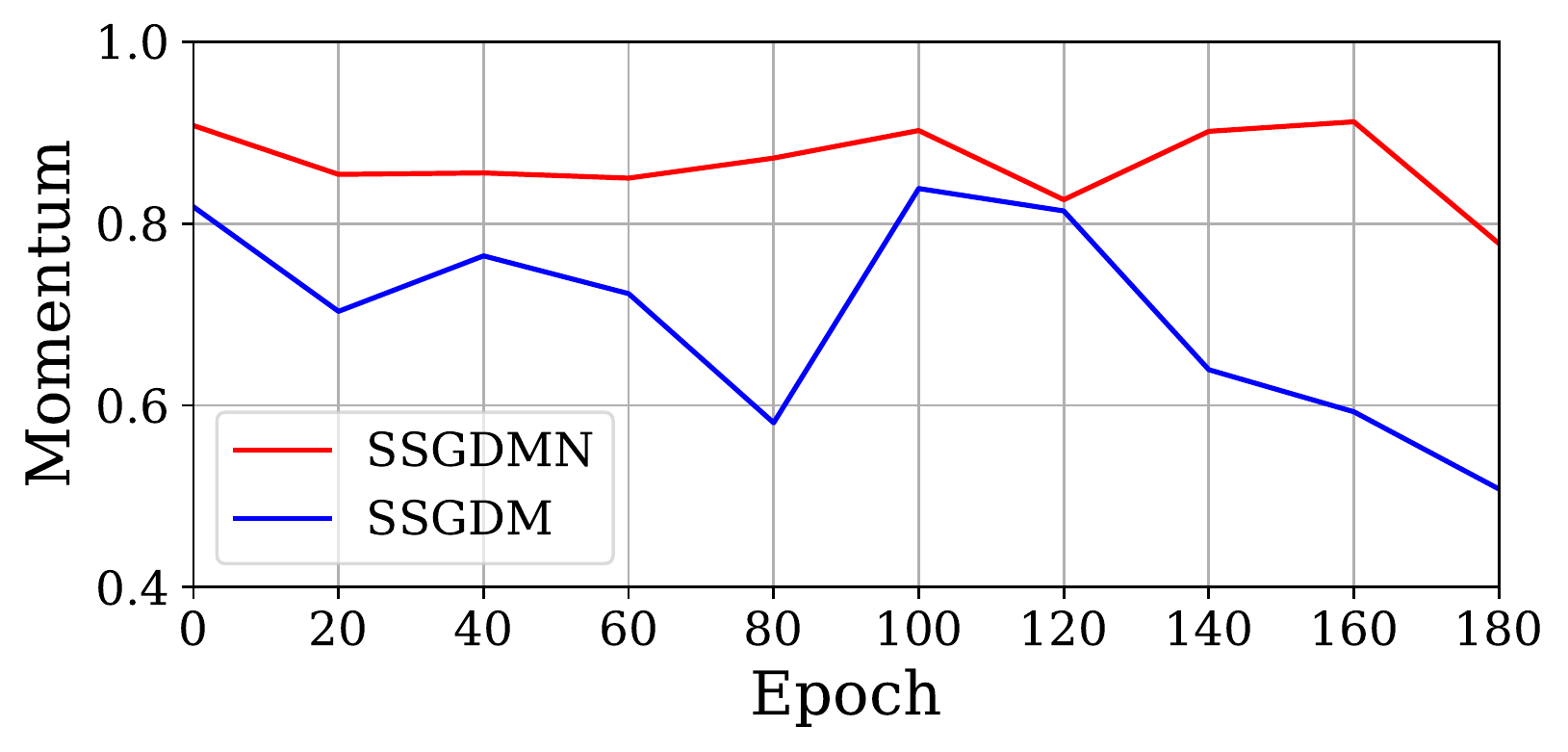}
	\caption{Momentum vs Epoch}
	\label{subfig:p110learnedmom}
\end{subfigure}
\caption{Learning Rates and Momenta learned during training for the PreResNet-$110$ on the CIFAR-$100$ dataset, with weight decay $\gamma=5e^{-4}$.}
\label{fig:p110learnedhyp}
\end{figure}
\end{Correction}

\begin{Correction}

\section{Square root learning rate scaling for adaptive optimisers with small damping}
\label{sec:adamlrscale}

By considering the change in loss for a generic second-order optimiser, where we precondition the gradients with some approximation of the Hessian $\mB$, we have
\begin{equation}
L(\vw_{k}-\alpha\mB^{-1}\nabla L(\vw_{k}))-L(\vw) = \alpha \nabla L(\vw_{k})^{T}\mB^{-1}\nabla L(\vw_{k}) + \frac{\alpha^{2}}{2}\nabla L(\vw_{k})^{T}\mB^{-1}\mH \mB^{-1}\nabla L(\vw_{k}).
\end{equation}
\nonumber
Writing $\mH_{emp} = \sum_{i}\lambda_{i}\bm{\psi}_{i}\bm{\psi}^{T}$ and writing the noisy estimated eigenvalue/eigenvector pair from the optimiser as 
$\mB = \sum_{j}\eta_{j}\vphi_{j}\vphi_{j}^{T}$, making use of orthogonal bases, we have,
\begin{equation}
	\label{eq:secorderlosschange}
	L(\vw_{k+1})-L(\vw) = \sum_{i}^{P}\frac{\alpha_{0}|\vphi_{i}^{T}\nabla L(\vw)|^{2}}{\eta_{i}+\delta}\bigg(1- \frac{\alpha_{0}}{2(\eta_{i}+\delta)}\sum_{\mu}\lambda_{\mu}|\bm{\psi}_{u}^{T}\vphi_{i}|^{2})\bigg). 
\end{equation}
This is a more complicated expression than the resulting equation for SGD (Equation \ref{eq:losschange}), as it involves both the eigenvalue/eigenvector pairs of the batch Hessian and that of the preconditioning matrix. Whereas for SGD, movement in the eigenvectors corresponding to the largest eigenvalues result in the greatest increase in loss, 
\emph{for adaptive optimisers, division by the inverse of the preconditioner eigenvalue means that an increase in the loss function could be due to the optimiser moving direction of lower curvature.}

\paragraph{Potential Boost for Edge Eigenvectors:}  Consider the simplified case of $|\bm{\psi}_{u}^{T}\vphi_{i}|^{2} = \delta_{u,i}$, where we assume perfect eigenvector estimation but potentially imperfect eigenvalue estimation. To consider imperfect eigenvector estimation, we can rewrite $\sum_{j}\eta_{j}\vphi_{j} \equiv \sum_{j}\eta^{*}_{j}\psi_{j}$ and hence imperfect eigenvector estimation can be reframed as perfect eigenvector estimation under a transformed set of eigenvalues.
We then consider an eigenvalue from the batch Hessian which is at the edge of the bulk distribution and thus an outlier. The loss will be larger moving in this "flat" direction iff,
\begin{equation}
	\label{eq:edgevsoutlier}
	\frac{\sqrt{P}\sigma}{\sqrt{\mathfrak{b}}(\eta_{i}+\delta)} >  \frac{\lambda_{j}+\frac{P\sigma^{2}}{\mathfrak{b}\lambda_{j}}}{(\eta_{j}+ \delta)} \; \; \; \text{i.e.} \; \; \; \frac{\eta_{j}+\delta}{\eta_{i}+\delta} >  \bigg( \frac{\lambda_{j}\sqrt{\mathfrak{b}}}{\sqrt{P}\sigma} + \frac{\sqrt{P}\sigma}{\sqrt{\mathfrak{b}}\lambda_{j}}\bigg).
\end{equation}
Hence an under-estimation of the bulk eigenvalue $\eta_{i}$, relative to outlier eigenvalue $\eta_{j}$, combined with a small damping coefficient (typically set at $10^{-8}$ for Adam) could result in this condition being satisfied.
There are many $O(P)$ eigenvalues near the edge of the bulk distribution, compared to the limited number of outliers and hence many edge eigenvalue/eigenvector pairs that need to be well estimated. 
\paragraph{Necessity of small numerical stability coefficient:}In the $\delta \rightarrow \infty$ limit, the l.h.s. of Equation \ref{eq:edgevsoutlier} is $1$, whereas as $\lambda_{j} > \sqrt{\frac{P}{\mathfrak{b}}\sigma}$ the r.h.s. is $>1$. Hence, Equation \ref{eq:edgevsoutlier}
cannot be satisfied. If we move in all eigendirections equally, then - since an outlier is, by definition, larger in magnitude than eigenvalues at the edge of the bulk - we cannot increase the loss more in a non-outlier direction than in an outlier direction.

\paragraph{Practical Implication:} Under the scenario of a small damping, $\delta$, with an adaptive method, we would expect to be able to scale the learning only proportionally to the square root of the batch size, since the bulk eigenvalue distribution scales as the square root of the batch size. Hence,
\begin{equation}
	\bigg(1-\frac{\alpha_{0}\sqrt{P}\sigma}{(\eta_{i}+\delta)\sqrt{\mathfrak{b}}}\bigg) > 0 \therefore \alpha_{0} < \frac{\sqrt{\mathfrak{b}}\kappa}{\sqrt{P}\sigma} \leq \frac{\sqrt{\mathfrak{b}}(\eta_{i}+\delta)}{\sqrt{P}\sigma},
\end{equation} 
where $\kappa = \eta_{min} + \delta$ and $\eta_{min}$ is the worst curvature estimate (transformed into the appropriate basis) of a bulk edge eigenvector. Note since the eigenvectors of the bulk edge all transform as $\propto \sqrt{\mathfrak{b}}$ we can simply absorb the constant into $\kappa$. Note further, that for small enough batch size - as the outlier eigenvalues scale proportionally with $\frac{1}{\mathfrak{b}}$ whereas the bulk distribution only grows proportional to $\sqrt{\frac{1}{\mathfrak{b}}}$ - we expect the condition to become harder to fulfil. This means that our misestimation of the bulk eigenvalue/eigenvector pairs needs to increase relative to the outliers in the event of smaller batch sizes. Hence, for very small batch sizes, the scaling could revert to being linear and the square root rule could break down.
\begin{figure}[h!]
	\centering
	\begin{subfigure}{0.32\linewidth}
		\includegraphics[trim={0cm 0cm 0cm 0cm},clip, width=1\textwidth]{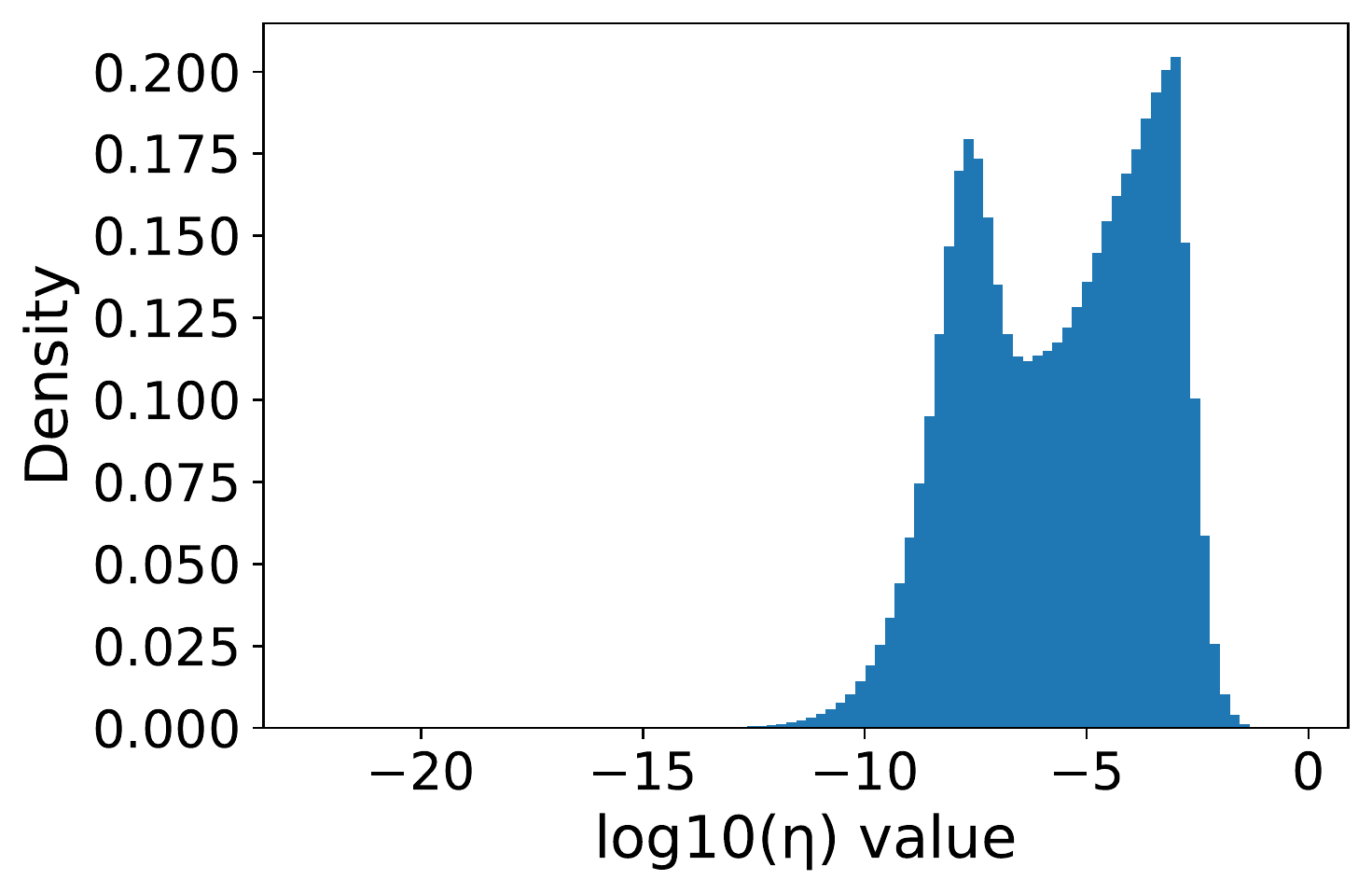}
		\caption{Epoch $25$}
		\label{subfig:ep25adameta}
	\end{subfigure}
	\begin{subfigure}{0.32\linewidth}
		\includegraphics[trim={0.0cm 0cm 0cm 0cm},clip, width=1\textwidth]{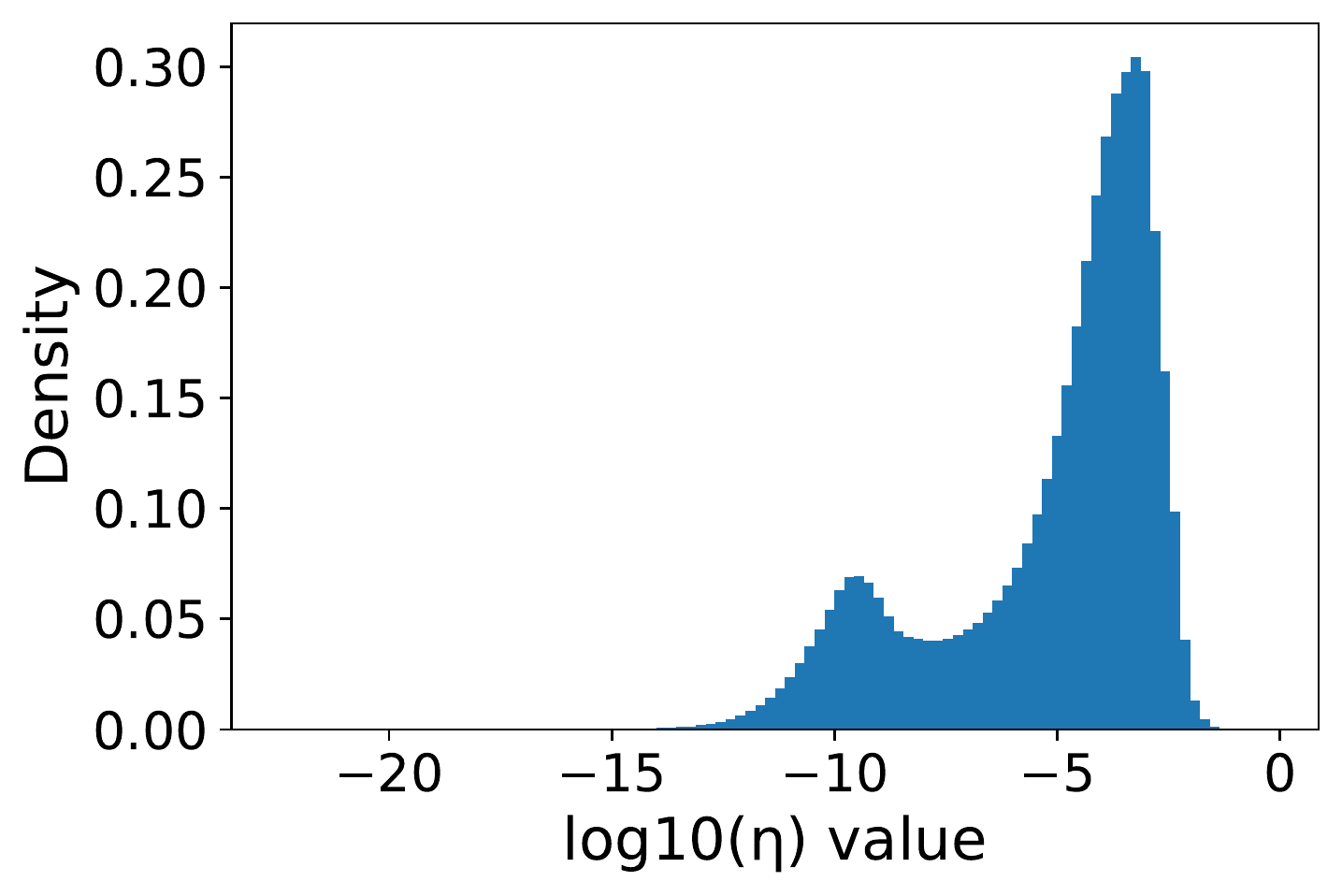}
		\caption{Epoch $50$}
		\label{subfig:ep50adameta}
	\end{subfigure}
	\begin{subfigure}{0.32\linewidth}
		\includegraphics[trim={0.0cm 0cm 0cm 0cm},clip, width=1\textwidth]{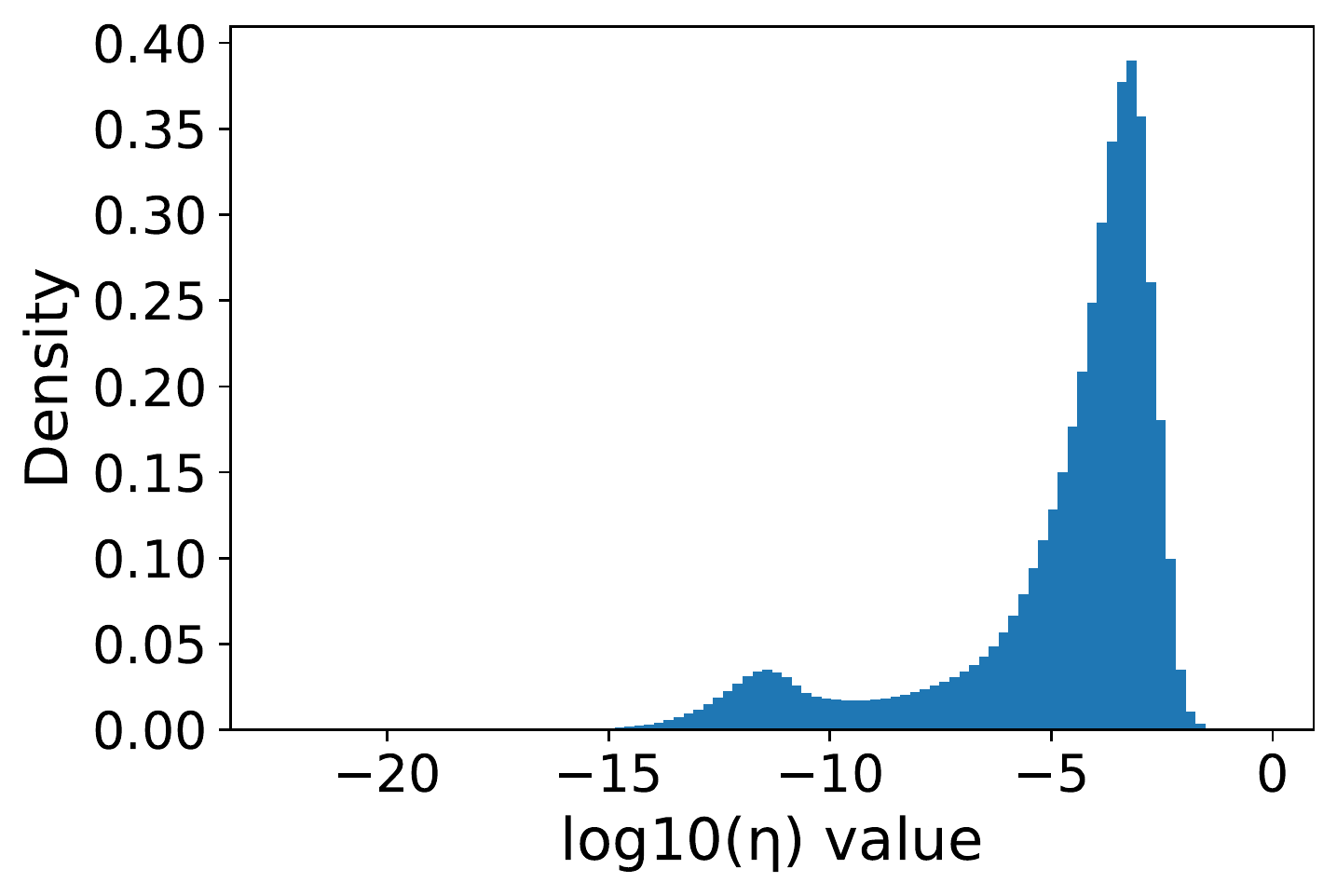}
		\caption{Epoch $75$}
		\label{subfig:ep75adameta}
	\end{subfigure}
	\caption{\textbf{Huge Variation in Scaling Coefficients for Adam.} Density of pseudo eigenvalues $\eta_{i}$ learned during training a VGG-$16$ on CIFAR-$100$ using the Adam optimiser for different epoch values, for $\alpha=0.0004,\gamma=0$ with a linear decay schedule from Section \ref{sec:experiments}.}
	\label{fig:adameta}
\end{figure}
In order to verify that the necessary conditions hold in the commonly used Adam optimiser for such a square root scaling to occur. We investigate the implied curvature eigenvalues $\eta_{i}$ from the Adam state dictionary \citep{chaudhari2016entropy} in the diagonal basis.
We plot the results for different epochs in Figure \ref{fig:adameta}. Note the huge range in value of $\eta$. With a maximum of $\approx 0.6$ and a practical minimum of $10^{-8}$ set by the damping coefficient.

In order to put this derived scaling rule to the test, we run experiments similar to those of Section \ref{sec:scaling}. We find the maximal initial learning rate for the VGG-$16$ on CIFAR-$100$ with no weight decay $\gamma=0$, which stably trains with the Adam optimiser. We use an initial learning rate of $\alpha_{0}=0.0004$ for a batch size of $B=128$ and then complete a linear learning rate decay schedule, as detailed in Section \ref{sec:experiments}. We then scale the learning rate with the square root of the batch size in either direction and plot the results. We drop the batch-size to $8$, to test the limits of our theory. The results are shown in Figure \ref{fig:vggadamsqrt}.
\begin{figure}[h!]
	\centering
	\begin{subfigure}{0.57\linewidth}
		\includegraphics[trim={0cm 0cm 0cm 0cm},clip, width=1\textwidth]{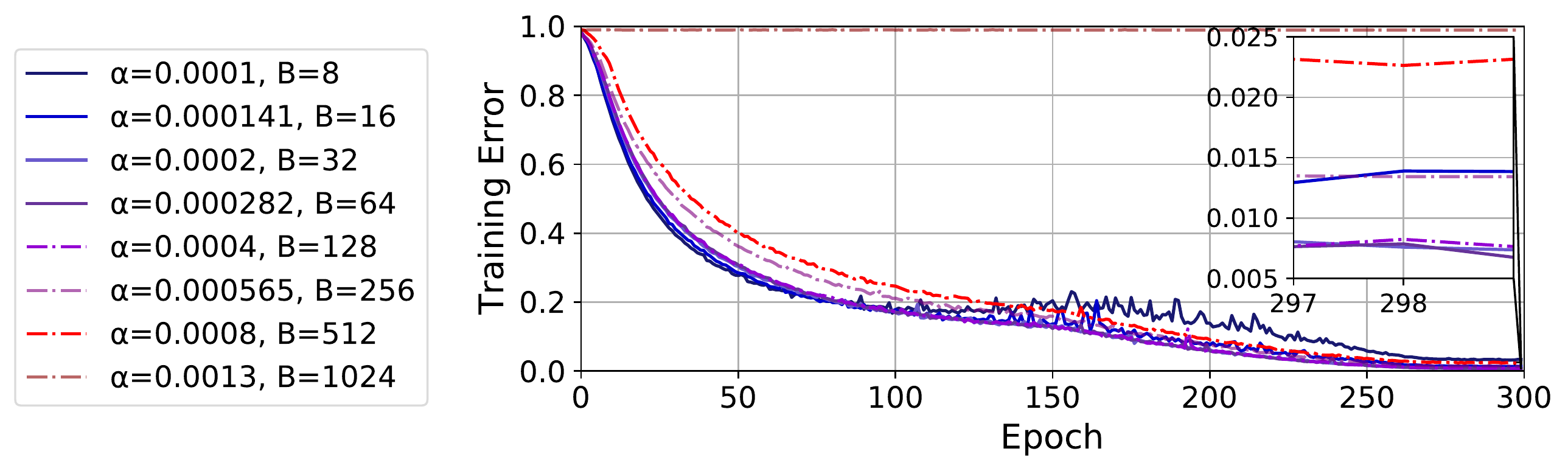}
		\caption{Training Error against Epoch}
		\label{subfig:sqrtadamvgg16train}
	\end{subfigure}
	\begin{subfigure}{0.42\linewidth}
		\includegraphics[trim={0.0cm 0cm 0cm 0cm},clip, width=1\textwidth]{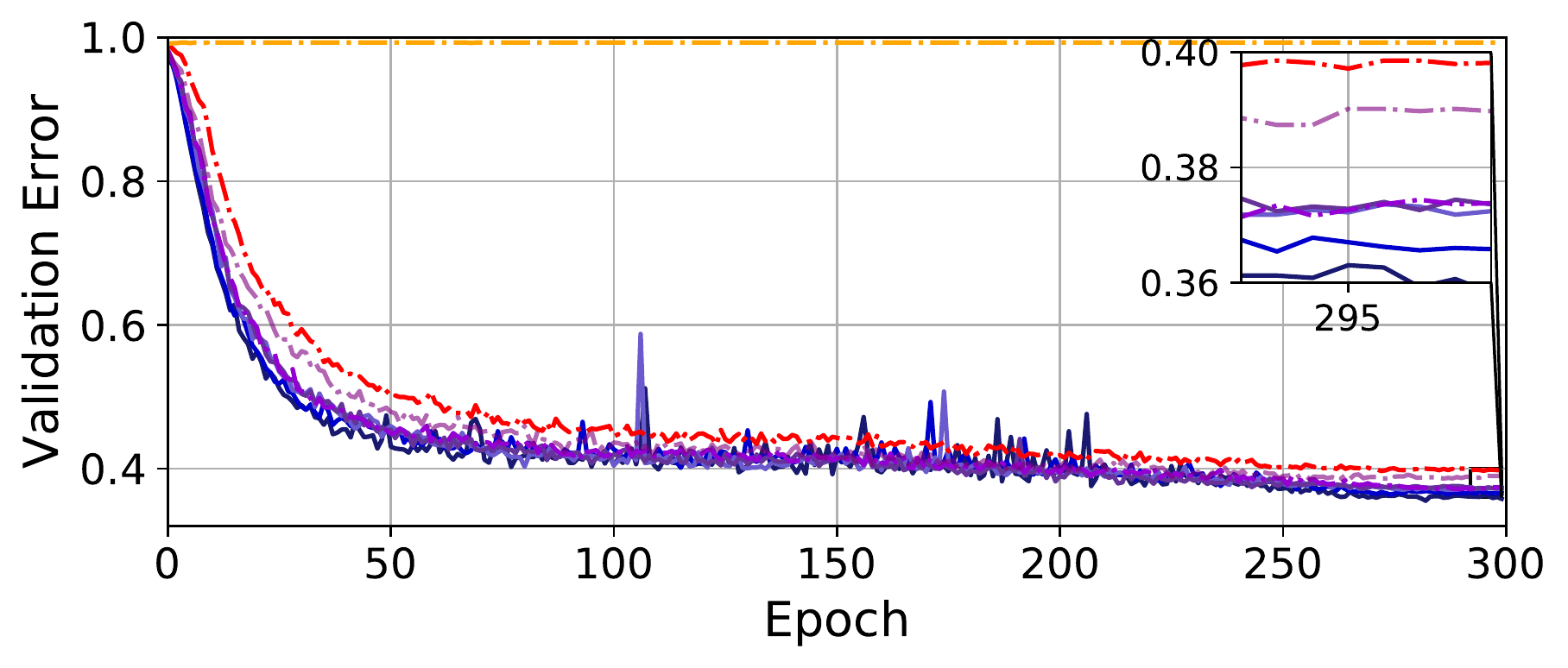}
		\caption{Validation Error against Epoch}
		\label{subfig:sqrtadamvgg16test}
	\end{subfigure}
	\caption{\textbf{Square root scaling for adaptive optimisers is consistent up to a threshold.} Training and Validation error of the VGG-$16$ architecture, without batch normalisation (BN) on CIFAR-$100$, with no weight decay $\gamma=0$ and initial learning rate $\alpha_{0}=\frac{0.004\sqrt{B}}{\sqrt{128}}$, which varies as a function of batch size $B$.}
	\label{fig:vggadamsqrt}
\end{figure}
We note excellent agreement down to $B=16$, with very small differences between training curves and validation performance. There is a slight instability in training for $B=8$, potentially indicating a regime where broadening of the outlier eigenvalues dominates the mis-estimation of the bulk distribution. Note that, when reducing the batch size, reducing the learning rate using square root scaling is a far more aggressive reduction schedule and hence, should the appropriate scaling be linear, training would quickly fail. As is shown for the SGD case, running the same learning rate of $0.01$ (from Section \ref{sec:scaling}) and reducing the learning rate using square-root scaling, leads to poor training and divergence, as shown in Figure \ref{fig:vggsgdsqrt}.

\begin{figure}[h!]
	\centering
	\begin{subfigure}{0.57\linewidth}
		\includegraphics[trim={0cm 0cm 0cm 0cm},clip, width=1\textwidth]{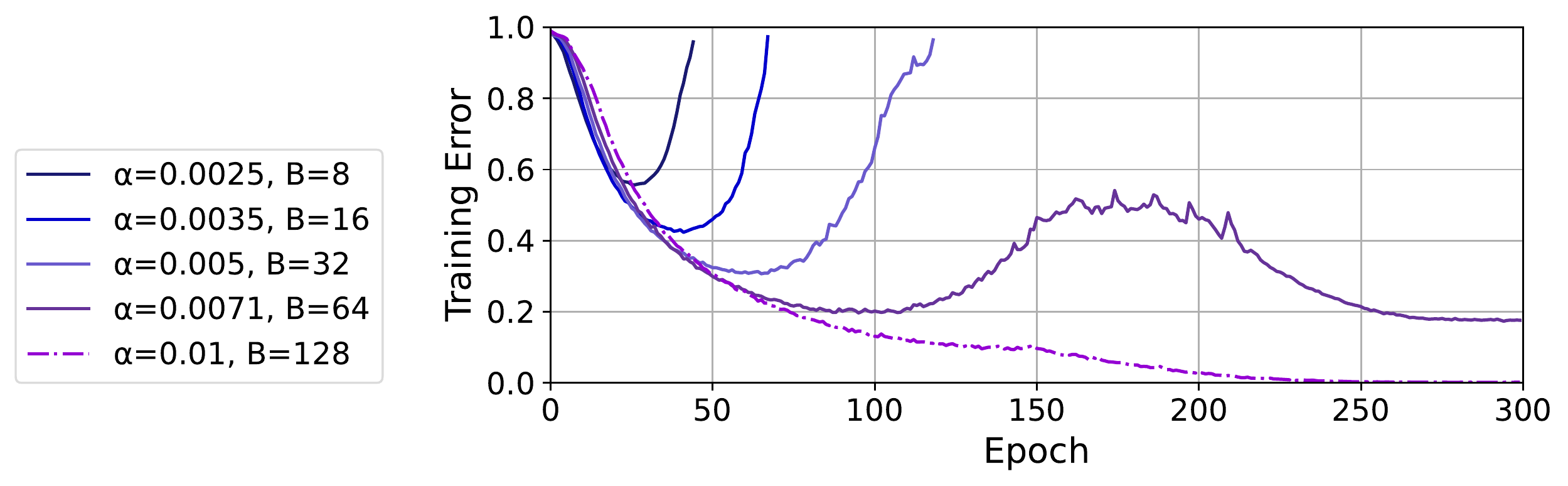}
		
		\caption{Training Error against Epoch}
		\label{subfig:vgg16trainsgdsqrt}
	\end{subfigure}
	\begin{subfigure}{0.42\linewidth}
		\includegraphics[trim={0.0cm 0cm 0cm 0cm},clip, width=1\textwidth]{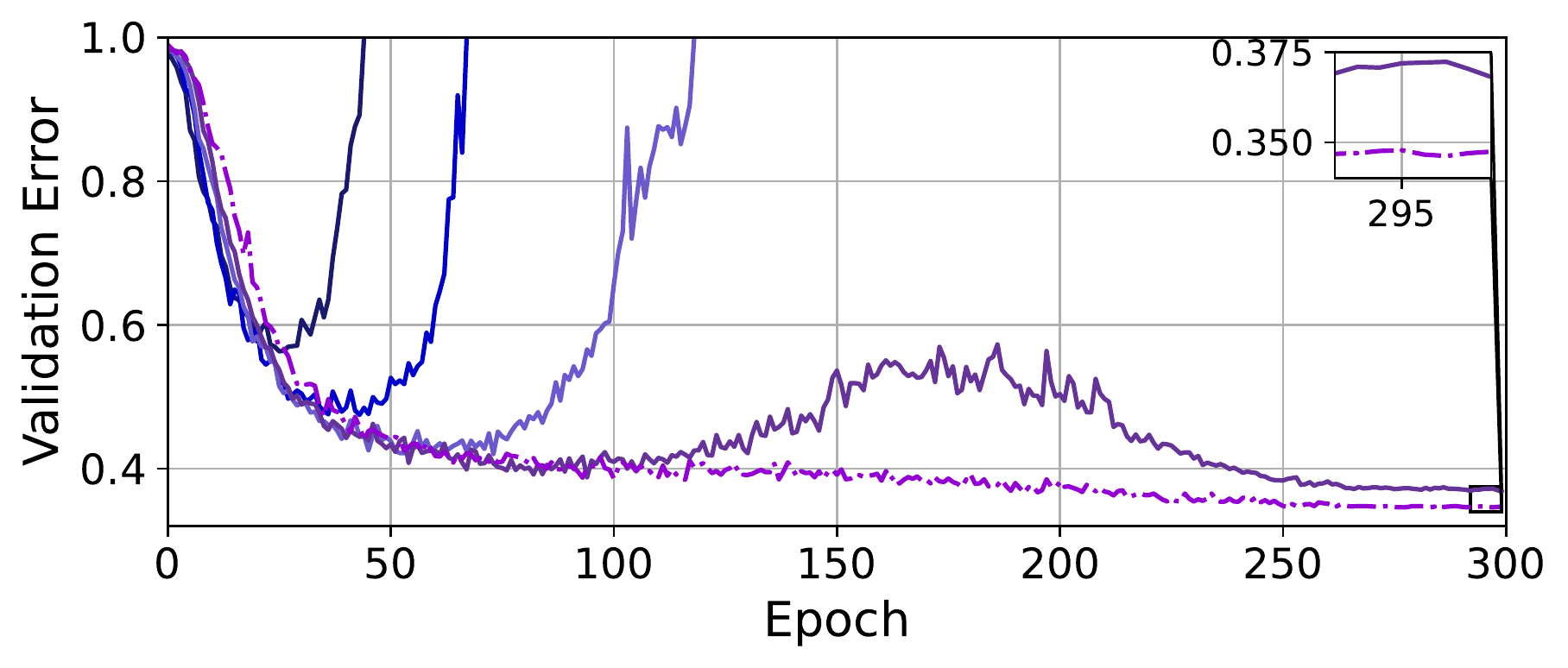}
		
		\caption{Validation Error against Epoch}
		\label{subfig:vgg16testsgdsqrt}
	\end{subfigure}
	
	\caption{\textbf{Square root scaling for SGD does not hold.} Training and Validation error of the VGG-$16$ architecture, without batch normalisation (BN) on CIFAR-$100$, with no weight decay $\gamma=0$ and initial learning rate $\alpha_{0}=\frac{0.01\sqrt{B}}{\sqrt{128}}$.}
	
	\label{fig:vggsgdsqrt}
\end{figure}
\begin{figure}[h!]
	\centering
	\begin{subfigure}{0.57\linewidth}
		\includegraphics[trim={0cm 0cm 0cm 0cm},clip, width=1\textwidth]{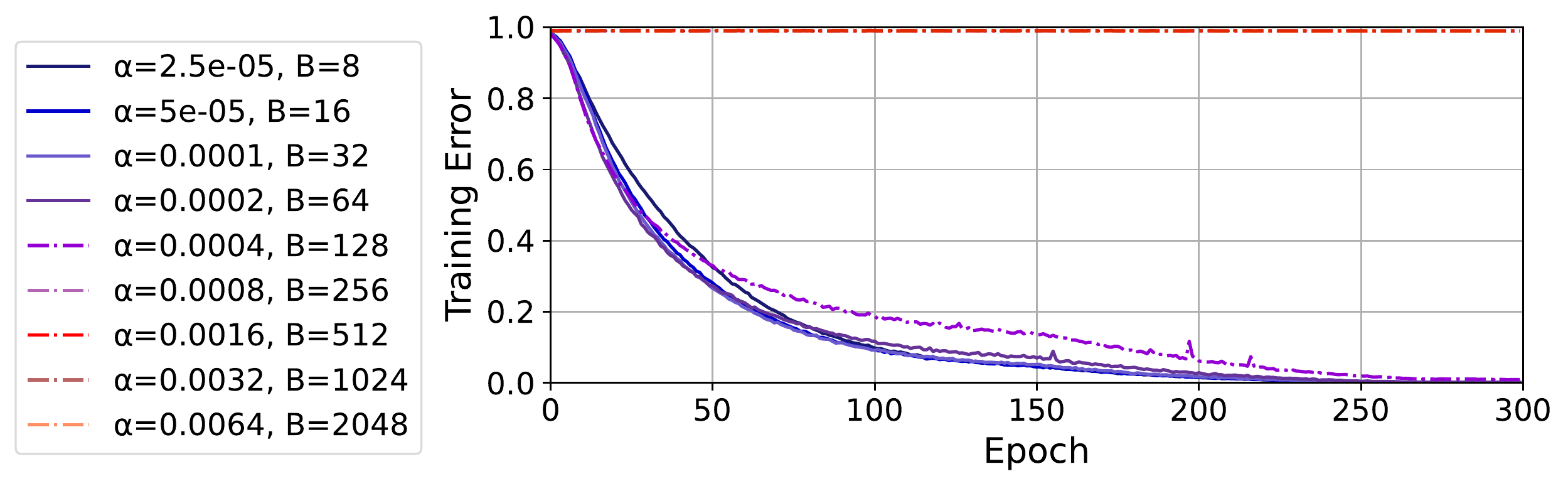}
		\caption{Training Error against Epoch}
		\label{subfig:vgg16trainadam4}
	\end{subfigure}
	\begin{subfigure}{0.42\linewidth}
		\includegraphics[trim={0.0cm 0cm 0cm 0cm},clip, width=1\textwidth]{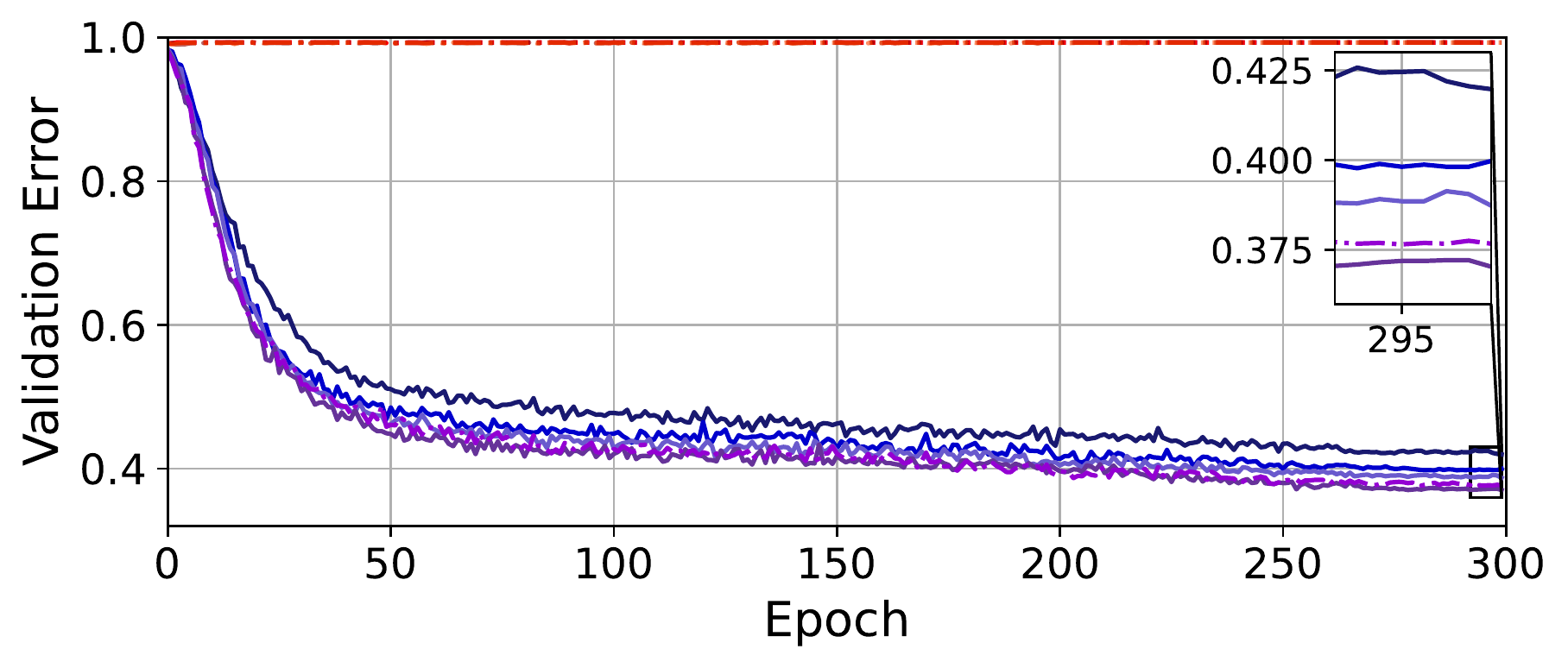}
		\caption{Validation Error against Epoch}
		\label{subfig:vgg16testadam4}
	\end{subfigure}
	\caption{\textbf{Linear scaling rate does not hold for Adam.} Training and Validation error of the VGG-$16$ architecture, without batch normalisation (BN) on CIFAR-$100$, with no weight decay $\gamma=0$ and initial learning rate $\alpha_{0}=\frac{0.0004B}{128}$, which varies as a function of batch size $B$.}
	\label{fig:vggadam4}
\end{figure}

We plot results from attempting to use the linear scaling rule for Adam in Figure \ref{fig:vggadam4}. Note that there is rapid divergence upon using a doubled batch size. Furthermore, we note that experiments with a reduced learning rate and lower batch size show different training loss profiles to that of the the initial setup (shown as the only dotted line that converges). Furthermore, upon inspecting the test error, we note that, unlike the case of square root scaling (which produces consistent test error estimates throughout the range), we see significant ($38.1 \% \pm 0.4$, $37.6 \% \pm 0.64$ $38 \% \pm 0.25$ $40.1 \% \pm 0.3$, $42.7 \% \pm 0.1$) increases in test error as we reduce the learning rate. 
We argue that, in expectation, the main effect of sub-sampling is a broadening of the spectrum (increasing eigenvalue magnitude)
at all points in weight space. Hence, for an equivalent
trajectory in expectation in weight space as we change the batch size, we would expect to scale the learning rate as the inverse of this increase in eigenvalue magnitude.
Hence large differences in test error suggest that we are not traversing the surface in an equivalent way and hence not settling to a minimum of similar sharpness/distance from initialisation.

Note that it is, in general, always possible to train the network with a lower learning rate. Indeed, in many such cases the training curves are almost indistinguishable. However, we find the test errors are often significantly worse, indicating an in-equivalence in trajectory traversal. We note that this is expected from our previous argument, in Section \ref{sec:testaccrmt}. There we argue that different learning rate schedules induce different trajectories across the loss surface and lead to minima of differing curvature. The latter, despite having similar training error, often have very different validation and test errors. We show this explicitly in Appendix \ref{sec:differentiatingschedules}.

\subsection{Validating the Accurate Estimation of the Large Eigenvalue/Eigenvector Pairs Assumption}
The past section relied heavily on the notion that, for adaptive optimisers, the sharpest eigenvalue/eigenvector pairs of the batch Hessian were better estimated than those at the edge of the bulk. Intuitively, it seems reasonable that these "pure noise" eigenvectors (which are distributed on the unit sphere) might change rapidly from iteration to iteration and since they occupy a sizeable fraction of the loss landscape (compared to the small number of outliers), some or many of them might be severely underestimated. However, given that it is unclear to what extent the empirical Fisher approximation \citep{kunstner2019limitations} faithfully approximates local curvature and noting that Adam serves as a running diagonal approximation to the empirical Fisher, it is unclear whether Adam learns any information about the top eigenvalue/eigenvector pairs of the batch Hessian. In order to test this empirically, we run Adam on the VGG-$16$ using CIFAR-$100$ with no weight decay and set the damping coefficient $\delta=1$. We use the linear learning rate schedule from Section \ref{sec:experiments}. Using the quadratic approximation from Equation \ref{eq:secorderlosschange}
we expect the largest possible learning rate before batch loss increases
to depend upon the ability, via $\mB$, to estimate the largest eigenvalue/eigenvector pairs of $\mH$\footnote{This follows as $\frac{\lambda_{i}}{\eta_{i}+1}$ is largest for large $\lambda_{i}/\eta_{i}$}. Hence, the largest learning rate achievable is a direct measure of the estimation accuracy of the sharpest eigenvalue/eigenvector pairs of the batch Hessian. We search for the highest stable learning rate, $\alpha$, along a logarithmic grid, with end points $\in (0.01,1)$. All methods incorporate a momentum of $\rho = 0.9$. We find that the largest \emph{stable} rates for SGD, Adam and KFAC are $0.01,0.12,0.32$ respectively, indicating that both KFAC and Adam are significantly better able to estimate sharp curvature directions than curvature agnostic SGD. Given that KFAC is a well-known second-order method that uses the Fisher approximation rather than an empirical measure, we believe this experiment indicatives that Adam reliably learns information about the largest eigenvalue/eigenvector pairs.

We show the training and testing error curves from these experiments in Figures \ref{fig:adamgradcov}, \ref{fig:sgdgradcov} and \ref{fig:kfacgradcov} for Adam, SGD and KFAC respectively. For SGD and Adam, overly large learning rates lead to returning to either near random or very low performance. We note that, certain curves for Adam and for all the KFAC curves, have more nuanced form. For KFAC we can even use a learning rate of $1$ without running into $NaN$ errors. However, even when annealing this learning rate by a factor of $100$ at the end of training, we do not converge in training or test error. 
We hence loosely define ``stable" as the largest permissible value which allows for training and in the event that a wide range of learning rates are permitted we take the value which gives the best test error. .
\begin{figure}[h!]
	\centering
	\begin{subfigure}[b]{0.48\textwidth}
		\includegraphics[width=\textwidth,trim={0cm 0.2cm 0 0},clip]{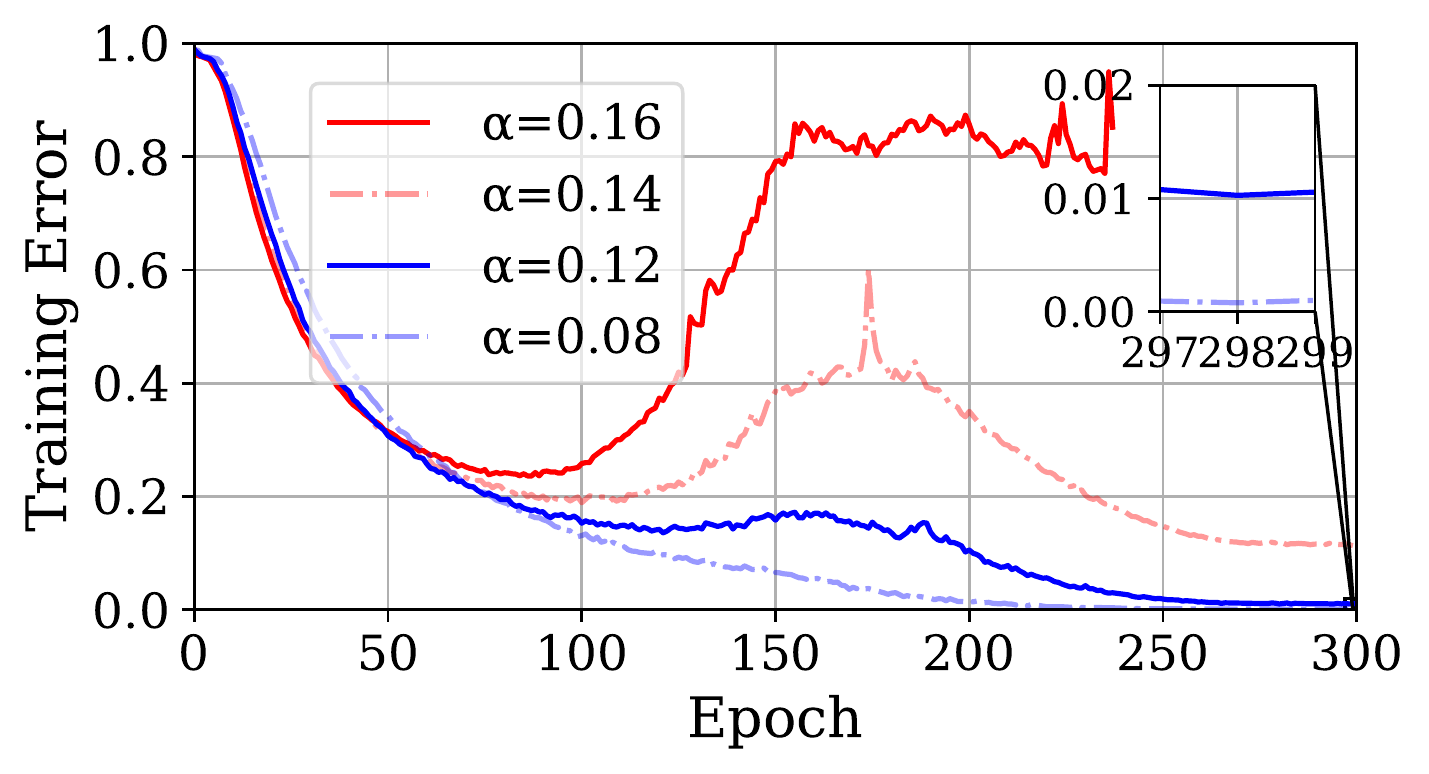}
	\end{subfigure}
	\begin{subfigure}[b]{0.48\textwidth}
		\includegraphics[width=\textwidth,trim={0cm 0.25cm 0 0},clip]{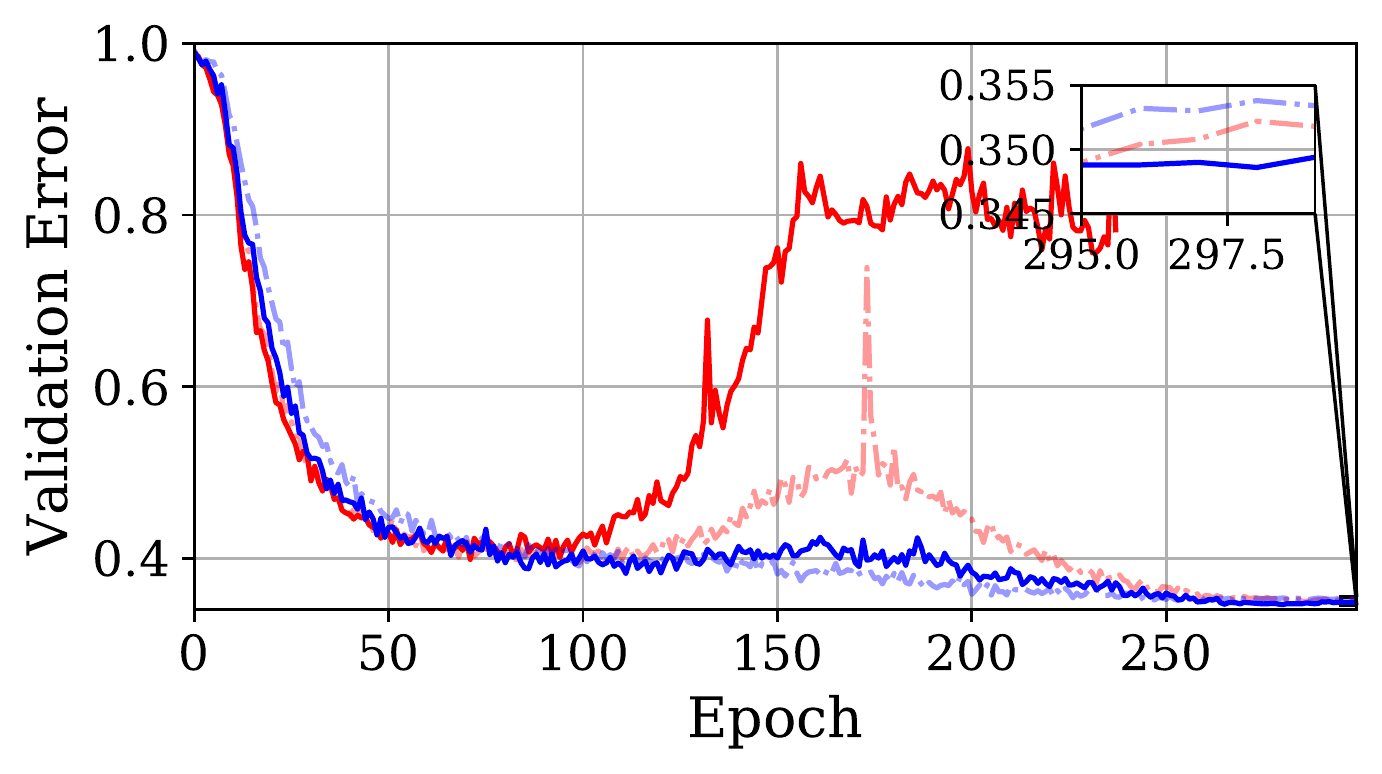}
	\end{subfigure}
	\caption{Training/Validation error of Adam using $\delta=1$ and $\alpha$ on the VGG-$16$ CIFAR-$100$}
	\label{fig:adamgradcov}
\end{figure}
\begin{figure}[h!]
	\centering
	\begin{subfigure}[b]{0.48\textwidth}
		\includegraphics[width=\textwidth,trim={0cm 0.2cm 0 0},clip]{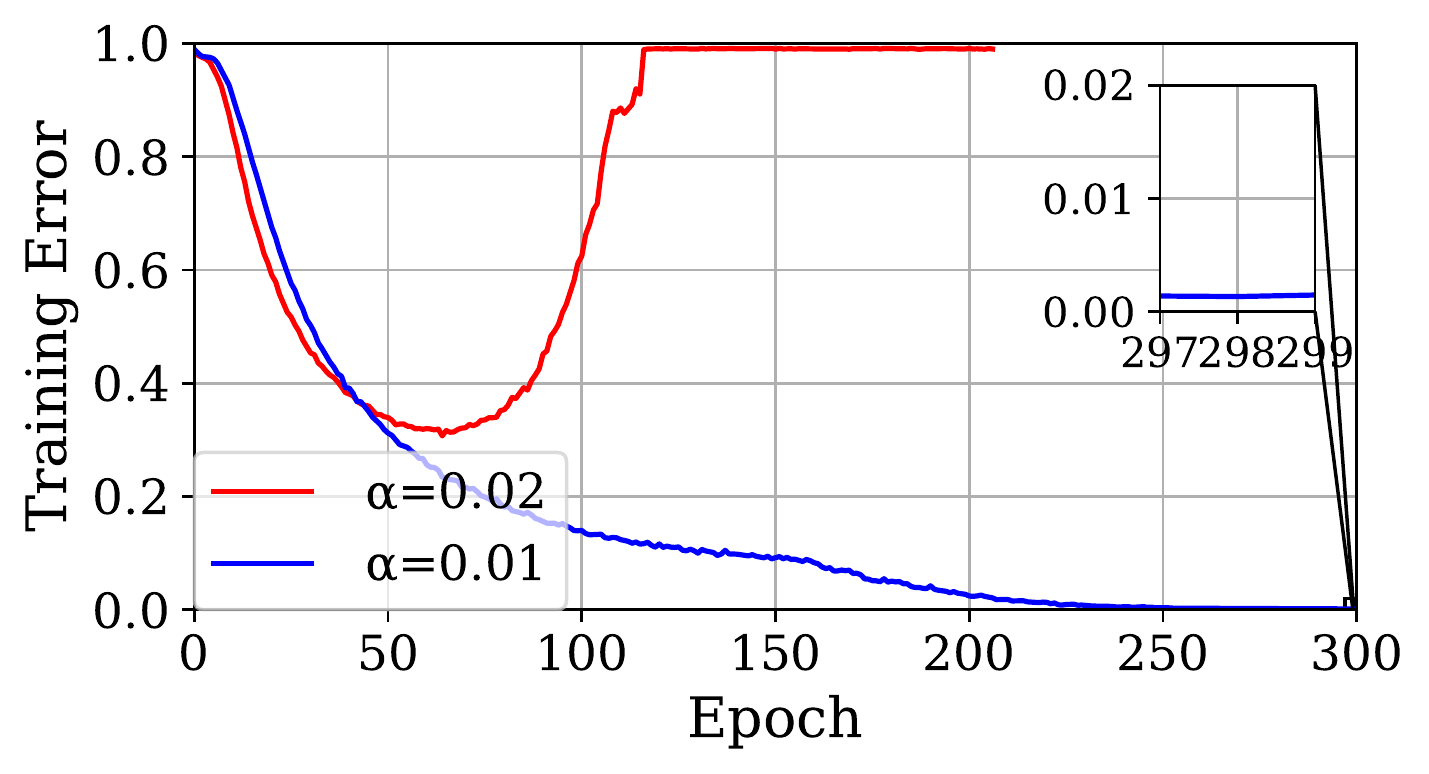}
	\end{subfigure}
	\begin{subfigure}[b]{0.48\textwidth}
		\includegraphics[width=\textwidth,trim={0cm 0.25cm 0 0},clip]{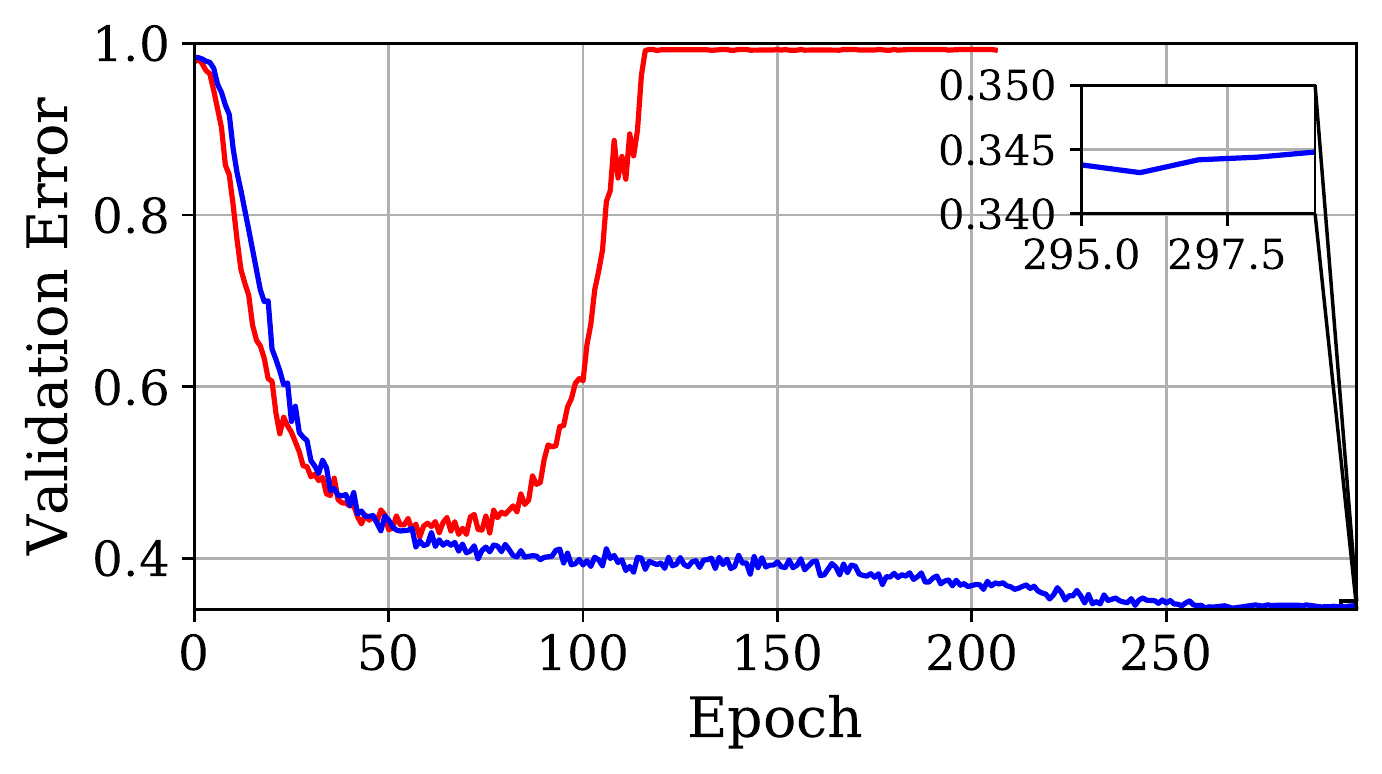}
	\end{subfigure}
	\caption{Training/Validation error of SGD using $\delta=1$ and $\alpha$ on the VGG-$16$ CIFAR-$100$}
	\label{fig:sgdgradcov}
\end{figure}
\begin{figure}[h!]
	\centering
	\begin{subfigure}[b]{0.53\textwidth}
		\includegraphics[width=\textwidth,trim={0cm 0.2cm 0 0},clip]{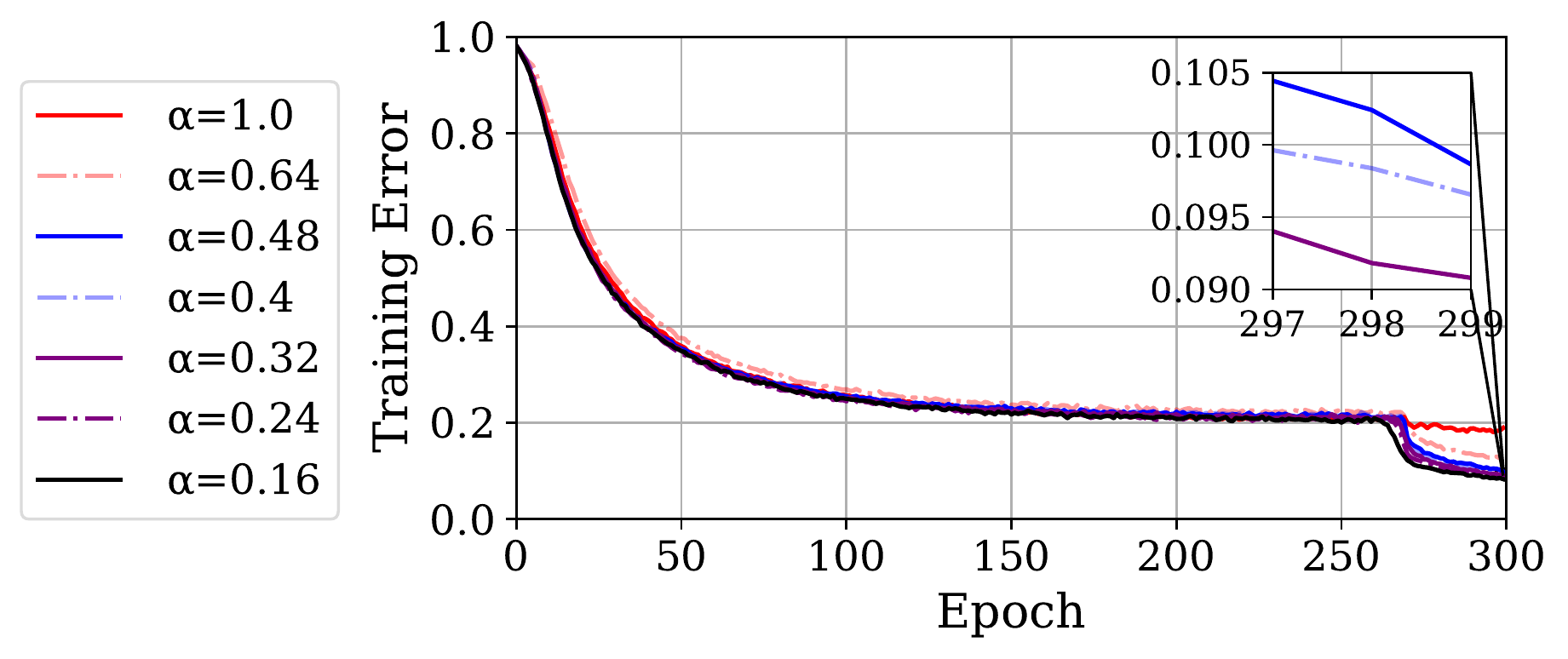}
	\end{subfigure}
	\begin{subfigure}[b]{0.42\textwidth}
		\includegraphics[width=\textwidth,trim={0cm 0.25cm 0 0},clip]{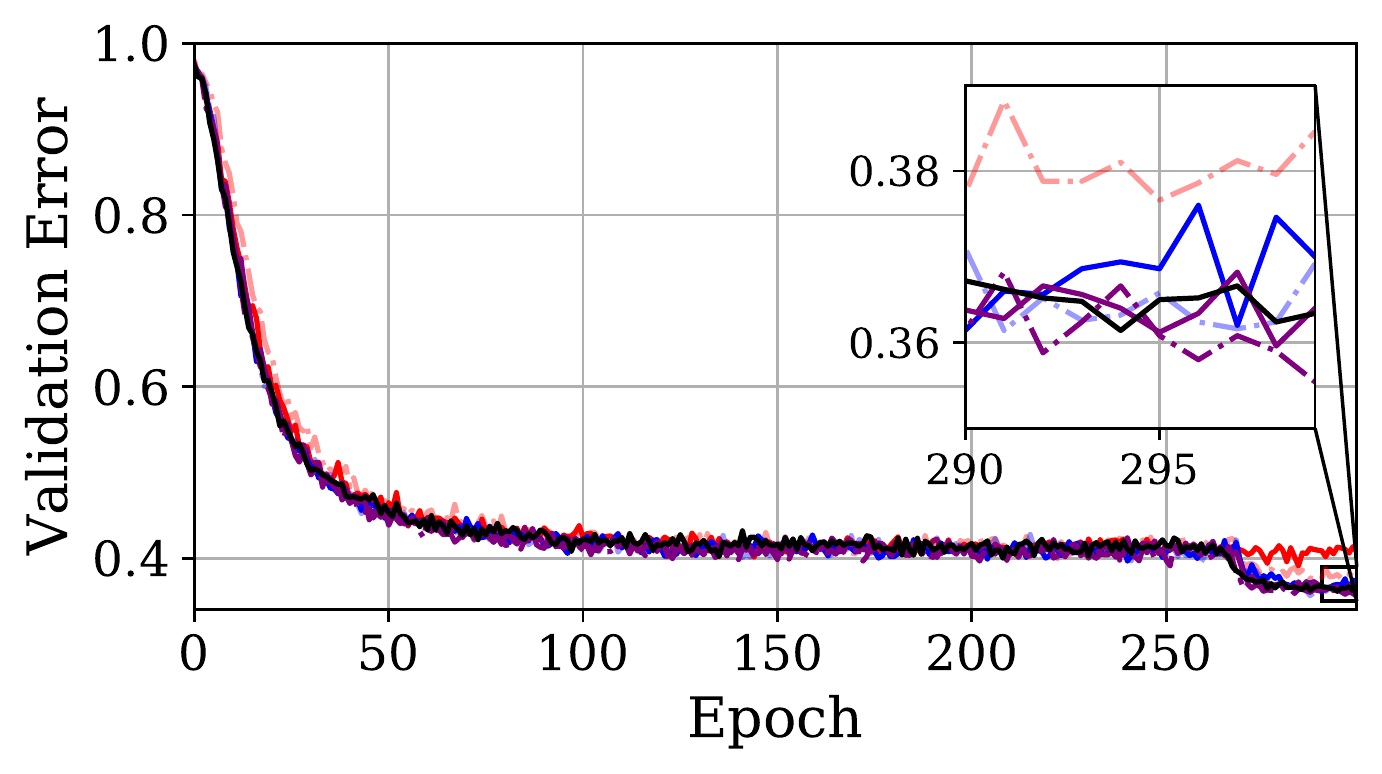}
	\end{subfigure}
	\caption{Training/Validation error of KFAC using $\delta=1$ and $\alpha$ on the VGG-$16$ CIFAR-$100$}
	\label{fig:kfacgradcov}
\end{figure}

\subsection{Comparison to Gradient Noise}
\label{subsec:gradnoise}

We consider in this section deriving similar scaling relations without a random matrix theory approach. This line of thought is  similar to the work in \citet{smith2017bayesian,smith2017don}. Consider a linear approximation to the change in loss from one step of SGD,
\begin{equation}
	L(\vw-\alpha\nabla L(\vw)) - L(\vw) = -\alpha||\nabla L(\vw)||^{2}.
\end{equation}
For a mini-batch of size $B$ (assuming independent samples), whilst the value of the change in loss remains unchanged in expectation, the variance changes as 
\begin{equation}
	\Var \bigg(L(\vw)-\alpha\bigg(\frac{1}{B}\sum_{i}^{B}\nabla L(\vw)_{i}\bigg)\bigg(\frac{1}{B}\sum_{j}^{B}\nabla L(\vw)_{j}\bigg)\bigg)^{T} \propto \frac{\alpha^{2}}{B^{2}},
\end{equation}
Under such a framework, where we take a linear approximation, we would scale the batch size linearly with the learning rate (within the limits of the approximation) in order to keep the loss variance similar. However, unlike our approach which considers optimal learning rates based on the curvature (and the changes in curvature with batch size),
it is somewhat unclear why we would want to keep the variance of the loss equal.
Furthermore, unlike our method, for which we can estimate the curvature of the batch from the batch and use this to inform our choice of learning rates and momenta for SGD, the variance of the loss gives no prescription as to which learning rates or momenta we should choose. Furthermore for adaptive methods, the corresponding variance is now 
\begin{equation}
	\alpha^{2} \Var\bigg(\mB^{-1}\bigg(\frac{1}{B}\sum_{i}^{B}\nabla L(\vw)_{i}\bigg)\bigg(\frac{1}{B}\sum_{j}^{B}\nabla L(\vw)_{j}\bigg)^{T}\bigg) \propto 
\end{equation}
which is a function of the preconditioning matrix and the covariance of the gradients and from which no clear scaling relationships follow (hence we do not give a scaling in the resulting equation). Considering a spiked random matrix model of the covariance of the gradients, would give similar results to our analysis, but again would require the use of random matrix theory. Whilst it is possible to ignore the effect of the pre-conditioning matrix as in \citet{smith2017don} and argue for a linear prescription, we show in this paper the sub-optimality of this naive approach.
\section{True Hessian}
\label{sec:truelosssurface}
We have, to this point, considered the empirical Hessian and the batch Hessian. We here consider the Hessian under the data generating distribution, originally (partially) investigated in \citet{granzioldeep2018}. For finite $P$ and $N \rightarrow \infty$, i.e. $q = P/N \rightarrow 0$, $|\meps(\vw)| \rightarrow 0$ the empirical Hessian would become the true Hessian. $P,N$ refer to the parameter count and dataset size respectively. Similarly, in this limit, the empirical risk converges almost surely to the true risk, i.e. we eliminate the \textit{generalisation gap}. 
However, in much deep learning, the network size eclipses the dataset size by orders of magnitude.\footnote{CIFAR datasets, which have $50,000$ examples, are routinely used to train networks with about $50$ million parameters.} This is similar to considering perturbations between the true covariance matrix and the noisy sample covariance matrix, extensively studied in mathematics and physics \citep{baik2006eigenvalues,bloemendal2016limits,bloemendal2016principal}. 
\begin{figure}[htbp]
	\centering
	\begin{subfigure}{0.32\linewidth}
		\includegraphics[trim={0cm 0cm 0cm 0cm},clip, width=1\textwidth]{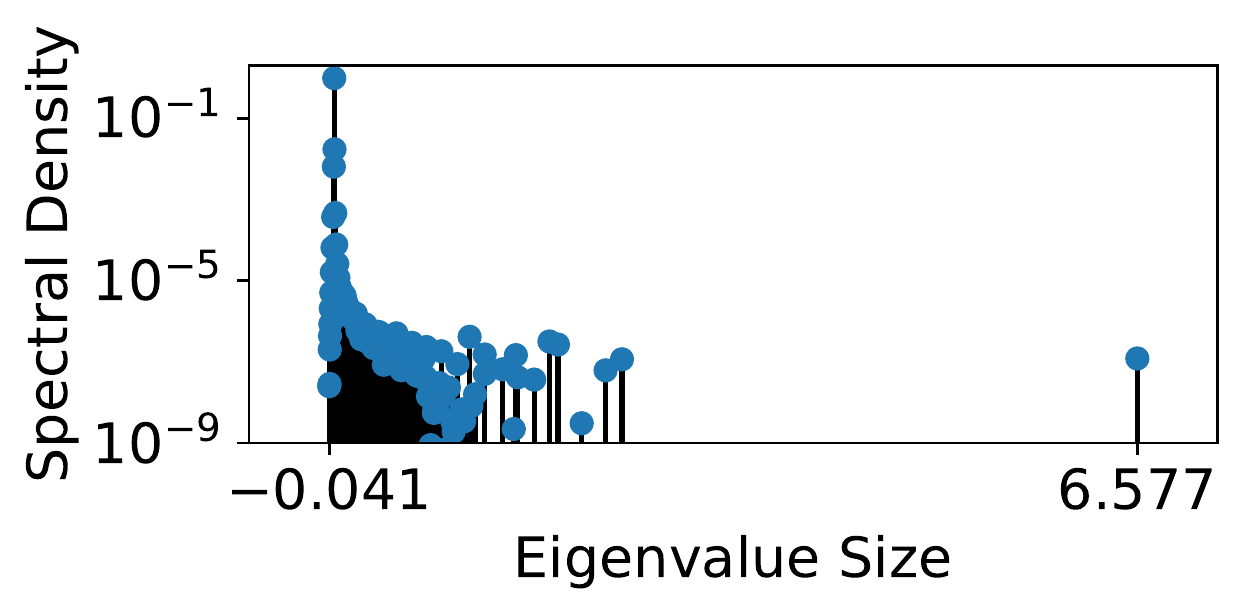}
		\caption{Emp-Hessian $N = 50,000$ }
		\label{subfig:vgg16emp}
	\end{subfigure}
	\begin{subfigure}{0.32\linewidth}
		\includegraphics[trim={0cm 0cm 0cm 0cm},clip,
		width=1\textwidth]{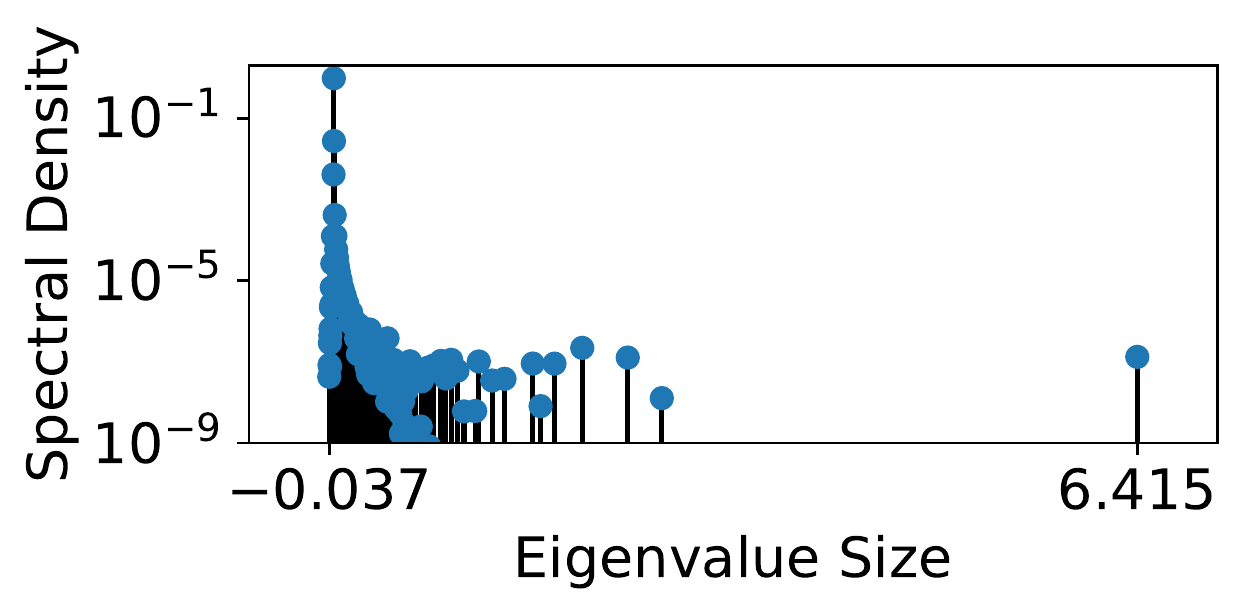}
		\caption{Aug-Hessian $N = 500,000$ }
		\label{subfig:vgg16aug1}
	\end{subfigure}
	\begin{subfigure}{0.32\linewidth}
		\includegraphics[trim={0cm 0cm 0cm 0cm},clip, width=1\textwidth]{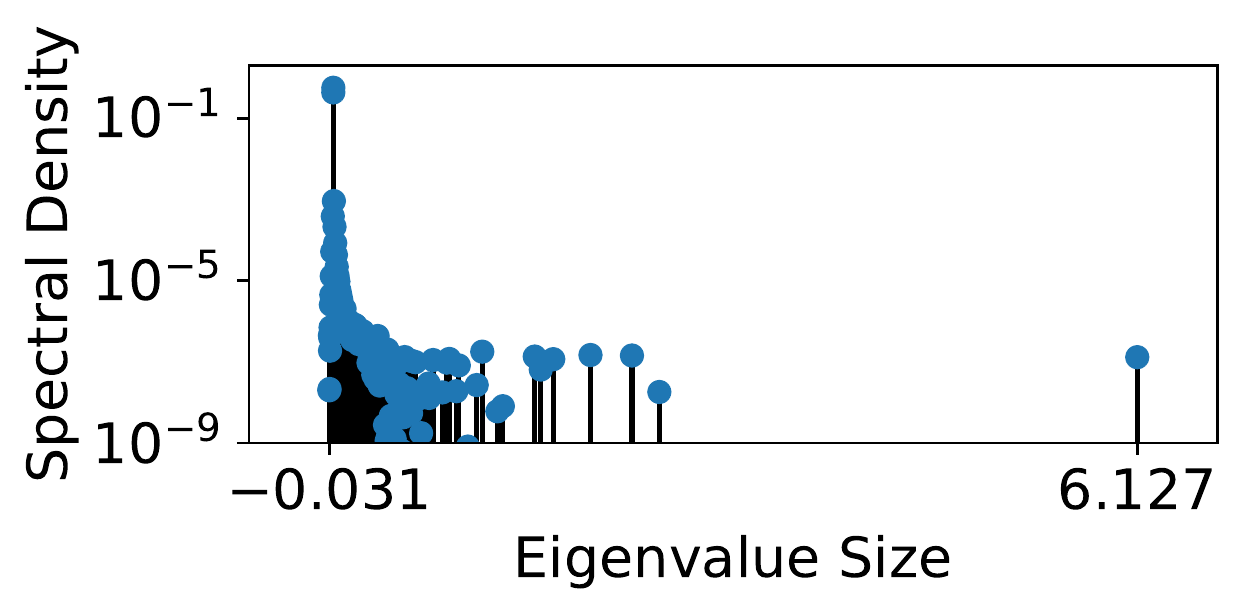}
		\caption{Aug-Hessian $N = 5,000,000$ }
		\label{subfig:vgg16aug2}
	\end{subfigure}
	\caption{Hessian Spectral Density at epoch $300$ , on a VGG-$16$ on the CIFAR-$100$ dataset, for different amounts of data-augmentation.}
	\label{fig:truehess}
\end{figure}
As the true Hessian and true risk are in practice unobservable, we consider whether the empirical Hessian provides a valid approximation to the true Hessian. We evaluate the Hessian by artificially increasing the number of samples through the use of data-augmentation. We simulate the effect of increasing the data-set size, with random horizontal flips, $4\times4$ zero padding and random $32\times 32$ crops. We then use the Pearlmutter trick \citep{pearlmutter1994fast} on the augmented dataset, combined with the Lanczos algorithm, to estimate the spectral density. While there is clear dependence between the augmented samples and original samples, intuitively we can consider the augmented dataset to be equivalent to an independent set, of larger size than the original dataset. This intuition is grounded in the observation that training without augmentation leads to significant performance decreases, similar to reducing the dataset size. Specifically, for the VGG-$16$ without augmentation, we achieve a testing accuracy of $48.8\%$ compared to $72.1\%$ on the CIFAR-$100$ dataset running the same schedule. As shown in Figure \ref{fig:truehess}, the extremal eigenvalues are reduced in size as the sample number is increased, in accordance with Theorem \ref{theorem:mainresult}.
We note from Figure \ref{fig:truehess} that, despite a factor of $100$ in augmentation, the differences between the most augmented and empirical Hessian are slight. There is a slight reduction in the extremal eigenvalues, but otherwise the general shape of the eigenspectrum remains unaffected - implying that the true Hessian may be similar (in its spectral properties) to the empirical Hessian. This is a different conclusion to that reached in \citep{granzioldeep2018}, where the authors conclude that the true Hessian is very different to that of the empirical Hessian. Note that, were we to consider our batch Hessian to be i.i.d. draws from the data generating distribution, we need only replace $\mathfrak{b}$ in Equation \ref{theorem:mainresult} with $B$, where $\mathfrak{b} > B$ and for $B\ll N$ then $\mathfrak{b} \approx B$.

\section{Why do DNN Spectra always have Outliers?}
\label{sec:whywehaveoutliers}
As is evident from Theorems \ref{theorem:mainresult} and \ref{theorem:batchtheoremgnn}, the scaling of the spectral norm as a function of batch size
depends on whether the spectrum has outliers or not. \citet{papyan2020traces} uses a formal  procedure to attribute various parts of the spectrum to within-class and cross-class covariances. These are then validated experimentally on the VGG-$11$ architecture with a subsampled CIFAR-$10$ dataset. The work analytically shows, for softmax regression and a $k$-class Gaussian Mixture Model, that the distance between the outliers (along with the distance between the outliers and the bulk) increases as a function of class separation.
Future work can look to extend these results to a $1$-hidden layer MLP. Note that, in Figure \ref{fig:hessiangmm},  both at initialisation and training end (where we achieve good class separation and training accuracy), distinct outliers in the spectrum are observed.  
\begin{figure}[h!]
	\centering
	\begin{subfigure}{0.46\linewidth}
		\includegraphics[width=1\linewidth,trim={0 0 0 0},clip]{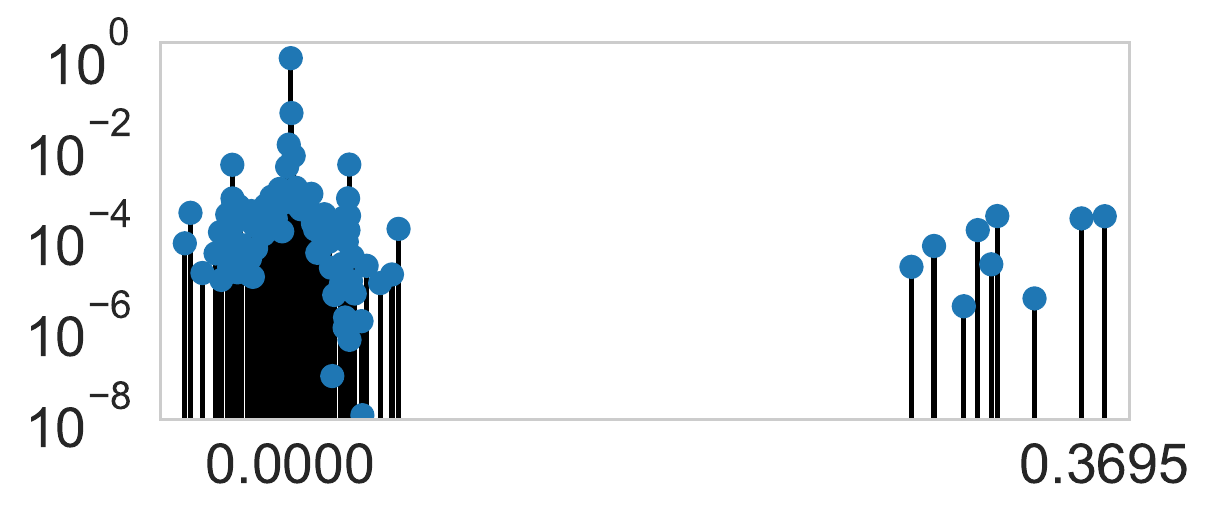}
		\caption{Spectrum at Initialisation}
		\label{subfig:gmmk=10init}
	\end{subfigure}
	\begin{subfigure}{0.46\linewidth}
		\includegraphics[width=1\linewidth,trim={0 0 0 0},clip]{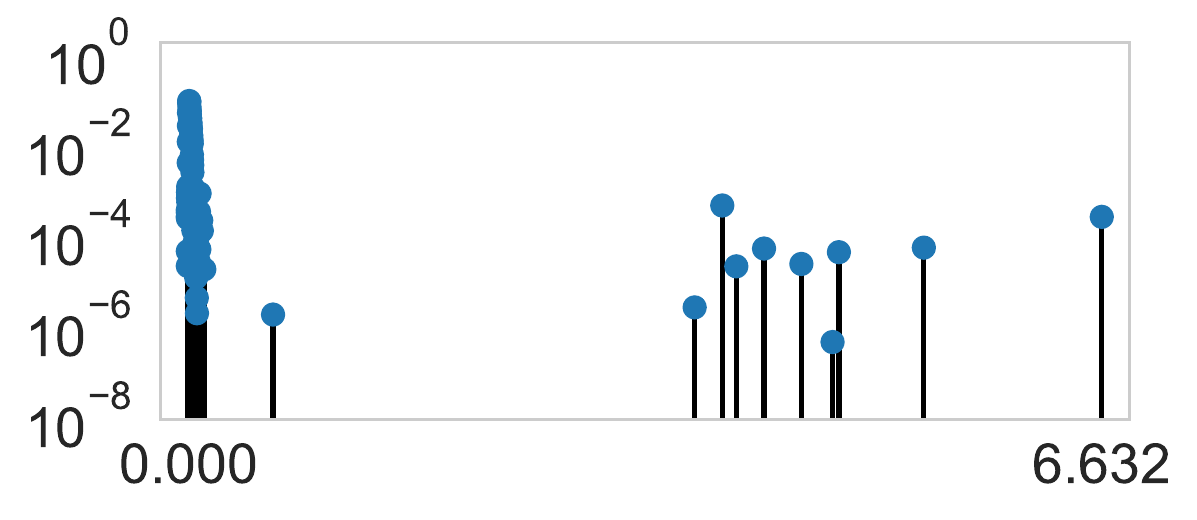}
		\caption{Spectrum at Training End}
		\label{subfig:gmmk=10end}
	\end{subfigure}
	\caption{Outlier persistence: Spectrum of a 1-Layer MLP on a 10 class Gaussian Mixture model at Initialisation (where performance is random, at $10\%$) and End of Training for $100$ Epochs with $\alpha=0.01$ and linear learning rate decay (where the performance is over $95\%$).}
	\label{fig:hessiangmm}
\end{figure}
We also note, at initialisation of the VGG network in Figure \ref{subfig:vggstart}, that we observe some outliers. 
Following \citet{papyan2020traces} we train with an extreme learning rate (we use $\alpha=0.2$ with $\gamma=0$ weight decay - compared to best performance values for this network of $\alpha=0.01$ and no weight decay). We find that this extreme training regime takes the network to a point of no return (i.e. we cannot improve training performance from this point, even by reducing the learning rate and increasing the weight decay factor). We see that, at the end of training, the largest eigenvalues are clustered together to form an outlying continuous spectral density, as shown in Figure \ref{subfig:vggnotrainend}. This implies \citep{papyan2020traces} that there is no class separation.
We plan to investigate the apparent need, during training, to retain singleton spectral outliers in future work.

These preliminary experiments, along with the work of \citet{papyan2020traces}, indicate that for the neural networks in the regime in which we are interested in (i.e. initialisation and training regimens where the network training and testing accuracy continually increase) outliers should be present in the spectrum. 
\begin{figure}[h]
	\centering
	\begin{subfigure}{0.46\linewidth}
		\includegraphics[width=1\linewidth,trim={0 0 0 0},clip]{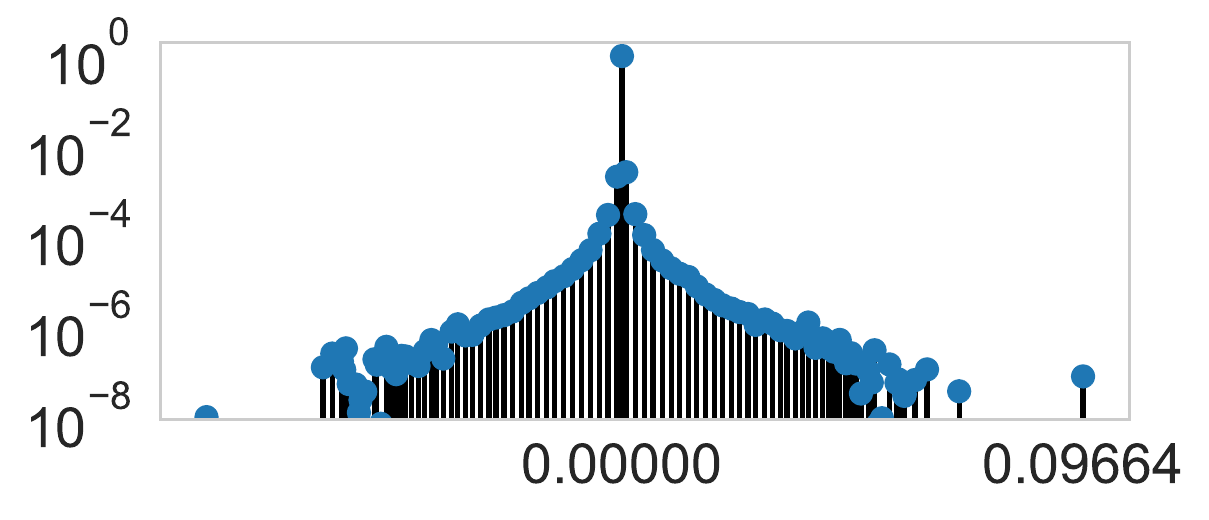}
		\caption{Spectrum at Initialisation}
		\label{subfig:vggstart}
	\end{subfigure}
	\begin{subfigure}{0.46\linewidth}
		\includegraphics[width=1\linewidth,trim={0 0 0 0},clip]{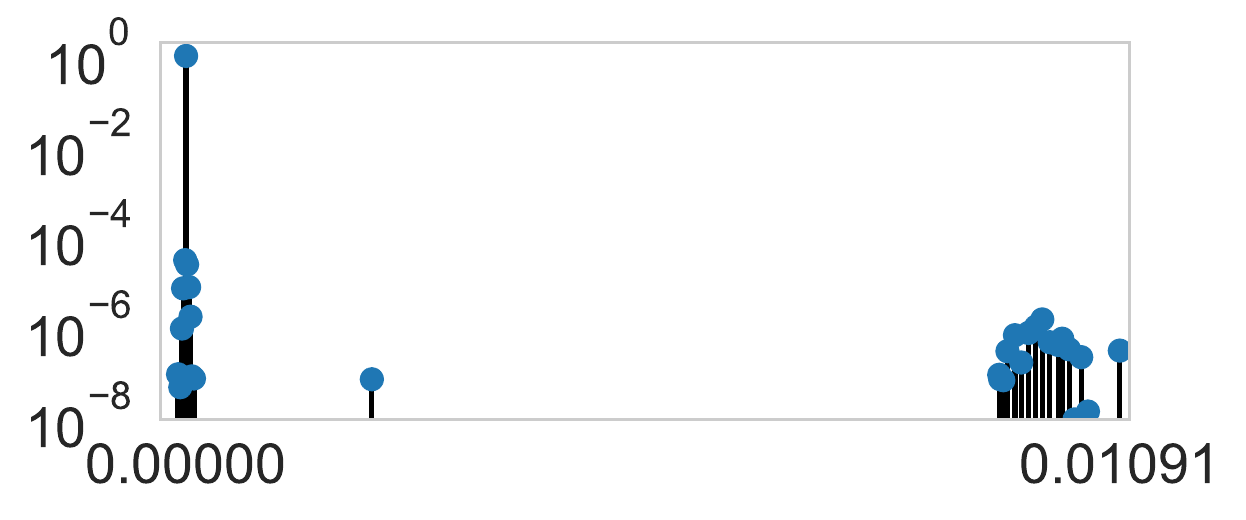}
		\caption{Spectrum at Training End}
		\label{subfig:vggnotrainend}
	\end{subfigure}
	\caption{\textbf{VGG Network Always contains outlayers}. Spectrum of a 16-Layer VGG Network on the 10 class CIFAR-$100$ dataset, at Initalisation (where performance is random at $1\%$) and end of Training for $100$ Epochs with $\alpha=0.2,\gamma=0$ and linear learning rate decay (where the performance remains random and is irrecoverable even with a learning rate drop).}
	\label{fig:hessianvggnotrain}
\end{figure}

\end{Correction}

\section{Conclusion}\label{sec:conclusion}
This paper shows that, under a spiked, field-dependent random matrix model, the extremal eigenvalues of the batch Hessian are larger than those of the empirical Hessian. The magnitude of the perturbation is inversely proportional to the batch size if there are well separated outliers in the Hessian spectrum and inversely proportional to the square root if not. The main implications of this work are that up to a threshold: 1) SGD learning rates should be scaled linearly with batch size; 2) Adam learning rates should be scaled with the square root of the batch size; 
\begin{Correction}
3)When trying to learn the learning rates and momenta directly from curvature, we should use the batch not empirical Hessian.
\end{Correction}
We extensively validate our predictions and associated implications on the VGG-$16$ network and CIFAR-$100$ dataset, across various hyper-parameter settings, including weight-decay and batch-normalisation. For SGD, we further validate our prediction on the WideResNet-$28\times10$ on both the CIFAR-$100$ and ImageNet-$32$ datasets. Given that our analysis is neither dataset nor architecture specific, we expect our results to hold generally outside of our experimental setup. This work can be used to better inform practitioners of how to adapt learning rate schedules for both small and large devices in a principled manner.
\begin{Correction}

\section{Acknowledgements}
The authors would like to thank Nicholas P Baskerville and Jon Keating for discussions surrounding the context behind this work and Random Matrix Theory, along with Timur Garipov and Dmitry Vetrov for the opportunity to develop software relevant to this line of research in Moscow.
The lead author would further like to thank the Oxford-Man Institute for funding the initial stages of this extensive research direction, Andrew Gittings from the JADE team for extensive computational resources and Hafiz Tiomoko \& Xingchen Wan for further discussions and interesting ideas on potential experimental validation.
\end{Correction}

\bibliography{bibliography}
\appendix

\section{Proof of the Central Lemma}
\label{sec:backgroundtheory}

A full proof of Theorem \ref{theorem:mainresult}, which rests heavily on disparate yet known results in the literature \citep{gotze2012semicircle,benaych2011eigenvalues,bun2017cleaning} would span many dozens of pages, repeating prior work. We hence adopt an alternative proof strategy, which we hope is understandable and relatable to a machine learning audience, for which this work is intended. We first introduce a minimum amount of necessary random matrix theory background. We then prove Theorem \ref{theorem:mainresult}, but under the stronger assumptions that the elements of the fluctuation matrix are i.i.d. Gaussian (the Gaussian Orthogonal Ensemble \citep{tao2012topics}). To understand why this makes sense, we consider the key ingredients of the proof
\begin{itemize}
\item The fluctuation matrix converges to the semi-circle law (which we introduce and explain in the next Section);
\item the spectral perturbation low-rank empirical Hessian by the fluctuation matrix can be computed analytically using perturbation theory; 
\item By Lemma \ref{lemma:normalelements}, the scaling relationships which characterise the extent of the noise perturbation as a function of batch size can be analysed.
\end{itemize}
Hence the only difference between the simplified proof and Theorem \ref{theorem:mainresult}, is that we have more general conditions for the convergence semicircle law, which are detailed extensively in \citet{gotze2012semicircle}. The other two key components proceed in an identical fashion.

\subsection{Background}
\label{sec:notationandproof}
Following the notation of \citep{bun2017cleaning} the resolvent of a matrix $H$ is defined as
\begin{equation}
\mG_{H}(z) = (z\mI_{N}-\mH)^{-1}
\end{equation}
with  $z = x + i\eta \in \mathbb{C}$. The normalised trace operator of the resolvent, in the $N \rightarrow \infty$ limit
\begin{equation}
\mathcal{S}_{N}(z) = \frac{1}{N}\text{Tr}[\mG_{H}(z)] \xrightarrow{N \rightarrow \infty} \mathcal{S}(z) = \int \frac{\rho(\lambda)}{z-\lambda}du
\end{equation}
is known as the Stieltjes transform of the eigenvalue density $\rho$. The functional inverse of the Siteltjes transform, is denoted the black transform $\mathcal{B}(\mathcal{S}(z)) = z$. The $\mathcal{R}$ transform is thence defined as 
\begin{equation}
\mathcal{R}(w) = \mathcal{B}(w) - \frac{1}{w}
\end{equation}
The following definition formally defines a Wigner matrix:
\begin{definition}{}
Let $\{Y_{i}\}$ and $\{Z_{ij}\}_{1\leq i\leq j}$ be two real-valued families of zero mean, i.i.d. random variables, Furthermore suppose that $\mathbb{E}|Z_{i,j}|^{2}=1$ and for each $k \in \mathbb{N}$
\begin{equation}
	\max(E|Z_{i,j}|^{k},E|Y_{1}|^{k})< \infty
\end{equation}
Consider an $n \times n$ symmetric matrix $\mM_{n}$, whose entries are given by
\begin{equation}
	\begin{cases}
		\mM_{n}(i,i) = Y_{i}\\
		\mM_{n}(i,j) = Z_{ij} = \mM_{n}(j,i) \\
	\end{cases} 
\end{equation}
The Matrix $\mM_{n}$ is known as a real symmetric Wigner matrix.
\end{definition}
\begin{theorem}
Let $\{\mM_{n}\}^{\infty}_{n=1}$ be a sequence of Wigner matrices, and for each $n$ denote $\mX_{n} = \mM_{n}/\sqrt{n}$. Then $\rho(\lambda)$, converges weakly, almost surely to the semicircle  distribution,
\begin{equation}
	d\mu(\lambda) = \rho(\lambda) d\lambda= \frac{1}{2\pi}\sqrt{4-\lambda^{2}}\mathbf{1}_{|\lambda|\leq 2} \, d\lambda.
\end{equation}
\end{theorem}
Crucially for our calculations, it is known that the $\mathcal{R}$ transform of the Wigner matrix $\mW$ where the variance is $\sigma^{2}$ instead of $1$ is given by:
\begin{equation}
\mathcal{R}_{W}(z) = \sigma^{2}z.
\end{equation}
\paragraph{Free random matrices:}
The property of freeness for non commutative random matrices can be considered analogously to the moment factorisation property of independent random variables. Let us denote the normalized trace operator, which is equal to the first moment of the spectral density
\begin{equation}
\psi(H) = \frac{1}{N}\text{Tr}\mH = \frac{1}{N}\sum_{i=1}^{N}\lambda_{i} = \int_{\lambda \in \mathcal{D}} d\mu(\lambda)\lambda
\end{equation}
We say matrices $\mA $ and $\mB$ for which $\psi(\mA) = \psi(\mB) = 0$ (We can always consider the transform $\mA - \psi(\mA)\mI$) are free if they satisfy for any integers $n_{1}..n_{k}$ with $k \in \mathbb{N}^{+}$
\begin{equation}
\psi(\mA^{n_{1}}\mB^{n_{2}}\mA^{n_{3}}\mB^{n_{4}}) = \psi(\mA^{n_{1}})\psi(\mB^{n_{2}})\psi(\mA^{n_{3}})\psi(\mA^{n_{4}}).
\end{equation}

\subsection{Proof of the Main Lemma}

Recall that we wish to derive the values of the extremal eigenvalues of the matrix sum $\mM = \mA+\meps(\vw)/\sqrt{P}$, where $\meps(\vw)$ has a limiting spectral density given by the semicircle law.
In this section we show by the definition of the Stieltjes transform (which has a one to one correspondence with the spectral density) that a finite rank perturbation of the Stieltjes transform (corresponding to the eigenvalues of the matrix $\mA$) can be dealt with perturbation theory.

The Stijeles transform of the matrix $\meps(\vw)$ with corresponding semicircle eigenvalue distribution can be written as \citep{tao2012topics}
\begin{equation}
\mathcal{S}_{\meps}(z) = \frac{z \pm \sqrt{z^{2}-4\sigma_{\epsilon}^{2}}}{2\sigma_{\epsilon}^{2}}.
\end{equation}
From the definition of the black transform, we hence have
\begin{equation}
\begin{aligned}
	\label{eq:usingtheblack}
	& z = \frac{\mathcal{B}_{\meps}(z) \pm \sqrt{\mathcal{B}^{2}_{\meps}(z)-4\sigma_{\epsilon}^{2}}}{2\sigma_{\epsilon}^{2}}\\
	& \mathcal{B}_{\meps}(z) = \frac{1}{z} + \sigma_{\epsilon}^{2}z \\
	&\mathcal{R}_{\meps}(z) =  \sigma_{\epsilon}^{2}z.\\
\end{aligned}
\end{equation}
Computing the $\mathcal{R}$ transform of the rank 1 matrix $\mA$ (which has a non-trivial eigenvalue $\lambda_{1} > \sigma_{\epsilon}$) using the Stieltjes transform \citep{bun2017cleaning}, we find the effect on the spectrum is given by:
\begin{equation}
\mathcal{S}_{\mA}(u) = \frac{1}{N}\frac{1}{u-\lambda_{1}}+\bigg(1-\frac{1}{N}\bigg)\frac{1}{u} = \frac{1}{u}\bigg[1 + \frac{1}{N}\frac{\lambda_{1}}{1-u^{-1}\lambda_{1}}\bigg] 
\end{equation}
We can use perturbation theory similar to in Equation (\ref{eq:usingtheblack}) to find the black and $\mathcal{R}$ transform which to leading order gives
\begin{equation}
\begin{aligned}
	& \mathcal{B}_{\mA}(\omega) = \frac{1}{w} + \frac{\lambda_{1}}{N(1-\omega \lambda_{1})} + \mathcal{O}(N^{-2}) \\
	&  \mathcal{R}_{\mA}(\omega) = \frac{\lambda_{1}}{N(1-\omega \lambda_{1})} + \mathcal{O}(N^{-2})  \\
\end{aligned}
\end{equation}
Setting $\omega = \mathcal{S}_{\mM}(z)$ so
\begin{equation}
z = \mathcal{B}_{\mA}(\mathcal{S}_{\mM}(z)) + \frac{\lambda_{1}}{N(1-\lambda_{1} \mathcal{S}_{\mM}(z))} + \mathcal{O}(N^{-2})
\end{equation}
using the ansatz of $\mathcal{S}_{\mM}(z) = \mathcal{S}_{0}(z) + \frac{\mathcal{S}_{1}(z)}{N} + \mathcal{O}(N^{-2})$ we find that $\mathcal{S}_{0}(z) = \mathcal{S}_{\epsilon(w)}(z)$ and using that $\mathcal{B}'_{\meps}(\mathcal{S}_{\meps})(z)) = \frac{1}{\mathcal{S}_{\meps}(z)}$, we conclude that
\begin{equation}
\mathcal{S}_{1}(z) = -\frac{\lambda_{1} \mathcal{S'}_{\epsilon(w)}(z)}{1-\mathcal{S}_{\epsilon(w)}(z)\lambda_{1}}
\end{equation}
and hence
\begin{equation}
\mathcal{S}_{\mM}(z) \approx \mathcal{S}_{\epsilon(w)}(z) - \frac{1}{N}\frac{\lambda_{1} \mathcal{S'}_{\epsilon(w)}(z)}{1-\mathcal{S}_{\epsilon(w)}(z)\lambda_{1}}
\end{equation}
In the large $N$ limit the correction only survives if $\mathcal{S}_{\epsilon(w)}(z) = 1/\lambda_{1}$
\begin{equation}
\begin{aligned}
	& \mathcal{S}_{\epsilon(w)}(z) = \frac{1}{\lambda_{1}} \\
	& \frac{2\sigma_{\epsilon}^{2}}{\lambda_{1}} = z \pm \sqrt{z^{2}-4\sigma_{\epsilon}^{2}} \\
	& \therefore z = \lambda_{1} + \frac{\sigma_{\epsilon}^{2}}{\lambda_{1}}  \\
\end{aligned}
\end{equation}
The same proof also holds for $\lambda_{P} < -2\sigma_{\epsilon}$ and hence the second part of the Lemma also follows.
\subsection{Overlap between Eigenvectors of the Batch and Empirical Hessian}
\begin{theorem}
The squared overlap $|\vphi_{i}'\vphi_{i}|^{2}$ between an eigenvector of the batch Hessian $\vphi_{i}'$ and that of the empirical Hessian $\vphi_{i}$ is given by
\begin{equation}
	|\vphi_{i}'\vphi_{i}|^{2} =  \left\{\begin{array}{lr}
		1 - \frac{P}{\mathfrak{b}}\frac{\sigma_{\epsilon}^{2}}{\lambda_{1}^{2}}, & \text{if } \lambda_{1} > \sqrt{\frac{P}{\mathfrak{b}}}\sigma_{\epsilon}\\
		0, & \text{otherwise } \\
	\end{array}\right\}
\end{equation}
\end{theorem}
For this theorem we utilise the following Lemma
\begin{lemma}
\label{lemma:eigenvectors}
Denote by $\phi_{1}'$ as the eigenvector associated with the largest eigenvalue of the matrix sum $\mM = \mA+\meps(\vw)/\sqrt{P}$, where $\mA \in \mathbb{R}^{P\times P}$ is a matrix of finite rank $r$ with largest eigenvalue $\lambda_{1}$ and $\meps(\vw) \in \mathbb{R}^{P\times P}$ with limiting spectral density $p(\lambda)$ satisfying the semicircle law $p(\lambda) = \frac{\sqrt{4\sigma_{\epsilon}^{2}-\lambda^{2}}}{2\pi\sigma_{\epsilon}^{2}}$. Then we have
\begin{equation}
	|\vphi_{i}'\vphi_{i}|^{2} = \left\{\begin{array}{lr}
		1 - \frac{\sigma_{\epsilon}^{2}}{\lambda_{1}^{2}}, & \text{if } \lambda_{1} > 2\sigma_{\epsilon}\\
		0, & \text{otherwise } \\
	\end{array}\right\} 
\end{equation}
\end{lemma}
The result and proof of this Lemma can be found in \citet{benaych2011eigenvalues}. Then the proof of the main theorem proceeds identically to that of Theorem~\ref{theorem:mainresult}.
\section{Non unit variance Marchenko-Pastur Stieltjes Transform}
\label{sec:nonunitmp}
We now derive the Stieltjes transform of the generalised non-unit variance Marchenko-Pastur density. This derivation closely follows \citep{feier2012methods}, but generalises the result. Note that \citet{feier2012methods} use a different convention for the Stieltjes transform
\begin{equation}
\mathcal{S}_{P}(z) = \int_{\mathbb{R}}\frac{1}{x-z} \rho (x)dx = \frac{1}{P}\Tr (\mM_{n}/\sqrt{P}-z\mI)^{-1}
\end{equation}
We consider a series of matrices 
\begin{equation}
\mX_{N} = \bigg(r^{s}_{i}/\sqrt{P}\bigg)_{1\leq i \leq P, \thinspace 1\leq s \leq N}
\end{equation}
where the entries $r^{s}_{i}$ are $0$ mean and variance $\sigma^{2}$  The Wishart matrix $\mW_{P} = \mX_{P}\mX_{P}^{T}$, where $(\mX_{P}\mX_{P})_{i,j} = \frac{1}{N}\sum_{s=1}^{N}r^{s}_{i}r^{s}_{j}$.. Clearly, $\mW_{P}$ can be written as the sum of rank-$1$ contributions $\mW_{P}^{s} = (r^{s}_{i}r^{s}_{j})_{1\leq i,j \leq P}$. Now as each element is of mean $0$ and variance $\sigma^{2}$, the expectation of the sum of the elements squared is given by $P^{2}\sigma^{4}/N^{2} = Tr([\mW_{P}^{s}]^{2}) = \lambda^{2}$ and hence the only eigenvalue (the contribution is rank $1$) is given by $\lambda = \frac{P}{T}\sigma^{2} = \beta \sigma^{2}$. For large $P$, by the weak law of large numbers, this is also true for a single realisation of $\mW_{P}^{s}$.  By the strong law of large numbers the column vectors $\vr^{s}= [r_{1}^{s}...r_{P}^{s}]^{T}$ and $\vr^{s'}$ are almost surely orthogonal as $P \rightarrow \infty$ and hence the matrices $\mW_{P}^{s}$ are asymptotically free \citep{voiculescu1992free}

The Stieltjes transform $\mathcal{S}(z)$ of $\mW_{P}^{s}$ 
\begin{equation}
\frac{1}{P}\Tr(\mW_{P}^{s}-z\mI)^{-1} = -\frac{1}{P}\sum_{k=0}^{\infty}\frac{\Tr(\mW_{P}^{s})^{k}}{z^{k+1}} = -\frac{1}{P}\bigg(\frac{P-1}{z}+\frac{1}{z-\beta\sigma^{2}}\bigg)
\end{equation}
Solving the quadratic for $z$, completing the square, dropping low order terms in $P$ and noting by the definition of the Stieltjes transform that for large $|z| \sim -\frac{1}{z}$
\begin{equation}
\begin{aligned}
	&z = \frac{P(s\beta\sigma^{2}-1)\pm \sqrt{P^{2}(s\beta\sigma^{2}-1)^{2}+4Ps(P-1)\beta\sigma^{2}}}{2Ps} =\frac{P(s\beta\sigma^{2}-1)\pm \sqrt{P^{2}(s\beta\sigma^{2}+1)^{2}-4Ps\beta\sigma^{2}}}{2Ps}\\
	&z \approx \frac{P(s\beta\sigma^{2}-1)- P(s\beta\sigma^{2}+1)-\frac{2Ps\beta\sigma^{2}}{s\beta\sigma^{2}+1}}{2Ps} = -\frac{1}{s}+\frac{\beta\sigma^{2}}{P(s\beta\sigma^{2}+1)}\\
\end{aligned}
\end{equation}
Hence as the $\mW_{P}$ is the free convolution \citep{voiculescu1992free} of the random matrices $\mW_{P}^{s}$ we simply multiply the $\mathcal{R}$ transform of each matrix by $N$ and as $\beta = P/N$ so:
\begin{equation}
\begin{aligned}
	& \mathcal{R}_{\mW_{P}}(s) = N \times \bigg(z-\frac{1}{s}\bigg) = \frac{N \beta\sigma^{2}}{P(s\beta\sigma^{2}+1)} = \frac{\sigma^{2}}{(s\beta\sigma^{2}+1)} \\
	& \mathcal{B}_{\mW_{P}}(s) = z = -\frac{1}{s}+ \frac{\sigma^{2}}{(s\beta\sigma^{2}+1)}
\end{aligned}
\end{equation}
and hence 
\begin{equation}
\mathcal{S}_{\mW_{P}} = \frac{-(z+\sigma^{2}(1-\beta))+\sqrt{(z+\sigma^{2}(1-\beta))^{2}-4\beta\sigma^{2}z}}{2\beta\sigma^{2}z}
\end{equation}
From here, using the definition of the Stieltjes transform and the relationship to the spectral density, $\text{Im}_{y \rightarrow 0}(\mathcal{S}_{\mW_{P}}(x+iy))/2\pi i$ we have the celebrated generalised Marchenko-Pastur result.
\begin{equation}
\rho(y) =  \frac{\sqrt{4\beta\sigma^{2}y-(y+\sigma^{2}(1-\beta))^{2}}}{2\beta\sigma^{2}y}
\end{equation}
Noting that the Stieltjes transform from \citet{feier2012methods} is reversed in the convention of the sign \citep{bun2017cleaning}, we take $z\rightarrow -z$. Now we apply the T transform, given by $\mathcal{T}(z)=z\mathcal{S}(z)-1$ and the result from \citet{benaych2011eigenvalues}, i.e $\lambda_{i}' = \mathcal{T}(\frac{1}{\lambda_{i}})$, where the dash denotes the eigenvalue corresponding to the batch Hessian (instead of the empirical which is fixed).
\begin{equation}
\mathcal{T}(z) = \frac{z-\sigma^{2}(1+\beta)-\sqrt{(z+\sigma^{2}(1-\beta))^{2}-4\beta\sigma^{2}z}}{2\beta\sigma^{2}}.
\end{equation}






\section{Differentiating Learning Rate Schedules}
\label{sec:differentiatingschedules}
As shown in Figure \ref{fig:traindiff}, in the instances where training occurs, the training curves between slightly larger or smaller learning rates is indistinguishable. Note that the curves for $B=128,64$ crosses at seemingly identical points, despite the learning rate being $25\% smaller$. Clearly the training profile is different between curves that train and those that don't (when a too large a learning rate has been used and training never commences).
\begin{figure}[h!]
	\centering
	\begin{subfigure}{0.47\linewidth}
		\includegraphics[trim={0cm 0cm 0cm 0cm},clip, width=1\textwidth]{sqrt_adam/Batch_optimal_linear_4_train_adam_VGG16_CIFAR100_simple.pdf}
		\caption{Training Error against Epoch}
		\label{subfig:vgg16trainadam4_app}
	\end{subfigure}
	\begin{subfigure}{0.47\linewidth}
		\includegraphics[trim={0cm 0cm 0cm 0cm},clip, width=1\textwidth]{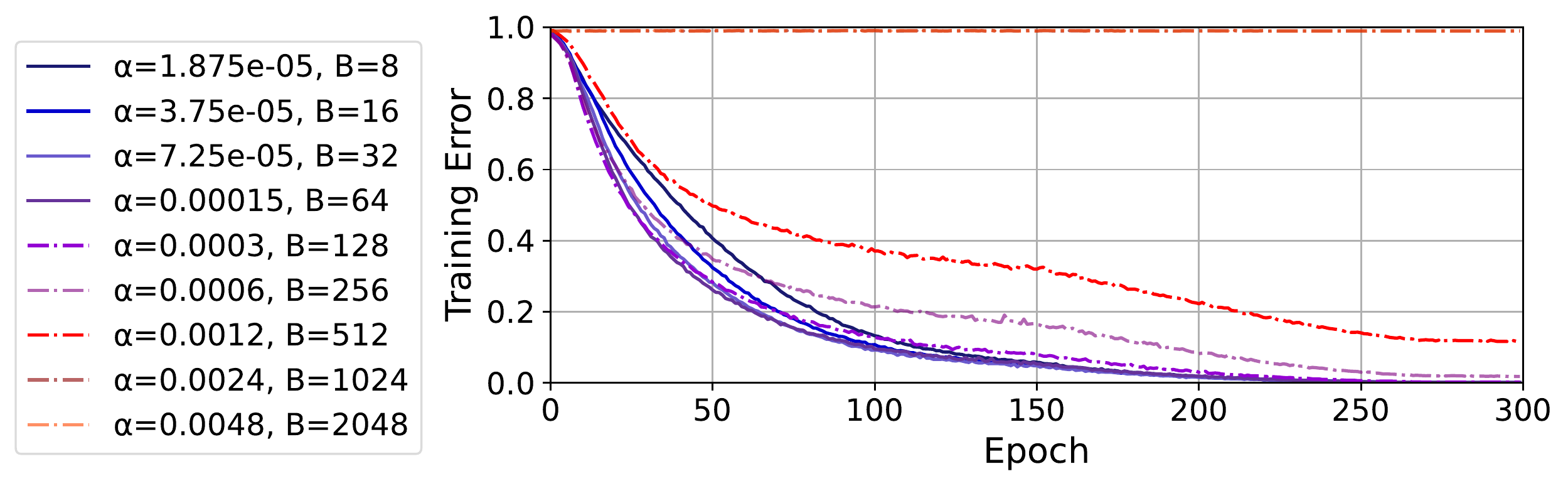}
		\caption{Training Error against Epoch}
		\label{subfig:vgg16train3adam_app}
	\end{subfigure}
	\caption{\textbf{Training Error trajectories do not well differentiate between learning rates if training occurs.} Training error of the VGG-$16$ architecture, without batch normalisation (BN) on CIFAR-$100$, with no weight decay $\gamma=0$ and initial learning rate $\alpha_{0}$}
	\label{fig:traindiff}
\end{figure}
However note that even for schedules with train indistinguishable training curves,the validation curve, shown in Figure \ref{fig:valdiff} can differ significantly. Quite specifically for $B=8,16,32$ there is an upwards shift in validation error of around $2\%$, for learning rates which have been decreased by $25\%$, despite indistinguishable training curves.
\begin{figure}[h!]
	\centering
	\begin{subfigure}{0.47\linewidth}
		\includegraphics[trim={0.0cm 0cm 0cm 0cm},clip, width=1\textwidth]{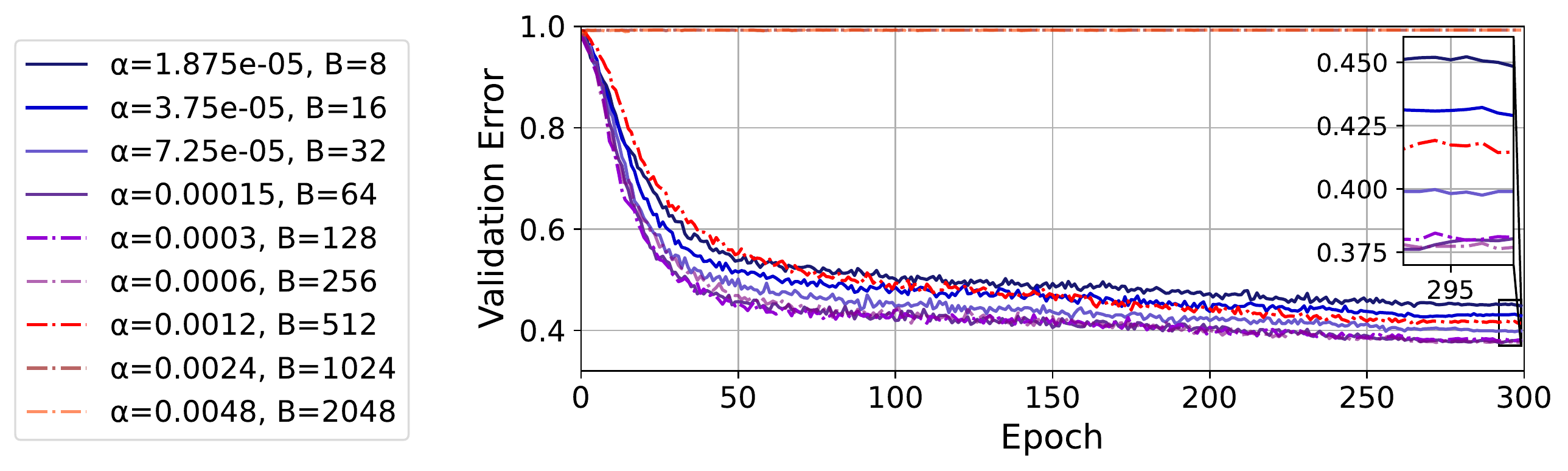}
		\caption{Validation Error against Epoch}
		\label{subfig:vgg16test3adam_app}
	\end{subfigure}
	\begin{subfigure}{0.47\linewidth}
		\includegraphics[trim={0.0cm 0cm 0cm 0cm},clip, width=1\textwidth]{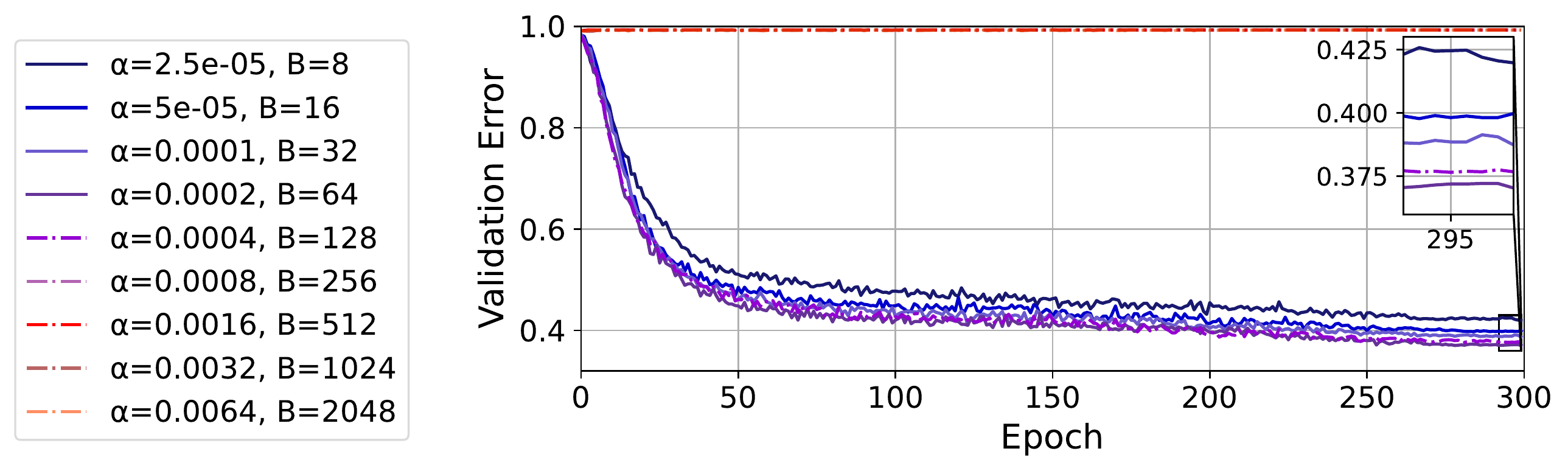}
		\caption{Validation Error against Epoch}
		\label{subfig:vgg16testadam4_app}
	\end{subfigure}
	\caption{\textbf{Validation Error trajectories distinguish between learning rates used.} Validation error of the VGG-$16$ architecture, without batch normalisation (BN) on CIFAR-$100$, with no weight decay $\gamma=0$ and initial learning rate $\alpha_{0}$.}
	\label{fig:valdiff}
\end{figure}
We expect schedules parameterised by a different learning rates to traverse the non-convex loss surface in a different way. Specifically larger initial learning rate schedules would be expected to escape sharper local minima earlier in training before the learning rate decay kicks in. Given that high capacity neural networks which are capable of easily memorising the data and hence there are many points of low training loss/error in the training risk, which may not be the result of following a similar trajectory in the loss surface, we consider the validation error plot (and final value) as more indicative to discriminate between effective scaling regimens.  
\subsection{Different Initailisations Give Similar Performance}
\label{sec:diffinit}
\begin{figure}
\begin{minipage}[t]{1\textwidth}
		\centering
		\begin{subfigure}[b]{0.32\textwidth}
			\includegraphics[width=\textwidth]{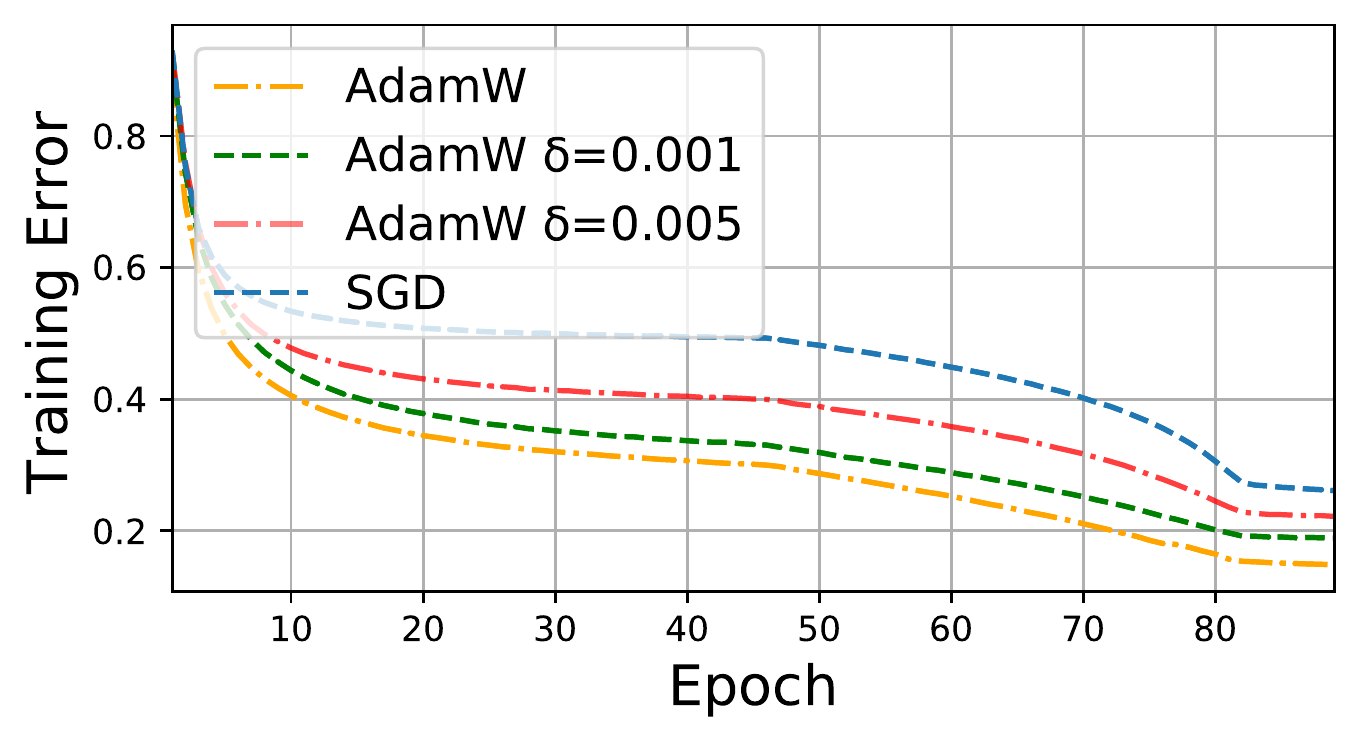}
			\caption{ResNet-$50$ Training Error}
			\label{subfig:r50train}
		\end{subfigure}
		\begin{subfigure}[b]{0.32\textwidth}
			\includegraphics[width=\textwidth]{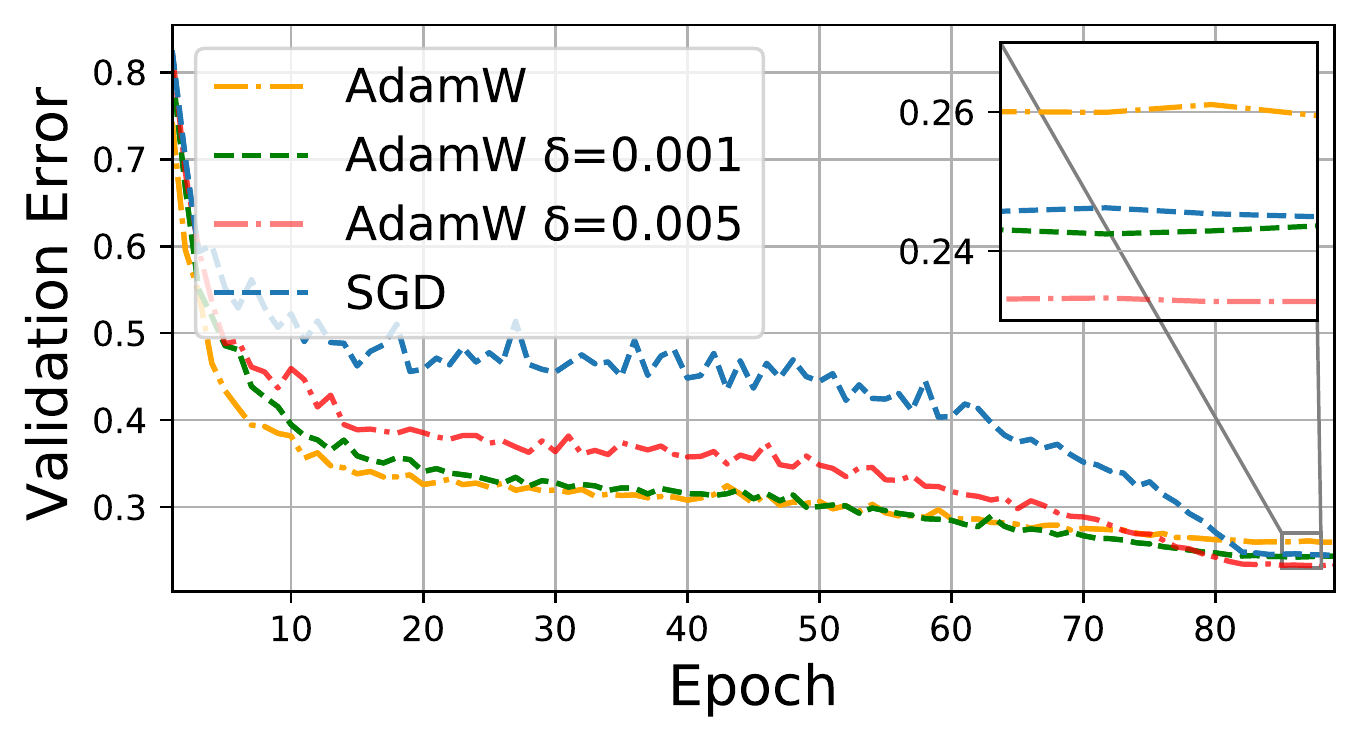}
			\caption{ResNet-$50$ Testing Error}
			\label{subfig:r50test}
		\end{subfigure}
		\begin{subfigure}[b]{0.31\textwidth}
		\begin{tiny}
			\begin{tabular}{@{}llll@{}}
				\toprule
				\textbf{Data/Model} & \textbf{SGD} & \textbf{Adam-D} & \textbf{Adam} \\ \midrule
				\textbf{C100/VGG16} & 65.3 $\pm$ 0.6 & 65.5 $\pm$ 0.7 & 61.9 $\pm$ 0.4 \\
				\midrule
				
				\textbf{ImgNet/Res50} & 75.7 $\pm$ 0.1 & 76.6 $\pm$ 0.1 & - \\ \bottomrule
			\end{tabular}
        \end{tiny}
        \vspace{15pt}
        \caption{Statistical Significance}
        \label{tab:seeds}
		\end{subfigure}
\end{minipage}
\label{fig:res50adamw}
				
\caption{(a-b) The influence of $\delta$ on the generalisation gap. Train/Val curves for ResNet-$50$ on ImageNet. The generalisation gap is completely closed with an appropriate choice of $\delta$. (c) Comparison of test accuracy across CIFAR 100 (5 seeds) and ImageNet (3 seeds). \textbf{Adam-D} denotes Adam with increased damping ($\delta=5e^{-3}$ for CIFAR-100, $\delta=1e^{-4}$ for ImageNet).}
\vspace{-15pt}
\end{figure}
We show here in \ref{tab:seeds} that for different seeds using SGD, Adam and Adam-$\delta$ where we simply tune the coefficient of the numerical stability coefficient to a larger value $\epsilon \rightarrow 10^{-4}$ instead of $10^{-8}$, gives very consistent performance across a set of datasets and networks.

\section{Lanczos algorithm}
\label{sec:lanczos}
In order to empirically analyse properties of modern neural network spectra with tens of millions of parameters $N = \mathcal{O}(10^{7})$, we use the Lanczos algorithm \citep{meurant2006lanczos}, provided for deep learning by \citet{granziol2019mlrg}. It requires Hessian vector products, for which we use the \emph{Pearlmutter trick} \citep{pearlmutter1994fast} with computational cost $\mathcal{O}(NP)$, where $N$ is the dataset size and $P$ is the number of parameters. Hence for $m$ steps the total computational complexity including re-orthogonalisation is $\mathcal{O}(NPm)$ and memory cost of $\mathcal{O}(Pm)$. In order to obtain accurate spectral density estimates we re-orthogonalise at every step \citep{meurant2006lanczos}. We exploit the relationship between the Lanczos method and Gaussian quadrature, using random vectors to allow us to learn a discrete approximation of the spectral density.  A quadrature rule is a relation of the form,
\begin{equation}
\label{eq:quadraturerule}
\int_{a}^{b}f(\lambda)d\mu(\lambda) = \sum_{j=1}^{M}\rho_{j}f(t_{j})+R[f]
\end{equation}
for a function $f$, such that its Riemann-Stieltjes integral and all the moments exist on the measure $d\mu(\lambda)$, on the interval $[a,b]$ and where $R[f]$ denotes the unknown remainder. The nodes $t_{j}$ of the Gauss quadrature rule are given by the Ritz values and the weights (or mass) $\rho_{j}$ by the squares of the first elements of the normalized eigenvectors of the Lanczos tri-diagonal matrix \citep{golub1994matrices}. The main properties of the Lanczos algorithm are summarized in the theorems \ref{theorem:lanczoseigenvalues},\ref{theorem:lanczosspectrum}
\begin{theorem}
\label{theorem:lanczoseigenvalues}
Let $H^{N\times N}$ be a symmetric matrix with eigenvalues $\lambda_{1}\geq .. \geq \lambda_{n}$ and corresponding orthonormal eigenvectors $z_{1},..z_{n}$. If $\theta_{1}\geq .. \geq \theta_{m}$ are the eigenvalues of the matrix $T_{m}$ obtained after $m$ Lanczos steps and $q_{1},...q_{k}$ the corresponding Ritz eigenvectors then
\begin{equation}
	\begin{aligned}
		& \lambda_{1} \geq \theta_{1} \geq \lambda_{1} - \frac{(\lambda_{1}-\lambda_{n})\tan^{2}(\theta_{1})}{(c_{k-1}(1+2\rho_{1}))^{2}} \\
		& \lambda_{n} \leq \theta_{k} \leq \lambda_{m} + \frac{(\lambda_{1}-\lambda_{n})\tan^{2}(\theta_{1})}{(c_{k-1}(1+2\rho_{1}))^{2}} \\
	\end{aligned}	
\end{equation}
where $c_{k}$ is the Chebyshev polynomial of order $k$
\end{theorem}
Proof: see \citep{golub2012matrix}.
\begin{theorem}
\label{theorem:lanczosspectrum}
The eigenvalues of $T_{k}$ are the nodes $t_{j}$ of the Gauss quadrature rule, the weights $w_{j}$ are the squares of the first elements of the normalized eigenvectors of $T_{k}$
\end{theorem}
Proof: See \citep{golub1994matrices}. The first term on the RHS of Equation \ref{eq:quadraturerule} using Theorem \ref{theorem:lanczosspectrum} can be seen as a discrete approximation to the spectral density matching the first $m$ moments $v^{T}H^{m}v$ \citep{golub1994matrices,golub2012matrix}, where $v$ is the initial seed vector. Using the expectation of quadratic forms, for zero mean, unit variance random vectors, using the linearity of trace and expectation
\begin{equation}
\begin{aligned}
	\mathbb{E}_{v}\text{Tr}(v^{T}H^{m}v) & =  \text{Tr}\mathbb{E}_{v}(vv^{T}H^{m}) = \text{Tr}(H^{m})
	= \sum_{i=1}^{N}\lambda_{i} = N \int_{\lambda \in \mathcal{D}} \lambda d\mu(\lambda) \\
\end{aligned}
\end{equation}
The error between the expectation over the set of all zero mean, unit variance vectors $v$ and the Monte Carlo sum used in practice can be bounded \citep{hutchinson1990stochastic,roosta2015improved}. However in the high dimensional regime $N \rightarrow \infty$, we expect the squared overlap of each random vector with an eigenvector of $H$, $|v^{T}\phi_{i}|^{2} \approx \frac{1}{N} \forall i$, with high probability. This result can be seen by computing the moments of the overlap between Rademacher vectors, containing elements $P(v_{j} = \pm 1) = 0.5$. Further analytical results for Gaussian vectors have been obtained \citep{cai2013distributions}.
\section{Batch Normalisation Results}
\label{sec:batchnormresults}
Given that the vast majority of image classification are run in conjunction with normalisation methods such as batch normalisation \citep{ioffe2015batch} and previous literature observing that batch normalisation suppresses outliers \citep{ghorbani2019investigation} it is important to investigate whether the observations in terms of spectral structure and mini-batching effect are in any way invalidated with batch-normalisation. We hence present results on a variety of pre-activated residual networks. We show that the typical spectral density plots in the main text with well separated outliers, a large rank degeneracy and large increase in spectral width with mini-batching are visible also in batch normalised Resnets more commonly used in deep learning for both types of batch normalisation mode (explained in the next paragraph).
\begin{figure}[h!]
\centering
\begin{subfigure}{0.23\linewidth}
	\includegraphics[width=1\linewidth,trim={0 0 0 0},clip]{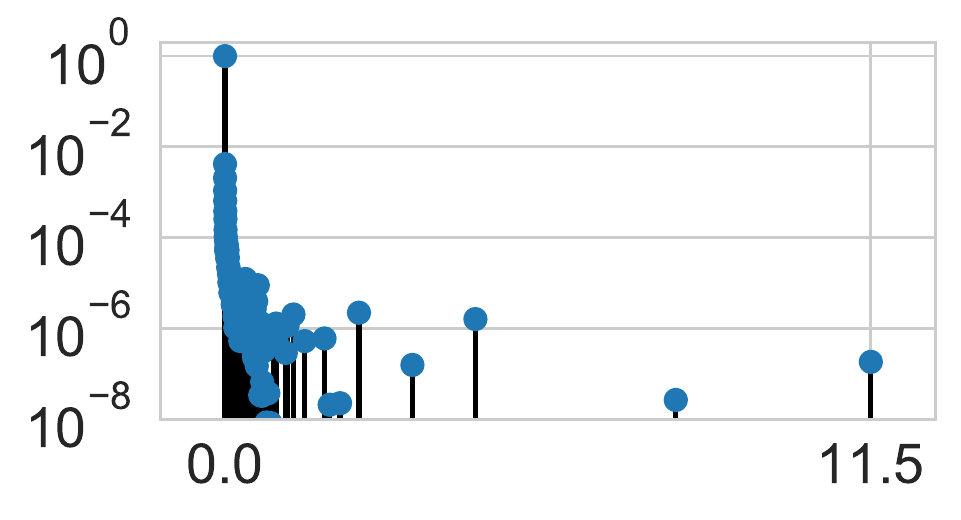}
	\caption{Epoch $0$}
	\label{subfig:c100p110ep0}
\end{subfigure}
\begin{subfigure}{0.23\linewidth}
	\includegraphics[width=1\linewidth,trim={0 0 0 0},clip]{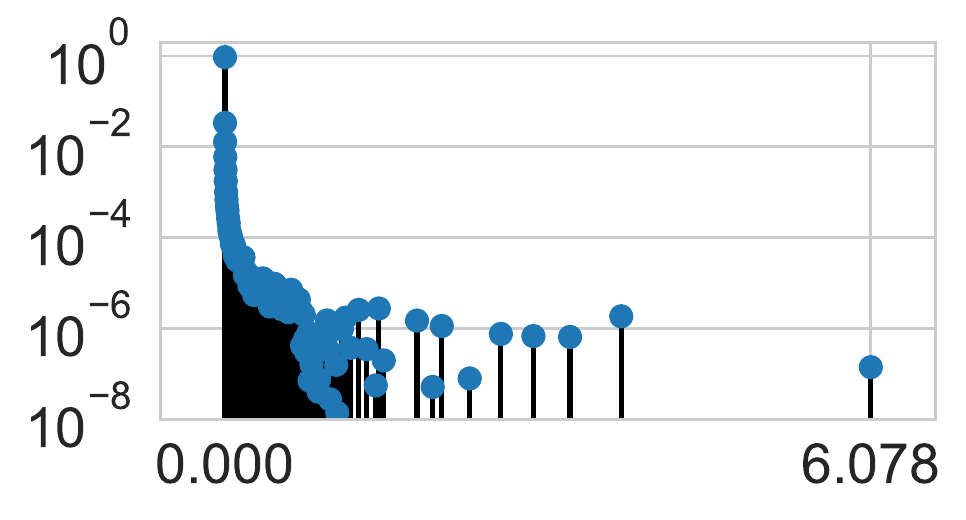}
	\caption{Epoch $25$}
	\label{subfig:c100p110ep25}
\end{subfigure}
\begin{subfigure}{0.23\linewidth}
	\includegraphics[width=1\linewidth,trim={0 0 0 0},clip]{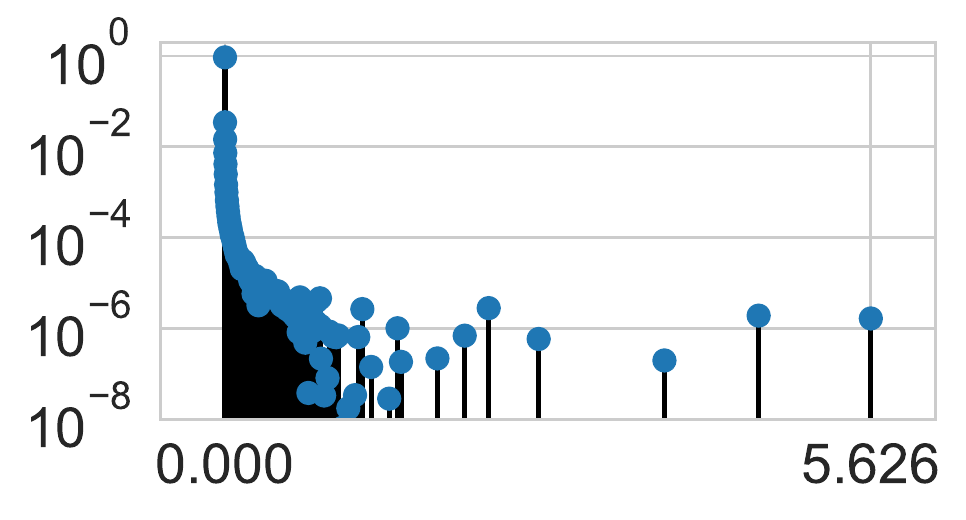}
	\caption{Epoch $50$}
	\label{subfig:c100p110ep50}
\end{subfigure}
\begin{subfigure}{0.23\linewidth}
	\includegraphics[width=1\linewidth,trim={0 0 0 0},clip]{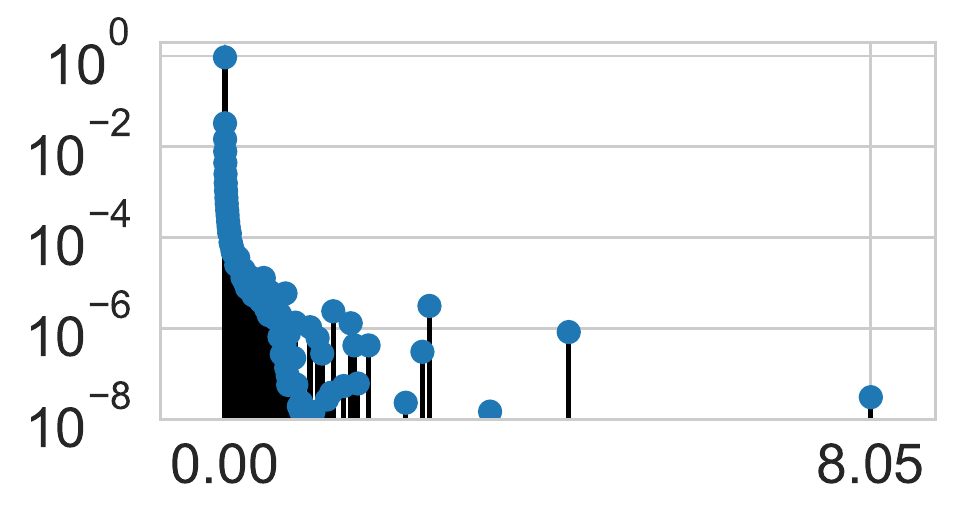}
	\caption{Epoch $75$}
	\label{subfig:c100p110ep75}
\end{subfigure}
\begin{subfigure}{0.30\linewidth}
	\includegraphics[width=1\linewidth,trim={0 0 0 0},clip]{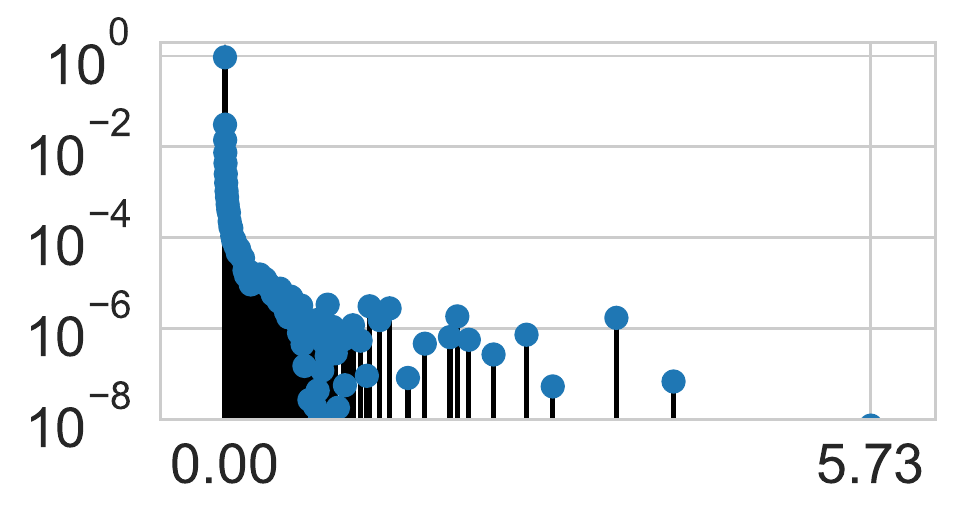}
	\caption{Epoch $100$}
	\label{subfig:c100p110ep100}
\end{subfigure}
\begin{subfigure}{0.30\linewidth}
	\includegraphics[width=1\linewidth,trim={0 0 0 0},clip]{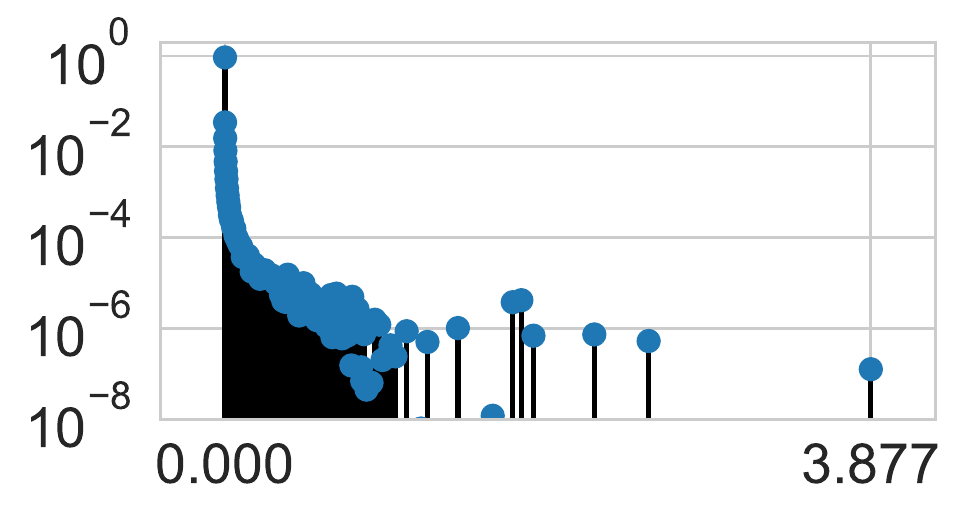}
	\caption{Epoch $125$}
	\label{subfig:c100p110ep125}
\end{subfigure}
\begin{subfigure}{0.30\linewidth}
	\includegraphics[width=1\linewidth,trim={0 0 0 0},clip]{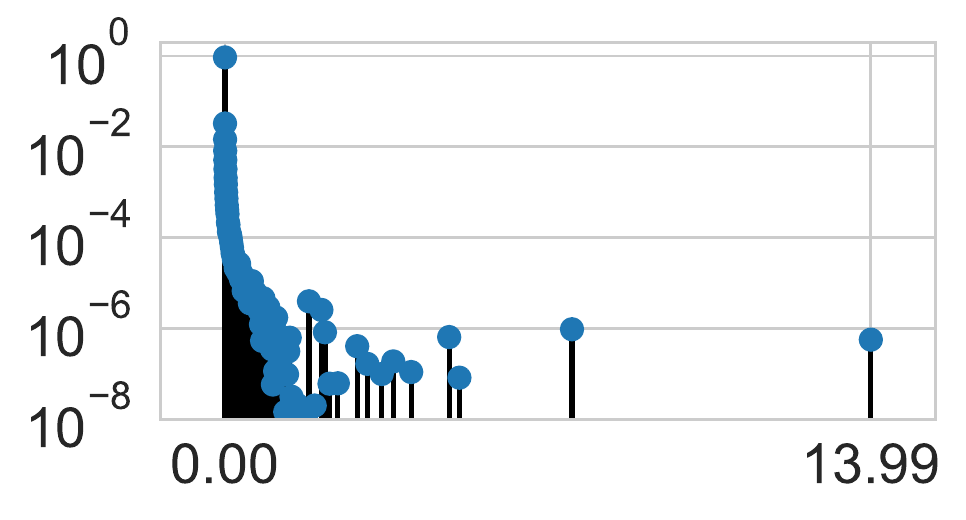}
	\caption{Epoch $150$}
	\label{subfig:c100p110ep150}
\end{subfigure}
\centering
\begin{subfigure}{0.30\linewidth}
	\includegraphics[width=1\linewidth,trim={0 0 0 0},clip]{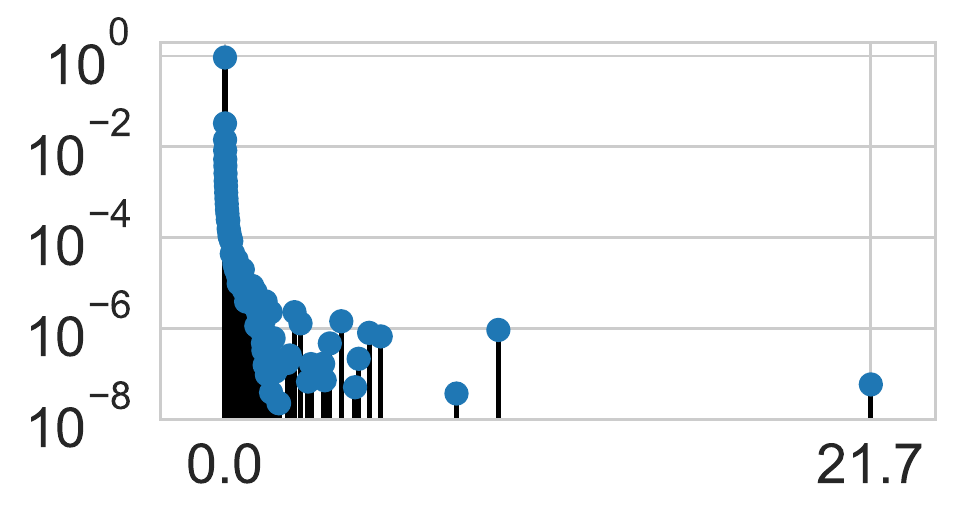}
	\caption{Epoch $175$}
	\label{subfig:c100p110ep175}
\end{subfigure}
\begin{subfigure}{0.30\linewidth}
	\includegraphics[width=1\linewidth,trim={0 0 0 0},clip]{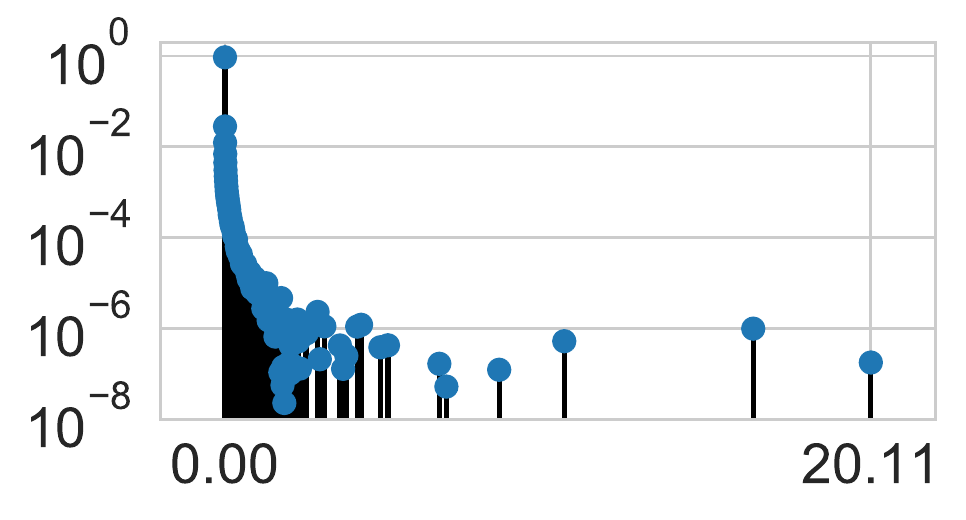}
	\caption{Epoch $200$}
	\label{subfig:c100p110ep200}
\end{subfigure}
\begin{subfigure}{0.30\linewidth}
	\includegraphics[width=1\linewidth,trim={0 0 0 0},clip]{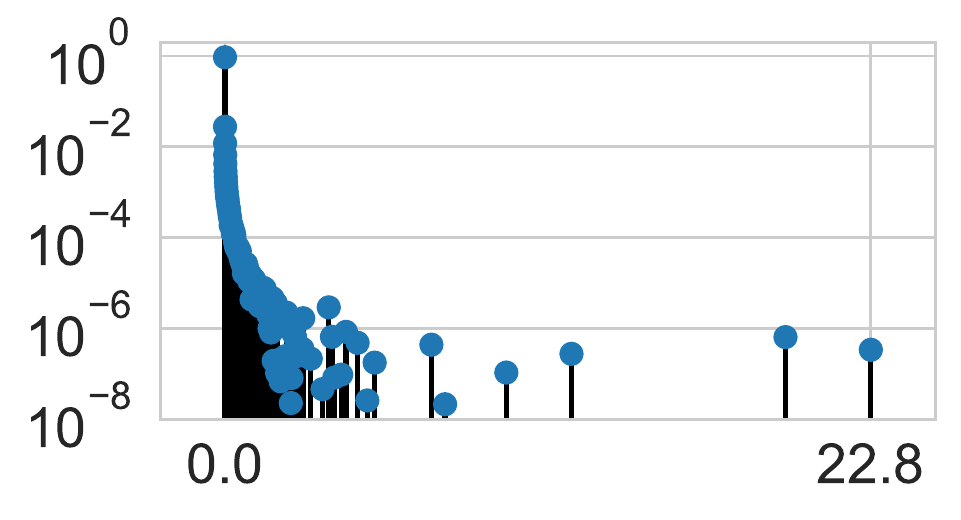}
	\caption{Epoch $225$}
	\label{subfig:c100p110ep225}
\end{subfigure}
\caption{Generalised Gauss-Newton matrix full empirical spectrum for the PreResNet-$110$ on the CIFAR-$100$ dataset, total training $225$ epochs, batch norm train mode}
\label{fig:p110c100ggn3}
\end{figure}
\paragraph{Technical point on batch norm:} Batch-normalisation, as alluded to in the main text function differently during training and during evaluation. When evaluating curvature, we thus have the option of choosing the setting of this functionality. We denote the same properties as during training as batch norm train mode and those during evaluation as batch norm evaluation mode.We look at the Generalised Gauss-Newton matrix and Hessian of the PreResNet-$110$ in batch norm evaluation and training mode.
\subsection{Generalised Gauss-Newton matrix - batch normalisation train mode}
To show similarly that the Generalised Gauss-Newton matrix experiences severe spectral broadening when mini-batching, we take the same points in weight space as in Figure \ref{fig:p110c100ggn3} but instead take stochastic samples of size $B=128$, although the results are stochastic, they are stochastic around a significantly broadened spectrum, with some samples shown in  Figure \ref{fig:p110c100ggnbatch}, for comparison. Where we see significant broadening.
\begin{figure}[h!]
\centering
\begin{subfigure}{0.30\linewidth}
	\includegraphics[width=1\linewidth,trim={0 0 0 0},clip]{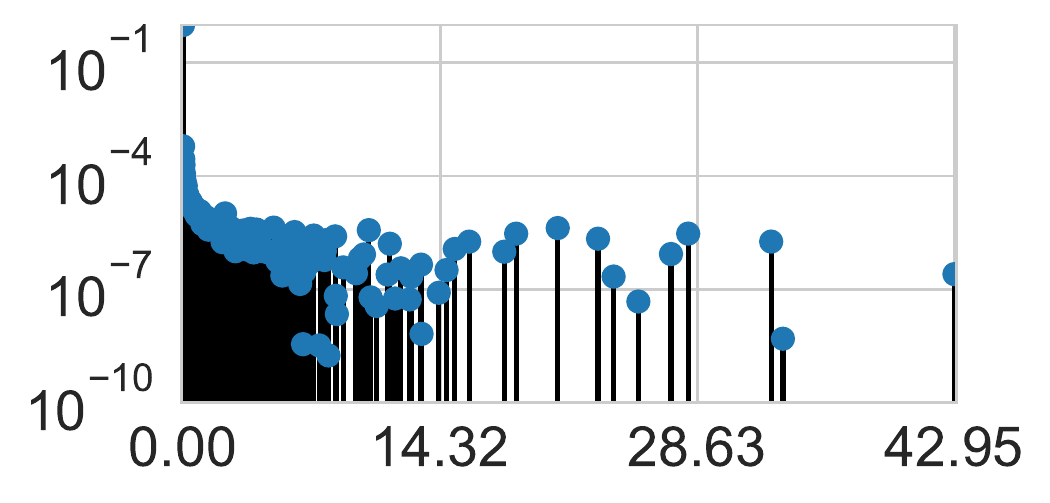}
	\caption{Epoch $175$, $B=128$}
	\label{subfig:c100p110ep175}
\end{subfigure}
\begin{subfigure}{0.30\linewidth}
	\includegraphics[width=1\linewidth,trim={0 0 0 0},clip]{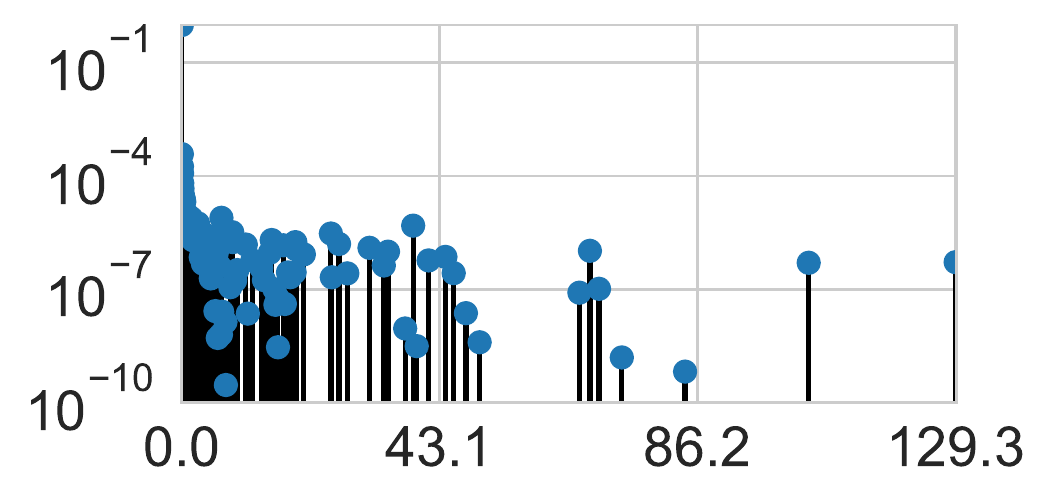}
	\caption{Epoch $200$, $B=128$}
	\label{subfig:c100p110ep200}
\end{subfigure}
\begin{subfigure}{0.30\linewidth}
	\includegraphics[width=1\linewidth,trim={0 0 0 0},clip]{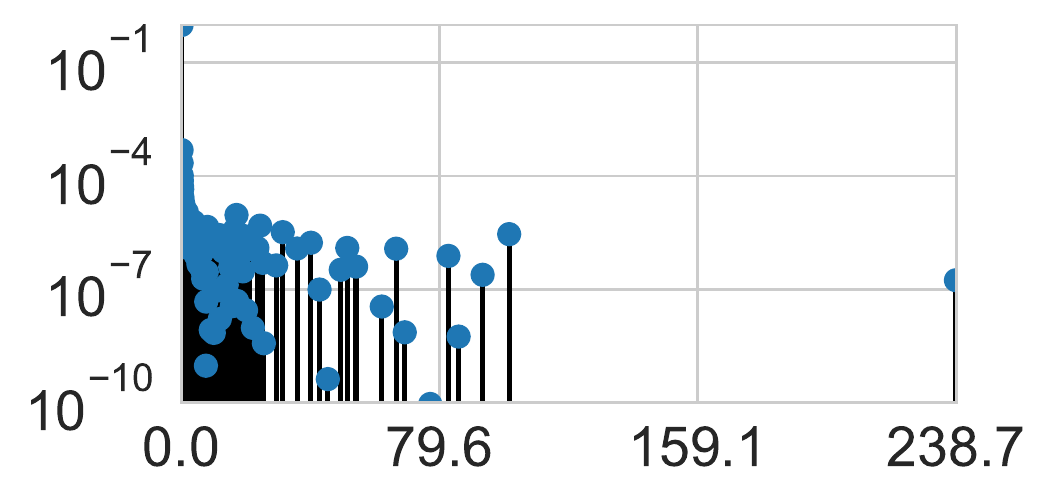}
	\caption{Epoch $225$, $B=128$}
	\label{subfig:c100p110ep225}
\end{subfigure}
\caption{Generalised Gauss-Newton matrix full empirical spectrum for the PreResNet-$110$ on the CIFAR-$100$ dataset, total training $225$ epochs, batch norm train mode, samples taken with a batch of $B=128$}
\label{fig:p110c100ggnbatch}
\end{figure}
\subsection{Generalised Gauss-Newton matrix - Evaluation Mode}
\begin{figure}[h!]
\centering
\begin{subfigure}{0.30\linewidth}
	\includegraphics[width=1\linewidth,trim={0 0 0 0},clip]{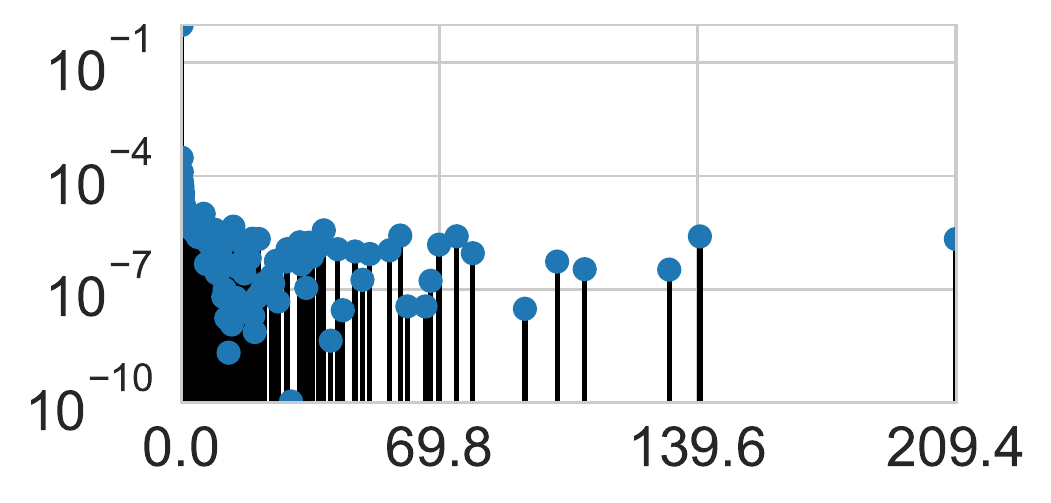}
	\caption{Epoch $175$, $B=128$}
	\label{subfig:c100p110ep175evalbatch}
\end{subfigure}
\begin{subfigure}{0.30\linewidth}
	\includegraphics[width=1\linewidth,trim={0 0 0 0},clip]{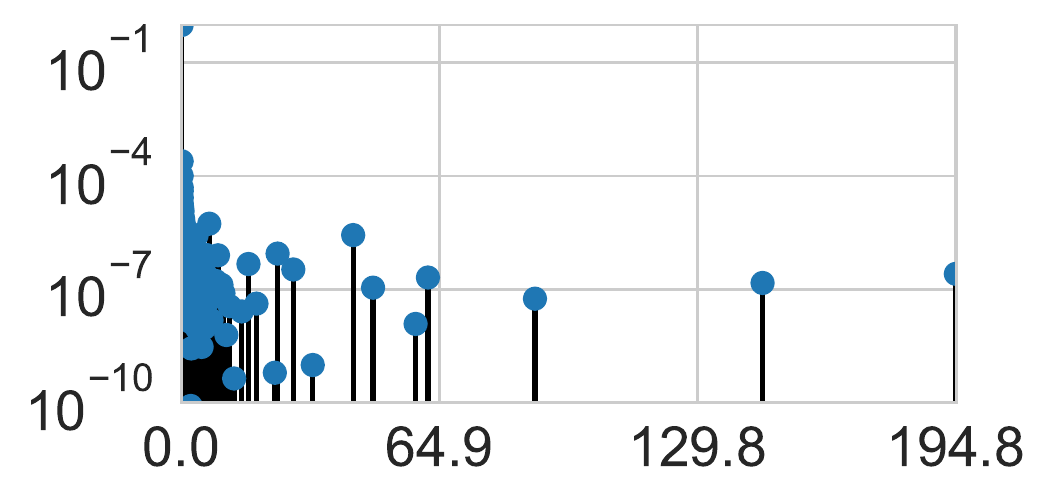}
	\caption{Epoch $200$, $B=128$}
	\label{subfig:c100p110ep200evalbatch}
\end{subfigure}
\begin{subfigure}{0.30\linewidth}
	\includegraphics[width=1\linewidth,trim={0 0 0 0},clip]{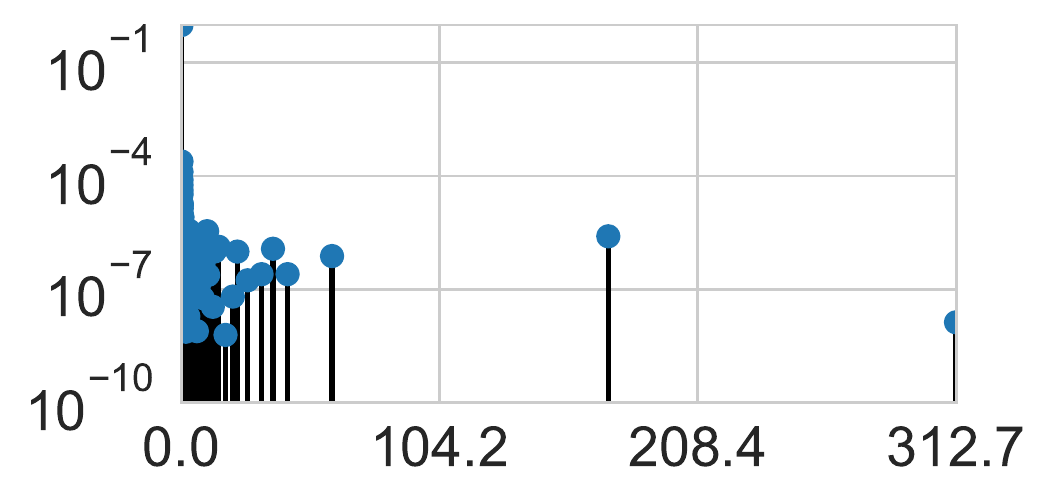}
	\caption{Epoch $225$, $B=128$}
	\label{subfig:c100p110ep225evalbatch}
\end{subfigure}
\caption{Generalised Gauss-Newton matrix full empirical spectrum for the PreResNet-$110$ on the CIFAR-$100$ dataset, total training $225$ epochs, batch norm eval mode, $B=128$ sub-sampled spectrum}
\label{fig:p110c100evalggn3batch}
\end{figure} 
Similarly to the previous section, we show in Figure \ref{fig:p110c100evalggn3batch} that even with batch normalisation in evaluation mode, the sub-sampling procedure induces extreme spectral broadening, compared to the same points in weight space with the full dataset, shown in Figure \ref{fig:p110c100evalggn}. Again although the results are stochastic, with large variance the trend is consistent. 
\begin{figure}[h!]
\centering
\begin{subfigure}{0.23\linewidth}
	\includegraphics[width=1\linewidth,trim={0 0 0 0},clip]{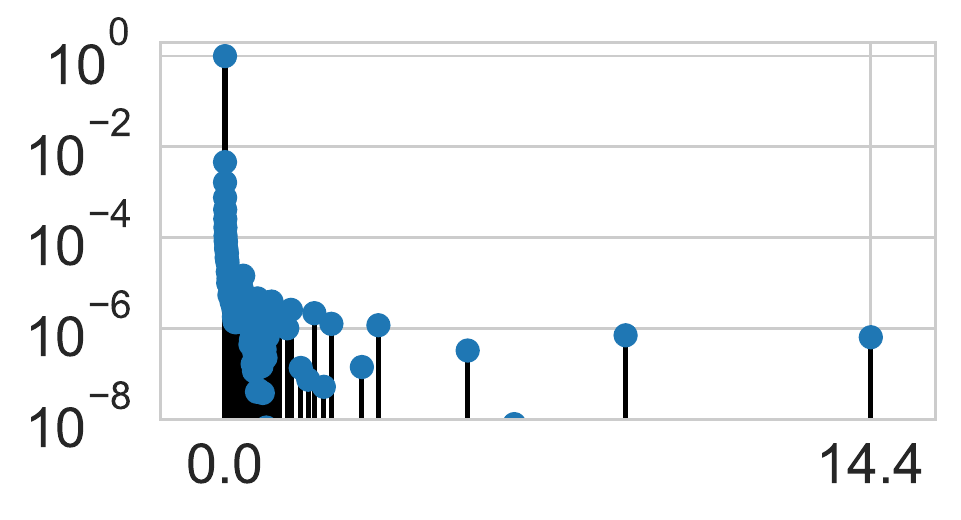}
	\caption{Epoch $0$}
	\label{subfig:c100p110ep0eval}
\end{subfigure}
\begin{subfigure}{0.23\linewidth}
	\includegraphics[width=1\linewidth,trim={0 0 0 0},clip]{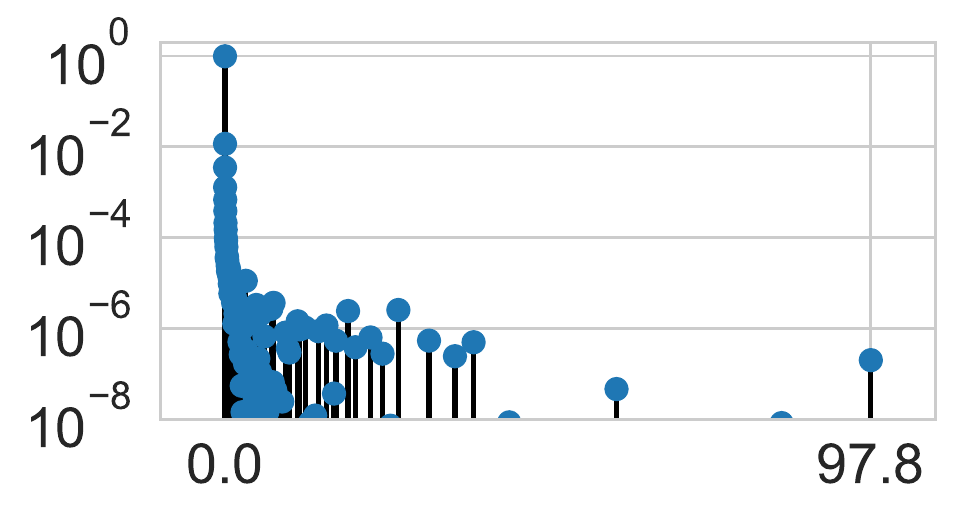}
	\caption{Epoch $25$}
	\label{subfig:c100p110ep25eval}
\end{subfigure}
\begin{subfigure}{0.23\linewidth}
	\includegraphics[width=1\linewidth,trim={0 0 0 0},clip]{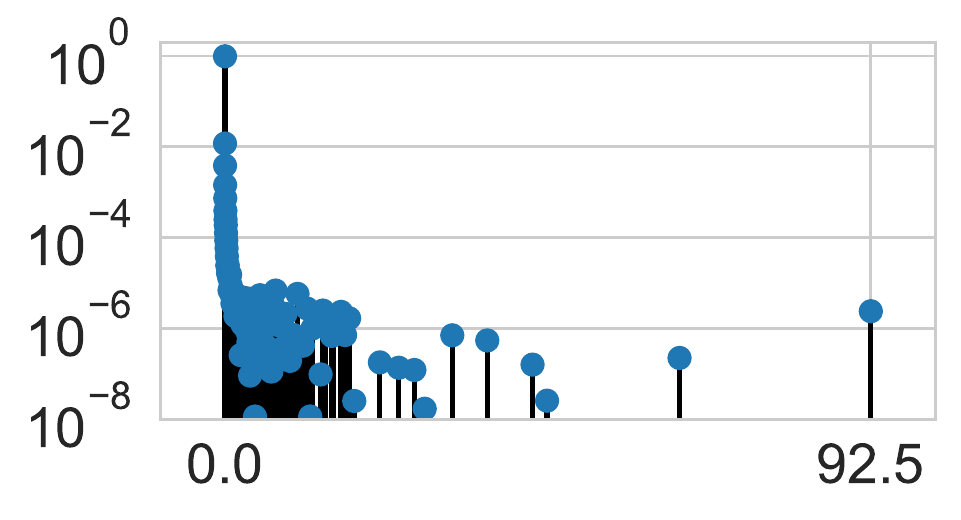}
	\caption{Epoch $50$}
	\label{subfig:c100p110ep50eval}
\end{subfigure}
\begin{subfigure}{0.23\linewidth}
	\includegraphics[width=1\linewidth,trim={0 0 0 0},clip]{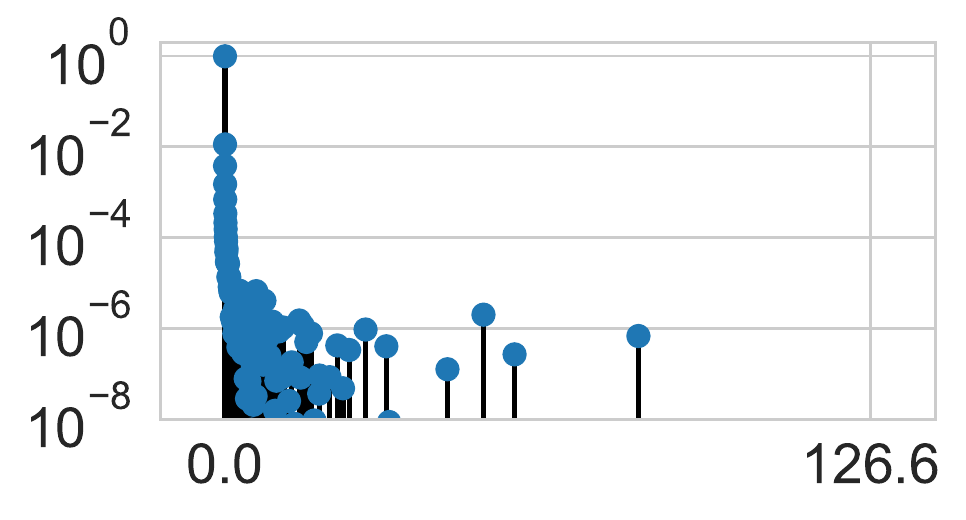}
	\caption{Epoch $75$}
	\label{subfig:c100p110ep75eval}
\end{subfigure}
\begin{subfigure}{0.23\linewidth}
	\includegraphics[width=1\linewidth,trim={0 0 0 0},clip]{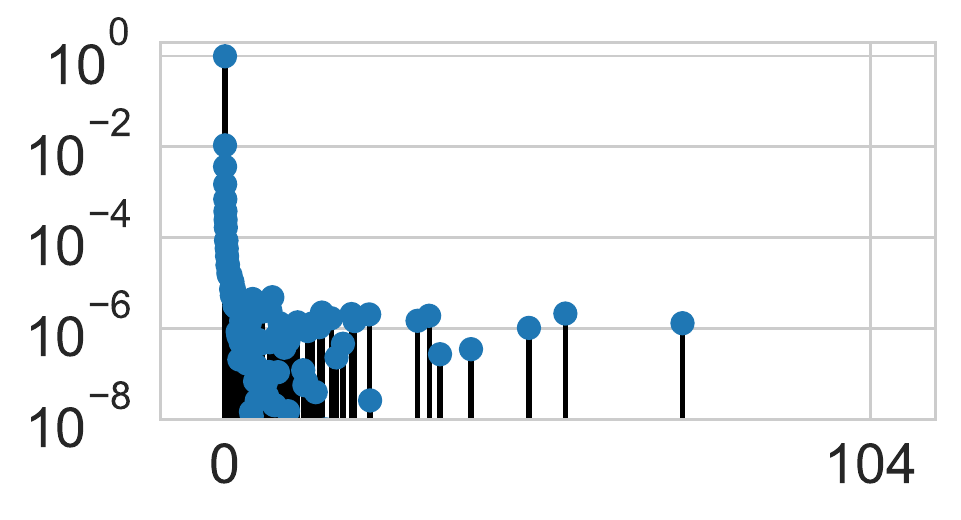}
	\caption{Epoch $100$}
	\label{subfig:c100p110ep0eval}
\end{subfigure}
\begin{subfigure}{0.23\linewidth}
	\includegraphics[width=1\linewidth,trim={0 0 0 0},clip]{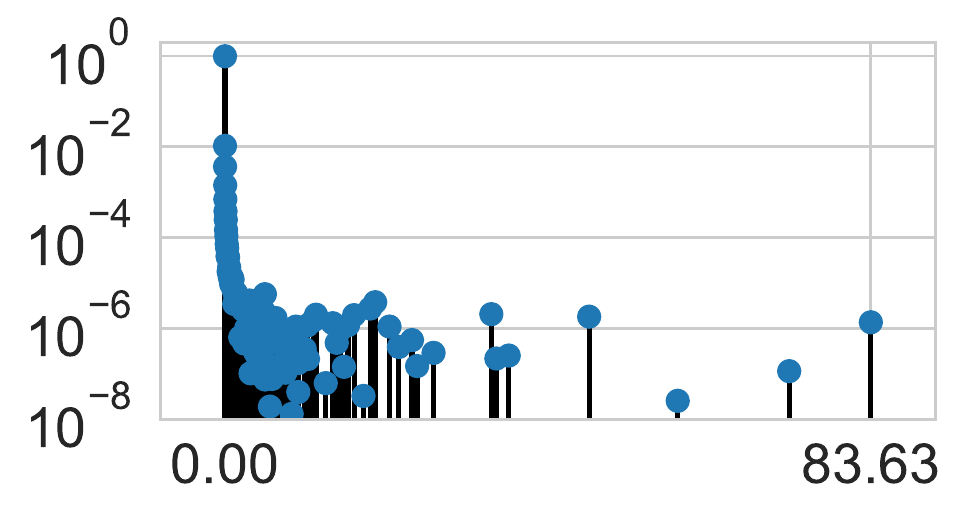}
	\caption{Epoch $125$}
	\label{subfig:c100p110ep25eval}
\end{subfigure}
\begin{subfigure}{0.23\linewidth}
	\includegraphics[width=1\linewidth,trim={0 0 0 0},clip]{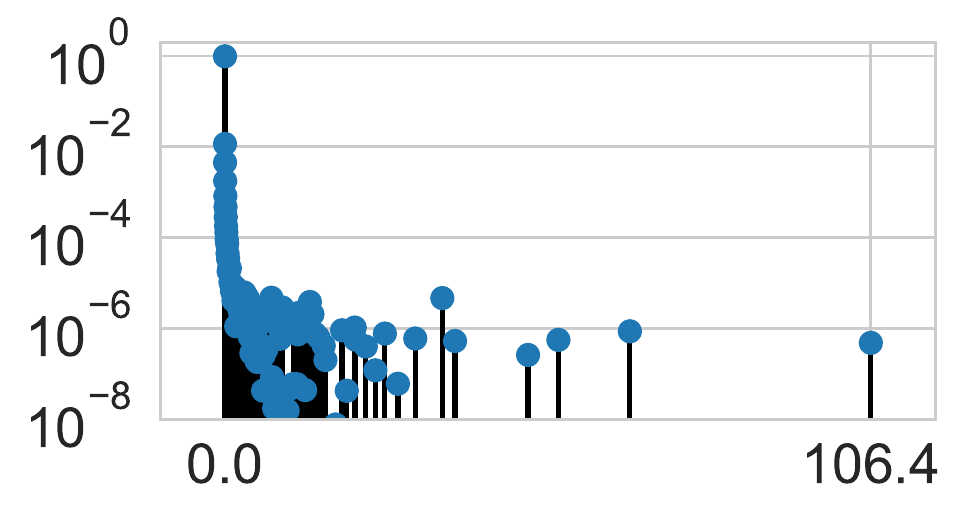}
	\caption{Epoch $150$}
	\label{subfig:c100p110ep50eval}
\end{subfigure}
\begin{subfigure}{0.23\linewidth}
	\includegraphics[width=1\linewidth,trim={0 0 0 0},clip]{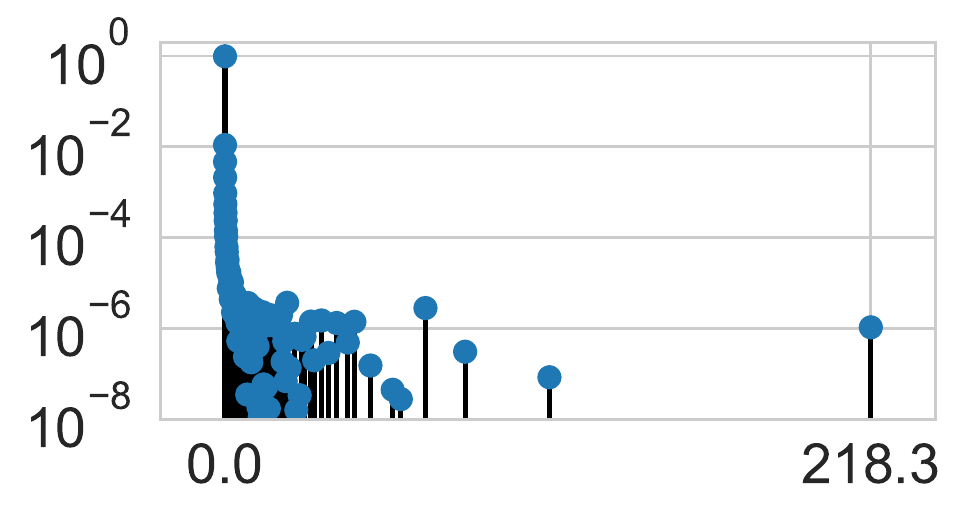}
	\caption{Epoch $175$}
	\label{subfig:c100p110ep75eval}
\end{subfigure}
\begin{subfigure}{0.30\linewidth}
	\includegraphics[width=1\linewidth,trim={0 0 0 0},clip]{P110/C100_PreResNet_GGN_eval_175.pdf}
	\caption{Epoch $175$}
	\label{subfig:c100p110ep175eval}
\end{subfigure}
\begin{subfigure}{0.30\linewidth}
	\includegraphics[width=1\linewidth,trim={0 0 0 0},clip]{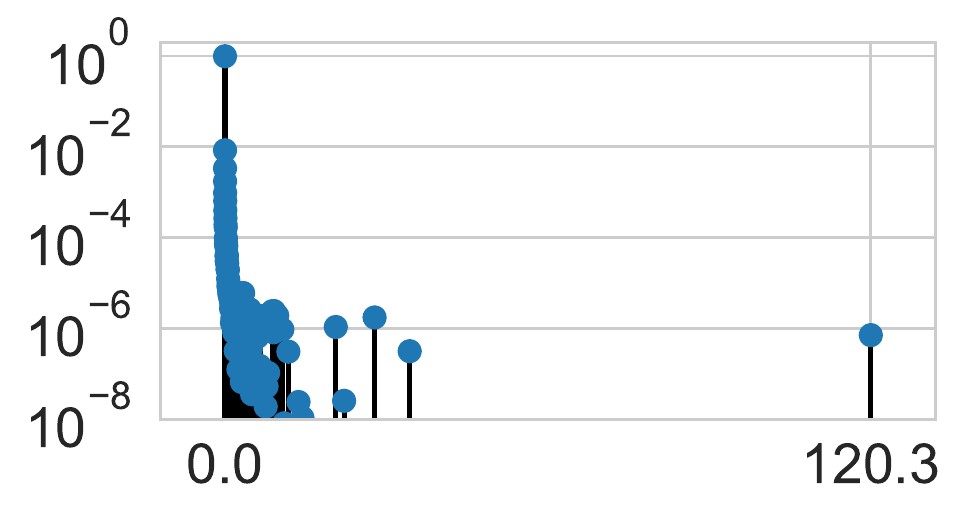}
	\caption{Epoch $200$}
	\label{subfig:c100p110ep200eval}
\end{subfigure}
\begin{subfigure}{0.30\linewidth}
	\includegraphics[width=1\linewidth,trim={0 0 0 0},clip]{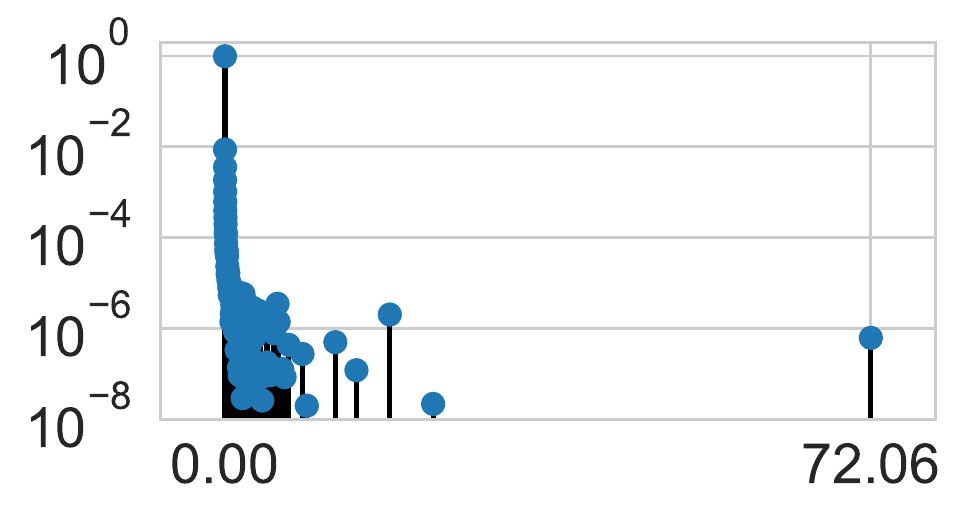}
	\caption{Epoch $225$}
	\label{subfig:c100p110ep225eval}
\end{subfigure}
\caption{Generalised Gauss-Newton matrix full empirical spectrum for the PreResNet-$110$ on the CIFAR-$100$ dataset, total training $225$ epochs, batch norm eval mode}
\label{fig:p110c100evalggn}
\end{figure}
\subsection{Hessian - Batch Normalisation Train Mode}
Similar to the Generalised Gauss-Newton matrix, the Hessian has well separated both negative and positive outliers from the spectral bulk and a large rank degeneracy, these observations are consistent throughout training.

\begin{figure}[h!]
\centering
\begin{subfigure}{0.23\linewidth}
	\includegraphics[width=1\linewidth,trim={0 0 0 0},clip]{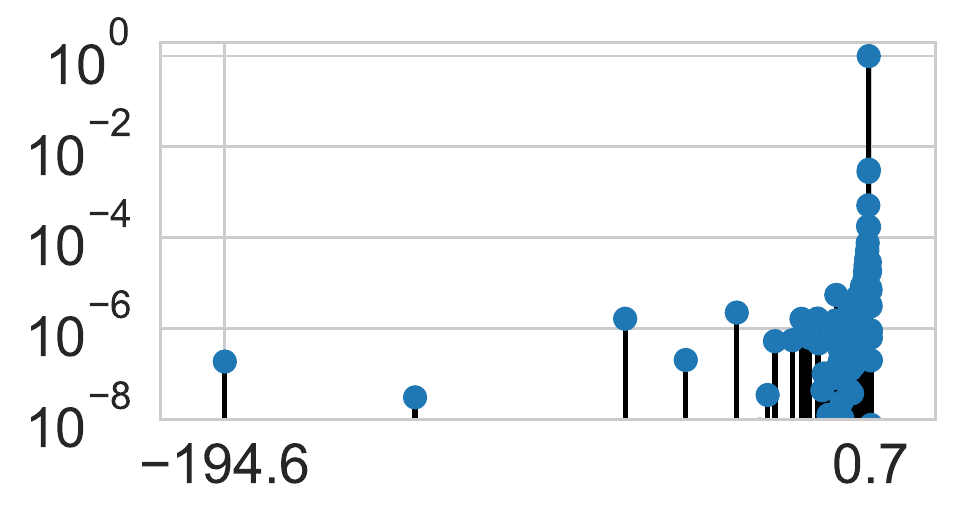}
	\caption{Epoch $0$}
	\label{subfig:hessc100p110ep0train}
\end{subfigure}
\begin{subfigure}{0.23\linewidth}
	\includegraphics[width=1\linewidth,trim={0 0 0 0},clip]{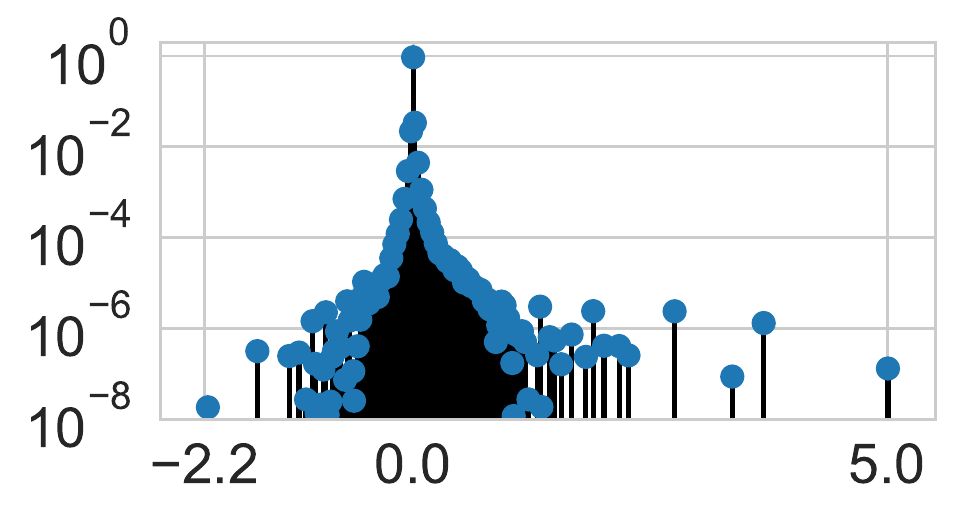}
	\caption{Epoch $25$}
	\label{subfig:hessc100p110ep25train}
\end{subfigure}
\begin{subfigure}{0.23\linewidth}
	\includegraphics[width=1\linewidth,trim={0 0 0 0},clip]{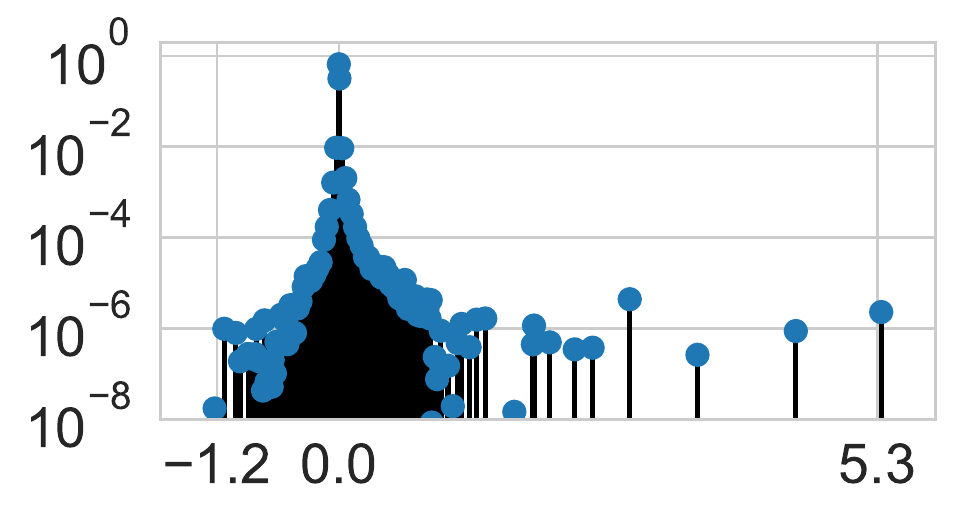}
	\caption{Epoch $50$}
	\label{subfig:hessc100p110ep50train}
\end{subfigure}
\begin{subfigure}{0.23\linewidth}
	\includegraphics[width=1\linewidth,trim={0 0 0 0},clip]{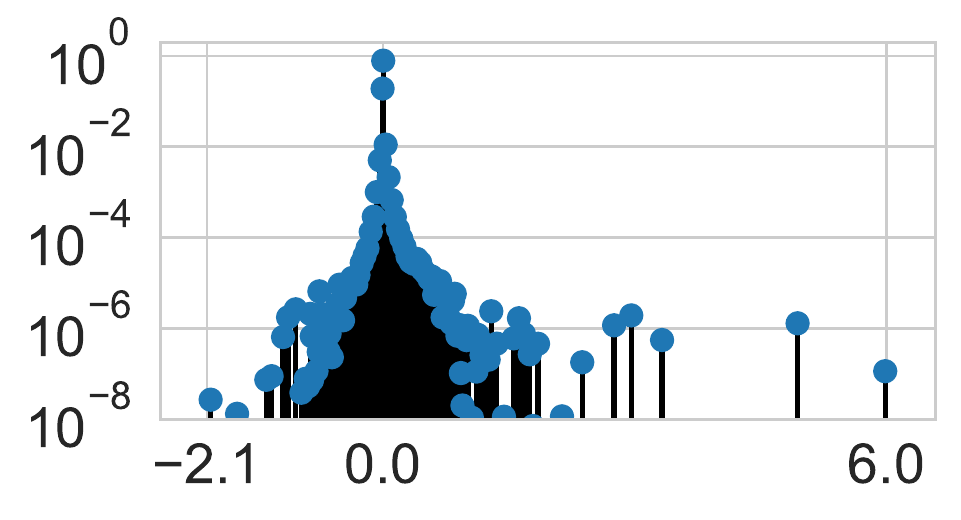}
	\caption{Epoch $75$}
	\label{subfig:hessc100p110ep75train}
\end{subfigure}
\begin{subfigure}{0.30\linewidth}
	\includegraphics[width=1\linewidth,trim={0 0 0 0},clip]{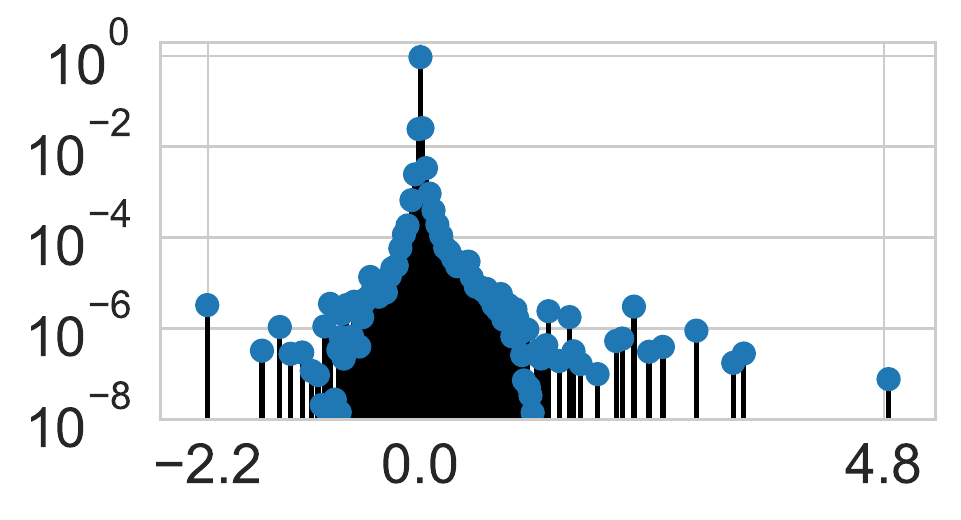}
	\caption{Epoch $100$}
	\label{subfig:hessc100p110ep100train}
\end{subfigure}
\begin{subfigure}{0.30\linewidth}
	\includegraphics[width=1\linewidth,trim={0 0 0 0},clip]{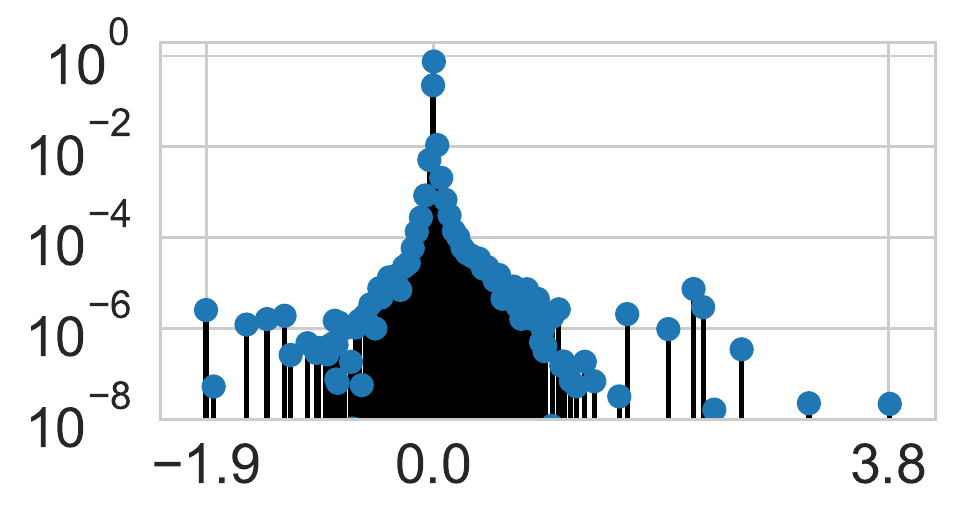}
	\caption{Epoch $125$}
	\label{subfig:hessc100p110ep125train}
\end{subfigure}
\begin{subfigure}{0.30\linewidth}
	\includegraphics[width=1\linewidth,trim={0 0 0 0},clip]{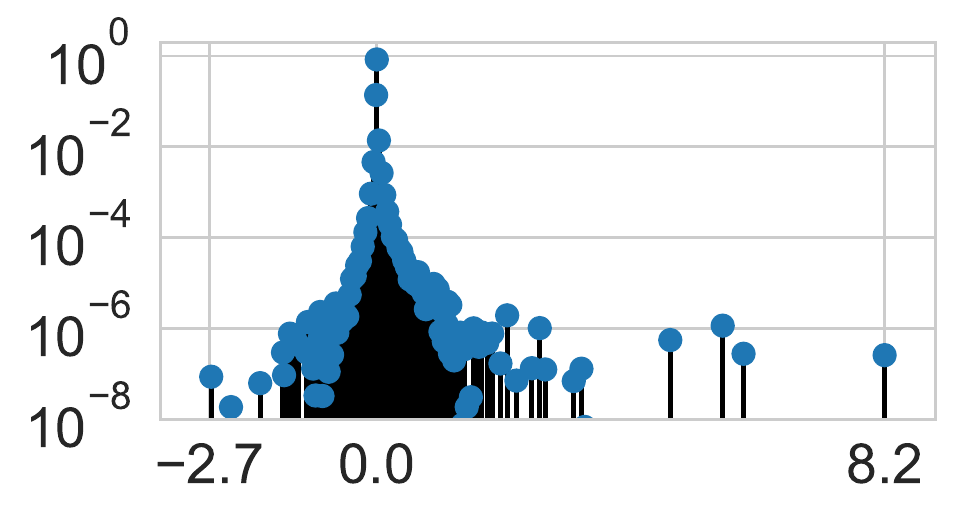}
	\caption{Epoch $150$}
	\label{subfig:hessc100p110ep150train}
\end{subfigure}
\begin{subfigure}{0.30\linewidth}
	\includegraphics[width=1\linewidth,trim={0 0 0 0},clip]{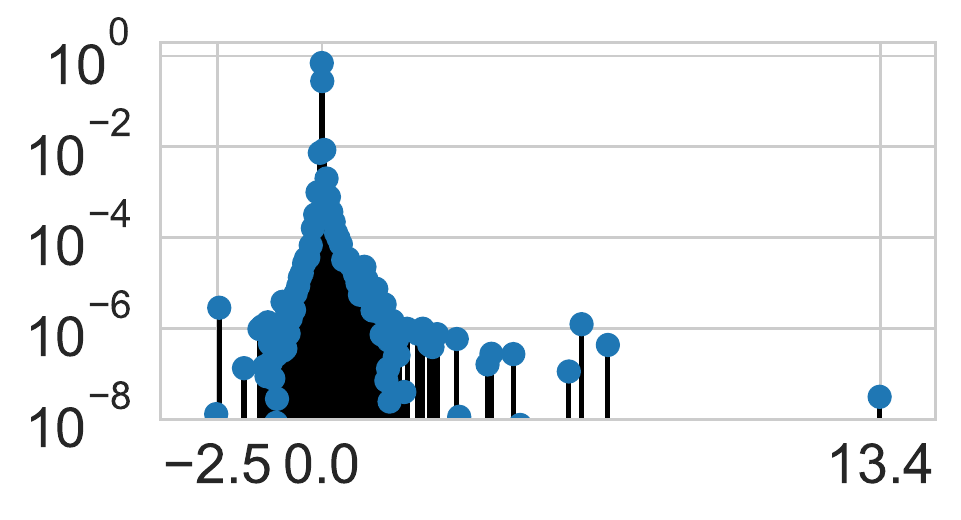}
	\caption{Epoch $175$}
	\label{subfig:hessc100p110ep175train}
\end{subfigure}
\begin{subfigure}{0.30\linewidth}
	\includegraphics[width=1\linewidth,trim={0 0 0 0},clip]{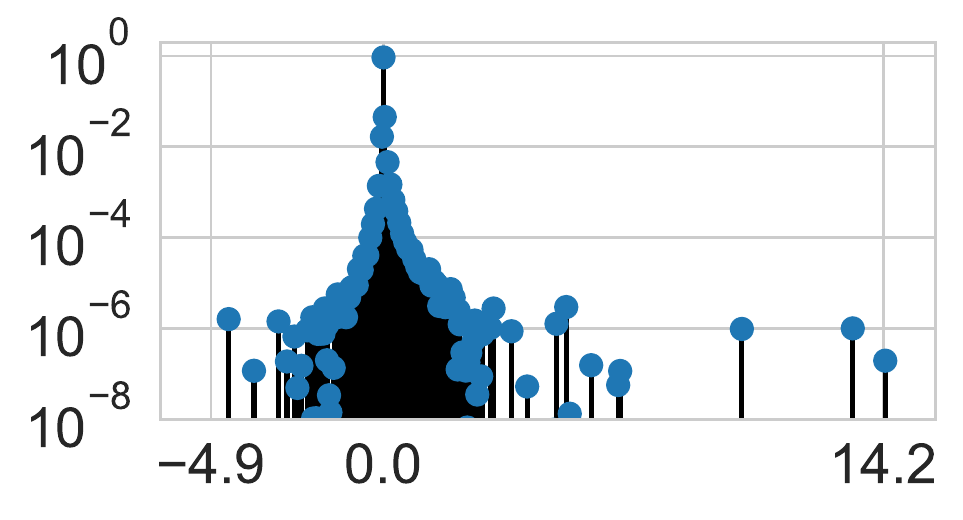}
	\caption{Epoch $200$}
	\label{subfig:hessc100p110ep200train}
\end{subfigure}
\begin{subfigure}{0.30\linewidth}
	\includegraphics[width=1\linewidth,trim={0 0 0 0},clip]{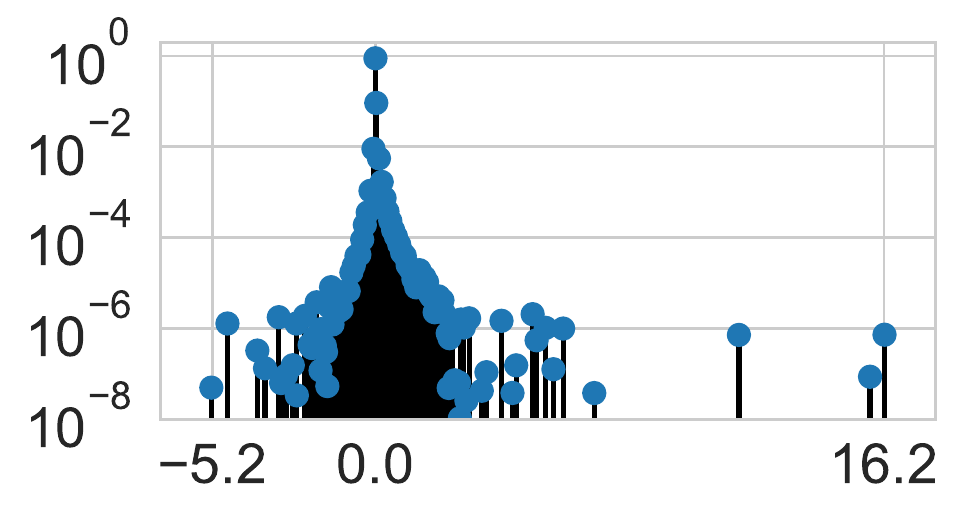}
	\caption{Epoch $225$}
	\label{subfig:hessc100p110ep225train}
\end{subfigure}
\caption{Hessian full empirical spectrum for the PreResNet-$110$ on the CIFAR-$100$ dataset, total training $225$ epochs, batch norm train mode}
\label{fig:hessp110c100train2}
\end{figure}
Similarly at all points in training, stochastic batch Hessians are shown to be significantly broadened, we see this by comparing the full data empirical Hessian spectrum \ref{fig:hessp110c100train2} compared to the Hessian at the same point in weight space but using only a batch size of $B=128$ in Figure \ref{fig:hessp110c100train2batch}. 
\begin{figure}[h!]
\centering
\begin{subfigure}{0.30\linewidth}
	\includegraphics[width=1\linewidth,trim={0 0 0 0},clip]{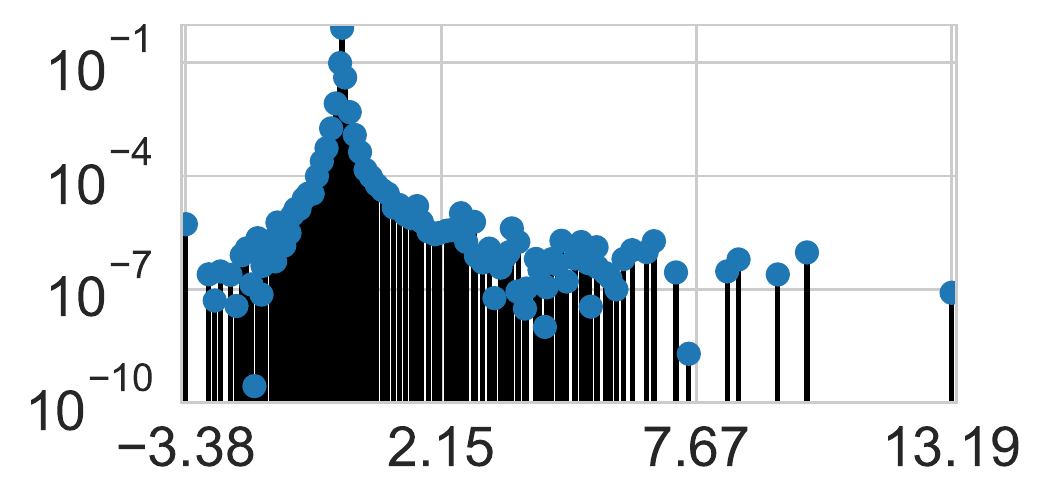}
	\caption{Epoch $100$, $B=128$}
	\label{subfig:hessc100p110ep100trainbatch}
\end{subfigure}
\begin{subfigure}{0.30\linewidth}
	\includegraphics[width=1\linewidth,trim={0 0 0 0},clip]{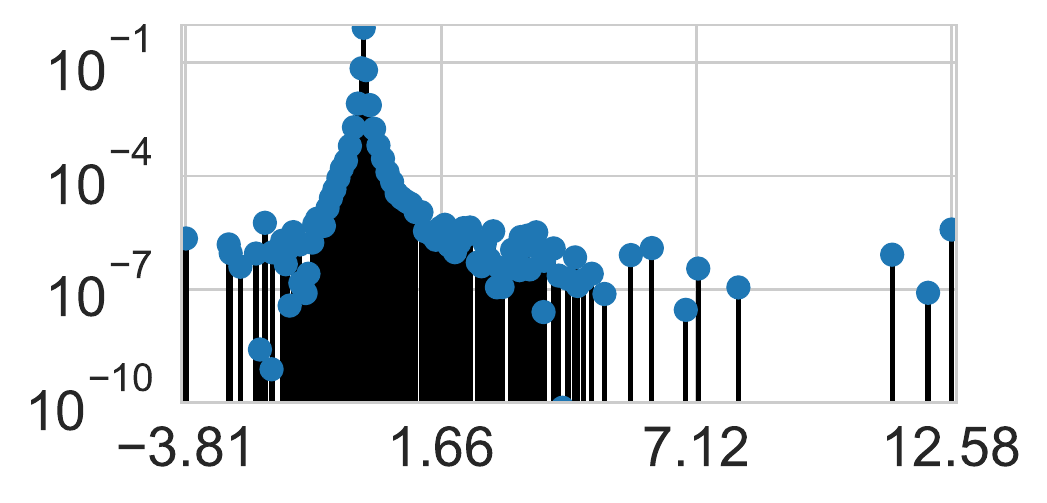}
	\caption{Epoch $125$, $B=128$}
	\label{subfig:hessc100p110ep125trainbatch}
\end{subfigure}
\begin{subfigure}{0.30\linewidth}
	\includegraphics[width=1\linewidth,trim={0 0 0 0},clip]{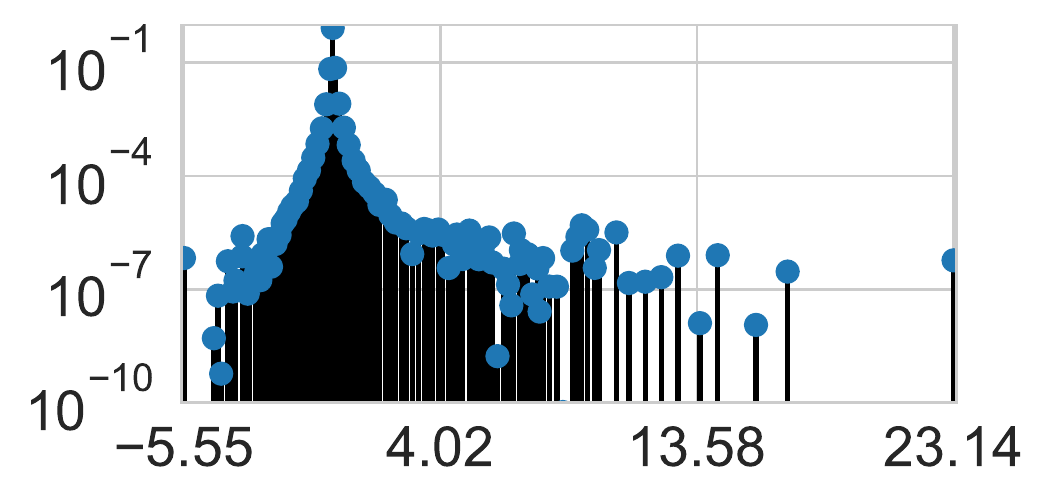}
	\caption{Epoch $150$, $B=128$}
	\label{subfig:hessc100p110ep150trainbatch}
\end{subfigure}
\caption{Hessian batch spectrum for the PreResNet-$110$ on the CIFAR-$100$ dataset, total training $225$ epochs, batch norm train mode, $B=128$}
\label{fig:hessp110c100train2batch}
\end{figure}
\subsection{Hessian - Batch Normalisation Evaluation Mode}
Similarly at all points in training, stochastic batch Hessians in evaluation mode are shown to be significantly broadened, we see this by comparing the full data empirical Hessian spectrum \ref{fig:hessp110c100eval2} compared to the Hessian at the same point in weight space but using only a batch size of $B=128$ in Figure \ref{fig:hessp110c100eval2batch}. 

\begin{figure}[h!]
\centering
\begin{subfigure}{0.23\linewidth}
	\includegraphics[width=1\linewidth,trim={0 0 0 0},clip]{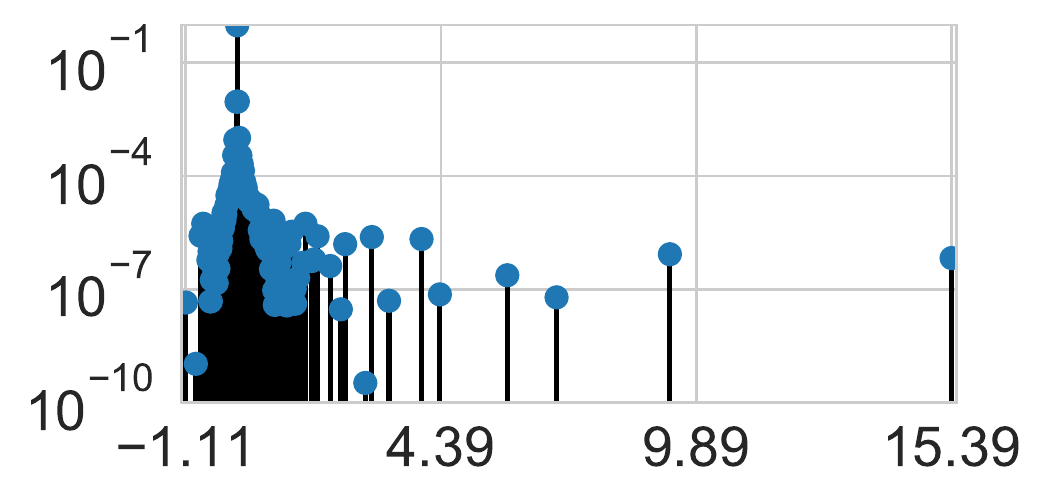}
	\caption{Epoch $0$}
	\label{subfig:hessc100p110ep0eval}
\end{subfigure}
\begin{subfigure}{0.23\linewidth}
	\includegraphics[width=1\linewidth,trim={0 0 0 0},clip]{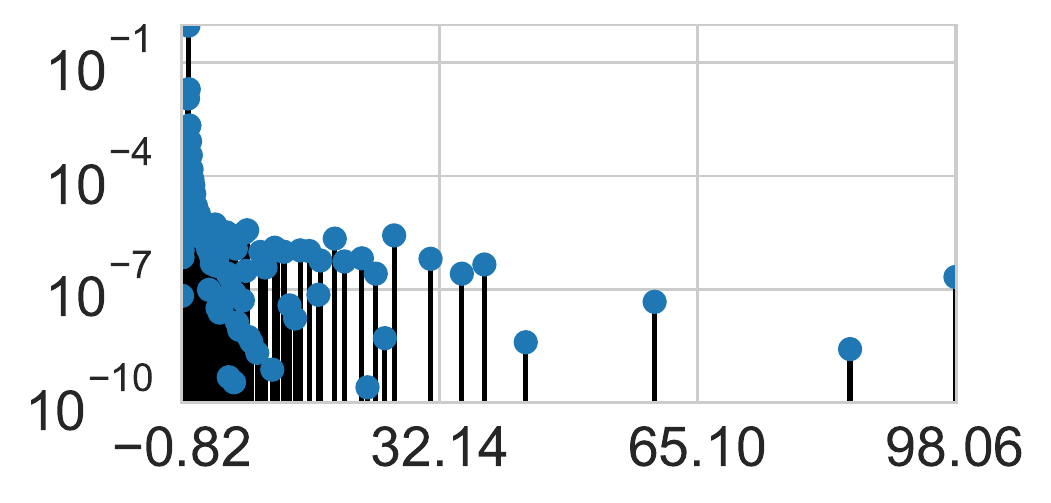}
	\caption{Epoch $25$}
	\label{subfig:hessc100p110ep25eval}
\end{subfigure}
\begin{subfigure}{0.23\linewidth}
	\includegraphics[width=1\linewidth,trim={0 0 0 0},clip]{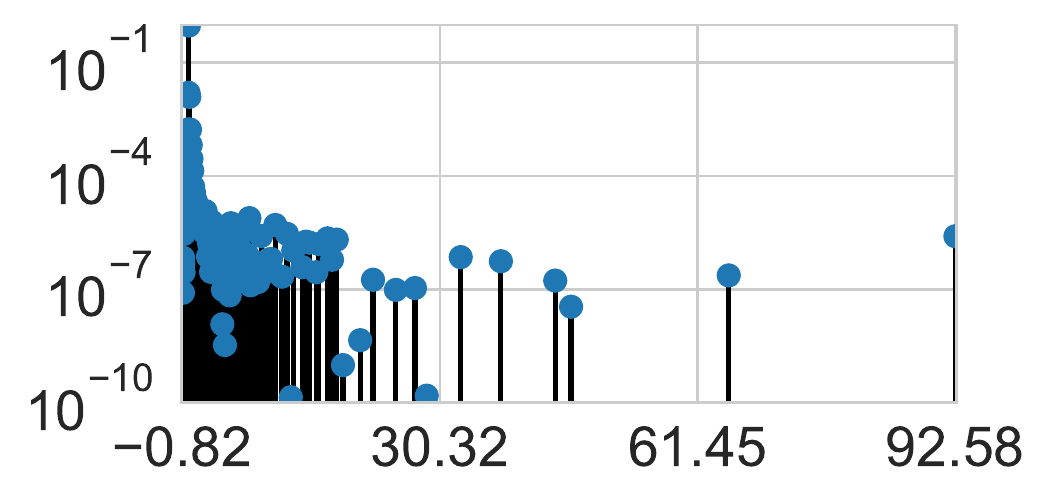}
	\caption{Epoch $50$}
	\label{subfig:hessc100p110ep50eval}
\end{subfigure}
\begin{subfigure}{0.23\linewidth}
	\includegraphics[width=1\linewidth,trim={0 0 0 0},clip]{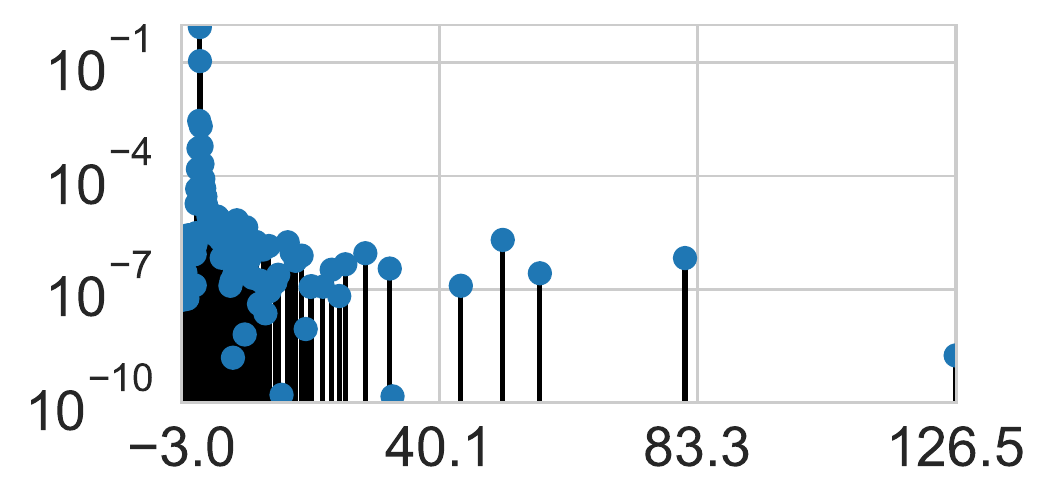}
	\caption{Epoch $75$}
	\label{subfig:hessc100p110ep75eval}
\end{subfigure}
\caption{Hessian full empirical spectrum for the PreResNet-$110$ on the CIFAR-$100$ dataset, total training $225$ epochs, batch norm evaluation mode}
\label{fig:hessp110c100eval}
\end{figure}

\begin{figure}[h!]
\centering
\begin{subfigure}{0.30\linewidth}
	\includegraphics[width=1\linewidth,trim={0 0 0 0},clip]{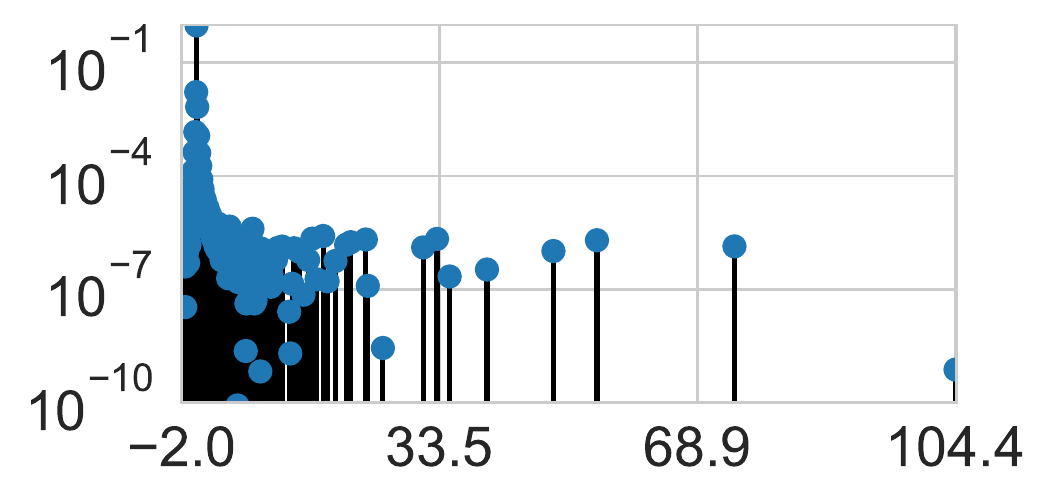}
	\caption{Epoch $100$}
	\label{subfig:hessc100p110ep100eval}
\end{subfigure}
\begin{subfigure}{0.30\linewidth}
	\includegraphics[width=1\linewidth,trim={0 0 0 0},clip]{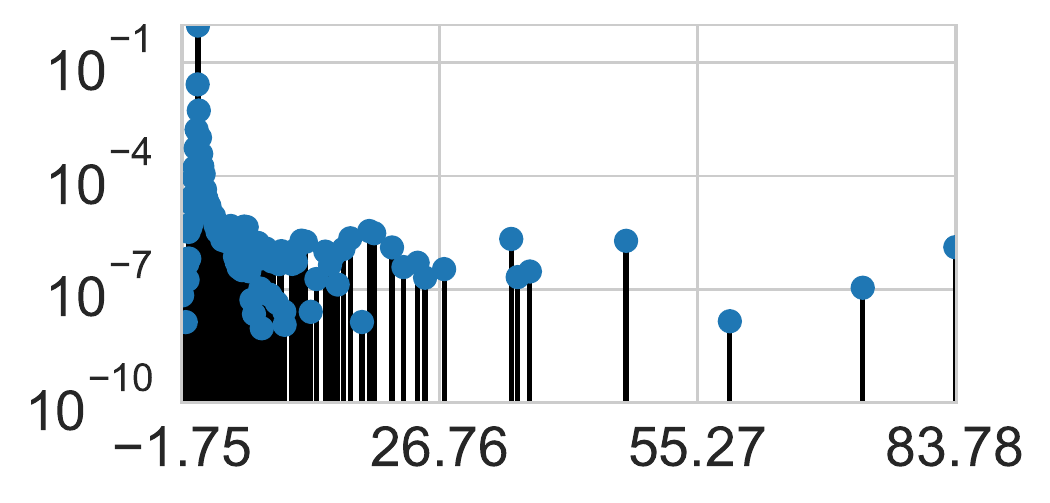}
	\caption{Epoch $125$}
	\label{subfig:hessc100p110ep125eval}
\end{subfigure}
\begin{subfigure}{0.30\linewidth}
	\includegraphics[width=1\linewidth,trim={0 0 0 0},clip]{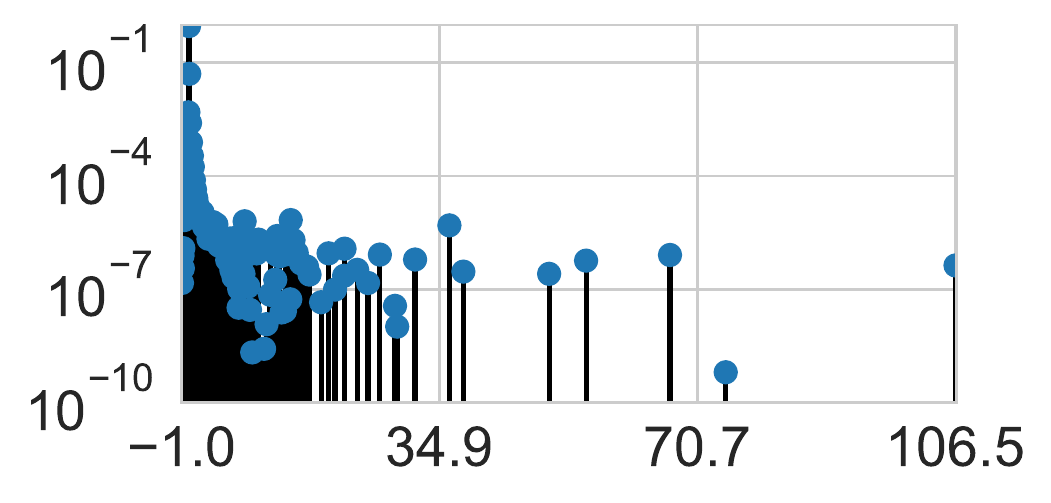}
	\caption{Epoch $150$}
	\label{subfig:hessc100p110ep150eval}
\end{subfigure}
\caption{Hessian full empirical spectrum for the PreResNet-$110$ on the CIFAR-$100$ dataset, total training $225$ epochs, batch norm evaluation mode}
\label{fig:hessp110c100eval2}
\end{figure}
\begin{figure}[h!]
\centering
\begin{subfigure}{0.30\linewidth}
	\includegraphics[width=1\linewidth,trim={0 0 0 0},clip]{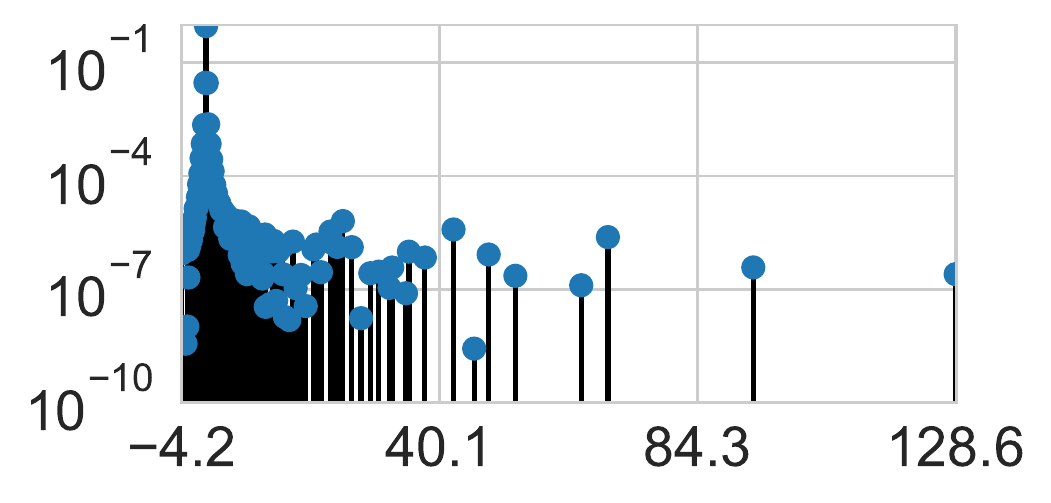}
	\caption{Epoch $100$, $B=128$}
	\label{subfig:hessc100p110ep100evalbatch}
\end{subfigure}
\begin{subfigure}{0.30\linewidth}
	\includegraphics[width=1\linewidth,trim={0 0 0 0},clip]{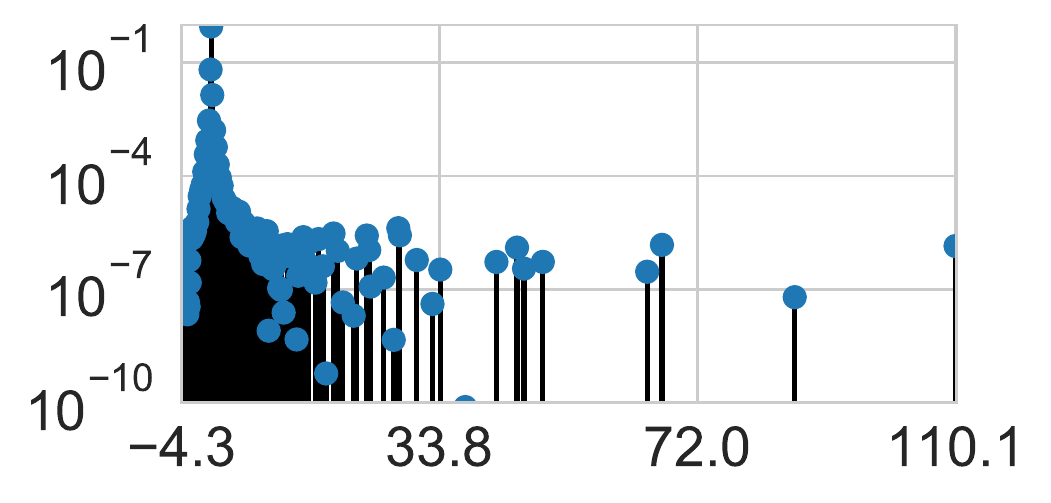}
	\caption{Epoch $125$, $B=128$}
	\label{subfig:hessc100p110ep125evalbatch}
\end{subfigure}
\begin{subfigure}{0.30\linewidth}
	\includegraphics[width=1\linewidth,trim={0 0 0 0},clip]{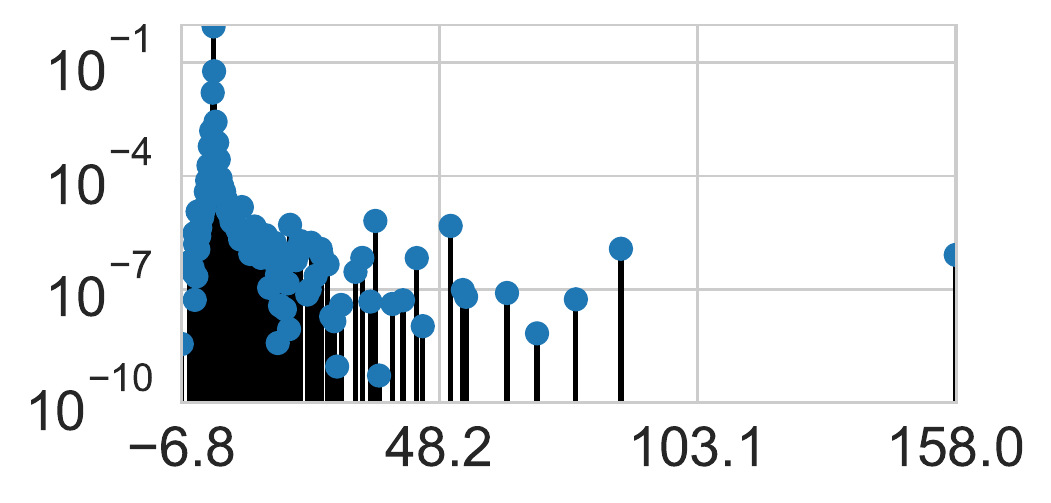}
	\caption{Epoch $150$, $B=128$}
	\label{subfig:hessc100p110ep150evalbatch}
\end{subfigure}
\caption{Hessian batch spectrum for the PreResNet-$110$ on the CIFAR-$100$ dataset, total training $225$ epochs, batch norm train mode, $B=128$}
\label{fig:hessp110c100eval2batch}
\end{figure}
\begin{figure}[h!]
\centering
\begin{subfigure}{0.30\linewidth}
	\includegraphics[width=1\linewidth,trim={0 0 0 0},clip]{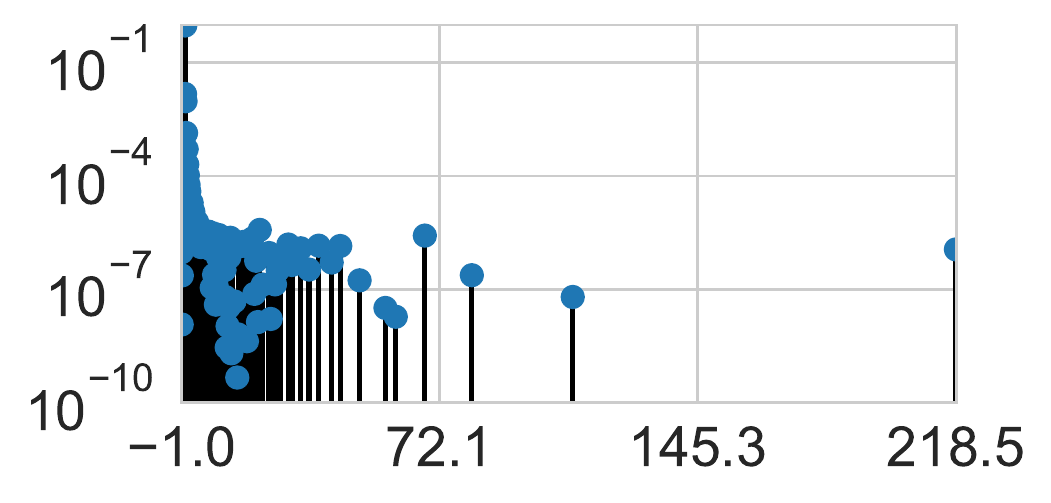}
	\caption{Epoch $175$}
	\label{subfig:hessc100p110ep175eval}
\end{subfigure}
\begin{subfigure}{0.30\linewidth}
	\includegraphics[width=1\linewidth,trim={0 0 0 0},clip]{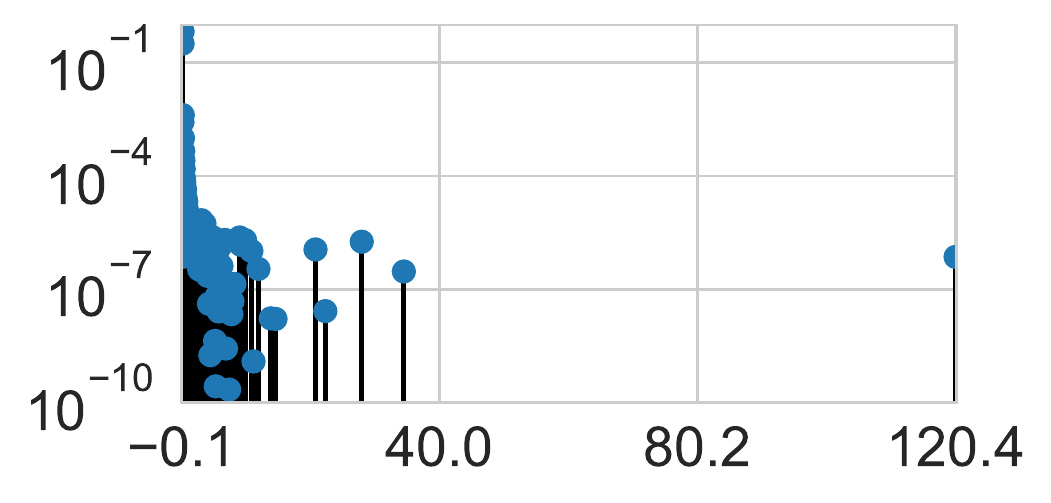}
	\caption{Epoch $200$}
	\label{subfig:hessc100p110ep200eval}
\end{subfigure}
\begin{subfigure}{0.30\linewidth}
	\includegraphics[width=1\linewidth,trim={0 0 0 0},clip]{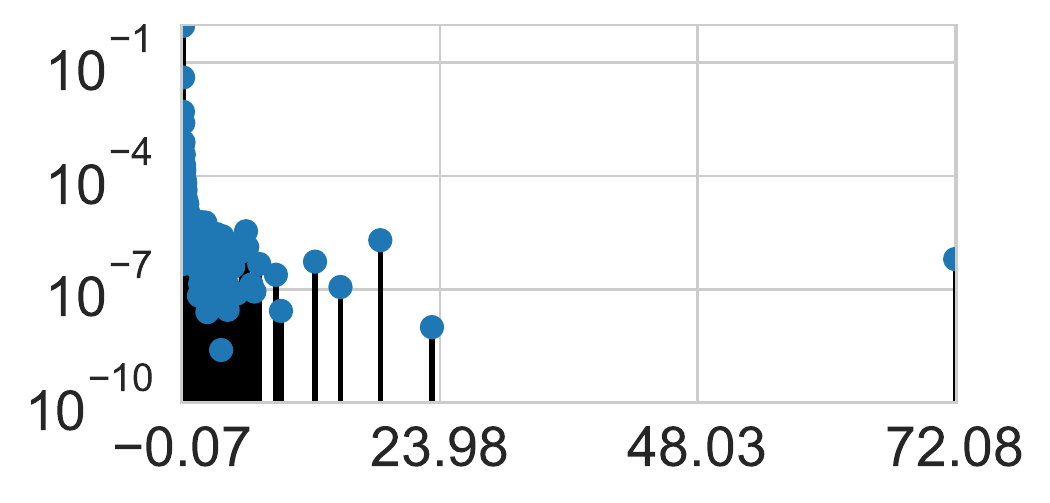}
	\caption{Epoch $225$}
	\label{subfig:hessc100p110ep225eval}
\end{subfigure}
\caption{Hessian full empirical spectrum for the PreResNet-$110$ on the CIFAR-$100$ dataset, total training $225$ epochs, batch norm evaluation mode}
\label{fig:hessp110c100eval3}

\end{figure}

\section{Alternative learning rate schedules and initialisation distance importance}
\label{sec:initdist}
\begin{figure}[h!]
\centering
\begin{subfigure}{0.48\linewidth}
	\includegraphics[trim={0cm 0cm 0cm 0cm},clip, width=1\textwidth]{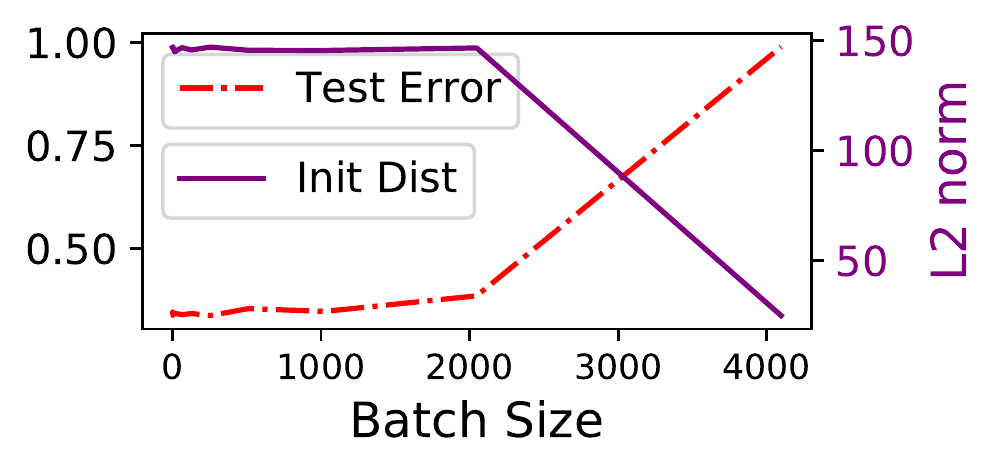}
	
	\caption{VGG-$16$}
	\label{subfig:vgg16err}
\end{subfigure}
\begin{subfigure}{0.48\linewidth}
	\includegraphics[trim={0cm 0cm 0cm 0cm},clip, width=1\textwidth]{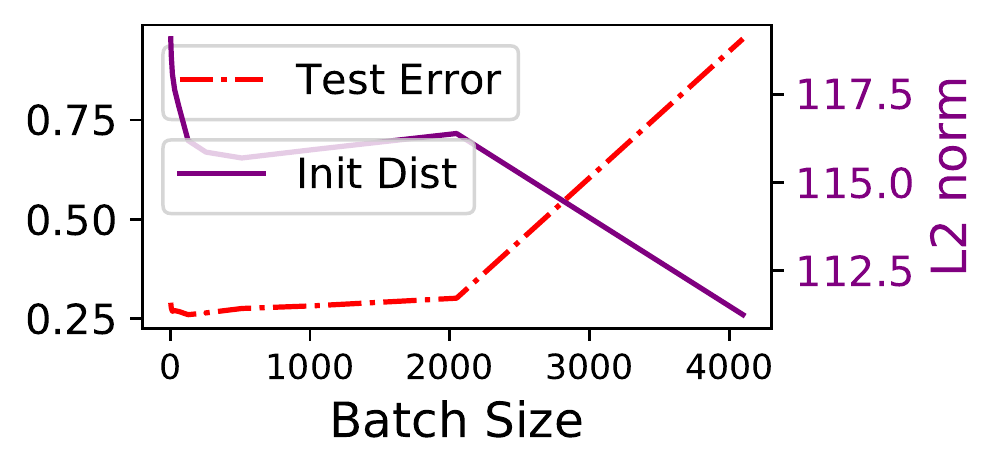}
	
	\caption{VGG-$16$BN}
	\label{subfig:vgg16bnerr}
\end{subfigure}
\caption{Test error as a function of initialisation distance for both the VGG-$16$ and VGG-$16$BN, for the CIFAR-$100$ dataset. Learning rate is scaled linearly with batch size.}
\label{fig:distinitvgg}
\end{figure}
One implicit assumption in Section \ref{sec:scaling} is that the largest learning rate which trains stably gives the best result. This informs our work as to how we should scale this rate as the batch size is increased. However it makes sense to consider how alternative more conservative scaling rules might fare and whether they impact performance. In this section we also consider whether increased distance from Initialisation, as posited in \cite{hoffer2017train} is relevant for generalisation.We do this for the VGG-$16$ on the CIFAR-$100$ dataset. Against a baseline validation accuracy of $65.82 \%$ for $B=128$. For the $B=1024$ case, our theoretically justified linearly increased learning rate of $0.08$ gives an accuracy of $64.35 \%$, whereas using the square root rule \citep{hoffer2017train} suggestion of $0.028$ only gives $61.08 \%$, we note from Figure \ref{fig:hesspert} that there are many well separated outliers. On the held out test set, the linear scaling solution has an error of $34.64 \%$ and a distance of $145.76$ in $L2$ norm from the initialisation, whereas the square root scaled solution has an error of $37.55 \%$ and a distance of $67.44$ from initialization. This indicates that as argued in \citet{hoffer2017train} that distance from initialisation seems to play an important role for generalisation. We test this further by looking at the initialisation distance across the set of similar test performing solutions for a constant learning rate to batch size ratio. Interestingly for the VGG-$16$ without batch normalisation as shown in Figure \ref{subfig:vgg16batchggn} there is a strong link between initialisation distance in $L2$ norm and the test error. This relationship is much weaker and much smaller in magnitude when batch normalisation is utilised, as shown in Figure \ref{subfig:vgg16empggn}. We even see the initialisation distance increasing as test error also increases.

\newpage

\end{document}